\documentclass[10pt]{article}
\usepackage[utf8]{inputenc}
\usepackage[stable]{footmisc}
\usepackage{amsmath}  
\allowdisplaybreaks
\usepackage{graphicx} 
\usepackage{amsthm}
\usepackage{tikz-cd} 
\usepackage{amssymb}
\usepackage{bm}
\usepackage{bbm}
\usepackage{booktabs}
\usepackage{pifont}
\usepackage{xcolor}
\usepackage{array}
\usepackage{makecell}
\usepackage{dsfont}
\usepackage{graphicx}
\usepackage{subcaption}
\usepackage{blindtext}
\usepackage{url}
\usepackage{algorithm}
\usepackage{algorithmic}
\usepackage{thmtools}
\usepackage{thm-restate}
\usepackage{setspace}
\usepackage{enumitem}
\usepackage[margin=1in]{geometry}
\usepackage{xcolor}
\usepackage{cite} 
\usepackage[final]{hyperref} 
\hypersetup{
	colorlinks=true,       
	linkcolor=blue,        
	citecolor=blue,        
	filecolor=magenta,     
	urlcolor=blue         
}
\usepackage{cleveref}
\usepackage{tcolorbox}
\usepackage{tabularx}
\usepackage{threeparttable}
\usepackage{authblk}
\usepackage{etoc}
\usepackage[export]{adjustbox}

\newtheorem{theorem}{Theorem}[section]
\newtheorem{proposition}[theorem]{Proposition}
\newtheorem{lemma}[theorem]{Lemma}
\newtheorem{corollary}[theorem]{Corollary}
\newtheorem{assumption}[theorem]{Assumption}

\theoremstyle{definition}

\theoremstyle{definition}

\newtheorem{remark}[theorem]{Remark}

\numberwithin{equation}{section}

\newcommand{\R}{\mathbb{R}}

\newcommand{\M}{\mathcal{M}}

\newcommand{\BBS}{\mathbb{S}}
\newcommand{\BBC}{\mathbb{C}}

\newcommand{\BBT}{\mathbb{T}}

\newcommand{\CE}{\mathcal{E}}
\newcommand{\CR}{\mathcal{R}}

\newcommand{\C}{\mathbb{C}}
\renewcommand{\S}{\mathbb{S}}

\newcommand{\IM}[1]{\mathrm{Im}\left(#1\right)}
\newcommand{\RE}[1]{\mathrm{Re}\left(#1\right)}
\newcommand{\im}[1]{\mathrm{Im} \left(#1\right)}
\newcommand{\re}[1]{\mathrm{Re}\left(#1\right)}

\newcommand{\alphaBar}{\overline{\alpha}}

\newcommand{\aDPlus}{a_+^{\text{d}}}
\newcommand{\aDMinus}{a_-^{\text{d}}}
\newcommand{\aOPlus}{a_+^{\text{o}}}
\newcommand{\aOMinus}{a_-^{\text{o}}}
\newcommand{\vDPlus}{v_+^{\text{d}}}
\newcommand{\vDMinus}{v_-^{\text{d}}}
\newcommand{\vOPlus}{v_+^{\text{o}}}
\newcommand{\vOMinus}{v_-^{\text{o}}}

\renewcommand{\P}{\mathsf{P}}

\newcommand{\Law}{\mathrm{Law}}

\newcommand{\dummy}{\mathord{\color{black!33}\bullet}}
\newcommand{\WC}{\mathrm{WC}}

\newcommand{\tr}[1]{\mathrm{tr}(#1)}

\newcommand{\Dtheta}{\Dot{\theta}}
\newcommand{\Dxi}{\Dot{\xi}}
\newcommand{\Drho}{\Dot{\rho}}
\newcommand{\Dphi}{\Dot{\phi}}
\newcommand{\Dpsi}{\Dot{\psi}}
\newcommand{\Dalpha}{\Dot{\alpha}}
\newcommand{\DalphaBar}{\Dot{\overline{\alpha}}}
\newcommand{\Dvartheta}{\Dot{\vartheta}}
\newcommand{\Deta}{\Dot{\eta}}
\newcommand{\Dgamma}{\Dot{\gamma}}
\newcommand{\DPhi}{\Dot{\Phi}}
\newcommand{\rhoTH}{\rho_{\mathrm{th}}}
\newcommand{\gammaTH}{\gamma_{\mathrm{th}}}

\newcommand{\lambdaMax}{\lambda_{\mathrm{max}}}

\newcommand{\thetaMF}{\Theta}
\newcommand{\tb}{\widetilde{b}}
\newcommand{\rOverline}{\overline{r}}

\newcommand{\Arg}{\operatorname{Arg}}
\newcommand{\zeq}{z_{\mathrm{eq}}}

\newcommand{\eq}{\mathrm{eq}}
\newcommand{\Wvarphi}{\widetilde{\varphi}}
\newcommand{\DeltaGamma}{\Delta_{\gamma}}
\newcommand{\DeltaPhi}{\Delta_{\Phi}}
\newcommand{\BL}{\mathrm{BL}}

\newcommand{\PhiA}{\Phi_{A}}
\newcommand{\PhiV}{\Phi_{V}}

\newcommand{\phiV}{\phi_{V}}

\newcommand{\WPhi}{\widetilde{\Phi}}
\newcommand{\PhiTH}{\Phi_{\mathrm{th}}}
\newcommand{\BargammaTH}{\overline{\gamma}_{\mathrm{th}}}
\newcommand{\wz}{\widetilde{z}}
\newcommand{\rhoEq}{\rho_{\mathrm{eq}}}
\newcommand{\WGammaTH}{\widetilde{\gamma}_{\mathrm{th}}}

\makeatletter
\def\url@leostyle{%
	\@ifundefined{selectfont}{\def\UrlFont{\sf}}{\def\UrlFont{\small\ttfamily}}}
\makeatother
\definecolor{darkgreen}{rgb}{0,0.4,0}

\newcommand{\cmark}{\textcolor{green!60!black}{\ding{51}}}
\newcommand{\xmark}{\textcolor{red!70!black}{\ding{55}}}

\newcommand{\Unif}{\mathrm{Unif}}
\newcommand{\OA}{\mathrm{OA}}

\newcommand{\SX}[1]{\marginpar{\scriptsize\textcolor{darkgreen}{SX: {#1}}}}

\title{{\usefont{OT1}{bch}{b}{n}
	\LARGE  On the Diverse Dynamical Behaviors Arising in\\Deep Linear Transformers}}

\author[1]{Sixu Li\thanks{Email: \texttt{sli739@wisc.edu}}}
\author[2]{Thomas Jacob Maranzatto\thanks{Email: \texttt{tmaran@umd.edu}}}
\author[3]{Jan Peszek\thanks{Email: \texttt{j.peszek@uw.edu.pl}}}
\author[4]{Trevor Teolis\thanks{Email: \texttt{tt111@rice.edu}}}
\author[2]{Semih Akkoc\thanks{Email: \texttt{akkoc@umd.edu}}}
\author[5]{Konstantin Riedl\thanks{Email: \texttt{Konstantin.Riedl@maths.ox.ac.uk}}}
\author[2]{Sennur Ulukus\thanks{Email: \texttt{ulukus@umd.edu}}}
\author[1]{Nicol\'as {Garc\'ia Trillos}\thanks{Email: \texttt{garciatrillo@wisc.edu}}}

\affil[1]{University of Wisconsin--Madison, Department of Statistics}
\affil[2]{University of Maryland, Department of Electrical and Computer Engineering}
\affil[3]{University of Warsaw, Institute of Applied Mathematics and Mechanics}
\affil[4]{Rice University, Department of Computational and Applied Mathematics \& Operations Research}
\affil[5]{University of Oxford, Mathematical Institute}
\date{\today}

\begin{document}
\maketitle
\vspace{-1.0em}
\begin{abstract}
\noindent
We study the inference-time behavior of deep linear encoder-only transformers through the lens of interacting particle systems.
In this perspective, tokens are modeled as particles that interact dynamically through successive linear self-attention layers. 
We show that in embedding dimension two, for any key, query, and value matrices, the dynamics can be reformulated as a generalized Kuramoto-type model with pure second-harmonic coupling. 
This formulation is amenable to Watanabe--Strogatz theory which reveals the dynamics are intrinsically low-dimensional regardless of the parameter matrices. 
For a class of token initializations associated with the Ott--Antonsen (OA) manifold, we show that the parameter matrices induce a diverse variety of long-time behaviors in linear transformers, including clustering, oscillations, and bifurcations.
The oscillations and bifurcations are characterized by uncovering a hidden Hamiltonian structure in the dynamics.
By establishing a structural stability result, we further show that dynamics initialized near the OA manifold exhibit the same long-time behavior as those initialized exactly on the manifold.
Motivated by our theory in dimension two, we conduct numerical experiments for analogous parameter regimes in higher-dimensional transformers. Our numerical experiments suggest that the long-time behaviors characterized in our theoretical results persist in higher dimensions.
\end{abstract}



%
%

\newpage 
\tableofcontents

\etocdepthtag.toc{main}
\etocsettagdepth{main}{subsection}
\etocsettagdepth{appendix}{none}



\section{Introduction}\label{sec:intro}

The transformer architecture, through its central role in modern large language and computer vision models~\cite{bert,dosovitskiy2021an,gemini2023,gpt,llama} and, more broadly, in state-of-the-art foundation models~\cite{bommasani2021opportunities}, has become a fundamental component of contemporary machine learning and artificial intelligence.
At its core lies the self-attention mechanism~\cite{vaswani2017attention}, which enables the model to capture long-range, context-aware dependencies within the input data.
This mechanism distinguishes transformers from earlier neural network architectures such as feedforward \cite{krizhevsky2012imagenet} or recurrent neural networks \cite{hochreiter1997long}, where semantic interactions within the data are either absent or propagated sequentially through a number of hidden states.
While the capabilities of transformer models have experienced remarkable improvements in performance since their introduction,
apart from phenomenological explanations it remains mostly unclear how these large-scale models process data, what the internal transformation of the data leads to, and, more generally, why these architectures turn out to be highly capable for the wide ranges of tasks they are employed for. 

This paper is motivated by the general goal of developing a mathematical understanding of transformer architectures, and specifically the self-attention mechanism.
Toward that goal, and as in recent theoretical works on the mathematical analysis of the inference dynamics of encoder-only transformers~\cite{sander2022sinkformers,geshkovski2023mathematical,bruno2024emergence,castin2025unified,burger2025analysis,bruno2025multiscale,alcalde2026quantifying}, 
we consider an idealized version of a deep encoder-only transformer
and view the evolution of a sequence of input tokens (prompt)
as a continuous-time process described by a system of coupled ODEs of the form
\begin{equation}
    \label{eq:transformer}
    \Dot{x}_k(t) = \P^{\perp}_{x_k(t)} \left( \frac{1}{Z_k(t)} \sum_{j=1}^n h\left(\beta \langle Qx_k(t), K x_j(t) \rangle \right) V x_j(t) \right), \qquad k=1, \dots, n.
\end{equation}
Here, $x_k(t) \in \S^{d-1}$ is interpreted as the output of the $t$-th self-attention layer for the $k$-th token representation, which we will thereafter also refer to as the $k$-th particle.
The matrices $Q, K\in \R^{d_\text{int}\times d}$, and $V \in \R^{d\times d}$ (the query, key, and value parameters of the transformer) are real-valued matrices assumed to have been previously determined at the moment of training.
$Z_k$ is a normalization factor
that in principle can be token dependent; different choices of the normalization factor lead to  different, although in many cases qualitatively similar, models, see, e.g., \cite{geshkovski2023mathematical,burger2025analysis}.
The operator $\P_x^{\perp}$ denotes the projection onto the tangent space $T_x\BBS^{d-1}$ of the sphere at a point $x\in\BBS^{d-1}$, i.e., $\P^{\perp}_x y := y - \langle x, y \rangle x$.
This projection map ensures that the dynamics~\eqref{eq:transformer} remains on the sphere and can be regarded as a model for layer normalization (in particular, RMSNorm), see, e.g., \cite[Section~1.2]{burger2025analysis}.
The function $h:\R \mapsto \R$ is a suitable attention kernel that may take a variety of forms. Most popularly, $h$ takes exponential~\cite{vaswani2017attention} and linear forms~\cite{katharopoulos2020transformers}, which lead (provided suitable normalization) to softmax and linear self-attention, respectively.
For simplicity,
we use the notation $A = Q^\top K$ in the sequel.

After formulating the transformer dynamics \eqref{eq:transformer}, a central question is to characterize its long-time behavior. More specifically, we ask the following question:
\begin{center}
\begin{minipage}{0.95\textwidth}
\centering
\textit{For which choices of the weight matrices $A$ and $V$ does the associated dynamics converge to consensus and stabilize, and under which choices can other long-time behaviors emerge?}
\end{minipage}
\end{center}
By answering this question, researchers hope to clarify the structure and qualitative properties of the token representations produced after repeated application of self-attention layers.

One class of settings in which this question is relatively well-understood is when \eqref{eq:transformer} admits a Wasserstein gradient flow structure on the space of probability measures.
In particular, \cite{geshkovski2023mathematical,burger2025analysis} study the case $h(y) = \exp(y)$ and shows that, if $A^{\top} = A$ and $V = cA$ with some constant $c \in \R$, then \eqref{eq:transformer} can be written as a Wasserstein gradient flow for a suitable energy functional.
A further analysis of the corresponding energy landscape shows that, under additional assumptions on $A$ and $V$, for instance $A = V = I_d$, the dynamics exhibits clustering/synchronization behavior, in the sense that all token representations converge in the long-time limit to a single point. 

In practical transformer models, however, the matrices $A$ and $V$ do not necessarily satisfy the structural assumptions required for the Wasserstein gradient flow formulation.
This naturally leads to the following question:
\begin{center}
\begin{minipage}{0.95\textwidth}
\centering
\textit{What is the long-time behavior of the transformer dynamics \eqref{eq:transformer} under more general choices of the weight matrices $A$ and $V$?}
\end{minipage}
\end{center}

In the present work, we use a different analytical framework, distinct from the Wasserstein gradient flow approach, that enables us to study the transformer dynamics \eqref{eq:transformer} for general matrices $A$ and $V$, albeit with a different choice of attention kernel $h$. 
In particular, we focus on the \textit{linear} self-attention model (LSA), corresponding to the choice $h(y)=y$. 
This model has received considerable attention in recent years given its advantage on fast inference \cite{katharopoulos2020transformers,peng2021random,choromanski2021rethinking,schlag2021linear,shen2021efficient}.
It is also particularly well-suited for the analytical approach developed in this paper, for reasons that will become clear later. 
In this case, the model \eqref{eq:transformer} takes the form 
\begin{equation}\label{eq:LSA} \tag{LSA}
    \Dot{x}_k(t) = \P^{\perp}_{x_k(t)} \left( \frac{1}{n} \sum_{j=1}^n \beta \langle x_k(t), A x_j(t) \rangle V x_j(t) \right), \qquad k=1, \dots, n,
\end{equation}
where we have set the normalization factor $Z_k$ equal to $n$ for all tokens, following \cite[Equation~(USA)]{geshkovski2023mathematical}. 
Without loss of generality, we henceforth set $\beta = 1$, since this parameter can be absorbed into either $A$ or $V$.

As the (USA) model studied in \cite{geshkovski2023mathematical, burger2025analysis}, that is, model \eqref{eq:transformer} with $h(y) = \exp(y)$, the model \eqref{eq:LSA}  also admits a Wasserstein gradient flow formulation when $A^{\top} = A$ and $V = cA$ for some constant $c$. 
A rigorous justification is given in Lemma \ref{lem:model_gen_h_WGF}. However, our goal is to go beyond the restrictive structural assumptions required by the Wasserstein gradient flow framework.
The key observation is that, at least in the case of dimension $d=2$, the Fourier representation of \eqref{eq:LSA} takes the form of a \textit{pure} second-harmonic; see \eqref{eq:finite_sys} and \eqref{eq:vf_b}. 
This type of interaction has an intrinsic low-dimensional structure, which in turn allows us to apply dimensionality-reduction techniques developed in the study of the Kuramoto model \cite{skardal2011cluster,ott2008low}, a classical mathematical model for synchronization. By ``low-dimensional structure” we mean that both the full $n$-particle system and the mean-field PDE which will be introduced shortly are completely governed by a small number of effective ODEs.

To make this idea more precise, we represent each token $x_k$ as a complex number on the unit circle, written as $e^{i\theta_k}$. In these variables, the equations of motion \eqref{eq:LSA} become
\begin{equation} \label{eq:particle_interact_intro}
    \dot{\theta}_k(t) = b(\theta_k(t), \CR_2^n(t)), \qquad k=1, \dots, n,
\end{equation}
for a function $b$ that depends implicitly on the matrices $A$ and $V$; see Theorem \ref{thm:finite_sys}.
Importantly, the particles interact only through the \emph{second order parameter} $\CR_2^n(t)$ of the empirical distribution of tokens $f^n_t$,
\begin{equation*}
f^n_t = f^n(t, \dummy) := \frac{1}{n} \sum_{k=1}^n \delta_{\theta_k(t)}
\end{equation*}
and
\begin{equation}
    \CR_2^n(t)
    := \int_{0}^{2\pi} e^{i2\theta} f^n_t(d\theta)
    = \frac{1}{n} \sum_{j=1}^n e^{i2\theta_j(t)} \, .
\end{equation}
An analogous characterization continues to hold for the dynamics of a ``typical" particle in the mean-field limit, that is, as $n \rightarrow \infty$ and under the assumption that the initial tokens are drawn from a fixed distribution $f_0$. More precisely, one may consider the characteristic flow 
\begin{equation*}
    \dot{\thetaMF}(t) = b \left(\thetaMF(t), \CR_2(t) \right)
\end{equation*}
to describe the evolution of this typical particle, where $f := \Law(\thetaMF)$ solves the continuity equation
\begin{equation*}
    \partial_t f_t + \partial_{\theta} \left( b(\dummy, \CR_2(t))) f_t \right) = 0, \qquad \text{with } \CR_2 (t) = \int_0^{2\pi} e^{i2\theta} f_t( d\theta).
\end{equation*}

The above formulation shows that each token is influenced by the others only through the order parameter $\CR_2$, highlighting the central role of this quantity in the analysis of transformer dynamics. 
Moreover, $\CR_2$ quantifies the degree of synchronization of the system.
Indeed, the modulus $|\CR_2(f)|$ is close to $1$ if and only if the distribution $f$ is concentrated near two antipodal points; see Figure~\ref{fig:R2behavior} for an illustration.
This observation motives tracking the evolution of the order parameter $\CR_2$.

\begin{figure}
\centering
\begin{tcolorbox}[colframe=black, colback=white, arc=0mm, sharp corners, boxrule=0.5pt, left=1mm, right=1mm, top=1mm, bottom=1mm]
\begin{subfigure}[t]{.24\textwidth}
\includegraphics[trim={5.3cm 4cm 22cm 4cm},clip,width=\textwidth]{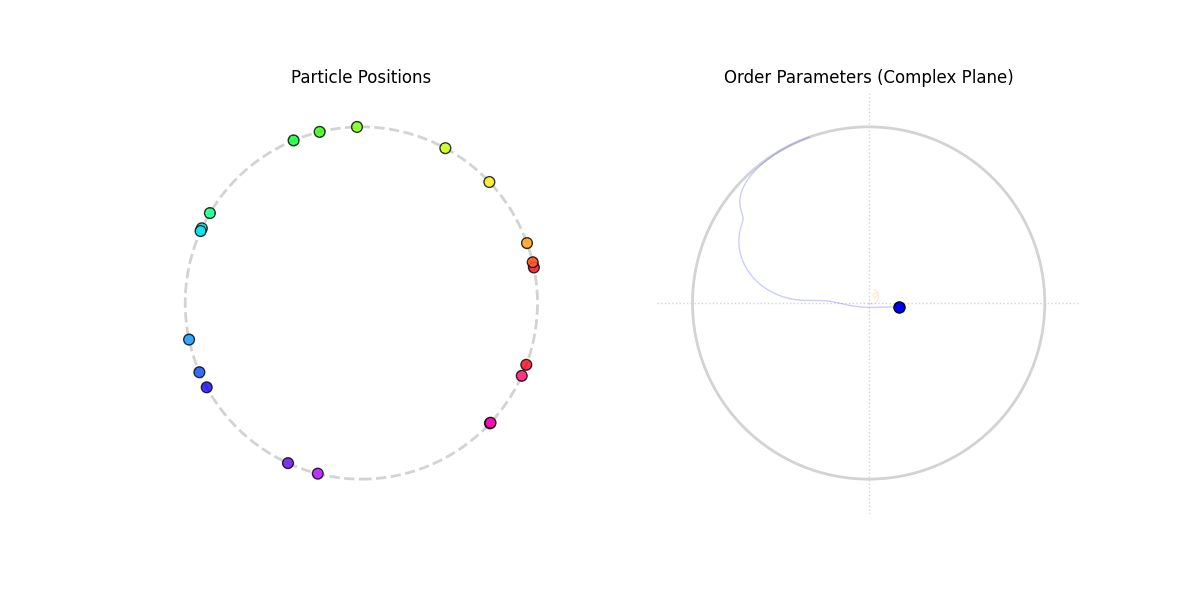}
\end{subfigure}%
\begin{subfigure}[t]{.24\textwidth}
\includegraphics[trim={5.3cm 4cm 22cm 4cm},clip,width=\textwidth]{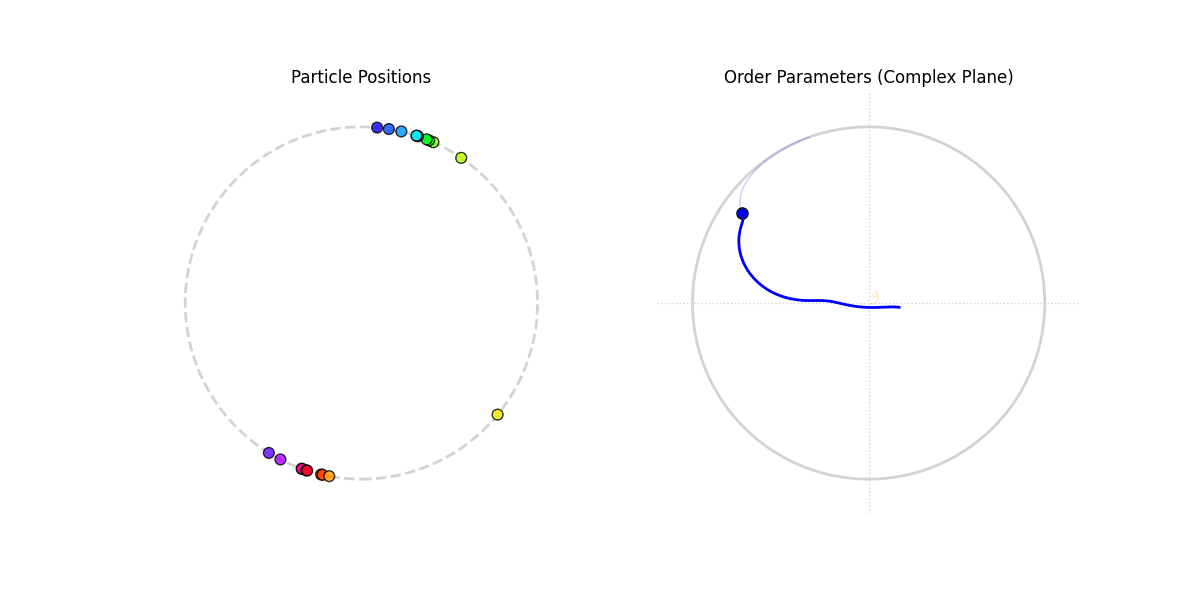}
\end{subfigure}%
\begin{subfigure}[t]{.24\textwidth}
\includegraphics[trim={5.3cm 4cm 22cm 4cm},clip,width=\textwidth]{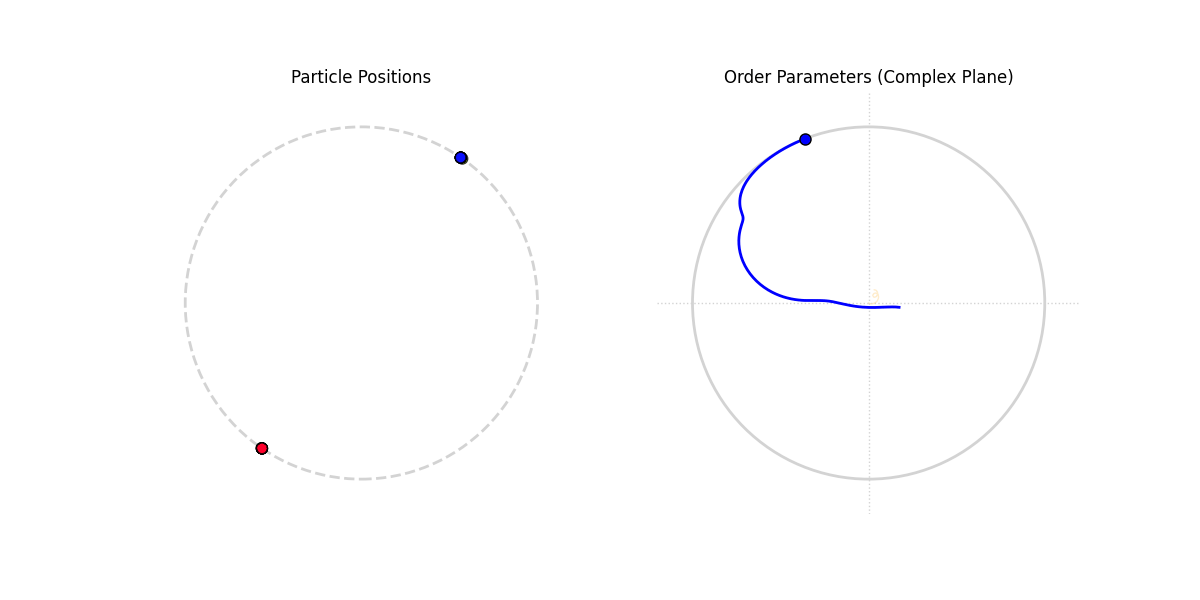}
\end{subfigure}%
{\unskip\ \vrule\ }
\begin{subfigure}{.24\textwidth}
\includegraphics[trim={22cm 4cm 5.3cm 4cm},clip,width=\textwidth]{figs/Numerics/converge2.png}
\end{subfigure}
\end{tcolorbox}
\vspace{.1pt}
\begin{tcolorbox}[colframe=black, colback=white, arc=0mm, sharp corners, boxrule=0.5pt, left=1mm, right=1mm, top=1mm, bottom=1mm]
\begin{subfigure}[t]{.24\textwidth}
\includegraphics[trim={5.3cm 4cm 22cm 4cm},clip,width=\textwidth]{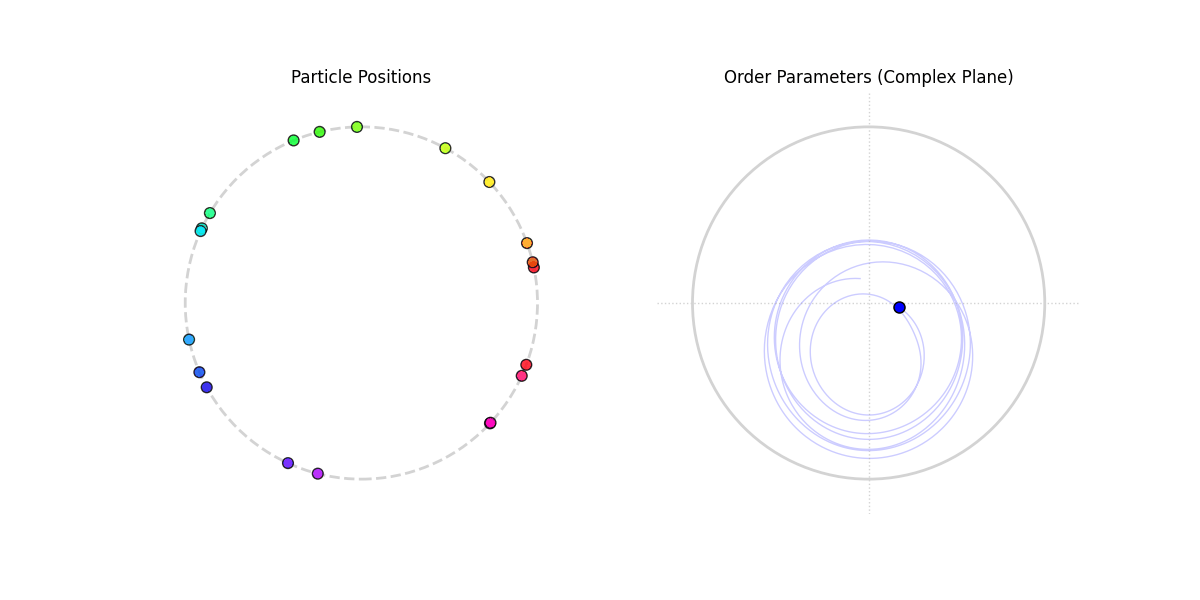}
\end{subfigure}%
\begin{subfigure}[t]{.24\textwidth}
\includegraphics[trim={5.3cm 4cm 22cm 4cm},clip,width=\textwidth]{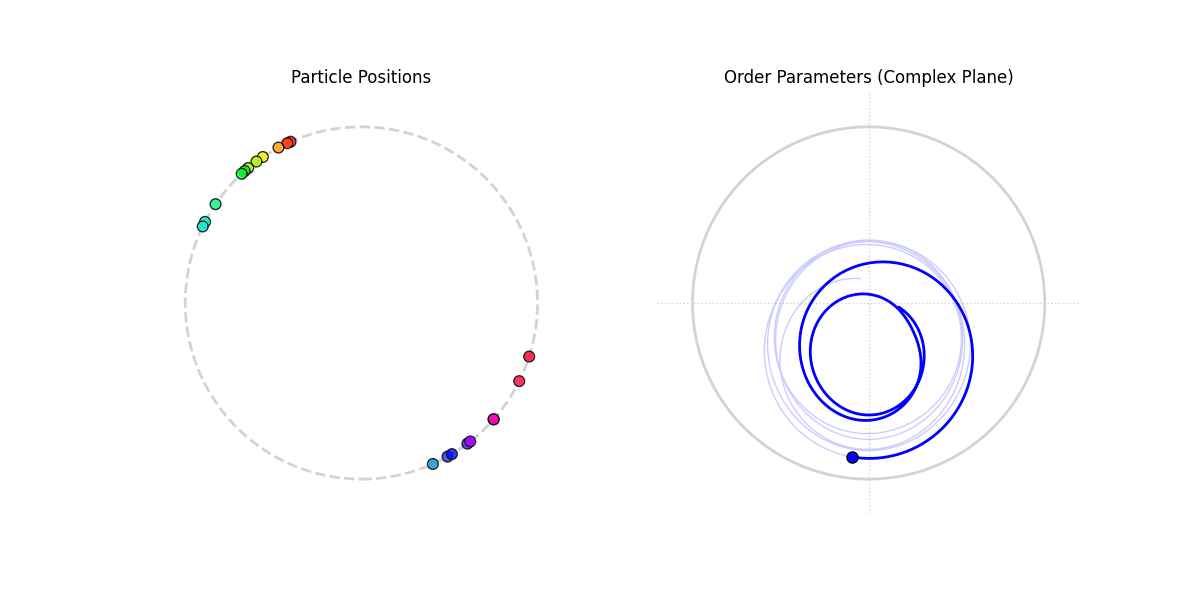}
\end{subfigure}%
\begin{subfigure}[t]{.24\textwidth}
\includegraphics[trim={5.3cm 4cm 22cm 4cm},clip,width=\textwidth]{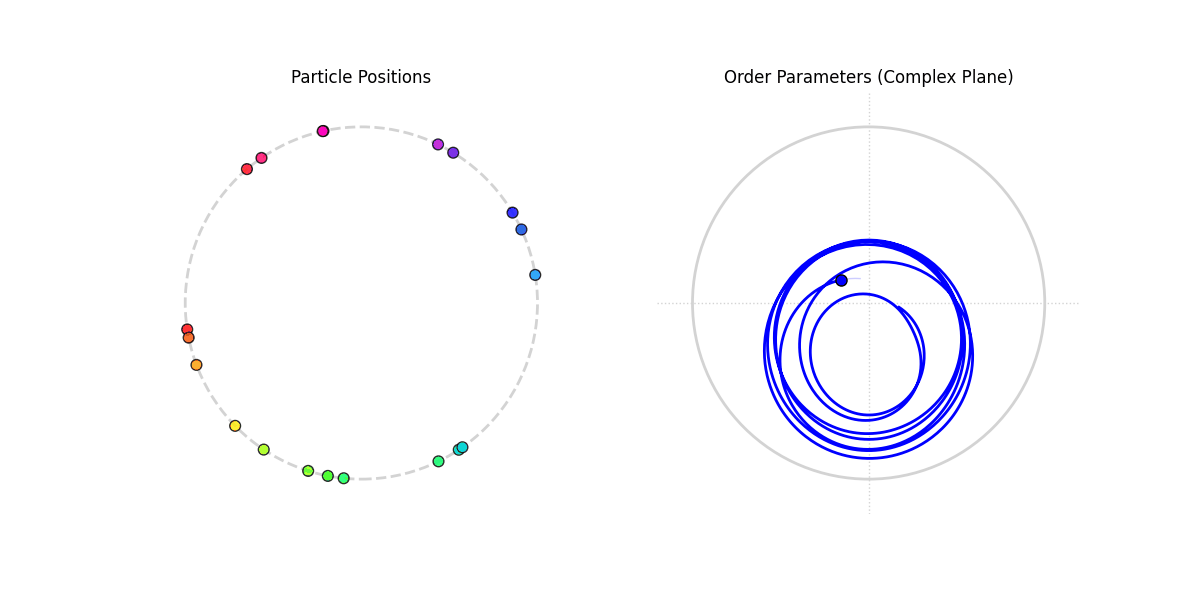}
\end{subfigure}%
{\unskip\ \vrule\ }
\begin{subfigure}{.24\textwidth}
\includegraphics[trim={22cm 4cm 5.3cm 4cm},clip,width=\textwidth]{figs/Numerics/noconverge2.png}
\end{subfigure}
\end{tcolorbox}

\caption{Snapshots of two systems with $20$ particles evolving according to \eqref{eq:LSA} under different choices of the matrices $A$ and $V$. 
The left three columns show the particle positions at three points in time, and the right-most column shows the corresponding time evolution of the order parameter $\CR_2^n(t)$. 
In the system at the top, the dynamics form two antipodal clusters, and the associated quantity $|\CR_2^n(t)|$ converges to $1$. 
In the system at the bottom, the particles neither cluster nor converge, and the associated quantity $|\CR_2^n(t)|$ does not tend to $1$.}
\label{fig:R2behavior}
\end{figure}

The behavior of $\CR_2$ becomes considerably more tractable for a special class of initialization of the angle $\Theta$.
More precisely, we prove that there exists a \textit{finite-dimensional} submanifold $\M_{\mathrm{OA}}$ of the space of probability measures over $\BBT^1 := \R/2\pi \mathbb{Z}$, referred to as the \emph{Ott--Antonsen (OA) manifold} \cite{ott2008low}, that is \emph{invariant} under the dynamics of the double-angle variable $\Xi := 2\Theta$, interpreted modulo $2\pi$.
Furthermore, within this submanifold, the dynamics of $\Xi$ are completely determined by a single ODE of the order parameter $\CR_2$ in the complex plane. This ODE takes the form 
\begin{equation}\label{eq:r2_dyn_in_intro}
\Dot{\CR}_2(t) = 2i \left( B(\CR_2(t)) (\CR_2(t))^2 + C(\CR_2(t)) \CR_2(t) + \overline{B}(\CR_2(t)) \right),
\end{equation}
where the functions $B$ and $C$ are linear in $\CR_2$ and its complex conjugate, with coefficients depending on the matrices $A$ and $V$; see \eqref{eq:coeff_B_C}. 

This reduction allows us to study the influence of the matrices $A$ and $V$ on the behavior of the mean-field PDE through the simpler ODE \eqref{eq:r2_dyn_in_intro} for the order parameter $\CR_2$, when the dynamics is initialized on the OA manifold.
On top of that, we further establish structural stability results showing that, for certain choices of $A$ and $V$, dynamics initialized near the OA manifold exhibit the same long-term behavior as those initialized exactly on it. 


\subsection{Contributions}
The main contributions of this paper are the following.

\begin{enumerate}[label=\textup{(\roman*)}]
\item \textbf{A pure second-harmonics Kuramoto formulation of linear self-attention.}
We show that, in dimension $d=2$, the linear self-attention model \eqref{eq:LSA} can be formulated as a generalized Kuramoto model with pure second harmonics; see Theorem \ref{thm:finite_sys}.
This formulation reveals an intrinsic low-dimensional structure generated by the Watanabe--Strogatz transformation.
Importantly, this property holds \emph{independently} of the matrices $A$ and $V$,
and forms the basis of an analytical framework that differs fundamentally from the Wasserstein gradient flow approach.

\item \textbf{Diverse dynamics on the OA manifold.} We prove that, for a class of initializations associated with the Ott--Antonsen (OA) manifold, the linear self-attention model \eqref{eq:LSA} exhibits clustering behavior for a broad class of matrices $A$ and $V$, going substantially beyond the assumptions considered in the existing literature.
We further show that, in other parameter regimes, the model can display a rich range of dynamical behaviors, including oscillations and bifurcations.
These oscillations and bifurcations are characterized by uncovering a hidden Hamiltonian structure in the dynamics.
The results are summarized in Table \ref{tab:case_study}. 
To the best of our knowledge, these are the first theoretical characterizations linking such diverse dynamical behaviors in transformer models to explicit parameter choices. 
    
\begin{table}[!htb] 
\centering
\begin{threeparttable}
\renewcommand{\arraystretch}{1.5}
\begin{tabular}{>{\raggedright\arraybackslash}p{0.09\textwidth}
                >{\centering\arraybackslash}p{0.24\textwidth}
                >{\centering\arraybackslash}p{0.12\textwidth}
                >{\centering\arraybackslash}p{0.14\textwidth}
                >{\centering\arraybackslash}p{0.27\textwidth}
                }
\toprule
\textbf{} & \textbf{Class of matrices} & \textbf{Clustering} & \textbf{Convergence} & \textbf{Assumptions} \\
\midrule

\textbf{Case 1} (Thm.\@~\ref{thm:case_1})
& $A$ arbitrary, $V = I$ 
& \cmark
& \cmark 
& $\tr{A} > 0$ or $\det(A) \leq 0$  \\

\textbf{Case 2} (Thm.\@~\ref{thm:case_2})
& $A = I$, $V$ symmetric 
& \cmark
& \cmark
& $\lambdaMax(V) \geq 0$ \\

\textbf{Case 3}  (Thm.\@~\ref{thm:case_3})
& $A$ diagonal, $V$ with two degrees of freedom 
& \cmark
& \xmark 
& $A$ and $V$ satisfy \eqref{eq:case3_cond} \\

\textbf{Case 4}  (Thm.\@~\ref{thm:case_4})
& $A = I$, $V = \left(\begin{smallmatrix} v_{11} & v_{12} \\ -v_{12} & -v_{11} \end{smallmatrix}\right)$ 
& \xmark
& \xmark
& $v_{11} < v_{12}$  \\



\bottomrule
\end{tabular}
\caption{Summary of clustering and convergence regimes across four case studies.}
\label{tab:case_study}
\end{threeparttable}
\end{table}

    \item \textbf{Stability beyond the OA manifold.} We establish a structural stability result showing that, for certain choices of $A$ and $V$, dynamics initialized near the OA manifold exhibit the same long-time behavior as those initialized exactly on the manifold.
    The corresponding results are presented in Theorems \ref{thm:case_1stab} and \ref{thm:case_2stab}.
    Our argument goes beyond a standard Grönwall estimate by exploiting the structural properties of the dynamics induced by the linear self-attention model, which allows us to obtain the stability over an infinite time horizon.
    The techniques developed here may also be of independent interest for the study of Kuramoto-type collective dynamics.

    \item \textbf{Numerical evidence.}
    We provide numerical experiments in higher dimensions $d$ showing that the
    qualitative behavior proved for linear self-attention persists beyond the
    $2$-dimensional setting. We also observe analogous behavior for the softmax self-attention;
    see Section~\ref{sec:experiments}.
\end{enumerate}

\subsection{Related Works}

\paragraph{Mathematical analysis of the transformer inference dynamics.}
One fundamental question in the mathematical study of transformer architectures is to characterize the behavior of token representations as they are propagated through the layers of the transformer model \cite{sander2022sinkformers,noci2022signal,geshkovski2023mathematical}. 
We refer to this evolution of token representations as the transformer inference dynamics.

The approach we adopt is to model the inference dynamics of encoder-only transformers with self-attention modules as interacting multi-particle systems, an idea pioneered in \cite{lu2019understanding,dutta2021redesigning}.
This viewpoint offers a continuous-time interpretation of the transformer architecture and, in the long-context regime \cite{peng2024yarn,yang2025qwen3,swan2025puvvada,qu2026tabiclv2},
naturally leads to mean-field perspectives, 
a strategy that we adopted also in this work.
Earlier studies in this direction have focused on self-attention-only transformers, due to the self-attention modules constituting the central novelty of the transformer architecture.
They have connected the self-attention dynamics to transport-type PDEs \cite{sander2022sinkformers}
and have
developed a mathematical framework for analyzing transformer architectures through the lens of Wasserstein gradient flows \cite{sander2022sinkformers,geshkovski2023mathematical,burger2025analysis}
and optimization perspectives \cite{alcalde2026quantifying}.

Building on this dynamical-system perspective, one of the main behaviors that has been characterized is \emph{token concentration} in the self-attention dynamics: after many self-attention layers, all token representations concentrate at a lower-dimensional region of the embedding space. 
In case of collapse to a single fixed point, this is  referred to as token uniformity, over-smoothing \cite{chen2022principle,guo2023contranorm,wu2023demystifying,chen2025residual}, or rank collapse \cite{dong2021attention,feng2022rank,cowsik2025geometric,giorlandino2026two}.
By analyzing the energy landscape of the Wasserstein gradient flow induced by the \emph{softmax} self-attention layers, earlier works \cite{geshkovski2023mathematical,criscitiello2024synchronization,polyanskiy2025synchronization} have shown that, under relatively restrictive assumptions on the matrices $A$ and $V$, such as $A = V = I$, the self-attention dynamics converge to a Dirac distribution.
Subsequent works have relaxed these assumptions, either by refining the analytical techniques within the Wasserstein gradient flow framework \cite{chen2025quantitative} or by adopting a different geometric approach from control theory \cite{abella2024asymptotic,abella2025consensus,altafini2026multistability} and optimization~\cite{alcalde2026quantifying},
leading to more general statements about concentration phenomena in mean-field transformers.
These types of cluster formation result have also been extended to hardmax-attention \cite{alcalde2025clustering,alcalde2025attention}, causal self-attention \cite{karagodin2024causal,abella2024asymptotic,abella2025consensus}, multi-head attention \cite{abella2024asymptotic,abella2025consensus,massucco2026multi}, transformers with fully-connected MLP layers \cite{alvarez2026perceptrons,koubbi2026homogenized,fedorov2026clustering,agazzi2026stochastic,alcalde2026reachability}, and noisy transformer dynamics \cite{balasubramanian2025structure,alvarez2026perceptrons,koubbi2026homogenized,fedorov2026clustering,agazzi2026stochastic,engel2026random,mun2026phaseA,mun2026phaseB,mun2026bphaseC}. 

Beside token representation collapse to a single point, several other theoretical properties of the transformer dynamics have also been established, including metastability \cite{geshkovski2024dynamic,bruno2024emergence,bruno2025multiscale}, connections to optimization algorithms \cite{alcalde2025attention,alcalde2026quantifying}, and scaling limits of hyperparameters \cite{chen2025critical,bruno2026scaling,giorlandino2026two}.
We refer the reader to \cite{rigollet2025mean} for a more systematic survey.

Our results in Theorems \ref{thm:case_1} and \ref{thm:case_2} also fall within the general direction of characterizing token clustering.
However, we diverge from existing lines of work by developing an \emph{entirely new analytical framework} by connecting the \emph{linear} self-attention dynamics to Kuramoto-type models with pure second-harmonic coupling; see the next paragraphs for a review.
This framework allows us to characterize cluster formation under more general assumptions on the parameter matrices $A$ and $V$.
More importantly, it enables us to identify parameter regimes in which the associated self-attention dynamics exhibit behaviors beyond cluster formation, including oscillation and bifurcation, as demonstrated by Theorems \ref{thm:case_3} and \ref{thm:case_4}.
This broader treatment comes at the cost of restricting our current analysis to the linear self-attention model in dimension $d=2$. In Section \ref{sec:Conclusions} we discuss some potential directions to extend our analysis to higher dimensional settings.  

\paragraph{Kuramoto model and transformer dynamics.}
The Kuramoto model~\cite{kuramoto1975,kuramoto1984chemical} is a classical mathematical model for studying synchronization in large populations of coupled oscillators, with broad applications in the natural sciences~\cite{emerging2002istvan,generative2010breakspear,acebron2005kuramoto}, engineering~\cite{wiesenfeld1998frequency,filatrella2008analysis,dorfler2014synchronization}, and more recently, machine learning~\cite{bohm2010clustering,nguyen2024coupled,miyato2025artificial,song2026kuramoto}. In its simplest form the model takes the form
\begin{equation}
    \label{eq:related_kuramoto}
\dot \theta_k = \omega_k + \frac{\kappa}{n} \sum_{j=1}^n \sin (\theta_j - \theta_k)
\end{equation}
for $n$ oscillators $\theta_k$ with natural frequencies $\omega_k$ and a global coupling constant $\kappa$.
One important generalization of the original model is the Kuramoto--Daido model~\cite{daido1992order,generic1994daido}, in which the sine coupling, corresponding to the first Fourier harmonic, is replaced by a general periodic coupling function that may contain higher-order or multiple Fourier modes.
A recent connection has been established between transformer dynamics and Kuramoto--Daido dynamics~\cite{geshkovski2023mathematical}, where in dimension $d=2$, the model \eqref{eq:transformer} with matrices $Q^{\top} K = A = V = I$ and a general attention function $h$ reduce to the Kuramoto--Daido model.
Subsequent works identified sufficient conditions on $h$ for synchronization~\cite{criscitiello2024synchronization,polyanskiy2025synchronization}, and analyzed bifurcation and phase-transition phenomena in noisy Kuramoto--Daido models through the associated McKean--Vlasov equations~\cite{balasubramanian2025structure,mun2026phaseA}.
However, in the context of transformer dynamics \eqref{eq:transformer}, these analyses are restricted to the special matrix choice $A = V = I$.

One important mathematical advance in the analysis of Kuramoto-type models comes from identifying low-dimensional structures in Kuramoto--Daido dynamics when the coupling function is a \emph{pure} Fourier harmonic.
This direction was pioneered by Watanabe and Strogatz~\cite{watanabe_integrability}, who developed a theory to explain phase coherence in coupled arrays of Josephson junctions.  
Their technique applies to the more general Kuramoto--Daido models, but it took some time for the underlying mechanism of the reduction to be discovered as a M\"obius group action~\cite{marvel2009identical}.
Independently, Ott and Antonsen~\cite{ott2008low,Ott_2009} showed that the mean-field Kuramoto dynamics further reduce to a single ODE for the order parameter when the oscillator positions are initially chosen from a certain family.
It turns out that these reductions are closely related: the Ott--Antonsen ansatz corresponds to a special initialization of oscillator positions within the Watanabe--Strogatz framework~\cite{marvel2009identical}.
These dimension-reduction techniques have also been extended to Kuramoto--Daido models whose coupling function are pure higher-order Fourier harmonics~\cite{skardal2011cluster,Gong_2019}.

In this paper, we make a new observation connecting the Kuramoto--Daido model~\cite{daido1992order} and the transformer dynamics \eqref{eq:transformer}.
Specifically, Theorem~\ref{thm:finite_sys} shows that the transformer dynamics \eqref{eq:transformer} with the attention function $h(y) = y$ can be reformulated as a Kuramoto--Daido model with \emph{pure second-harmonic coupling} for \emph{arbitrary} matrices $A$ and $V$.
This observation allows us to apply the Watanabe--Strogatz transformation and Ott--Antonsen ansatz to reduce the mean-field linear self-attention dynamics to a single complex-valued governing ODE. We then use this reduced dynamics to study how different classes of matrices $A$ and $V$ determine the long-time behavior of the linear self-attention dynamics, going substantially beyond the previously studied parameter regimes.  We also note that Ott and Antonsen demonstrated the OA manifold is globally attracting under mild assumptions on the oscillator frequencies~\cite{Ott_2009}.  Unfortunately, these assumptions do not apply to our model, hence we conduct a separate stability analysis in Section~\ref{sec:stability}.

\subsection{Outline} The rest of the paper is organized as follows. 
In Section~\ref{sec:2}, we revisit the interacting particle system induced by the linear self-attention model~\eqref{eq:LSA} in dimension $d=2$ and show that it can be reformulated as a Kuramoto-type model with pure second-order harmonics.
We further show that this system admits a low-dimensional structure through the Watanabe--Strogatz (WS) transformation and the Ott--Antonsen (OA) ansatz, leading to reduced and trackable dynamics.
In Section~\ref{sec:case_study}, we present case studies for different choices of the matrices $A$ and $V$, and analyze the long-time behavior of the reduced dynamics under the OA ansatz in each case.
In Section~\ref{sec:stability}, we establish structural stability results beyond the OA ansatz for some choices of $A$ and $V$, thereby describing the asymptotic behavior of the dynamics without requiring special initializations. 
In Section \ref{sec:experiments}, we present a variety of numerical experiments where we explore the validity of our theoretical findings from Section \ref{sec:case_study} for the $2$-dimensional linear self-attention model \eqref{eq:LSA} in higher dimensions. 
We observe that the behaviors identified in the case $d = 2$ also arise in higher dimensions. 
We conclude the paper in Section \ref{sec:Conclusions}.
The main notation used throughout the manuscript is summarized in Table \ref{tab:notation}.



\section{Linear Self-Attention Model on the Circle and its Low-Dimensional Structure}\label{sec:2}

In this section,
we show that the linear self-attention model \eqref{eq:LSA} admits an intrinsic low-dimensional structure
by rewriting it, in dimension $d = 2$, as a Kuramoto-type model with pure second-order harmonics.
In Section~\ref{subsec:LSA}, we establish this connection and discuss the evolution of the mean-field's second order parameter $\CR_2$. Thereafter, in Section~\ref{subsec:WS_OA},
we apply the Watanabe--Strogatz transformation~\cite{watanabe_integrability} to show that the original system of $n$ coupled equations is completely governed by \emph{two} coupled ODEs.
For a special class of initializations, corresponding to the Ott--Antonsen (OA) ansatz~\cite{ott2008low}, this WS-reduced system further simplifies, leading to a single closed complex-valued ODE capturing the behavior of the full linear self-attention dynamics~\eqref{eq:LSA}.


\subsection{Linear Self-Attention Model on the Circle}\label{subsec:LSA}
We consider the linear self-attention model \eqref{eq:LSA} in the $2$-dimensional setting $d=2$, i.e., on the unit circle $\BBS^{1}$.
In this case, each token $x_k(t) \in \BBS^1$ is completely determined by an angle $\theta_k(t) \in \BBT^1 = \R/2\pi \mathbb{Z}$ since $x_k(t) = \cos(\theta_k(t)) e_1 + \sin(\theta_k(t)) e_2$, where $e_1 = (1,0)^\top$ and $e_2 = (0,1)^\top$ denote the standard basis vectors in $\R^2$. In what follows, recall that we have set $\beta=1$ without loss of generality.

For each $k \in \{1,\dots,n\}$, the induced dynamics for the angle $\theta_k$ is given by
\begin{equation}
\label{eq:LSA:1d}
\begin{split}
    \Dtheta_k(t)
    &= -\frac{1}{n \sin (\theta_k(t))} \left( \sum_{j=1}^n  \left\langle x_k(t), A x_j(t) \right\rangle  \big\langle V x_k(t), e_1 - \left\langle x_k(t), e_1 \right\rangle x_k(t) \big\rangle \right).
\end{split}
\end{equation}
This expression follows by differentiating the identity $\cos(\theta_k(t)) = \langle x_k(t), e_1 \rangle$ w.r.t.\@ time $t$, which yields $\langle \Dot{x}_k(t), e_1 \rangle=\frac{d}{dt}\cos(\theta_k(t)) = -\sin(\theta_k(t))\Dtheta_k(t)$, and then substituting the evolution equation~\eqref{eq:LSA} together with the definition of the projection~$\P_{x_k(t)}^{\perp}$.

To expand~\eqref{eq:LSA:1d}, we write the matrices $A$ and $V$ as
\begin{equation}\label{eq:matrices_A_V_general}
A := \left(\begin{array}{ll}
    a_{11} & a_{12} \\
    a_{21} & a_{22} 
\end{array} \right)
    \qquad
    \text{and}
    \qquad
V := \left( \begin{array}{ll}
    v_{11} & v_{12}  \\
    v_{21} & v_{22} 
\end{array} \right).
\end{equation}
With this notation, we show that the dynamics governed by \eqref{eq:LSA:1d} can be written as a Kuramoto-type model with pure second-order harmonics, as stated in the following proposition.
The proof is deferred to Appendix~\ref{app:finite_sys}.

\begin{theorem}
    \label{thm:finite_sys}
    Let $\{\theta_k(t)\}_{k=1}^n \in \BBT^1$ denote the angular variables associated with the LSA dynamics \eqref{eq:LSA} on $\BBS^1$.
    Define the second order parameter~$\CR_2^n$ as
    \begin{equation}\label{eq:order_param_finite}
        \CR_2^n(t) := \frac{1}{n} \sum_{j=1}^n e^{i2\theta_j(t)} .
    \end{equation}
    Then, for each $k \in \{1,\dots,n\}$, the angle $\theta_k$ satisfies the closed ODE
    \begin{equation}\label{eq:finite_sys}
        \Dot{\theta}_k(t) =  b \left(\theta_k(t), \CR^{n}_2 (t) \right), 
    \end{equation}
    where the vector field $ b: \BBT^1 \times \BBC \mapsto \R$ is given by
    \begin{equation}
        \label{eq:vf_b}
        b(\theta, r)
        := 2 \re{B(r) e^{i2\theta}} + C(r) .
    \end{equation}
    Here $B: \BBC \mapsto \BBC$ and $C: \BBC \mapsto \R$ are the coefficient functions defined in \eqref{eq:coeff_B_C}.
\end{theorem}
The resulting system \eqref{eq:finite_sys}, with vector field $b$ of the form \eqref{eq:vf_b}, is a Kuramoto-type model with \textit{pure} second-order harmonics~\cite{skardal2011cluster} (see Remark \ref{remark:modelling} below for a more detailed comparison with the standard Kuramoto model). 
A key feature of this class of models is its intrinsic low-dimensional structure.
In particular, the $n$-particle dynamics \eqref{eq:finite_sys} can be completely described by just \emph{two} macroscopic variables satisfying a closed low-dimensional system of coupled ODEs.
This leads to a substantial reduction in complexity: instead of analyzing $n$ coupled ODEs, or a PDE in the mean-field limit, one can study a system of only two coupled ODEs.
We discuss this reduction in Section~\ref{subsec:WS_OA}.

Before doing so, let us formulate the mean-field limit of the particle dynamics \eqref{eq:finite_sys}, which facilitates the analysis conducted in the subsequent sections.
Specifically, by letting the number of particles $n \rightarrow \infty$, we formally consider a process $\Theta = (\Theta(t))_{t \geq 0}$ evolving according to
\begin{equation}\label{eq:mf_sys}
    \Dot{\thetaMF}(t)
    = b(\thetaMF(t), \CR_2(t)), \qquad \Theta(0) \sim f(0, \dummy),
\end{equation}
where the second mean-field order parameter~$\CR_2$ is defined, accordingly, as
\begin{equation}
    \label{eq:mf_order_param}
    \CR_2(t)
    := \int_0^{2\pi} e^{i2\theta} f(t, d\theta) .
\end{equation}
Here, $f$ denotes the law of $\thetaMF$, i.e., $f := \Law (\thetaMF)$, and satisfies the continuity equation
\begin{equation}\label{eq:cont_eq}
    \partial_t f + \partial_{\theta} \left( b(\dummy, \CR_2(t)) f \right) = 0
\end{equation}
with the vector field $b$ as defined in \eqref{eq:vf_b}. Throughout the paper, we use either $f_t$ or $f(t, \dummy)$ to denote the distribution of $\Theta(t)$.

\begin{remark}
The dependence of the functions $B$ and $C$ on the matrices $A$ and $V$ is in general very complicated (see~\eqref{eq:coeff_B_C} for a full definition).  
In our analysis we have used computer algebra software\footnote{The associated code can be found in~\url{https://github.com/tmaranzatto/attention_dynamics}.} to aid in writing the ODE systems explicitly. 
Even in the simplest case $A = V = I$, it is not obvious that the system reduces to the standard Kuramoto-type model with pure second harmonic coupling. We choose to view the complicated dependence on $A$ and $V$ as a positive---our case analysis demonstrates the richness of dynamics induced by the linear self-attention layers.
\end{remark}

To expose the low-dimensional structure of both the finite particle system \eqref{eq:finite_sys} and its mean-field counterpart \eqref{eq:mf_sys}, and to track the behavior of the order parameter $\CR_2$, it is convenient to work with the \emph{double angles} $\xi_k:=2\theta_k$.
Specifically, we introduce the change of variable $\xi_k = 2\theta_k \in \BBT^1$ and $\Xi := 2\Theta \in \BBT^1$, with both variables understood modulo $2\pi$. 
The dynamics of the transformed particle system $\{\xi_k\}_{k=1}^n$ and its mean-field limit can then be derived directly from the corresponding equations~\eqref{eq:finite_sys} and \eqref{eq:mf_sys} for $\{\theta_k\}_{k=1}^n$ and $\Theta$, respectively.
In particular, the variables $\{\xi_k\}_{k=1}^n$ satisfy
\begin{equation}\label{eq:xi_finite_sys}
\Dxi_k(t) = \tb \left(\xi_k(t), \CR_2^n(t) \right),
\end{equation}
where the vector field $\tb$ is induced by $b$ in \eqref{eq:vf_b} through
\begin{equation}
    \label{eq:vf_tb}
\tb (\xi, r) := 2 b (\xi/2, r).
\end{equation}
Moreover, the second order parameter $\CR_2^n$ can be expressed purely in terms of the double angles $\{\xi_k\}_{k=1}^n$ since
\begin{equation}\label{eq:order_param_finite_in_xi}
\CR_2^n(t) = \frac{1}{n} \sum_{j=1}^n e^{i2\theta_j(t)} = \frac{1}{n} \sum_{j=1}^n e^{i\xi_j(t)}.
\end{equation}
Analogously, at the mean-field we obtain a process $\Xi$ evolving according to
\begin{equation}\label{eq:xi_mf_sys}
\Dot{\Xi} (t) = \tb \left( \Xi (t), \CR_2(t) \right), \qquad \Xi(0) \sim g(0, \dummy).
\end{equation}
Here, $g$ denotes the law of $\Xi$, i.e., $g := \Law (\Xi)$, which satisfies the continuity equation
\begin{equation}
\partial_t g + \partial_{\xi}\, ( \tb(\dummy, \CR_2(t)) g ) = 0
\label{eqn:ContEqDoubled}
\end{equation}
with the vector field $\tb$ as defined in \eqref{eq:vf_tb}.
The second mean-field order parameter $\CR_2$ defined in \eqref{eq:mf_order_param} admits the equivalent representation
\begin{equation}\label{eq:order_param_mf_in_xi}
\CR_2(t) = \int_0^{2\pi} e^{i2\theta} f(t, d\theta) = \int_0^{2\pi} e^{i\xi} g (t, d\xi),
\end{equation}
where we recall that $f = \Law(\Theta)$ is the law of the original angle variable $\Theta$.
The relationship between the distributions $f$ and $g$ is given by
\begin{equation}\label{eq:relation_f_and_g}
g(t, \xi) = \frac{1}{2} \left( f(t, \xi/2) + f(t, \xi/2 + \pi) \right),
\end{equation}
which is a consequence of the change of variables formula on $[0, 2\pi)$ due to $\Xi = 2\Theta \in \BBT^1$.

In the next section, we apply the Watanabe--Strogatz transformation \cite{watanabe_integrability} and the Ott--Antonsen ansatz \cite{ott2008low} to the double-angle dynamics \eqref{eq:xi_finite_sys} (and to its mean-field limit \eqref{eq:xi_mf_sys}) to reveal its underlying low-dimensional structure. This will enable a more tractable mathematical analysis.


\begin{remark}\label{remark:modelling}
    The classical Kuramoto model~\eqref{eq:related_kuramoto} can be expressed as $\dot \theta_k(t) = \omega_k + \operatorname{Im}({H(t)e^{i\theta_k(t)}})$, where $H$ is typically a function of the first order parameter $\mathcal{R}_1^n(t) := \frac{1}{n}\sum_{j=1}^n e^{i\theta_j(t)}$.
    Our model differs in a few ways.
    As stated above, our angles are coupled via the second-order harmonics $e^{i2\theta}$.
    Furthermore, our field extracts the real part of the second harmonic instead of the imaginary part. 
    The term $C(\mathcal{R}_2^n)$ can be interpreted as a time-varying external forcing which acts on all particles and is controlled by the second order parameter $\mathcal{R}_2^n$. Lastly, our model assumes that all intrinsic frequencies are zero, i.e., $\omega_k = 0$ for all $k=1,\dots,n$. Despite these differences the techniques used to analyze the standard Kuramoto system still apply to our model.
\end{remark}

\subsection{Watanabe--Strogatz Transformation and Ott--Antonsen Ansatz}\label{subsec:WS_OA}
In this section, we investigate the low-dimensional structure of the finite-particle double-angle dynamics for $\{\xi_k\}_{k=1}^n$ in \eqref{eq:xi_finite_sys} and its mean-field counterpart $\Xi$ in $\eqref{eq:xi_mf_sys}$ using the Watanabe--Strogatz (WS) transformation \cite{watanabe_integrability,marvel2009identical} and the Ott--Antonsen (OA) ansatz \cite{ott2008low}.
In particular, the WS transformation shows that the $n$-particle system \eqref{eq:xi_finite_sys} can be fully described by \emph{two} macroscopic variables, denoted by $\alpha$ and $\eta$, which satisfy a coupled system of ODEs.
In the mean-field regime, under a special choice of initialization for the mean-field variable $\Xi$, the dynamics of \eqref{eq:xi_mf_sys} reduces further to a \emph{single} complex variable $\alpha$, which, in this case, coincides with the order parameter $\CR_2$, that is, $\alpha = \CR_2$.
This is the Ott--Antonsen ansatz.

\subsubsection{Watanabe--Strogatz Transformation}\label{subsec:WS}
We begin by introducing the Watanabe--Strogatz (WS) transformation, which reveals the intrinsic low-dimensional structure of the systems \eqref{eq:xi_finite_sys} and \eqref{eq:xi_mf_sys}.
The corresponding result is stated in the following proposition, whose proof is deferred to Appendix~\ref{app:WS_OA}.

For fixed $\alpha\in\BBC$ and $\eta \in \R$, we define the Möbius transformation~$M_{\alpha, \eta}: \BBC \mapsto \BBC$ by
\begin{equation}\label{eq:Mobius_tf}
M_{\alpha, \eta} (u) := \frac{\alpha + e^{i\eta} u}{1 + \alphaBar e^{i\eta} u}.
\end{equation}

\begin{proposition}[Watanabe--Strogatz Transformation~\cite{marvel2009identical}]\label{prop:WS_transform}
Let $\{\xi_k\}_{k=1}^n$ be the solution to \eqref{eq:xi_finite_sys}.
Then there exist time-independent angles $\{\varphi_k\}_{k=1}^n$ such that, for each $k \in \{1,\dots,n\}$, it holds
\begin{equation}\label{eq:action_of_M_finite_sys}
e^{i \xi_k(t)} = M_{\alpha^n(t), \eta^n(t)} \left( e^{i \varphi_k} \right), \quad t \geq 0. 
\end{equation} 
Here, $\alpha^n: [0, \infty) \mapsto \BBC$ and $\eta^n: [0, \infty) \mapsto \R$ satisfy the system of ODEs
\begin{subequations}\label{eq:WS_variable_finite_sys}
\begin{align}
\dot{\alpha}^n(t) &= F(\alpha^n(t), \CR_2^n(t)), \label{eq:alpha_finite_sys}\\
\dot{\eta}^n(t) &= G(\alpha^n(t), \CR_2^n(t)),
\end{align}
\end{subequations}
where the functions $F$ and $G$ are defined for $\alpha,r\in\BBC$ and $\eta\in\R$ by
\begin{equation}\label{eq:F_and_G}
F(\alpha, r) := 2i \left( B(r) \alpha^2 + C (r) \alpha + \overline{B} (r)\right), \qquad G(\alpha, r) := 2\left(2 \re{B(r) \alpha} + C(r) \right),
\end{equation}
with $B$ and $C$ as in \eqref{eq:coeff_B_C}, and $\CR_2^n$ is the second order parameter defined in \eqref{eq:order_param_finite_in_xi}.
\end{proposition} 

The previous proposition says that the evolution of the $n$-particles system \eqref{eq:xi_finite_sys} is completely characterized by the two variables $\alpha$ and $\eta$, together with $n$ time-independent angles $\{\varphi_k\}_{k=1}^n$.
This reduces the description of the dynamics from $n$ coupled differential equations in \eqref{eq:xi_finite_sys} to the two equations in \eqref{eq:WS_variable_finite_sys}.
For later reference, we refer to $\alpha^n$ and $\eta^n$, as well as $\alpha$ and $\eta$ introduced shortly, as the \emph{WS variables}; to $\{\varphi_k\}_{k=1}^n$ as the \emph{constants of motion}; and to system \eqref{eq:WS_variable_finite_sys} (thus also \eqref{eq:mf_WS_param}) as the \emph{WS system}. 

In the mean-field limit, an analogous representation holds.
The double angle $\Xi$, which evolves according to the mean-field dynamics \eqref{eq:xi_mf_sys}, can be written in terms of a time-varying Möbius transformation on the complex unit circle, i.e.,
\begin{equation}\label{eq:action_of_M_mf_sys}
    e^{i \,\Xi (t)} = M_{\alpha(t), \eta(t)} \left( e^{i\varphi} \right),
\end{equation}
where $\varphi$ is a time-independent random variable (in analogy with the finite particle system, we will continue to refer to it as a constant of motion) and where the Möbius transformation $M_{\alpha, \eta}$ is as in \eqref{eq:Mobius_tf}.
The WS variables $\alpha$ and $\eta$ satisfy the coupled ODE system
\begin{subequations}\label{eq:mf_WS_param}
\begin{equation}\label{eq:alpha_dyn}
\dot{\alpha}(t) = F(\alpha(t), \CR_2(t)),
\end{equation}
\begin{equation}\label{eq:eta_dyn}
\dot{\eta}(t) = G(\alpha(t), \CR_2(t)) ,
\end{equation}
\end{subequations}
where $F$ and $G$ are as in \eqref{eq:F_and_G}, and where $\CR_2$ is the mean-field's second order parameter defined in \eqref{eq:order_param_mf_in_xi}.

\paragraph{Möbius transformation.}
Let us now discuss the Möbius transformation \eqref{eq:Mobius_tf}.  Below we reuse the same variables as in Section~\ref{subsec:WS}, however for the current discussion these should be seen as arbitrary constants. It is straightforward to verify that, when $|\alpha| \not =1$, $M_{\alpha, \eta}$ defines a bijection (and a diffeomorphism) from the unit circle $\BBS^1$ in the complex plane~$\BBC$ onto itself.
Its inverse is given by
\begin{equation}\label{eq:inverse_M}
M_{\alpha, \eta}^{-1} (v) := e^{-i\eta} \frac{v - \alpha}{1 - \alphaBar v},
\end{equation}
which can also be written as $M_{\tilde \alpha, \tilde \eta}$ for $\tilde \eta := -\eta$ and $\tilde{\alpha} := - e^{-i\eta} \alpha$. If, on the other hand, $|\alpha|=1$, the transformation $M_{\alpha, \eta}$ is constant and equal to $\alpha$. This can be seen from the fact that, when $|\alpha| = 1$, one can write $\alpha = e^{i\Phi}$ (for $\Phi \in \R$) and the Möbius map takes the form
\begin{equation}
    M_{\alpha, \eta} (u) = \frac{\alpha + u e^{i\eta}}{1 + \alphaBar u e^{i\eta}} = e^{i\Phi} \frac{1 + e^{-i\Phi} u e^{i\eta}}{1 + e^{-i\Phi} u e^{i\eta}} = \alpha,
    \label{eqn:MobiusUnitNormAlpha}
\end{equation}
provided $ue^{i\eta} \neq -\alpha$.

From the above discussion we conclude that, regardless of the values of $\alpha$ and $\eta$, each Möbius transformation $M_{\alpha, \eta}$ induces a map 
\begin{equation}\label{eq:map_T}
T_{\alpha, \eta} : [0, 2\pi) \mapsto [0, 2\pi),\qquad T_{\alpha,\eta}(\varphi) = \Arg\!\left(
\frac{\alpha+e^{i(\eta+\varphi)}}{1+\bar\alpha e^{i(\eta+\varphi)}}
\right)
\end{equation}
as illustrated in the diagram in Figure~\ref{fig:Mobius_tf}.
Therefore, any probability distribution $\nu$ on the constant-of-motion variable $\varphi$ induces a corresponding family of distributions $g$ on the physical variable $\Xi$ via pushforwards by the family of maps $\{ T_{\alpha, \eta}\}_{\alpha \in \C \setminus \S^1 , \eta \in \R}$. More precisely, each distribution in this family takes the form
\begin{equation}\label{eq:g_and_nu}
    g = (T_{\alpha, \eta})_{\sharp} \nu,
\end{equation} 
where $(T_{\alpha, \eta})_{\sharp} \nu$ denotes the pushforward of the measure $\nu$ under the map $T_{\alpha, \eta}$. 

\begin{figure}[!htb]
    \centering
    \begin{tikzcd}[column sep=5.0em, row sep=4.5em]
    \varphi \in [0,2\pi) \arrow[r, "p"] \arrow[d, dashed, "{T_{\alpha,\eta}}"'] 
    & u = e^{i\varphi} \in \mathbb{S}^1 \arrow[d, "M_{\alpha,\eta}"] \\
    \Xi \in [0,2\pi) 
    & e^{i \Xi} = M_{\alpha,\eta}(u) \in \mathbb{S}^1 \arrow[l, "p^{-1}"'] .
    \end{tikzcd}
    \caption{Diagram illustrating the Möbius transformation.}
    \label{fig:Mobius_tf}
\end{figure}
In light of \eqref{eq:action_of_M_mf_sys}, \eqref{eq:order_param_mf_in_xi} and \eqref{eq:Mobius_tf}, we can thus express the second order parameter $\CR_2(t)$ as 
\begin{equation}\label{eq:R2_push_foward}
\CR_2(t) = \int_{0}^{2\pi} e^{i \xi } g(t, d\xi) = \int_{0}^{2\pi} e^{i \xi} \left(T_{\alpha(t), \eta(t)}\right)_{\sharp} \nu  (d\xi) = \int_0^{2\pi} M_{\alpha(t), \eta(t)} (e^{i\varphi}) \nu (d\varphi) .
\end{equation}

We now discuss the interpretation of the WS variables $\alpha(t)$ and $\eta(t)$.
\begin{itemize}
    \item \textbf{Variable $\alpha(t)$.}
   The variable
    $\alpha(t)$ characterizes the level of synchronization of the particles.
    In particular, $|\alpha(t)| = 1$ implies $|\CR_2(t)| = 1$, which corresponds to complete synchronization (as discussed in previous sections). Indeed, when $|\alpha(t)| = 1$,  \eqref{eq:R2_push_foward} and \eqref{eqn:MobiusUnitNormAlpha} imply that $\CR_2(t) = \alpha(t)$ and in particular $|\CR_2(t)| = |\alpha(t)| = 1$. Conversely, it is straightforward to see that if the measure $\nu$ is assumed to have full support, then $|\CR_2(t)|=1$ only if $|\alpha(t)|=1$. 
    
    As discussed in Figure~\ref{fig:R2behavior}, we know that  $|\CR_2(t)|=1$ if and only if the corresponding particle distribution $f$ is supported on a single point or two antipodal points. For this reason, we will be interested in tracking the long-term behavior of the variable $|\alpha|$ and determining whether $|\alpha(t)| \rightarrow 1$ as $t \rightarrow \infty$.

    \item \textbf{Variable $\eta(t)$.}
    The variable $\eta(t)$ can be interpreted as a rotating frame acting on all particles. 
    Specifically, if we set $\alpha(t) \equiv 0$ for simplicity, then the Möbius transformation \eqref{eq:Mobius_tf} reduces to
    \begin{equation}
    e^{i\xi_k(t)} = M_{\alpha(t), \eta(t)} (e^{i \varphi_k}) = \frac{\alpha(t) + e^{i\eta(t)} e^{i \varphi_k}}{1 + \alphaBar(t) e^{i\eta(t)} e^{i \varphi_k}} = e^{i (\varphi_k + \eta(t))}
    \end{equation}
    Hence, in this case, $\eta(t)$ simply shifts all phases by the same amount, corresponding to a global rotation. In general, one can view $M_{\alpha(t), \eta(t)}$ as the composition of a rotation (counterclockwise by $\eta(t)$) and the transformation $M_{\alpha(t), 0}$.
\end{itemize}

Despite the simpler form of the WS system  \eqref{eq:mf_WS_param}, it remains difficult to analyze it in general since the variables $\alpha(t)$ and $\eta(t)$ are implicitly coupled through the second order parameter $\CR_2(t)$, whose dependence on $\alpha(t)$ and $\eta(t)$ can be quite intricate, as can be appreciated from equation \eqref{eq:R2_push_foward}.
However, as we will show in the next section, for a special choice of the distribution $\nu$, or, equivalently, for a corresponding special initialization of the double angle $\Xi(0) = 2\Theta(0)$, the WS system \eqref{eq:mf_WS_param} further decouples.
In this case, $\CR_2(t)$ coincides with the variable $\alpha(t)$ and becomes independent of the rotation variable $\eta(t)$.
This is precisely the setting of the Ott--Antonsen ansatz~\cite{ott2008low}.

\begin{remark}
    The intrinsic coupling to $\CR_2^n$ in the WS system~\eqref{eq:WS_variable_finite_sys} means there is no computational reduction compared to numerically evaluating the untransformed system.  In the simplest setting, using Euler's method to compute the next state requires $\mathcal{O}(n)$ operations for both the untransformed and WS systems.  This is straightforward to see for the original system; as for the WS system note $\CR_2$ is expressed in terms of the $n$ constants of motion $\varphi_k$ as seen in \eqref{eq:R2_push_foward}.  Savings may be found in the mean-field limit, where the computation of $\CR_2$ in the WS system involves integrating over the time-invariant measure  $\nu$, as opposed to the time-varying measure $f$.  If the integral \eqref{eq:R2_push_foward} can be efficiently approximated (or indeed analytically solved as in the next section), then there are potentially huge computational savings --- the WS transform reduces an infinite dimensional PDE into a finite dimensional ODE.  
    
\end{remark}

\subsubsection{Ott--Antonsen Ansatz}\label{subsec:OA}


Recall from the diagram in Figure~\ref{fig:Mobius_tf} that any probability distribution $\nu$ on the constant of motion $\varphi$ induces a corresponding family of distributions $g$ on the physical variable $\Xi$ via pushforwards by the maps $T_{\alpha, \eta}$.
In general, the resulting distributions $g=T_{\alpha, \eta\sharp} \nu$ may be quite complicated.
However, if $\nu$ is the uniform distribution on $[0, 2\pi)$, and $|\alpha|<1$, then the distribution of $\Xi$ is a \emph{wrapped Cauchy distribution} \cite{mardia2009directional}, associated with densities in the form of the Poisson kernel
\begin{equation}\label{eq:WC_dis}
    \WC (\xi; \alpha)
    := \frac{1}{2\pi} \frac{1 - |\alpha|^2}{|e^{i\xi} - \alpha|^2},
\end{equation}
parameterized by the single complex number $\alpha \in \BBC$ with $|\alpha|<1$; see Figure~\ref{fig:WC_dis} for an illustration. 

\begin{figure}[!htb]
    \centering
    \includegraphics[width=0.65\linewidth]{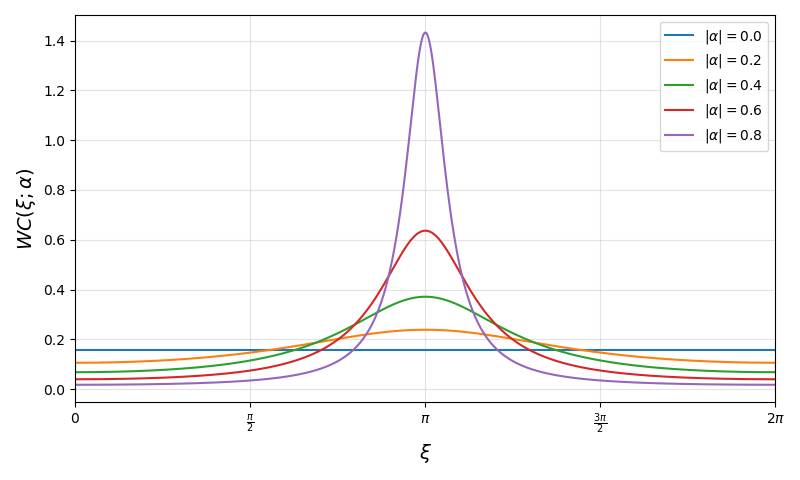}
    \caption{The family of wrapped Cauchy distributions generated by fixing the argument $\Arg(\alpha) = \pi$ and varying the magnitude $|\alpha|$.}
    \label{fig:WC_dis}
\end{figure}
In particular, $g = (T_{\alpha, \eta})_{\sharp}\mathrm{Unif}([0, 2\pi)) = \WC(\dummy; \alpha)$ when $|\alpha|<1$ (see Lemma \ref{lem:WS_to_OA} below). Note that the dependence on $\eta$ of the induced measures disappears because the uniform distribution is invariant under rotations. Also, note that $ \WC(\dummy; \alpha)$ for $\alpha=0$ is equal to $\mathrm{Unif}([0, 2\pi))$ and that $\WC(\dummy; \alpha)$ for $|\alpha|=1$ can be identified with a Dirac delta at $\Arg(\alpha)$. We refer to the induced family of distributions parameterized by $\alpha \in \C$ with $|\alpha|\leq1$ as the \emph{Ott--Antonsen (OA) manifold} \cite{Ott_2009}.

From formula \eqref{eq:R2_push_foward} it follows that the second order parameter $\CR_2$ depends only on $\alpha$ and is thus independent of the rotating-frame parameter $\eta$. As a consequence, the dynamics of $\alpha$ (recall \eqref{eq:alpha_dyn}) are independent of $\eta$.
This special case is exactly the setting of the Ott--Antonsen ansatz \cite{ott2008low, Ott_2009}.
Based on \cite{marvel2009identical}, we summarize the above discussion in the following lemma.
A detailed proof is deferred to Appendix~\ref{app:WS_OA}.
\begin{lemma}\label{lem:WS_to_OA}
Let $\Xi \in [0,2\pi)$ and $\varphi \in [0,2\pi)$ be two random variables satisfying the relation $\Xi = T_{\alpha, \eta} (\varphi)$ for some $\alpha \in \C$ with $|\alpha|<1$, where the map $T_{\alpha, \eta}$ is as in \eqref{eq:map_T}. Then the distribution $g(\dummy)$ of $\Xi$ is a wrapped Cauchy distribution of the form
\begin{equation}\label{eq:g_poisson}
g(\xi) = \WC(\xi; \alpha) 
\end{equation}
if and only if $\varphi \sim \nu =  \mathrm{Unif} ([0,2\pi))$.
\end{lemma}
Let us now discuss several direct consequences of Lemma~\ref{lem:WS_to_OA},
which highlight some favorable features of the setting of the Ott--Antonsen ansatz.

\begin{itemize}
    \item \textbf{The wrapped Cauchy family is invariant under the dynamics \eqref{eq:xi_mf_sys}.} 
    Because of the representation \eqref{eq:action_of_M_mf_sys}, the constant of motion $\varphi$ is completely determined by the initial phase $\Xi(0)$ together with the initial values $\alpha(0)$ and $\eta(0)$.
    More precisely, using \eqref{eq:action_of_M_mf_sys} together with the inverse Möbius transform \eqref{eq:inverse_M}, we obtain
    \begin{equation*}
    e^{i\varphi} = M_{\alpha(0), \eta(0)}^{-1} \left( e^{i \,\Xi(0)} \right) .
    \end{equation*}
   Consequently, for any fixed $\eta(0)$, $\varphi \sim \nu =   \mathrm{Unif}([0,2\pi))$ if and only if the initial data $\Xi(0)$ is distributed according to the wrapped Cauchy distribution $\WC(\xi; \alpha(0))$. 
    The value of $\eta(0)$ only rotates the constant of motion $\varphi$ and therefore does not affect this distributional equivalence.

    As a direct consequence, the mean-field dynamics \eqref{eq:mf_sys} leaves the wrapped Cauchy family invariant: if the initial distribution belongs to this family, then it remains in the same family for all later times.

    \item \textbf{Closed dynamics for $\CR_2$ and characterization of distribution $g$.} 
    If the distribution $g$ has the form $\WC(\xi; \alpha)$, then it admits the Fourier representation
    \begin{equation*}
    g(t, \xi) = \frac{1}{2\pi} \left(1 + 2\sum_{m=1}^{\infty} \re{\alphaBar^m e^{im \xi}} \right).
    \end{equation*}
    Consequently, the order parameter $\CR_2$ defined in \eqref{eq:order_param_mf_in_xi} satisfies
    \begin{equation}
       \CR_2 (t) = \alpha(t).
    \end{equation}

    In particular, $\CR_2$ depends only on $\alpha$, and thus the WS system \eqref{eq:mf_WS_param} decouples, as had already been mentioned earlier. Moreover, the second order parameter $\CR_2$ satisfies the closed evolution equation
    \begin{equation}\label{eq:dyn_r2}
    \Dot{\CR}_2(t) = F(\CR_2(t), \CR_2(t))
    = 2i \left( B(\CR_2(t)) (\CR_2(t))^2 + C(\CR_2(t)) \CR_2(t) + \overline{B}(\CR_2(t)) \right) .
    \end{equation}
    We further note that, in this case, the distribution $g$ is fully determined by the order parameter $\CR_2$. 
    Hence, tracking the behavior of $\CR_2$ is sufficient to characterize the full distribution of the double-angle variable $\Xi$.
\end{itemize}

We refer to the dynamics \eqref{eq:dyn_r2} as the \emph{OA system}.
Compared with the WS system~\eqref{eq:mf_WS_param}, the OA system \eqref{eq:dyn_r2} is generally more amenable to analysis.
Accordingly, in the subsequent Section \ref{sec:case_study},
we investigate the long-time behavior of the OA system for different choices of the matrices $A$ and $V$.
The resulting analysis allows us to infer the synchronization properties of the mean-field particle system \eqref{eq:xi_mf_sys} through the behavior of $\CR_2$.

It is important to note, however, that the finite-particle system~\eqref{eq:xi_finite_sys} cannot satisfy the OA ansatz exactly, since it is necessarily associated with an atomic initial distributions $g(0) = (T_{\alpha(0), \eta(0)})_{\sharp} \nu$, and hence with an atomic distribution $\nu$.


\begin{remark}\label{rem:g_and_f_wrapped_Cauchy}
Let us comment on the relation between the distribution $f$ of the original angle variable $\Theta$ and the distribution $g$ of the double-angle variable $\Xi = 2\Theta$ in the Ott--Antonsen setting.

Using \eqref{eq:relation_f_and_g} and the definition of the wrapped Cauchy distribution in \eqref{eq:WC_dis}, we observe that if $\Theta$ is distributed according to a wrapped Cauchy distribution with parameter $\alpha$, namely, $f = \WC (\dummy; \alpha)$, then the double-angle variable $\Xi$ is distributed according to a wrapped Cauchy distribution, now with parameter $\alpha^2$, that is, $g = \WC(\dummy; \alpha^2)$. 

The converse, however, does not hold in general: if $g$ is a wrapped Cauchy distribution, $f$ need not be a wrapped Cauchy.
For example, let $g$ be the uniform distribution $\mathrm{Unif}([0,2\pi)])$, which corresponds to the wrapped Cauchy distribution \eqref{eq:WC_dis} with $\alpha = 0$.
Define $f(\dummy) = \frac{1}{2\pi} (1 + \cos(\dummy))$.
Then $f$ and $g$ satisfy the relation \eqref{eq:relation_f_and_g}, but $f$ is not a wrapped Cauchy distribution.
More generally, the relation \eqref{eq:relation_f_and_g} does not uniquely determine $f$.
Indeed, whereas $g$ belonging to the wrapped Cauchy family necessarily has at most one local maximum, one may perturb $f$ by a sufficiently small $\pi$-antiperiodic function, provided the total mass and non-negativity are preserved, thereby obtaining admissible densities with arbitrary many local maxima.

Consequently, if the initial distribution of $\Theta(0)$ is a wrapped Cauchy distribution, then the induced initial distribution of $\Xi(0)$ is also a wrapped Cauchy distribution, and hence the dynamics of $\Xi$ starts on the OA manifold.
This manifold, however, is invariant only for the double angle dynamics~\eqref{eq:xi_mf_sys}, and not for the original dynamics~\eqref{eq:mf_sys} of $\Theta$.
In other words, under the setting of the Ott--Antonsen ansatz, only the distribution $g$ remains in the wrapped Cauchy family, whereas the distribution $f$ may not necessarily stay in this family over time. This can be illustrated by the fact that distributions in the Cauchy family are unimodal (see Figure \ref{fig:WC_dis}), while the distribution $f$ may concentrate around antipodal points.  
\end{remark}

\section{Long-Time Behavior Within the Ott--Antonsen Ansatz: Case Studies}
\label{sec:case_study}
In this section,
we investigate the long-time behavior of the dynamics~\eqref{eq:dyn_r2} of the second mean-field order parameter~$\CR_2$ under the OA ansatz,
with a particular emphasis on its dependence on the parameter matrices $A$ and $V$ of the transformer dynamics. 
We will show that, even for relatively simple choices of these parameter matrices, the system~\eqref{eq:dyn_r2} exhibits remarkably rich behavior. 
These different dynamical regimes of $\CR_2$ reveal the richness of the mean-field linear transformer dynamics~\eqref{eq:mf_sys}.

For the reader's convenience, let us recall the dynamics~\eqref{eq:dyn_r2}, which is given by
\begin{equation*}
    \Dot{\CR}_2(t) = F(\CR_2(t), \CR_2(t)) = 2i \left( B(\CR_2(t)) (\CR_2(t))^2 + C(\CR_2(t)) \CR_2(t) + \overline{B}(\CR_2(t)) \right) \, .
\end{equation*}
When the context is clear,
we will simply write $F(\CR_2(t))$ in place of $F(\CR_2(t), \CR_2(t))$.
We also introduce the polar representation of the order parameter, 
\begin{equation*}
    \CR_2(t) = \rho(t) e^{i\phi(t)},
\end{equation*}
where $\rho: [0, \infty) \mapsto [0,1]$ and $\phi: [0, \infty) \mapsto \BBT^1$.
Substituting this representation into \eqref{eq:dyn_r2},
we obtain the coupled ODE system:
\begin{equation}\label{eq:pc_dyn_general}
\Drho(t) = \re{e^{-i\phi(t)} F(\rho(t) e^{i\phi(t)})}, \qquad \Dphi(t) = \frac{1}{\rho(t)} \im{e^{-i\phi(t)} F(\rho(t) e^{i\phi(t)})} .
\end{equation}

In what follows, we present a series of case studies for four different choices of the matrices $A$ and $V$.
By analyzing either the complex ODE~\eqref{eq:dyn_r2} directly or the polar-coordinate system~\eqref{eq:pc_dyn_general} in each case, we show that $\CR_2(t)$ can display a rich variety of long-time behaviors.
These, in turn, correspond to distinct long-time behaviors of the mean-field linear transformer dynamics~\eqref{eq:mf_sys}.
We highlight three types of behavior for which we are able to draw definitive conclusions about the dynamics~\eqref{eq:mf_sys}, and we further explain their interpretations in the context of linear transformers, where the particles are viewed as token embeddings evolving through linear self-attention layers.

\begin{itemize}
    \item 
    \textbf{(Type 1).} Both $\rho = |\CR_2| \rightarrow 1$ and $\phi = \Arg(\CR_2) \rightarrow \phi_{\infty}$ converge.
    Correspondingly, the distribution $f$ of the particle $\Theta$ converges to a distribution supported on two fixed antipodal points on $\BBT^1$; see Case 1 and Case 2.

    In the linear transformer interpretation, this means that, after sufficiently many layers, the token embeddings cluster into two \emph{fixed} feature representations.

    \item \textbf{(Type 2).} While
    $\rho = |\CR_2| \rightarrow 1$ converges, $\phi = \Arg(\CR_2)$ does not.
    In this scenario, the distribution $f$ of the particle $\Theta$ converges to a distribution supported on two antipodal points on $\BBT^1$.
    However, these support points themselves continue to move and do not converge; see Case 3.

    In the linear transformer interpretation, this means that, after sufficiently many layers, the token embeddings cluster into two feature representations, but these representations continue to change as the number of layers increases.

    \item \textbf{(Type 3).}
    Neither $\rho = |\CR_2|$ or $\phi = \Arg(\CR_2)$ converges; instead, both continue to oscillate.
    Correspondingly, the distribution $f$ does not converge to a stationary distribution; see Case 4.

    In the linear transformer interpretation, this means that the token embeddings do not converge: even after arbitrarily many layers, the output continues to change as additional layers are applied.
\end{itemize}

Moreover, as a by-product of our analysis, we also identify bifurcations in the dynamics~\eqref{eq:dyn_r2}, highlighting the sensitivity of the long-time behavior to small changes in the matrices $A$ and $V$; see Case 4.

Before proceeding to the case studies, we introduce some notation.
We denote the unit disk and its boundary by
\begin{equation}
D := \left\{ z \in \BBC : |z| \leq 1 \right\}, \qquad \partial D := \left\{ z \in \BBC : |z| = 1 \right\} .
\end{equation}
For the entries of the matrices $A$ and $V$, we use the shorthand
\begin{equation*}
\begin{aligned}
&\aDPlus := a_{11} + a_{22}, \quad \aDMinus := a_{11} - a_{22}, \quad \aOPlus := a_{12} + a_{21}, \quad \aOMinus := a_{12} - a_{21},\\[5pt]
&\vDPlus := v_{11} + v_{22}, \quad \vDMinus := v_{11} - v_{22}, \quad \vOPlus := v_{12} + v_{21}, \quad \vOMinus := v_{12} - v_{21}.
\end{aligned}
\end{equation*}
The proofs of theorems in this section are deferred to Appendix \ref{app:case_study}.

\subsection{Case 1: \texorpdfstring{$V = I$, $A$}{VA} Arbitrary}
\label{sec:th_case1}

Let $A$ be a real-valued matrix such that $A \neq 0$ and $V$ be the identity matrix, 
\begin{equation}\label{eq:matrices_A_V_case_1}
A := \left(\begin{array}{ll}
    a_{11} & a_{12} \\
    a_{21} & a_{22} 
\end{array} \right), \qquad V := \left( \begin{array}{ll}
    1 & 0  \\
    0 & 1 
\end{array} \right) \,.
\end{equation}

In this case, the dynamics~\eqref{eq:dyn_r2} of the order parameter $\CR_2$ and its corresponding polar-coordinate system~\eqref{eq:pc_dyn_general} simplify to

\begin{equation}\label{eq:OAcase1}
\Dot{\CR}_2 = \frac{1}{4} \left( 1 - |\CR_2|^2 \right) \left( \left(\aDPlus + i\aOMinus \right) \CR_2 + \left( \aDMinus + i\aOPlus \right) \right),
\end{equation}
and
\begin{equation}\label{eq:pc_dyn_case_1}
\begin{aligned}
\Drho = \frac{1}{4} (1-\rho^2) \left( \aDPlus \rho + \aDMinus\cos \phi + \aOPlus \sin \phi \right), \qquad 
\Dphi = \frac{1}{4} (1-\rho^2) \left(\aOMinus - \frac{\aDMinus \sin \phi - \aOPlus \cos \phi}{\rho} \right) .
\end{aligned}
\end{equation}

The long-time behavior of the dynamical system~\eqref{eq:pc_dyn_case_1} is summarized as follows.

\begin{theorem}\label{thm:case_1}
Consider the dynamical system \eqref{eq:pc_dyn_case_1}.
\begin{itemize}
    \item[1.] If the matrix $A$ satisfies the condition:
    \begin{equation} \label{eq:case1_condition_E1}
    \tr{A} > 0 \quad \text{or} \quad  \det(A) \leq 0,
    \end{equation}
    then for almost every initial condition $(\rho(0), \phi(0)) \in (0,1] \times \BBT^1$, the solution of \eqref{eq:pc_dyn_case_1} converges to $(1, \phi_{\infty})$ for some $\phi_{\infty} \in \BBT^1$ as $t \rightarrow \infty$.

    \item[2.] If the symmetric part of $A$ is negative definite, i.e.
    \begin{equation}\label{eq:case1_condition_E2}
        \frac{A + A^{\top}}{2} \prec 0,
    \end{equation}
    then for every initial condition $(\rho(0), \phi(0)) \in (0,1) \times \BBT^1$, the solution of \eqref{eq:pc_dyn_case_1} converges to $(\rho_{\eq}, \phi_{\eq})$ as $t \rightarrow \infty$ with $\rho_{\eq} \in (0,1)$ and $\phi_{\eq} \in \BBT^1$; see the explicit form of $(\rho_{\eq}, \phi_{\eq})$ in \eqref{eq:rho*_phi*}.
\end{itemize}
\end{theorem}

\begin{figure}[h!]
\centering
\begin{subfigure}{.33\textwidth}
  \centering
  \includegraphics[trim={0cm 0cm 0cm 1.6cm},clip,width=\textwidth]{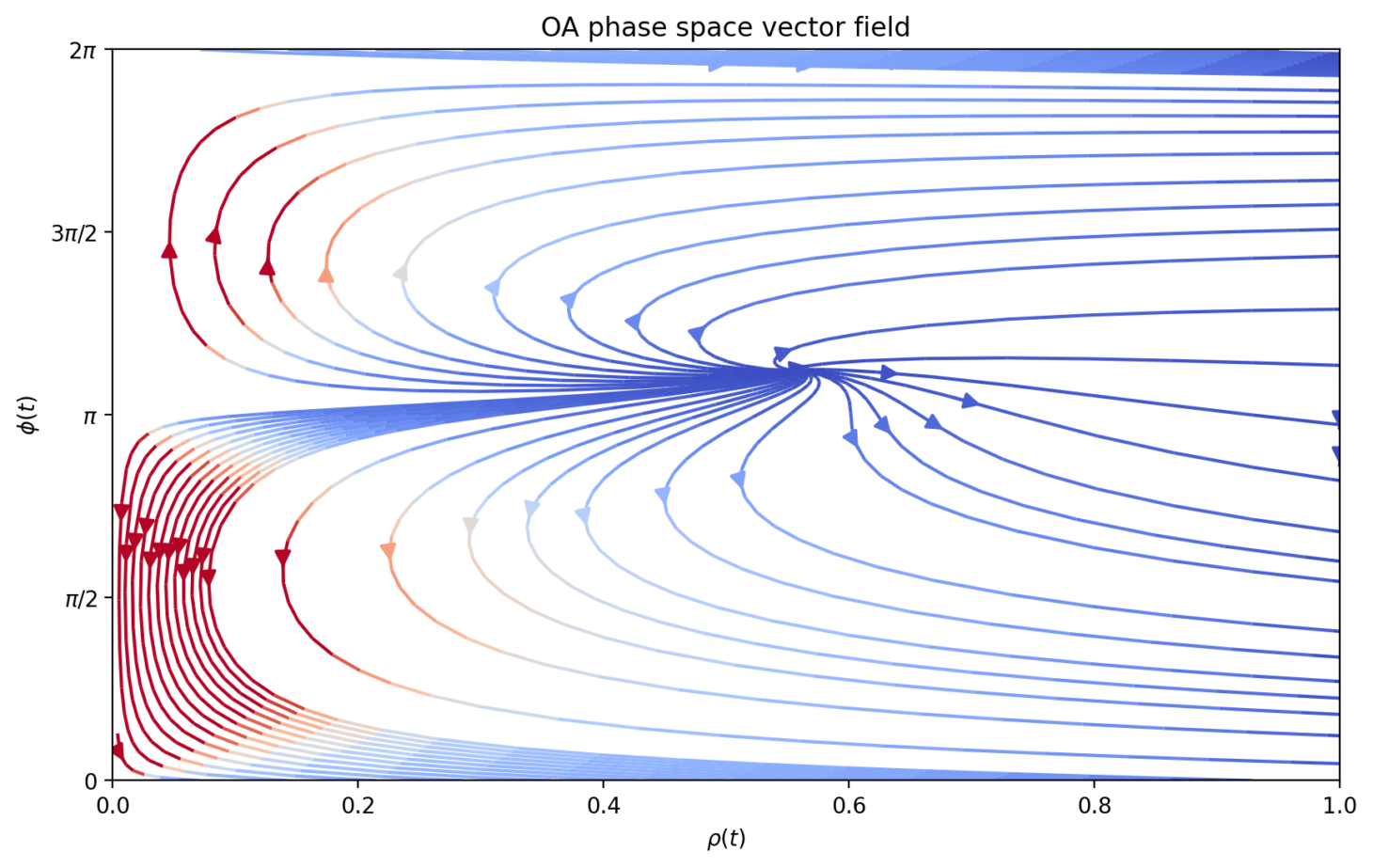}
  \subcaption{Phase portrait.}
  \label{fig:case_1_1_a}
\end{subfigure}%
\begin{subfigure}{.33\textwidth}
  \centering
  \includegraphics[trim={0cm 0cm 0cm 1.7cm},clip,width=\textwidth]{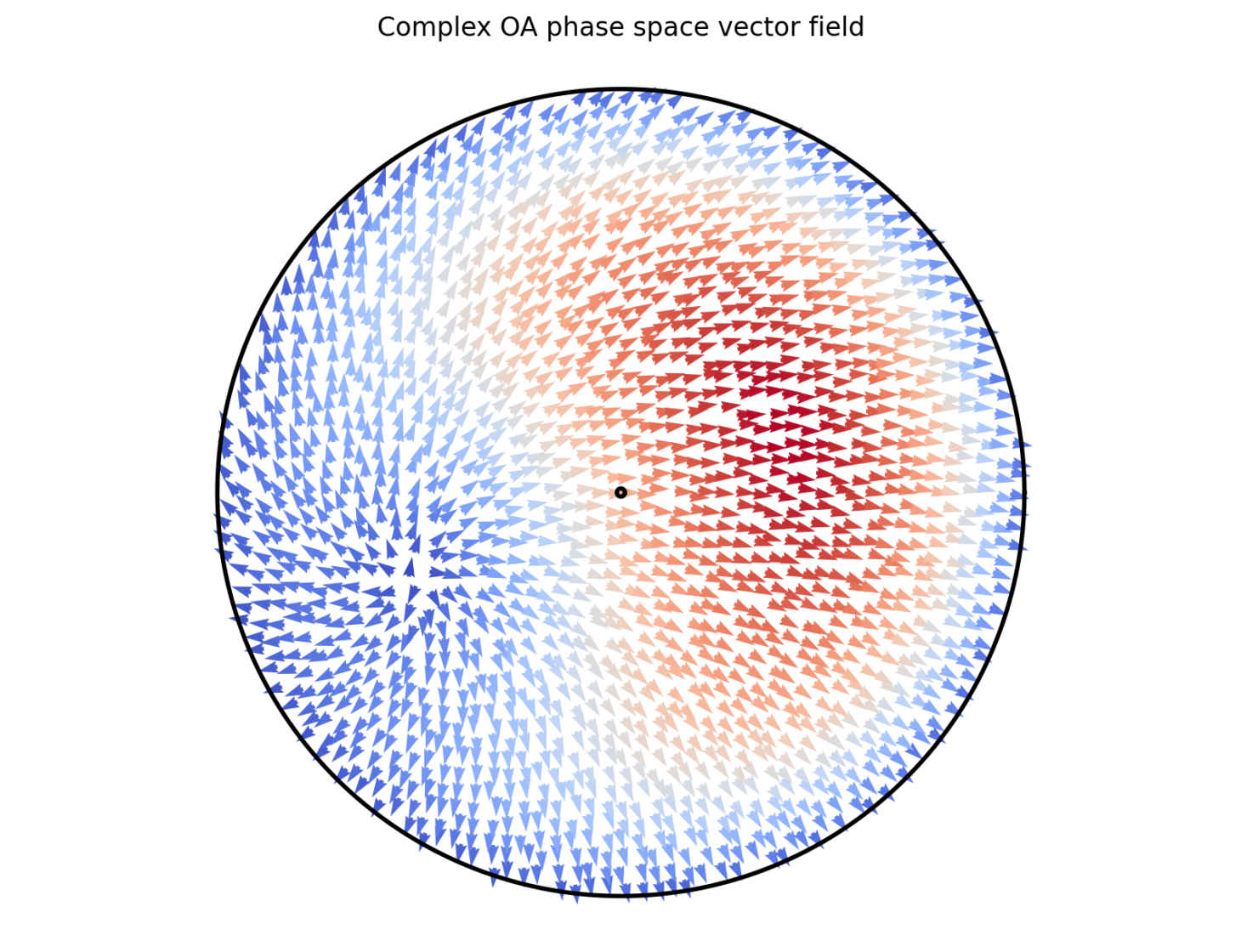}
  \subcaption{Complex phase portrait.}
  \label{fig:case_1_1_c}
\end{subfigure}%
\begin{subfigure}{.33\textwidth}
  \centering
  \includegraphics[trim={0cm 0cm 0cm 1.6cm},clip,width=\textwidth]{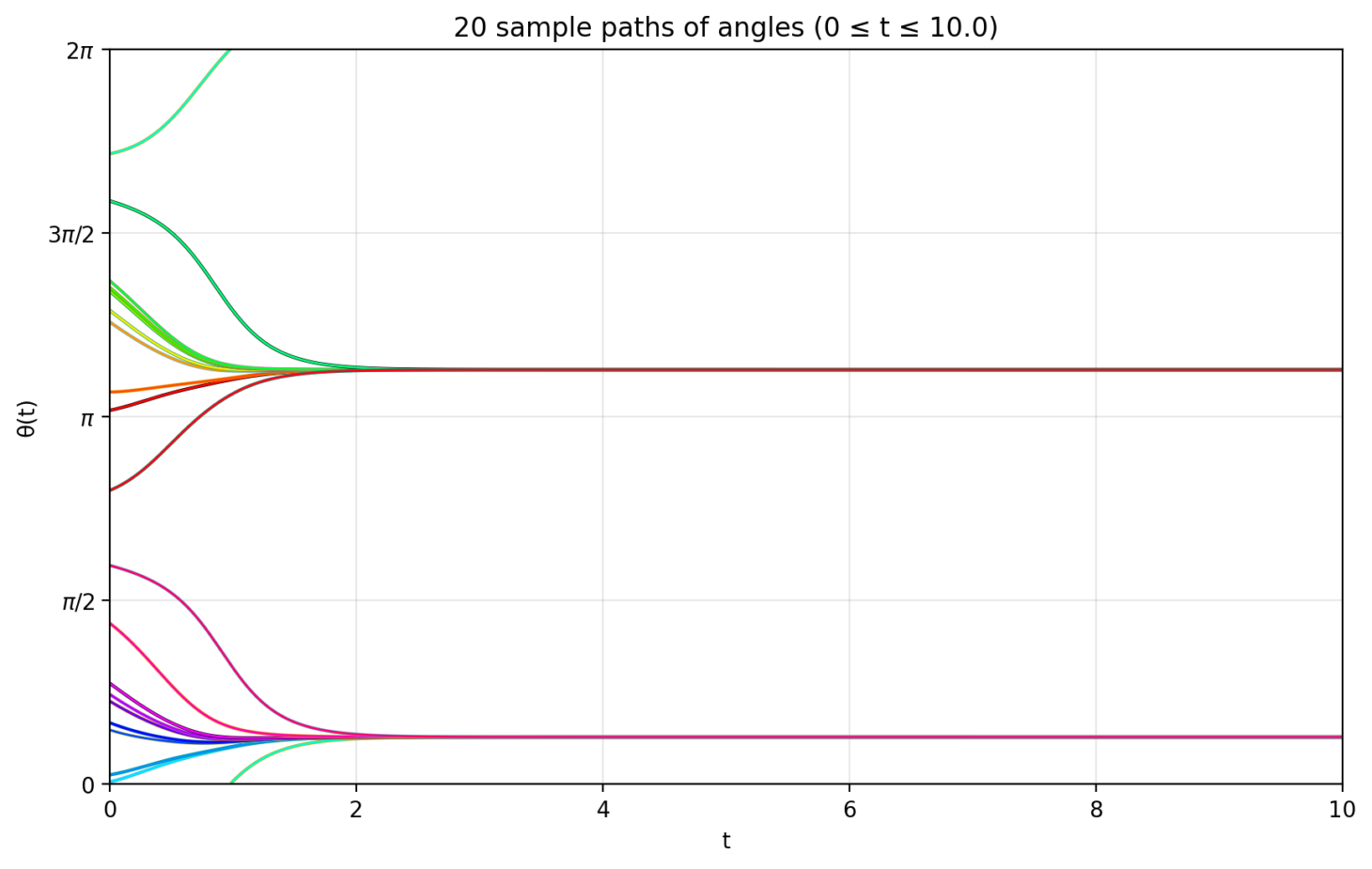}
  \subcaption{Particle dynamics.}
  \label{fig:case_1_1_b}
\end{subfigure}
\caption{Cartesian and complex phase portraits of $(\rho, \phi)$-dynamics \eqref{eq:pc_dyn_case_1} and the trajectories of particles evolving under the transformer dynamics \eqref{eq:finite_sys} for $V = I$, $A = \left(\begin{smallmatrix} 4 & -1 \\ 1 & 1 \end{smallmatrix}\right)$, which satisfies condition~\eqref{eq:case1_condition_E1}.  In the phase portraits, red flow lines or arrows indicate a fast velocity of $(\rho, \phi)$, and blue indicates a slow velocity.  The coloring in the cartesian and complex portraits do not align, due to the fast velocity near degenerate point $\rho = 0$.}
\label{fig:case_1_1}
\end{figure}
\begin{figure}[!htb]
\centering
\begin{subfigure}{.33\textwidth}
  \centering
  \includegraphics[trim={0cm 0cm 0cm 1.6cm},clip,width=\textwidth]{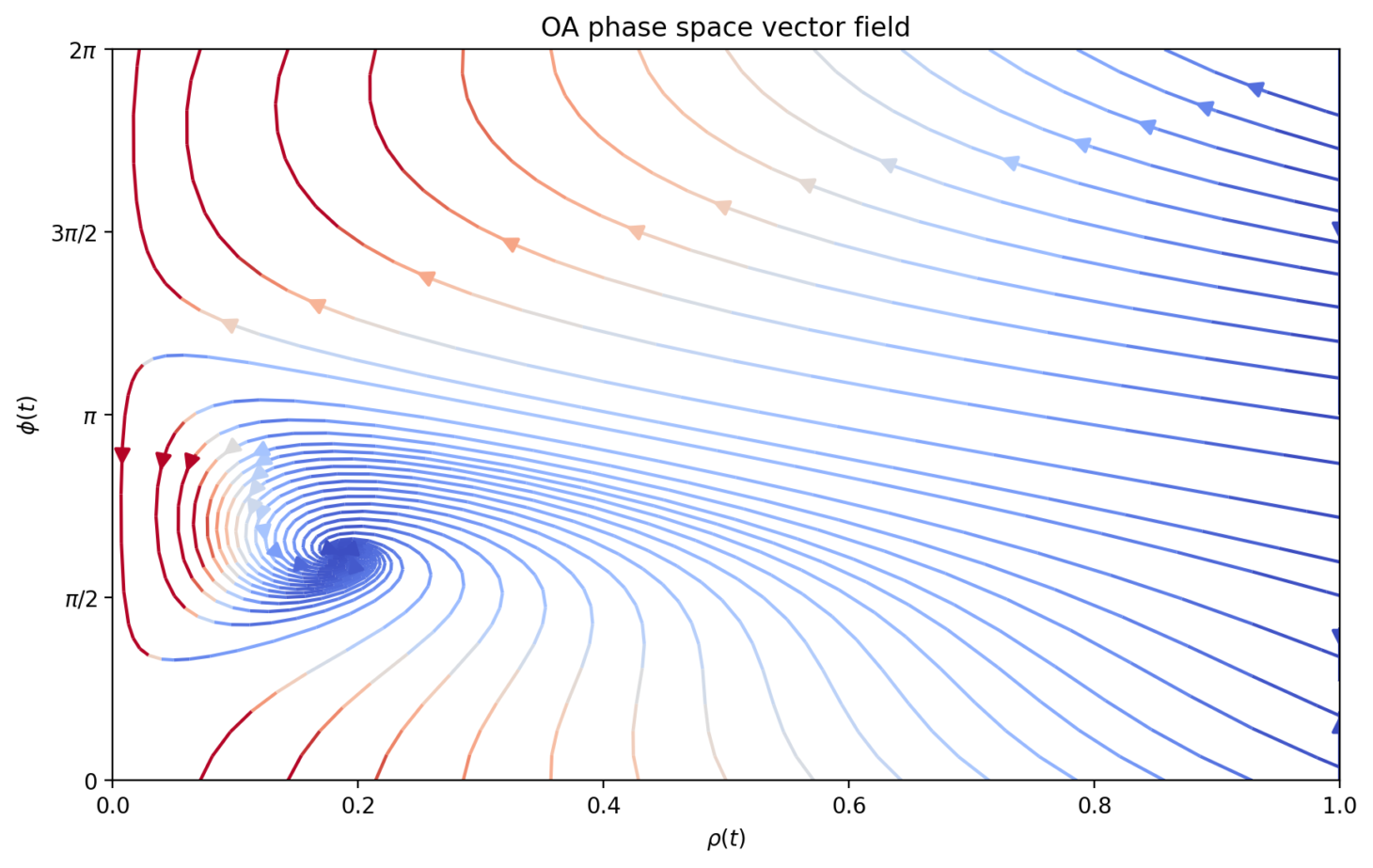}
  \subcaption{Phase portrait.}
  \label{fig:case_1_2_a}
\end{subfigure}%
\begin{subfigure}{.33\textwidth}
  \centering
  \includegraphics[trim={0cm 0cm 0cm 1.7cm},clip,width=\textwidth]{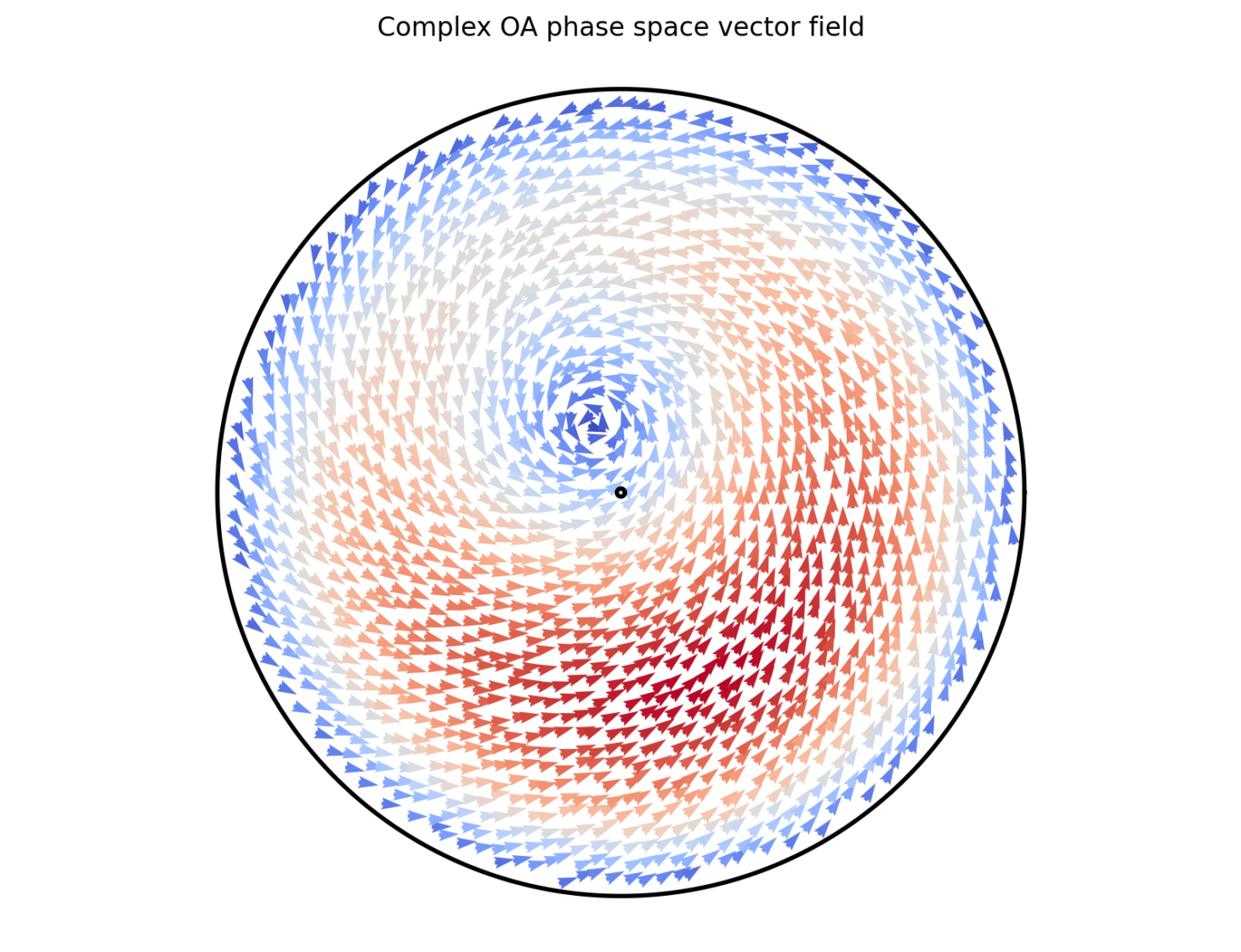}
  \subcaption{Complex phase portrait.}
  \label{fig:case_1_2_c}
\end{subfigure}%
\begin{subfigure}{.33\textwidth}
  \centering
  \includegraphics[trim={0cm 0cm 0cm 1.6cm},clip,width=\textwidth]{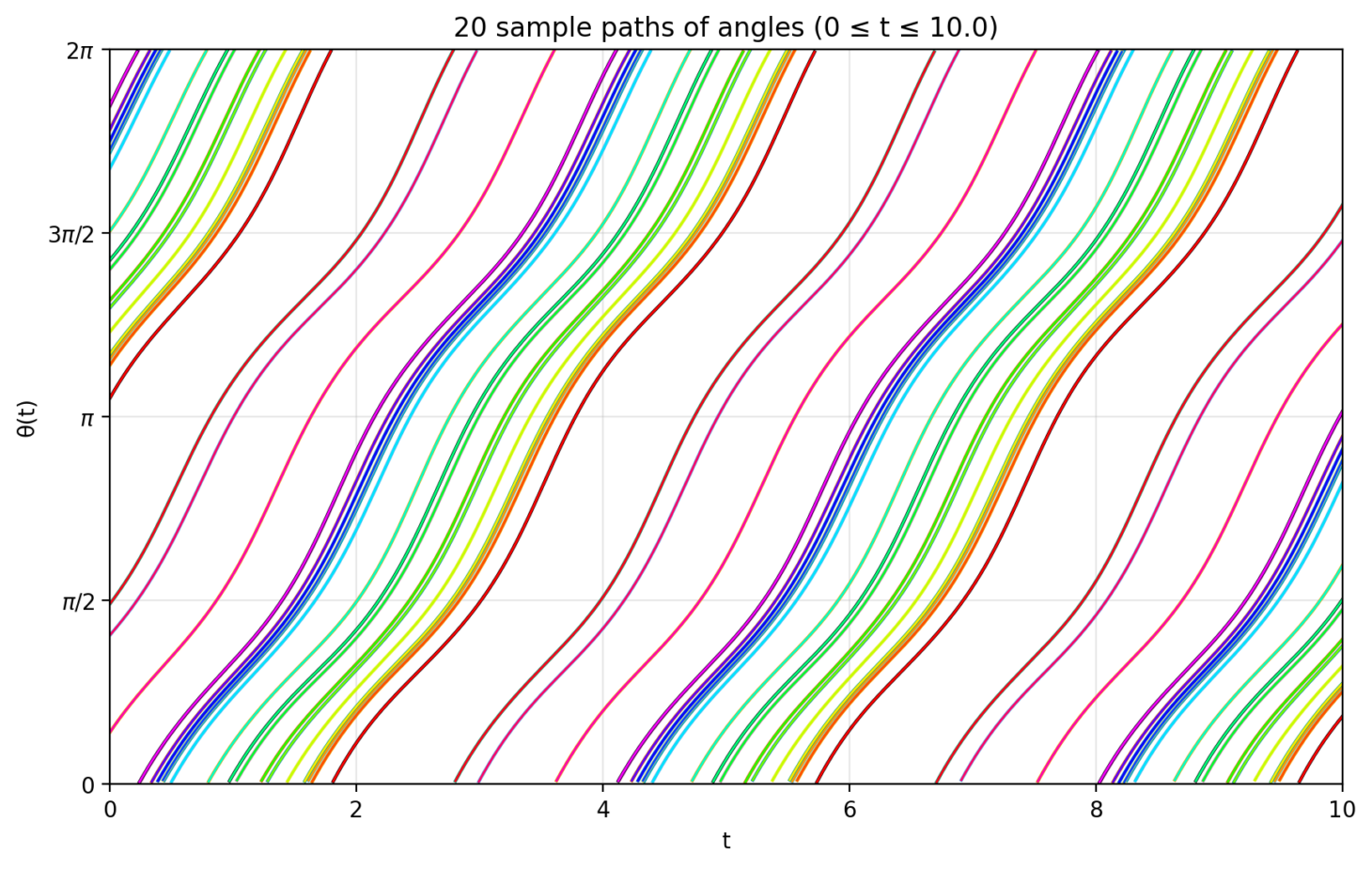}
  \subcaption{Particle dynamics.}
  \label{fig:case_1_2_b}
\end{subfigure}
\caption{Cartesian and complex phase portraits of $(\rho, \phi)$-dynamics \eqref{eq:pc_dyn_case_1} and the trajectories of particle evolving under the transformer dynamics \eqref{eq:finite_sys} for $V = I$, $A = \left(\begin{smallmatrix} -1 & 4 \\ -3 & -2 \end{smallmatrix}\right)$, which satisfies condition~\eqref{eq:case1_condition_E2}.  Again, red indicates fast flow and blue indicates slow flow in the phase diagrams.}
\label{fig:case_1_2}
\end{figure}
Theorem \ref{thm:case_1} states that, when $V = I$, if the matrix $A$ satisfies \eqref{eq:case1_condition_E1}, then the order parameter $\CR_2(t)$ converges to a fixed point on $\partial D$.
Through the relation between $\CR_2(t)$ and the distribution $f$ of the mean-field angle $\Theta$ in \eqref{eq:mf_order_param}, this further implies that $f$ converges to a stationary distribution supported on two antipodal points.

Figure~\ref{fig:case_1_1} illustrates an example in this regime.
The left panel shows the phase portrait of the $(\rho, \phi)$-dynamics in \eqref{eq:pc_dyn_case_1}, 
where the vector field points toward the boundary $\rho = 1$ throughout the phase space.
The right panel shows the trajectories of $20$ angle particles $\{\theta_k\}$ evolving under the finite-particle transformer dynamics \eqref{eq:finite_sys}, and these particles converge to two antipodal points.

On the other hand, when the symmetric part of $A$ is negative definite, the order parameter $\CR_2$ converges asymptotically to an interior point of the unit disk $D$.
In this case, since the distribution $g$ of the double angle $\Xi = 2\Theta$ remains a wrapped Cauchy distribution with parameter $\CR_2(t)$, it follows that $g$ converges to a stationary distribution in the wrapped Cauchy family.
However, as explained in Remark \ref{rem:g_and_f_wrapped_Cauchy}, the distribution $f$ of the original angle $\Theta$ does not necessary remain in the wrapped Cauchy family, nor does it necessarily converge. Figure~\ref{fig:case_1_2} illustrates an example in this regime.
In the left panel, the phase portrait shows that the vector field of the $(\rho, \phi)$-dynamics points toward the interior equilibrium, again consistent with the theoretical prediction.
Meanwhile, the right panel shows that the corresponding particle dynamics does not converge, in agreement with the theoretical insight that the distribution $f$ does not necessarily converge.


\begin{remark}
The union of the regimes described by conditions \eqref{eq:case1_condition_E1} and \eqref{eq:case1_condition_E2} does not cover the entire space $\R^{2\times 2}$.
In the remaining regimes, the long-time behavior of the order parameter $\CR_2$ depends on the initial data.
In particular, $\CR_2(t)$ may converge to a fixed point on the boundary $\partial D$, converge to the interior equilibrium in $D$, or approach a periodic orbit encircling the interior equilibrium. We refer to the proof of Theorem \ref{thm:case_1} in Appendix \ref{app:case_study} for a detailed discussion.
\end{remark}

\subsection{Case 2: \texorpdfstring{$A = I$, $V$}{AV} Symmetric}
\label{sec:th_case2}

Let $A$ be the identity matrix and $V$ be a real-valued symmetric matrix,
\begin{equation}
    \label{eq:matrices_A_V_case_2}
    A := \left(\begin{array}{ll}
        1 & 0 \\
        0 & 1
    \end{array} \right),
    \qquad
    V := \left( \begin{array}{ll}
        v_{11} & v_{12}  \\
        v_{12} & v_{22} 
    \end{array} \right) .
\end{equation}
In this case, the dynamics \eqref{eq:dyn_r2} of the order parameter $\CR_2$ and its corresponding polar-coordinate system \eqref{eq:pc_dyn_general} simplify to
\begin{equation}
\Dot{\CR}_2 = \frac{1}{4} \left( \vDPlus \left( 1 - |\CR_2|^2 \right) \CR_2 + 2\left( -\vOPlus + i \vDMinus \right) (\CR_2)^2 + \left( \vOPlus + i\vDMinus \right) \left(1 + |\CR_2|^2 \right) \right),
\end{equation}
and
\begin{equation}
    \label{eq:pc_dyn_case_2}
\begin{split}
    \Drho = \frac{1}{4} (1-\rho^2) \left(\vDPlus \rho + \vDMinus \cos \phi + \vOPlus \sin \phi \right), \qquad 
    \Dot{\phi} = - \frac{3\rho^2 + 1}{4\rho} \left( \vDMinus \sin \phi - \vOPlus \cos \phi \right).
\end{split}
\end{equation}
The long-time behavior of the dynamical system \eqref{eq:pc_dyn_case_2} is summarized as follows.



\begin{theorem}\label{thm:case_2}
Consider the dynamical system \eqref{eq:pc_dyn_case_2} and assume that $V$ is symmetric.
Let $\lambdaMax(V)$ denote the largest eigenvalue of $V$.
\begin{itemize}
    \item[1.] If $\lambdaMax (V) \geq 0$, then for almost every initial condition $(\rho(0), \phi(0)) \in (0,1] \times \BBT^1$, the solution of \eqref{eq:pc_dyn_case_2} converges to a boundary equilibrium $(1, \phiV)$ as $t \rightarrow \infty$.
    Here, $\phiV$ is defined in \eqref{eq:case2_phiV}.

    \item[2.] If $\lambdaMax(V) < 0$, then for almost every initial condition $(\rho(0), \phi(0)) \in (0,1) \times \BBT^1$, the solution of \eqref{eq:pc_dyn_case_2} converges to the interior equilibrium $(\rho_{\eq}, \phiV)$ as $t \rightarrow \infty$. 
    Here, $\rho_{\eq}$ is given in \eqref{eq:case2_rhoEq}.
\end{itemize}
\end{theorem}
\begin{figure}[!htb]
\centering
\begin{subfigure}{.33\textwidth}
  \centering
  \includegraphics[trim={0cm 0cm 0cm 1.6cm},clip,width=\textwidth]{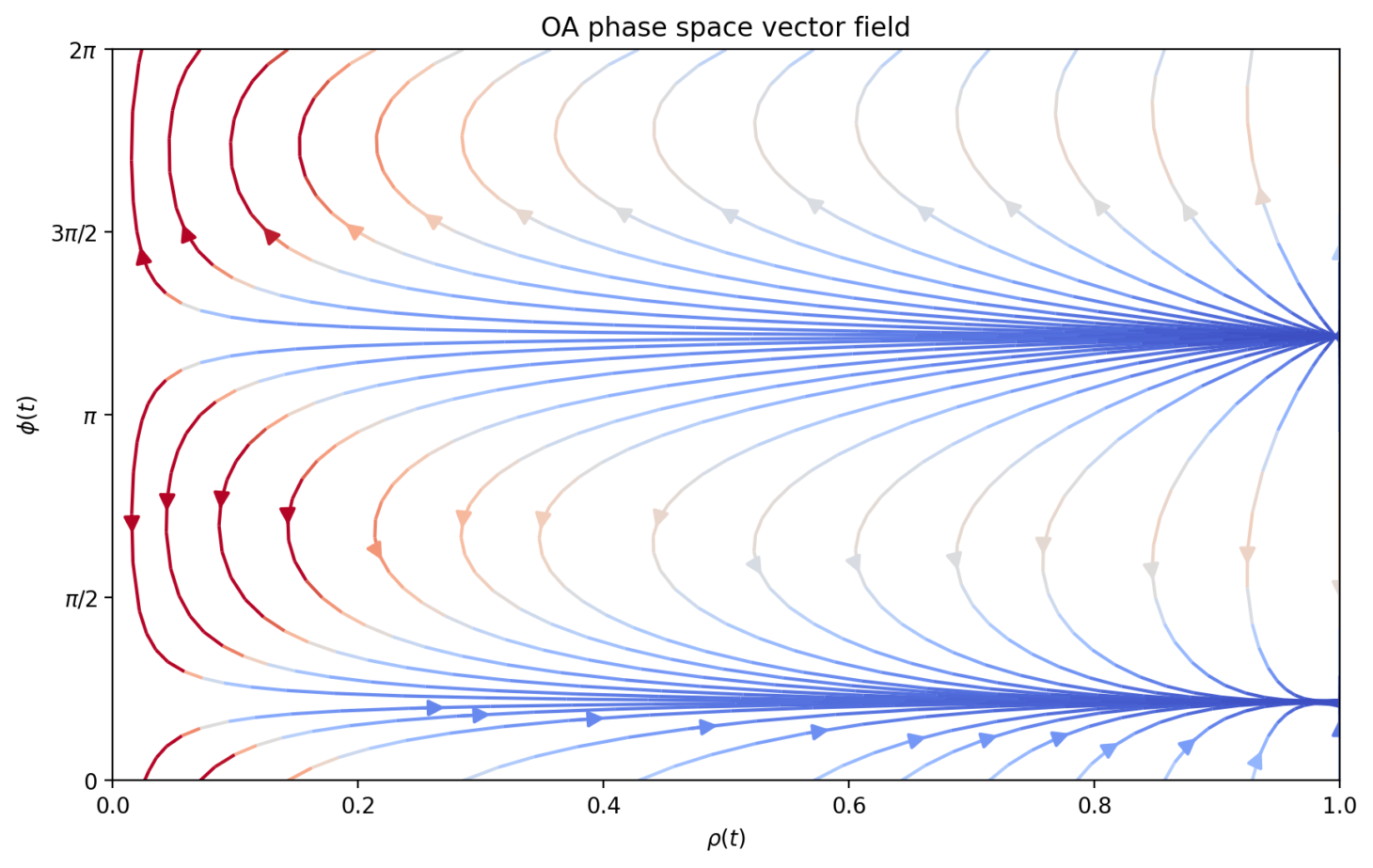}
  \subcaption{Phase portrait.}
  \label{fig:case_2_1_a}
\end{subfigure}%
\begin{subfigure}{.33\textwidth}
  \centering
  \includegraphics[trim={0cm 0cm 0cm 1.7cm},clip,width=\textwidth]{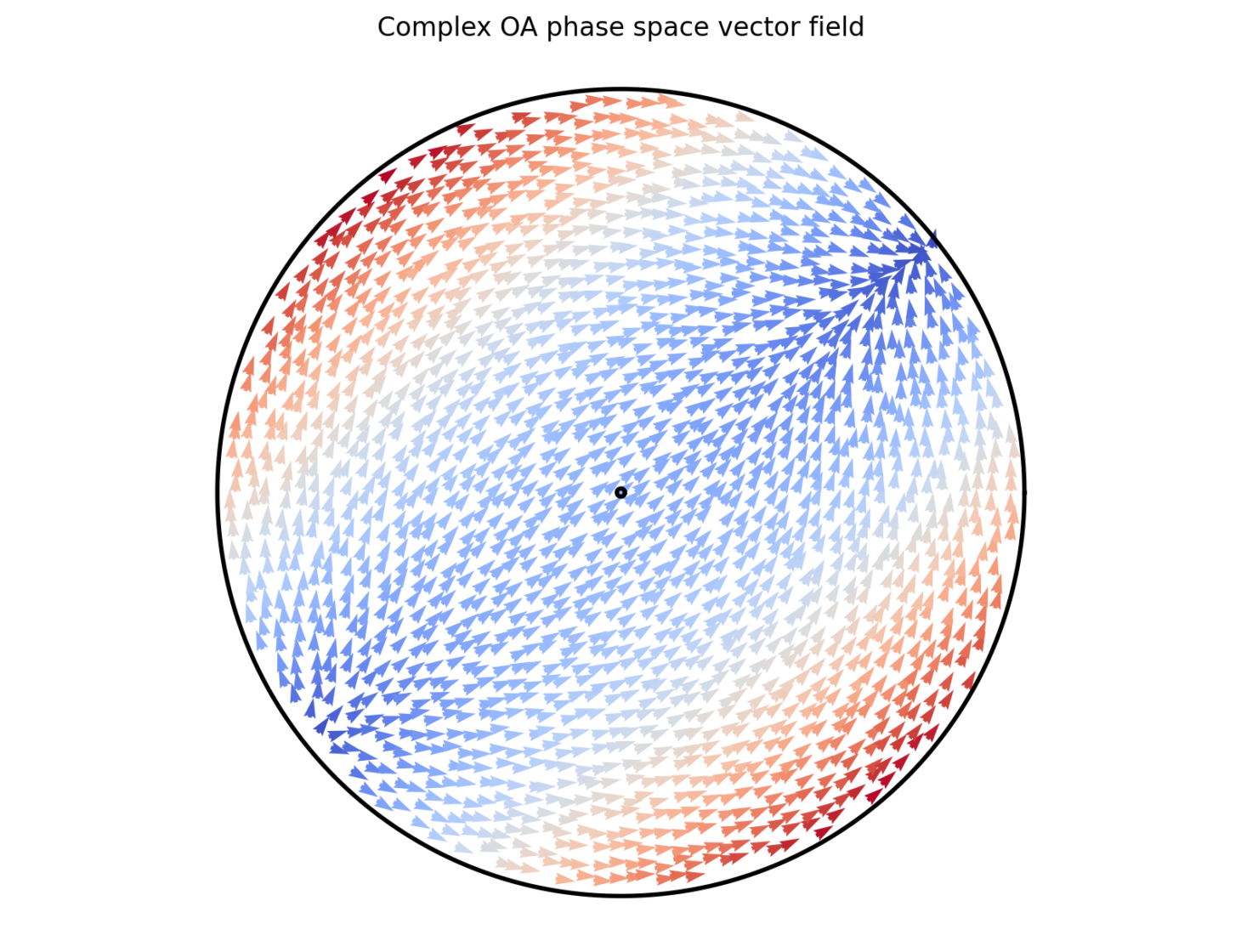}
  \subcaption{Complex phase portrait.}
  \label{fig:case_2_1_c}
\end{subfigure}%
\begin{subfigure}{.33\textwidth}
  \centering
  \includegraphics[trim={0cm 0cm 0cm 1.6cm},clip,width=\textwidth]{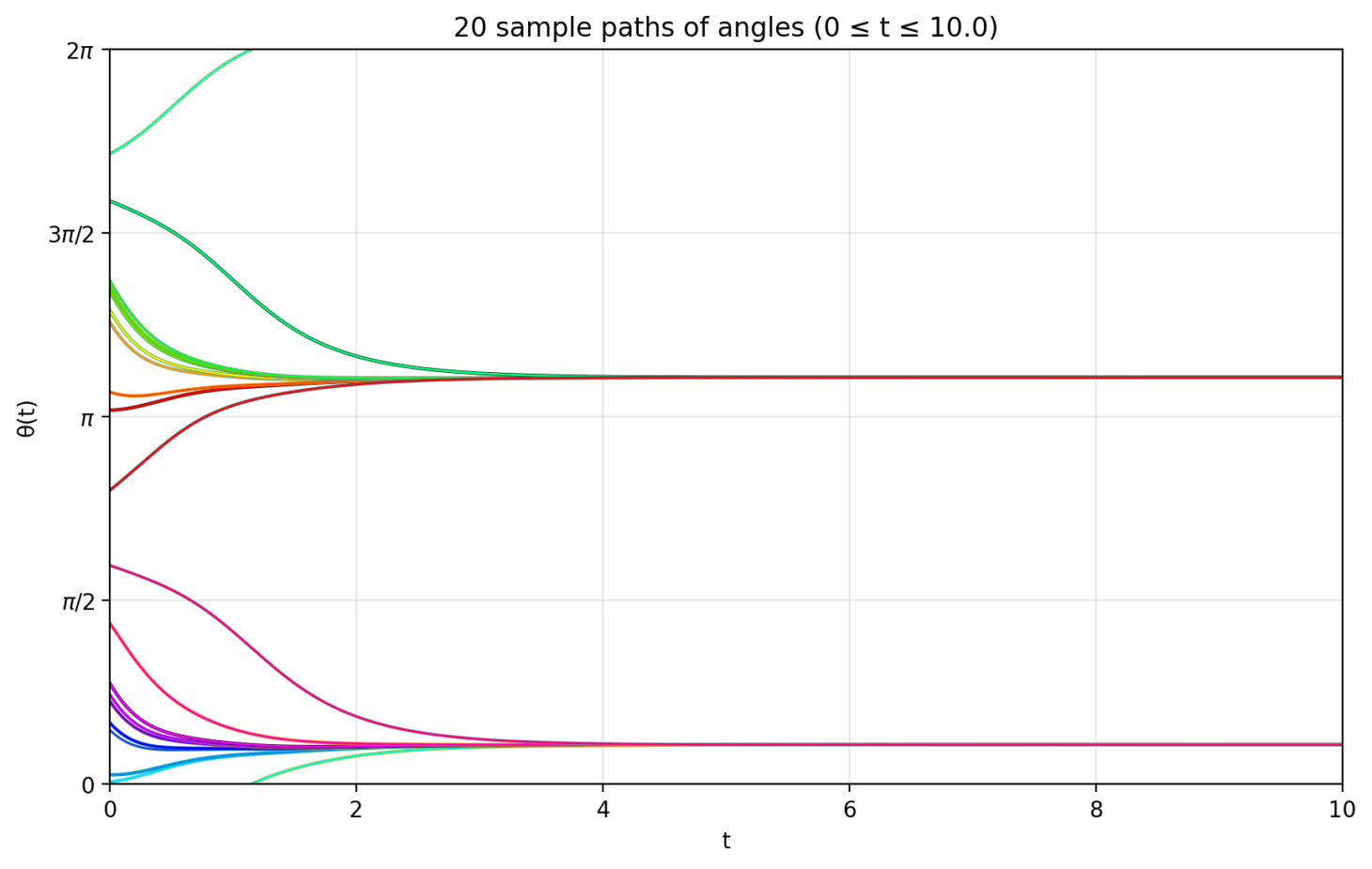}
  \subcaption{Particle dynamics.}
  \label{fig:case_2_1_b}
\end{subfigure}
\caption{Phase portrait of $(\rho, \phi)$-dynamics \eqref{eq:pc_dyn_case_2} and the trajectories of particle evolving under the transformer dynamics \eqref{eq:finite_sys} for $A = I$, $V = \left(\begin{smallmatrix}
    1 & 2 \\
    2 & -4
\end{smallmatrix}\right)$, which satisfies $\lambdaMax(V) \geq 0$.}
\label{fig:case_2_1}
\end{figure}
\begin{figure}[!htb]
\centering
\begin{subfigure}{.33\textwidth}
  \centering
  \includegraphics[trim={0cm 0cm 0cm 1.6cm},clip,width=\textwidth]{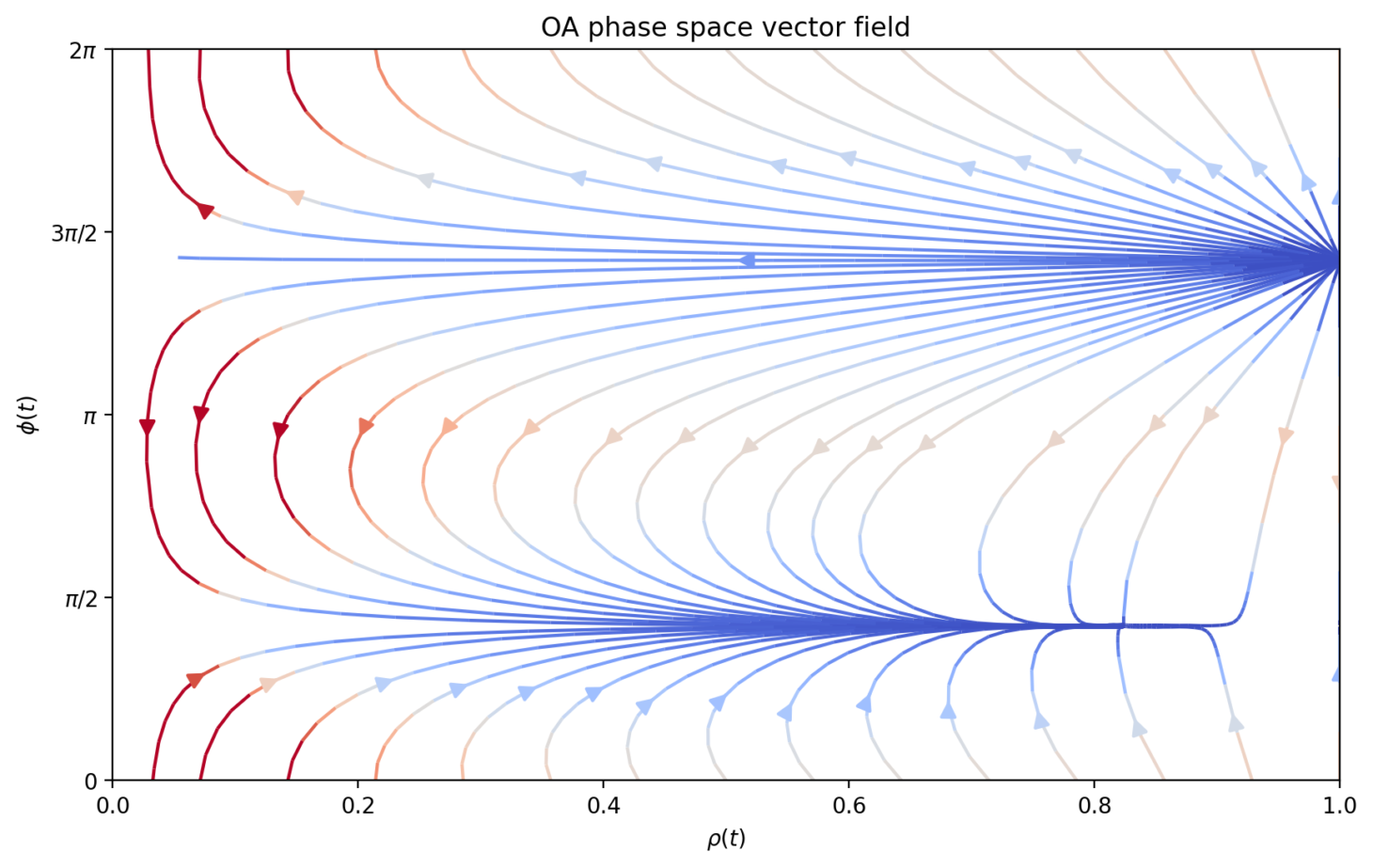}
  \subcaption{Phase portrait.}
  \label{fig:case_2_2_a}
\end{subfigure}%
\begin{subfigure}{.33\textwidth}
  \centering
  \includegraphics[trim={0cm 0cm 0cm 1.7cm},clip,width=\textwidth]{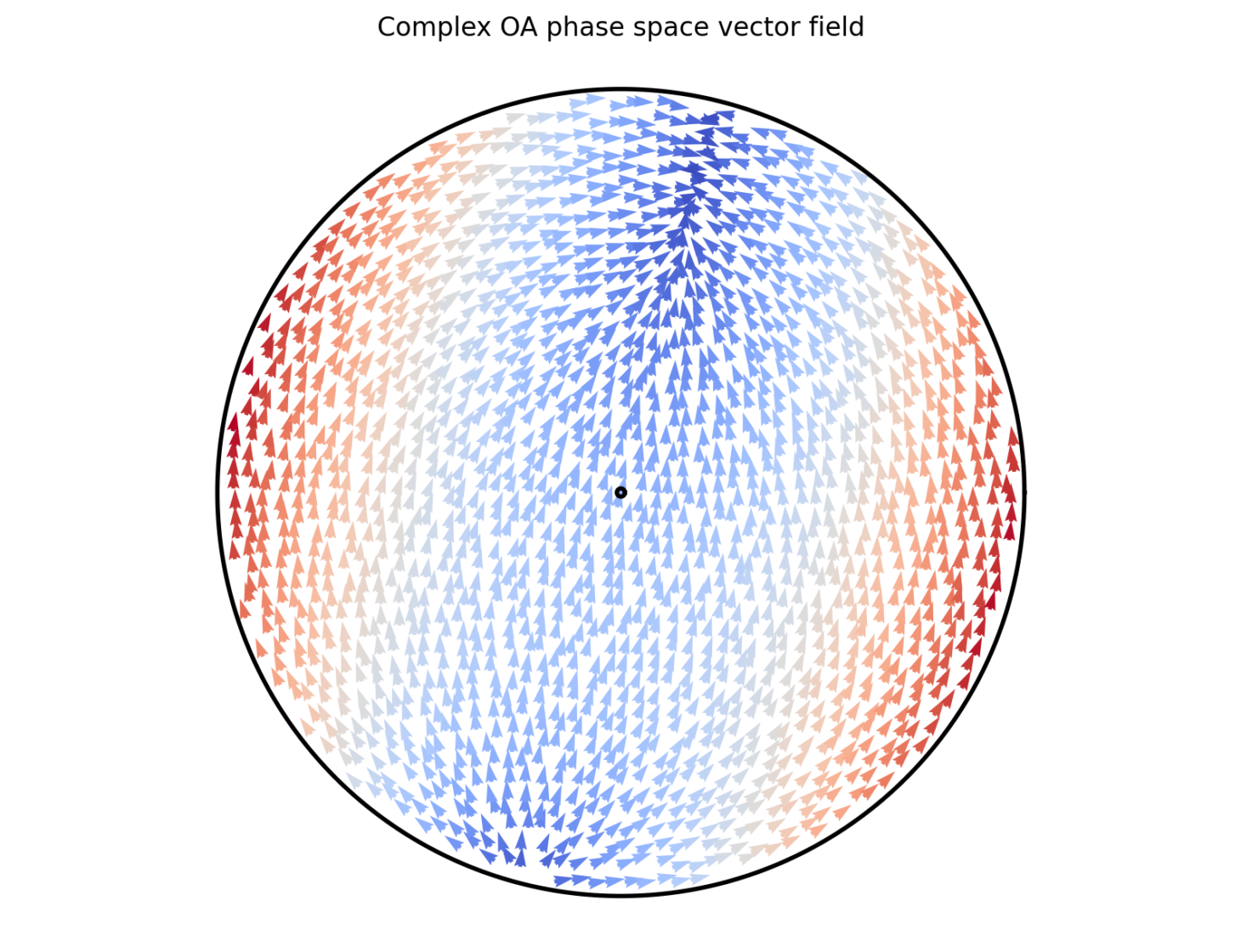}
  \subcaption{Complex phase portrait.}
  \label{fig:case_2_2_c}
\end{subfigure}%
\begin{subfigure}{.33\textwidth}
  \centering
  \includegraphics[trim={0cm 0cm 0cm 1.6cm},clip,width=\textwidth]{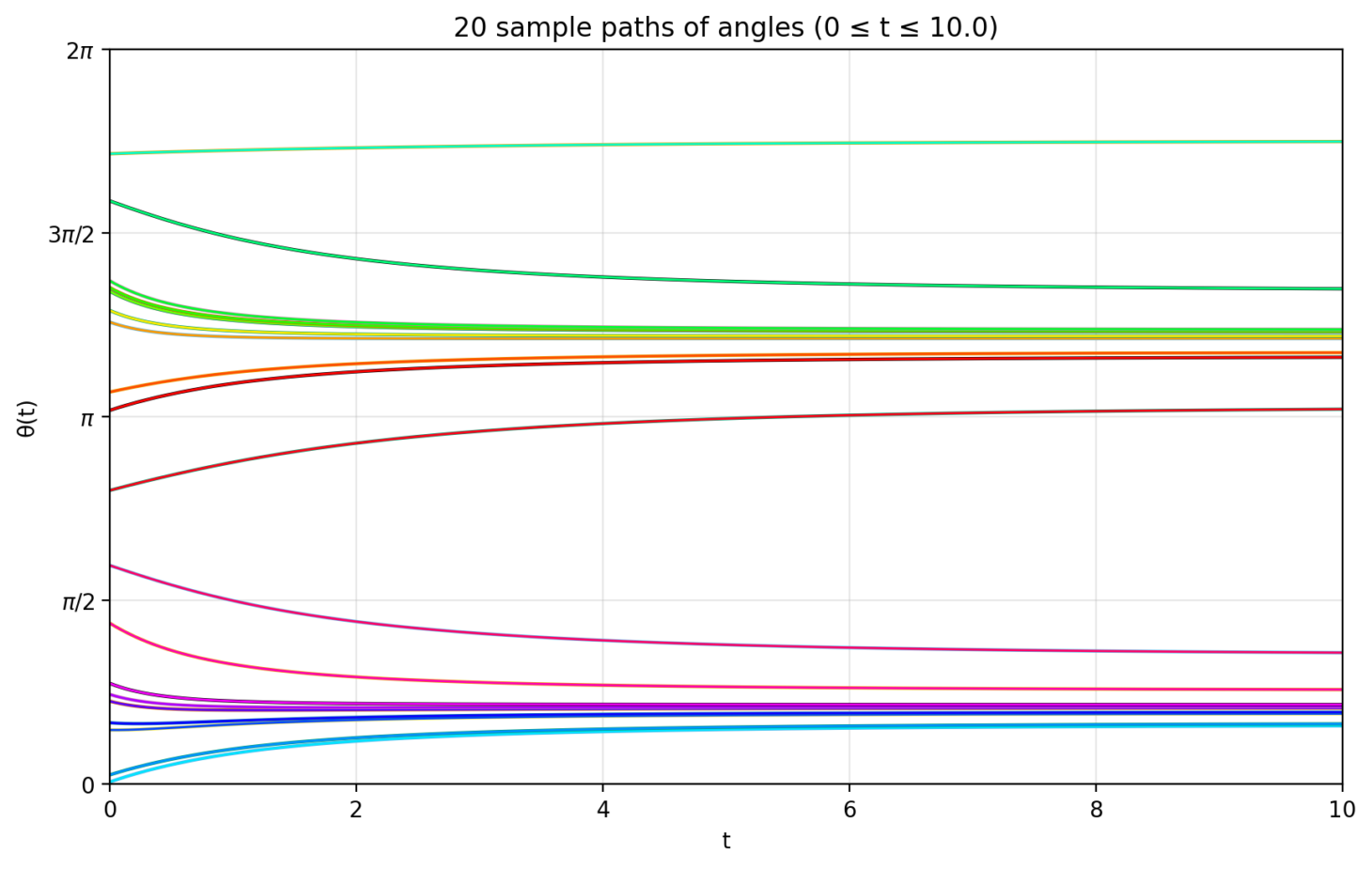}
  \subcaption{Particle dynamics.}
  \label{fig:case_2_2_b}
\end{subfigure}
\caption{Phase portrait of $(\rho, \phi)$-dynamics \eqref{eq:pc_dyn_case_2} and the trajectories of particle evolving under the transformer dynamics \eqref{eq:finite_sys} for $A = I$, $V = \left(\begin{smallmatrix}
    -2 & 2 \\
    2 & -3
\end{smallmatrix}\right)$, which satisfies $\lambdaMax(V) < 0$. 
}
\label{fig:case_2_2}
\end{figure}

Theorem \ref{thm:case_2} states that, when $A = I$ and the matrix $V$ is symmetric, the long-term behavior of dynamics \eqref{eq:pc_dyn_case_2} is entirely determined by the sign of the largest eigenvalue of $V$.
More precisely, if $\lambdaMax(V) \geq 0$, then the order parameter $\CR_2(t)$ converges to a fixed point on the boundary $\partial D$.
This in turn implies that the distribution $f$ converges to a stationary distribution supported on two antipodal points.
Figure~\ref{fig:case_2_1} illustrate this regime: the phase portrait shows that the vector field points toward a boundary fixed point on $\partial D$, and the particle dynamics converges to two antipodal points.

By contrast, if $\lambdaMax(V) < 0$, then the order parameter $\CR_2(t)$ converges to an interior equilibrium in $D$.
Figure~\ref{fig:case_2_2} illustrates an example in this regime.
The right panel shows that the particles converge, but not to a fully synchronized configuration.

\begin{remark}[Closest results available in the literature]
The closest results in the literature concerning the role of the matrix $V$ in the cluster formation of the transformer dynamics appear in \cite{abella2024asymptotic,abella2025consensus}.
In those works, the authors analyze the \emph{causal-attention} model and show that, when $A = I$ and $V$ is symmetric with $\lambdaMax(V) > 0$, the associated dynamics collapse to a single point.
This condition coincides with the one identified in our Theorem \ref{thm:case_2}. We emphasize, however, that our setting differs from theirs in two important respects.
First, we study the \emph{self-attention} model, whose dynamical properties differ from those of the \emph{causal-attention} model.
Second, we focus on the linear attention kernel $h(y)= y$, whereas \cite{abella2024asymptotic,abella2025consensus} consider the softmax kernel $h(y) = \exp(y)$.
\end{remark}

\subsection{Case 3: Non-commuting \texorpdfstring{$A$}{A} and \texorpdfstring{$V$}{V}}
\label{sec:th_case3}

Let $A$ and $V$ be given by
\begin{equation}\label{eq:matrices_A_V_case_3}
A := \left(\begin{array}{ll}
    a_{11} & 0 \\
    0 & a_{22}
\end{array} \right), \qquad  V := \left( \begin{array}{cc}
    v_{11} & v_{12}  \\[4pt]
    -\frac{a_{22}}{a_{11}} v_{12} & \frac{a_{11}}{a_{22}} v_{11} 
\end{array} \right),
\end{equation}
with $a_{11}, a_{22} \neq 0$. 

\begin{remark}
Suppose $A$ and $V$ are as in \eqref{eq:matrices_A_V_case_3}. A direct computation reveals that if $a_{22} \not = a_{11} $ and $v_{12} \not =0$, then the matrices $A, V$ do not commute. 
\end{remark}

In this case, the polar-coordinate system \eqref{eq:pc_dyn_general} simplifies to\footnote{We don't present the corresponding dynamics of $\CR_2$ in this case, since the resulting expression is rather complicated and offers limited insight.}

\begin{equation}\label{eq:pc_dyn_case_3}
\begin{aligned}
\Dphi = K(\phi) + (\rho-1) Q(\rho, \phi),
\end{aligned}
\end{equation}
where the functions $H, K$ and $Q$ depend on the matrices $A$ and $V$; see \eqref{eq:case3_pc_dyn_gen} in the Appendix for their explicit expressions.

In this example, our goal is to identify a parameter regime for which $\rho(t) \rightarrow 1$ while $\phi(t)$ does not converge.
This behavior is guaranteed under the sufficient conditions
\begin{equation*}
H(\phi) > 0, \qquad |K(\phi)| > 0, \qquad |Q(\rho, \phi)| < \infty,
\end{equation*}
for all $(\rho, \phi) \in [0,1] \times \BBT^1$.
The corresponding result is stated in the following theorem.

\begin{theorem}\label{thm:case_3}
Consider the dynamical system \eqref{eq:pc_dyn_case_3}.
If the matrices $A$ and $V$ defined in \eqref{eq:matrices_A_V_case_3} satisfy:
\begin{equation}\label{eq:case3_cond}
 \frac{v_{11}}{a_{22}} > 0, \qquad \frac{1}{8} \frac{a_{22}}{a_{11}} \frac{(a_{11} - a_{22})^4}{a_{11}^2 (a_{11}^2 + a_{22}^2)} < \frac{v_{11}^2}{v_{12}^2} < 4 \frac{a_{22}}{a_{11}} \frac{a_{22}^2}{(a_{11} - a_{22})^2},
\end{equation}
then for every initial condition $(\rho(0), \phi(0)) \in (0,1] \times \BBT^1$, the solution of \eqref{eq:pc_dyn_case_3} satisfies that $ \rho(t) \rightarrow 1$ as $t \rightarrow \infty$ while $\phi(t)$ never converges.
\end{theorem}
\begin{figure}[!htb]
\centering
\begin{subfigure}{.33\textwidth}
  \centering
  \includegraphics[trim={0cm 0cm 0cm 1.6cm},clip,width=\textwidth]{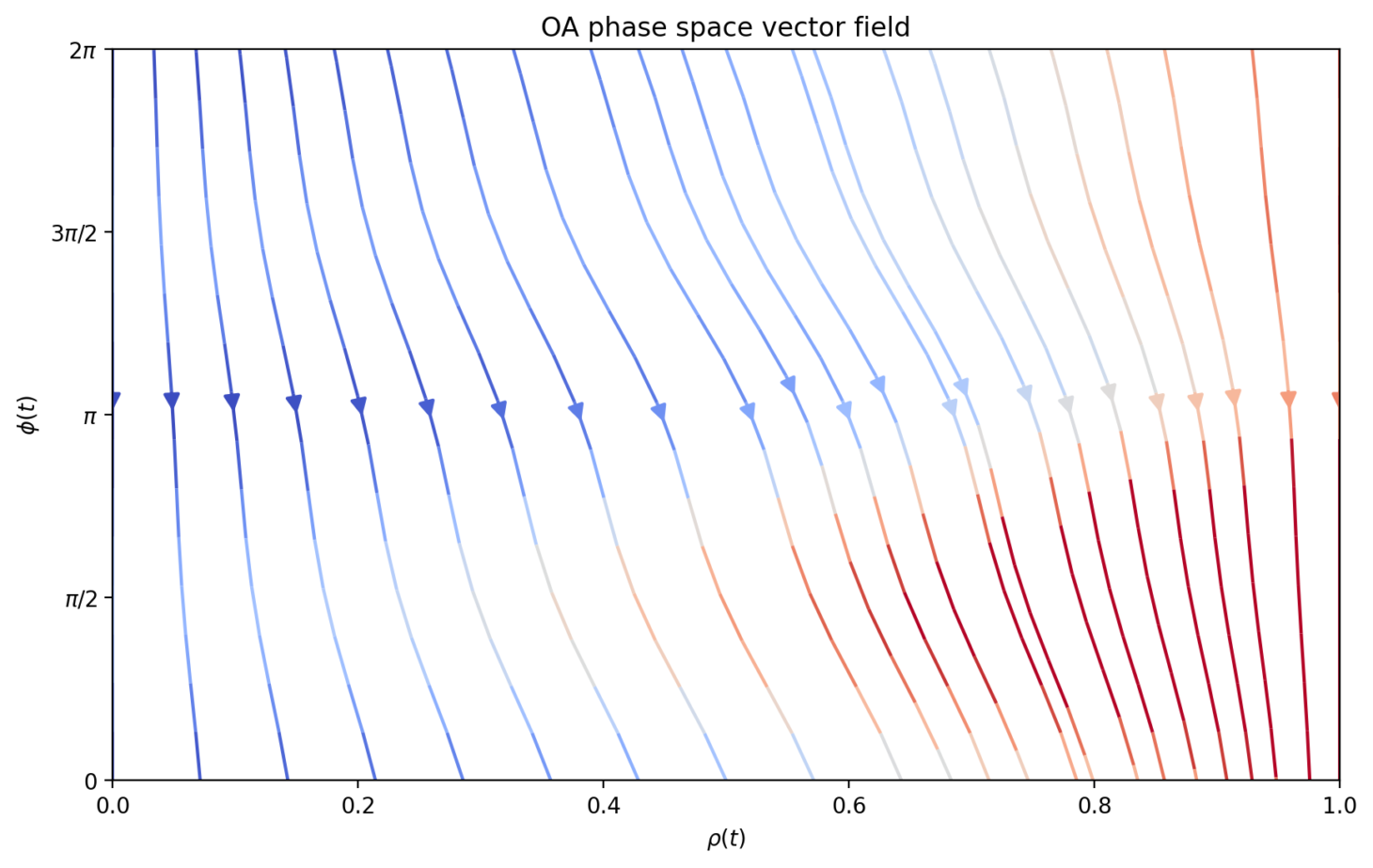}
  \subcaption{Phase portrait.}
  \label{fig:case_3_a}
\end{subfigure}%
\begin{subfigure}{.33\textwidth}
  \centering
  \includegraphics[trim={0cm 0cm 0cm 1.7cm},clip,width=\textwidth]{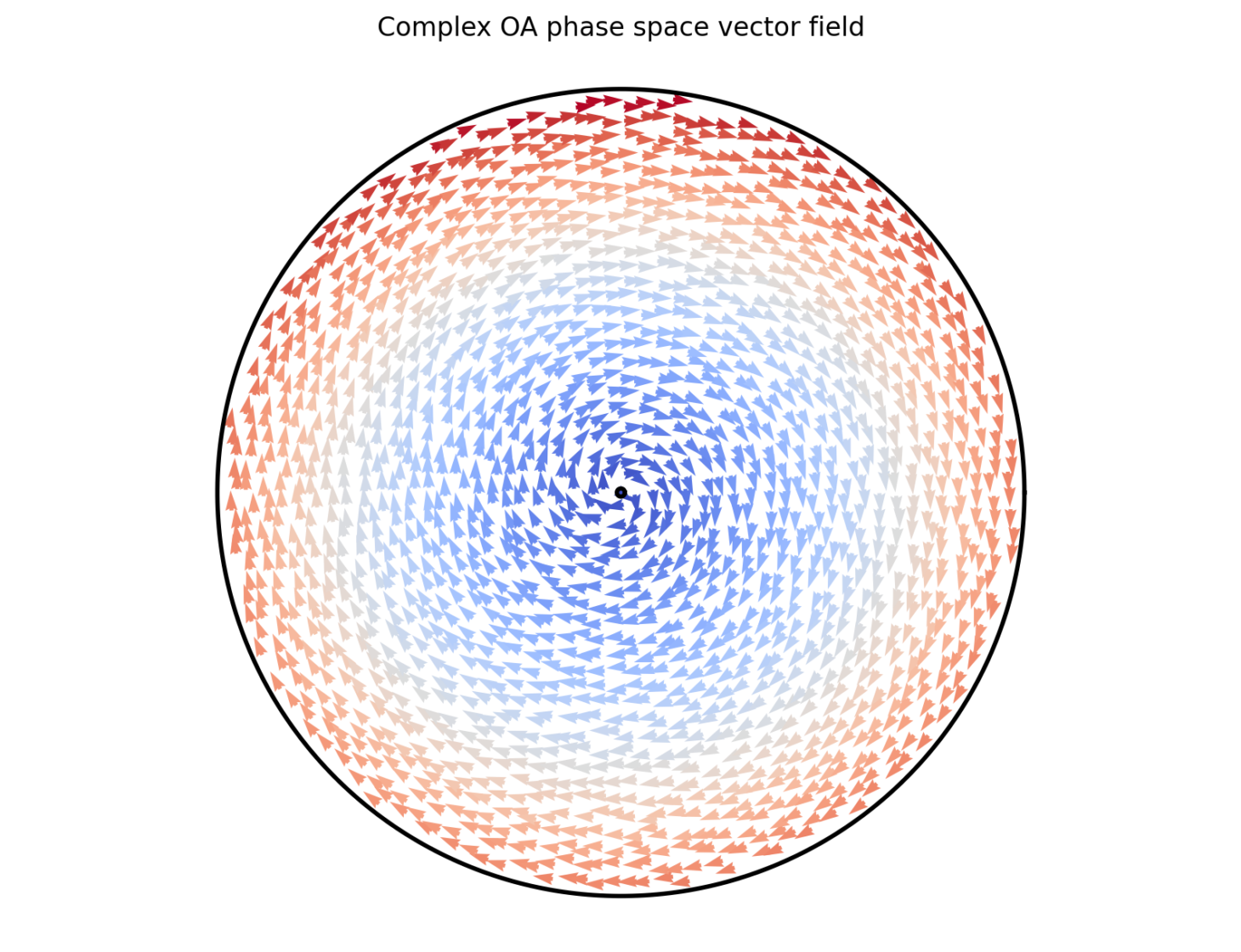}
  \subcaption{Complex phase portrait.}
  \label{fig:case_3_b}
\end{subfigure}%
\begin{subfigure}{.33\textwidth}
  \centering
  \includegraphics[trim={0cm 0cm 0cm 1.6cm},clip,width=\textwidth]{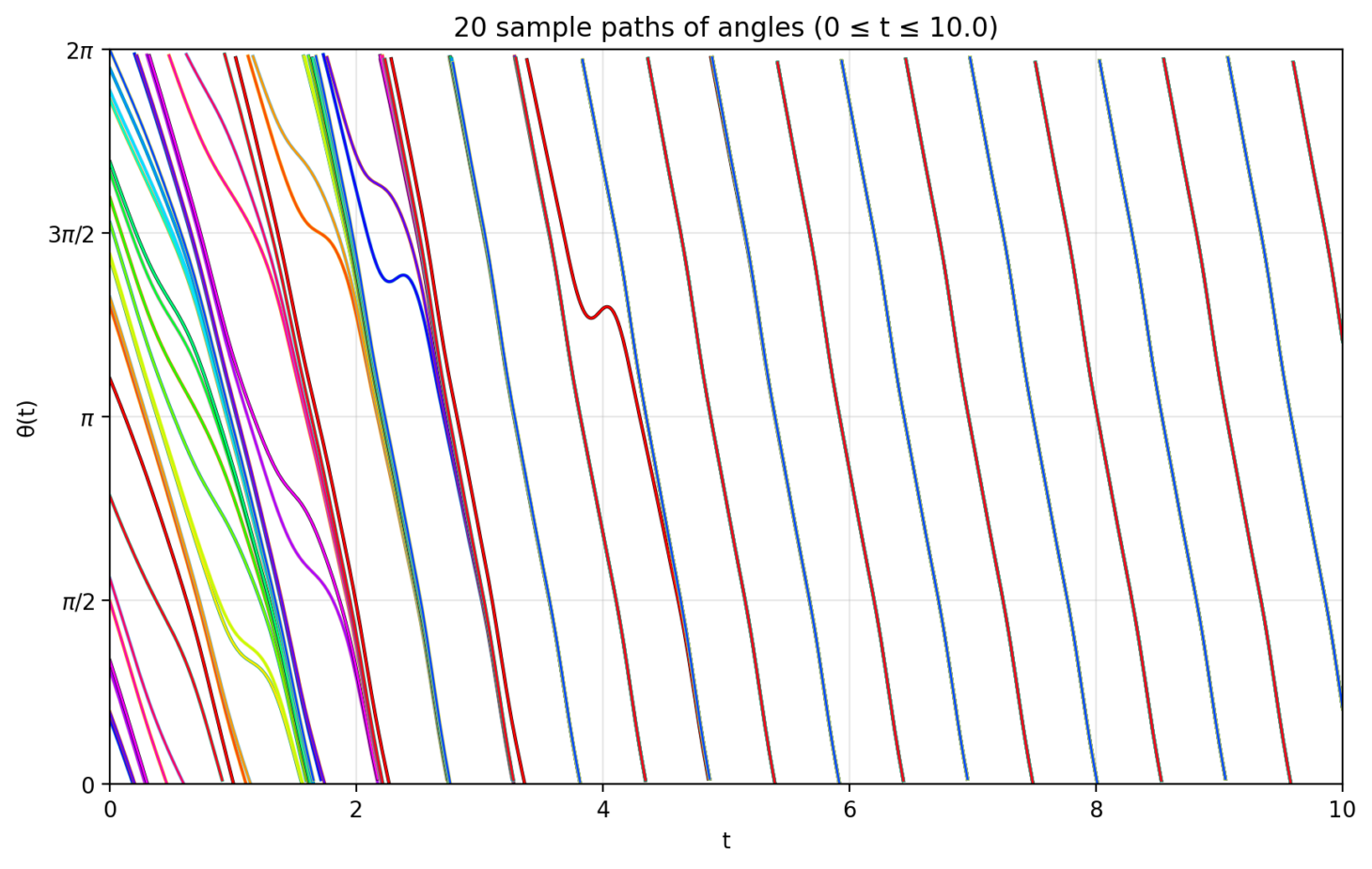}
  \subcaption{Particle dynamics.}
  \label{fig:case_3_c}
\end{subfigure}
\caption{Phase portrait of $(\rho, \phi)$-dynamics \eqref{eq:pc_dyn_case_3} and the trajectories of particle evolving under the transformer dynamics \eqref{eq:finite_sys} for $A =\left(\begin{smallmatrix} 1 & 0 \\ 0 & 2 \end{smallmatrix}\right)$, $V = \left(\begin{smallmatrix} 2 & 3 \\ -6 & 1 \end{smallmatrix}\right)$, which satisfies condition \eqref{eq:case3_cond}. }
\label{fig:case_3}
\end{figure}
Theorem \ref{thm:case_3} states that, when the matrices $A$ and $V$ are given by \eqref{eq:matrices_A_V_case_3}\footnote{The construction of the matrices $A$ and $V$ is explained in the proof of Theorem \ref{thm:case_3}.}, and their entries satisfy condition \eqref{eq:case3_cond}, the magnitude $\rho$ of the order parameter $\CR_2$ converges to $1$, while its argument $\phi$ fails to converge.
Through the relation between $\CR_2$ and the distribution $f$ of the mean-field particle $\Theta$ in \eqref{eq:mf_order_param}, this implies that, for large times, the support of $f$ is concentrated on two antipodal points.
However, these two support points continue to rotate, and hence do not converge to fixed locations.
Figure~\ref{fig:case_3} illustrates a numerical example in this regime.
In particular, the right panel shows the particle trajectories clustering around the trajectories of two antipodal points that continue to revolve around each other.

To the best of our knowledge, this is the first theoretical characterization of the long-time behavior of the transformer dynamics in a regime where both $A$ and $V$ are non-identity and non-commuting. 
The asymptotic behavior identified here is also new: the token embeddings asymptotically form two clusters, but the cluster locations continue to rotate and hence do not converge.

\begin{remark}[Connection to the practical behavior of LLMs]
We note that the clustering-without-convergence long-time behavior identified in Theorem \ref{thm:case_3} may not be merely an artifact of the mathematical model studied here. Indeed, the numerical results of \cite{blayney2026mechanistic} show that in Looped Transformers \cite{giannou2023looped,yang2024looped,geiping2026scaling}, where a trained LLM is recurrently applied, the output of the recurrent block follows a consistent cyclic trajectory.
This observation is qualitatively reminiscent of the clustering-without-convergence phenomenon identified in our analysis, and suggests interesting directions of future investigation in practical settings. 
\end{remark}

\begin{remark}\label{rem:case3prime}
The condition \eqref{eq:case3_cond} is somewhat complicated and not immediately transparent.
We therefore present here a simple family of matrices that falls into this parameter regime.
In particular, consider matrices $A$ and $V$ of the form,

\begin{equation}\label{eq:matrices_A_V_case3_special}
    A := \begin{pmatrix}
        1 & 0 \\
        0 & 1
    \end{pmatrix},  \qquad V := \begin{pmatrix}
        v_{11} & v_{12} \\
        -v_{12} & v_{11}
    \end{pmatrix}
\end{equation}
with $v_{11} > 0$ and $v_{12} \neq 0$. In this particular regime the matrices $A$ and $V$ commute, but we can more easily explain the mechanism behind the validity of Theorem \ref{thm:case_3}. Indeed, in this regime the coupled ODE system \eqref{eq:pc_dyn_general} simplifies to
\begin{equation}
\label{eq:pc_dyn_case_3_prime}
\begin{aligned}
\Drho = \frac{v_{11}}{2} (1 - \rho^2) \rho, \qquad \Dphi = \frac{v_{12}}{2} (\rho^2 + 3) .
\end{aligned}
\end{equation}
Since $v_{11} > 0$, $\rho \in (0,1]$, by Lemma \ref{lem:gamma-integral}, we know that $\rho(t) \rightarrow 1$ for all initial condition $\rho(0) \in (0,1]$. 
And since $|\Dphi(t)| > 0$ for $t \geq 0$, we infer that $|\phi|$ does not converge in $\BBT$ as $t \rightarrow \infty$. 
\end{remark}

\subsection{Case 4: Hidden Hamiltonian Dynamics}
\label{sec:th_case4}

Let $A$ and $V$ be of the form,
\begin{equation}\label{eq:matrices_A_V_case_4}
    A := \begin{pmatrix}
        1 & 0 \\
        0 & 1
    \end{pmatrix},  \qquad V := \begin{pmatrix}
        v_{11} & v_{12} \\
        -v_{12} & -v_{11}
    \end{pmatrix}
\end{equation}
with $v_{11}, v_{12} > 0$.  

In this case, the dynamics \eqref{eq:dyn_r2} of the order parameter $\CR_2$ and its corresponding polar-coordinate system \eqref{eq:pc_dyn_general} simplify to
\begin{equation}
\dot{\CR}_2 = \frac{v_{11}}{2} \left( 1 + |\CR_2|^2 - 2 (\CR_2)^2 \right) - \frac{iv_{12}}{2} \left( |\CR_2|^2 + 3 \right) \CR_2,
\end{equation}
and
\begin{equation}
\label{eq:pc_dyn_case_4}
\begin{aligned} 
        \dot \rho = \frac{v_{11}}{2} (1- \rho^2) \cos(\phi), \qquad        \dot \phi = -\frac{1}{2} \left( v_{12} (\rho^2 + 3) + v_{11} \left(3 \rho + \frac{1}{\rho} \right) \sin \phi \right) .
\end{aligned}
\end{equation}
Interestingly, the system~\eqref{eq:pc_dyn_case_4} turns out to be a Hamiltonian flow on the open cylinder $(0,1) \times \BBT^1$ after an appropriate time rescaling.  
More precisely, define the Hamiltonian
\[H(\rho, \phi) = -\frac{v_{11}\rho}{(1-\rho^2)^2}\sin(\phi) - \frac{v_{12}(\rho^2 + 1)}{2(1-\rho^2)^2} .
\]
Then, under the time change 
\begin{equation}
\frac{d\tau}{dt}=\frac{(1-\rho^2)^3}{2\rho},
\end{equation}
the dynamics \eqref{eq:pc_dyn_case_4} on the domain $(0,1) \times \BBT^1$ can be rewritten in Hamiltonian form as
\begin{equation}
\label{eq:hamiltonian_form}
\frac{d\rho}{d\tau} = - \partial_{\phi} H, \qquad \frac{d\phi}{d\tau} = \partial_{\rho} H .
\end{equation}
Moreover, this system exhibits a saddle-node bifurcation \cite{strogatz2024nonlinear}, indicating that even a small perturbation of $V$ may drastically alter the long-time behavior of the system.
The long-time clustering behavior of \eqref{eq:pc_dyn_case_4} is described in the following proposition.

\begin{theorem}\label{thm:case_4}
    Consider the dynamical system~\eqref{eq:pc_dyn_case_4}, and assume that $v_{11}, v_{12} > 0$.
    \begin{enumerate}
        \item If $v_{11} < v_{12}$, then for almost every initial condition $(\rho(0), \phi(0)) \in (0,1) \times \BBT^1$ the associated solution is bounded away from the boundary $\rho \in \{0,1\}$. Furthermore, the trajectory is periodic, and hence forms a closed orbit in $(0,1) \times \BBT^1$.
        \item If $v_{11} \geq v_{12}$, then for almost every initial condition $(\rho(0), \phi(0)) \in (0,1] \times \BBT^1$ the associated solution converges to a boundary equilibrium $(1, \phi_{\infty})$ as $t \to \infty$,  
        where $\phi_{\infty} := 2\pi - \arcsin(v_{12}/v_{11})$.
    \end{enumerate}
\end{theorem}

\if{false}{
\begin{proposition} \label{prop:OA_cycle_old}
    Let $a,b > 0$.  The dynamical system~\eqref{eq:pc_dyn_case_4} exhibits a bifurcation of a nonlinear stable center to a stable-unstable fixed point pair,
    \begin{itemize}
        \item When $a < b$, there is a nonlinear center at $\left(\rho^*, \frac{3\pi}{2}\right)$, with $\rho^*$ the unique solution to \begin{equation}\label{eq:OA_cycle_constraint}
            -b \rho^3 + 3a \rho^2 - 3b\rho + a = 0, \quad \rho \in (0,1).
        \end{equation}
        Furthermore, almost all trajectories form closed periodic orbits.    If $E(\rho_0, \phi_0) > -\frac{b}{2}$ and $(\rho_0, \phi_0) \not = (\rho^*, \frac{3\pi}{2})$, the trajectory of $(\rho_0, \phi_0)$ is a libration about the fixed point $(\rho, \frac{3\pi}{2})$. If $E(\rho_0, \phi_0) < -\frac{b}{2}$, the trajectory of $(\rho_0, \phi_0)$ is a rotation around the cylinder.  If $E(\rho_0, \phi_0) = -\frac{b}{2}$, the trajectory of $(\rho_0, \phi_0)$ is heteroclinic terminating at $(0, \pi)$.
        \SX{Based on my understanding from the proof, the different behaviors of the dynamics when $E(\rho_0, \phi_0) > -b/2$ and $E(\rho_0, \phi_0) < -b/2$ are not justified explicitly. Did I miss something here?}
        \item When $a = b$, all initial conditions satisfying $E(\rho_0, \phi_0) \not= -\frac{b}{2}$  converge to  the saddle point $\left( 1, \frac{3\pi}{2}\right)$.
        \SX{When $a = b$, $(1, \frac{3\pi}{2})$ seems not be a saddle point as the Jacobian at this point is a zero matrix. So it is nonhyperbolic.
        It seems that this particular fixed is called semistable fixed point in the literature.}
        \item When $a > b$, there is a stable-unstable fixed point pair $\{(1, 2\pi - r), (1, \pi + r)\}$ with $r = \arcsin \left(\frac{b}{a}\right)$.  All initial conditions satisfying $E(\rho_0, \phi_0) \not= -\frac{b}{2}$ converge to $(1, 2\pi - r)$.
    \end{itemize}
\end{proposition}
}\fi

\begin{figure}[!htb]
\centering
\begin{subfigure}{.33\textwidth}
  \centering
  \includegraphics[trim={0cm 0cm 0cm 1.6cm},clip,width=\textwidth]{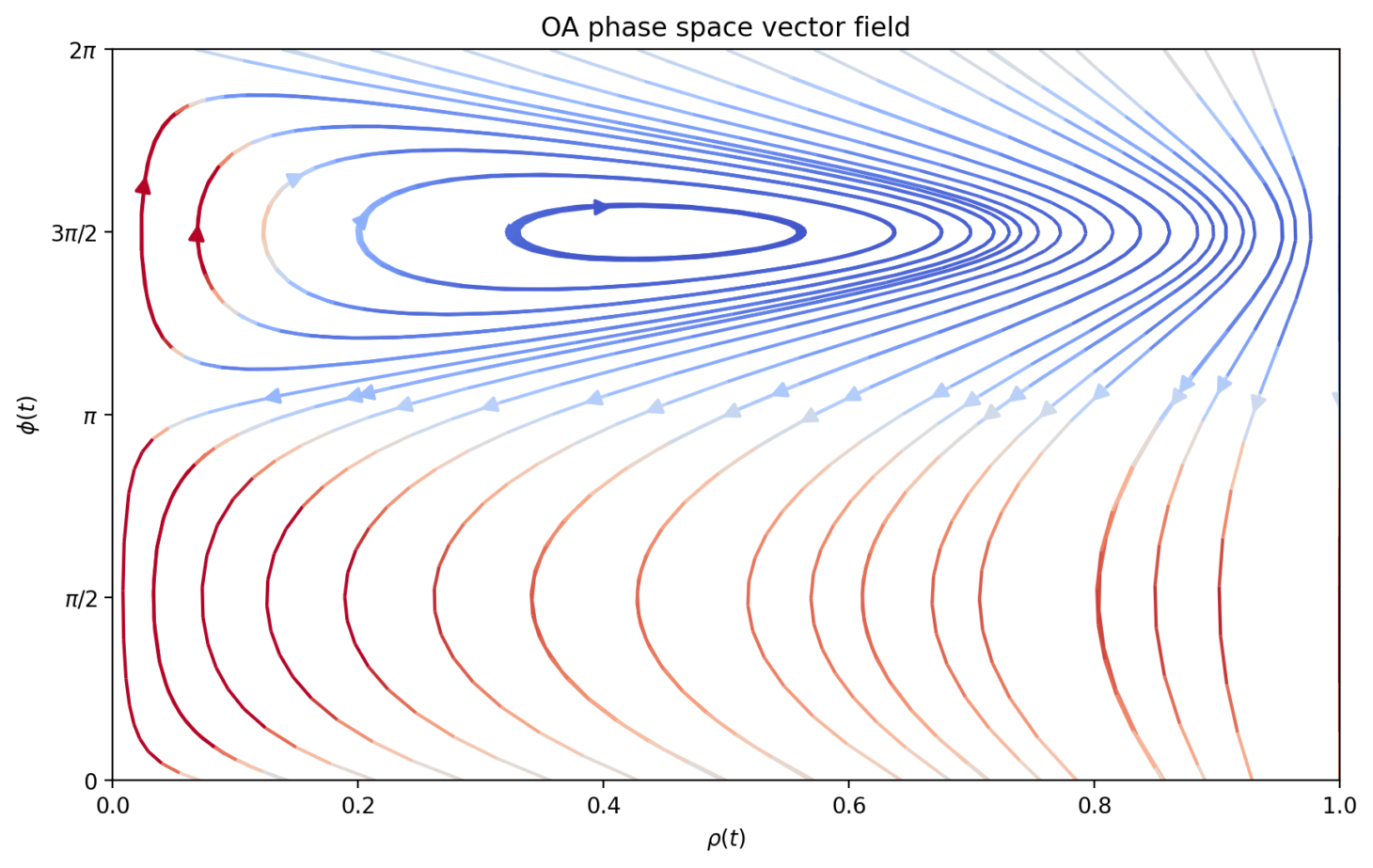}
  \label{fig:case_4_1_a_phase}
\end{subfigure}%
\begin{subfigure}{.33\textwidth}
  \centering
  \includegraphics[trim={0cm 0cm 0cm 1.7cm},clip,width=\textwidth]{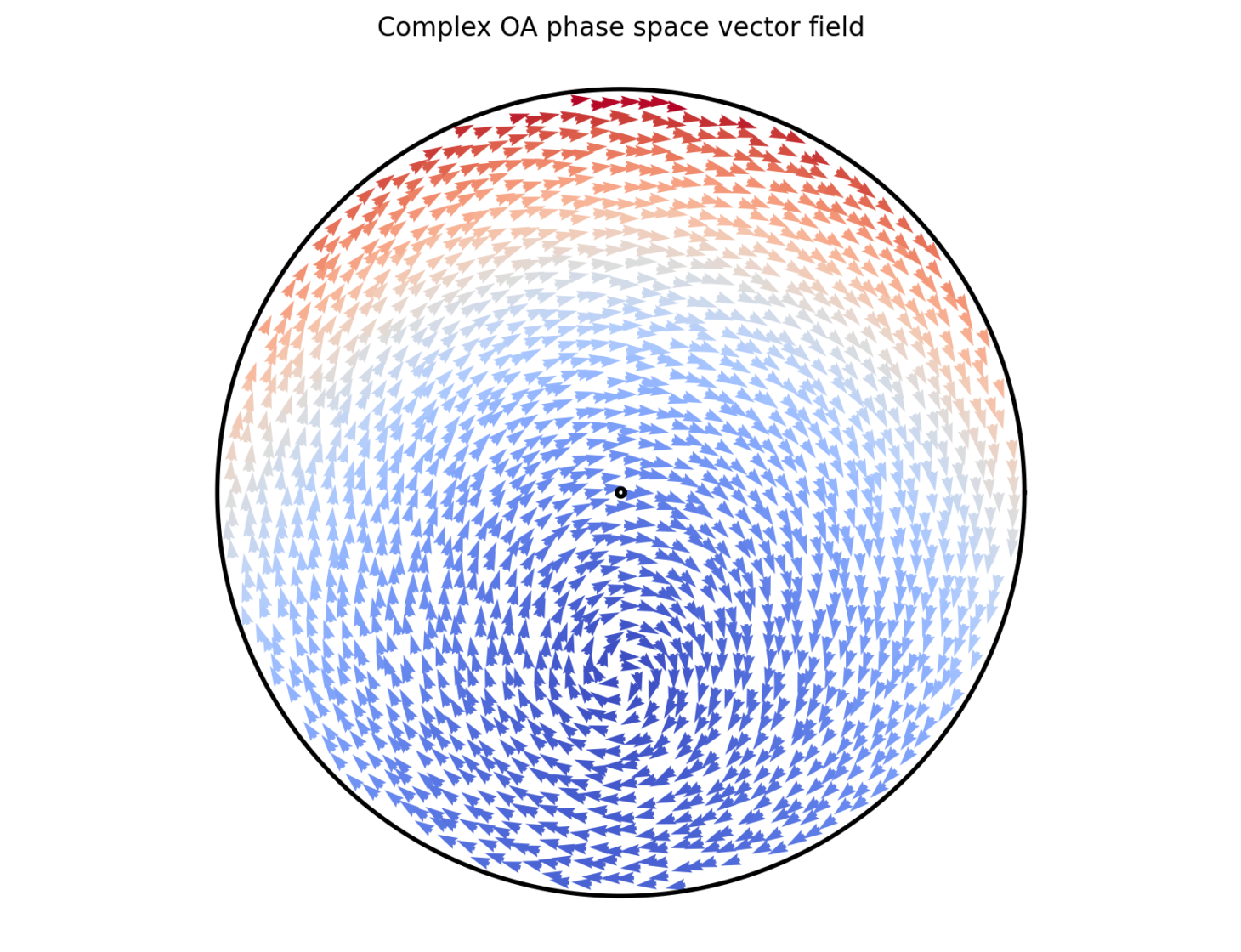}
  \label{fig:case_4_1_a_cphase}
\end{subfigure}%
\begin{subfigure}{.33\textwidth}
  \centering
  \includegraphics[trim={0cm 0cm 0cm 1.6cm},clip,width=\textwidth]{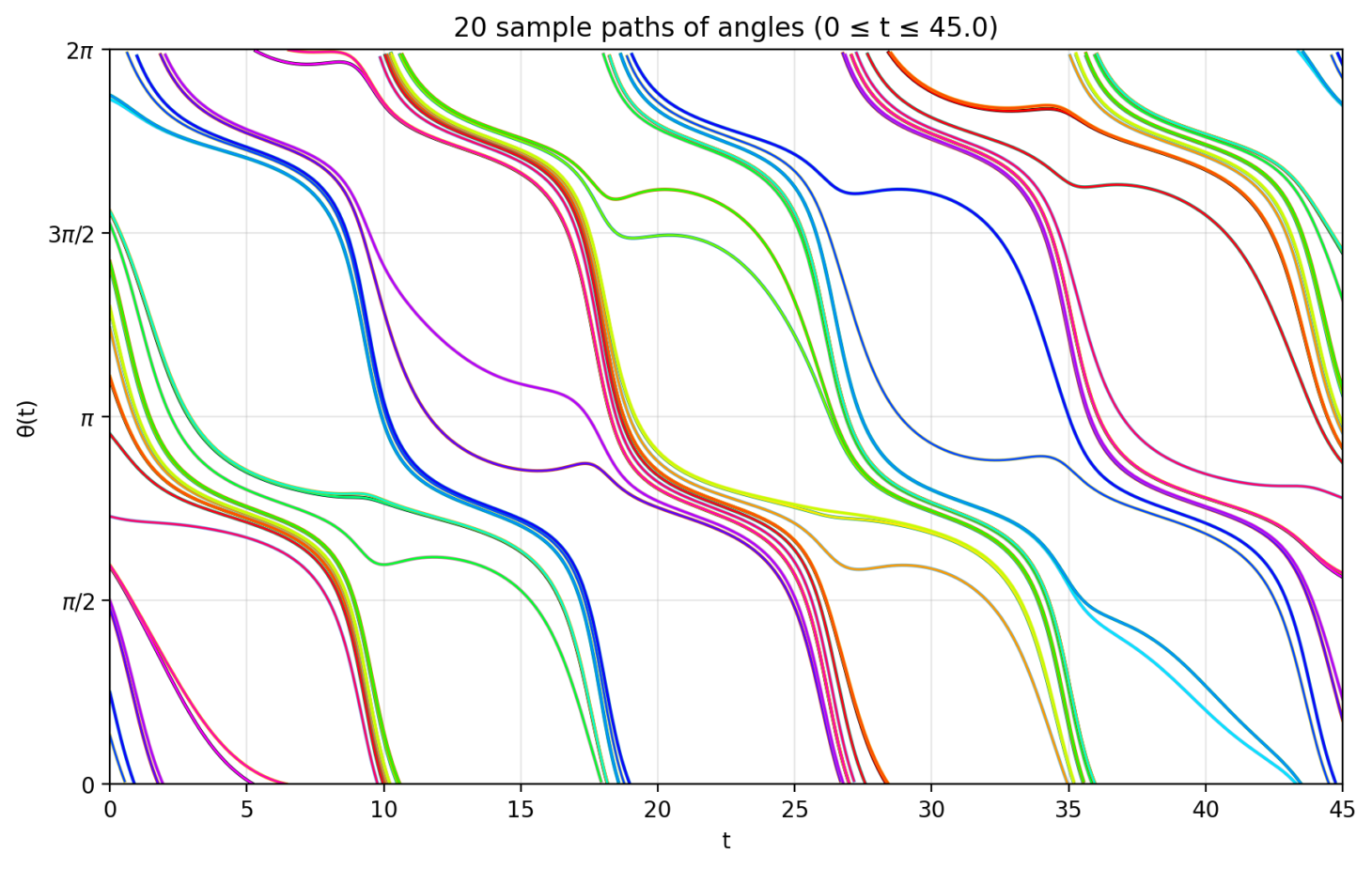}
  \label{fig:case_4_1_a_particle}
\end{subfigure}%

\medskip
\begin{subfigure}{.33\textwidth}
  \centering
  \includegraphics[trim={0cm 0cm 0cm 1.6cm},clip,width=\textwidth]{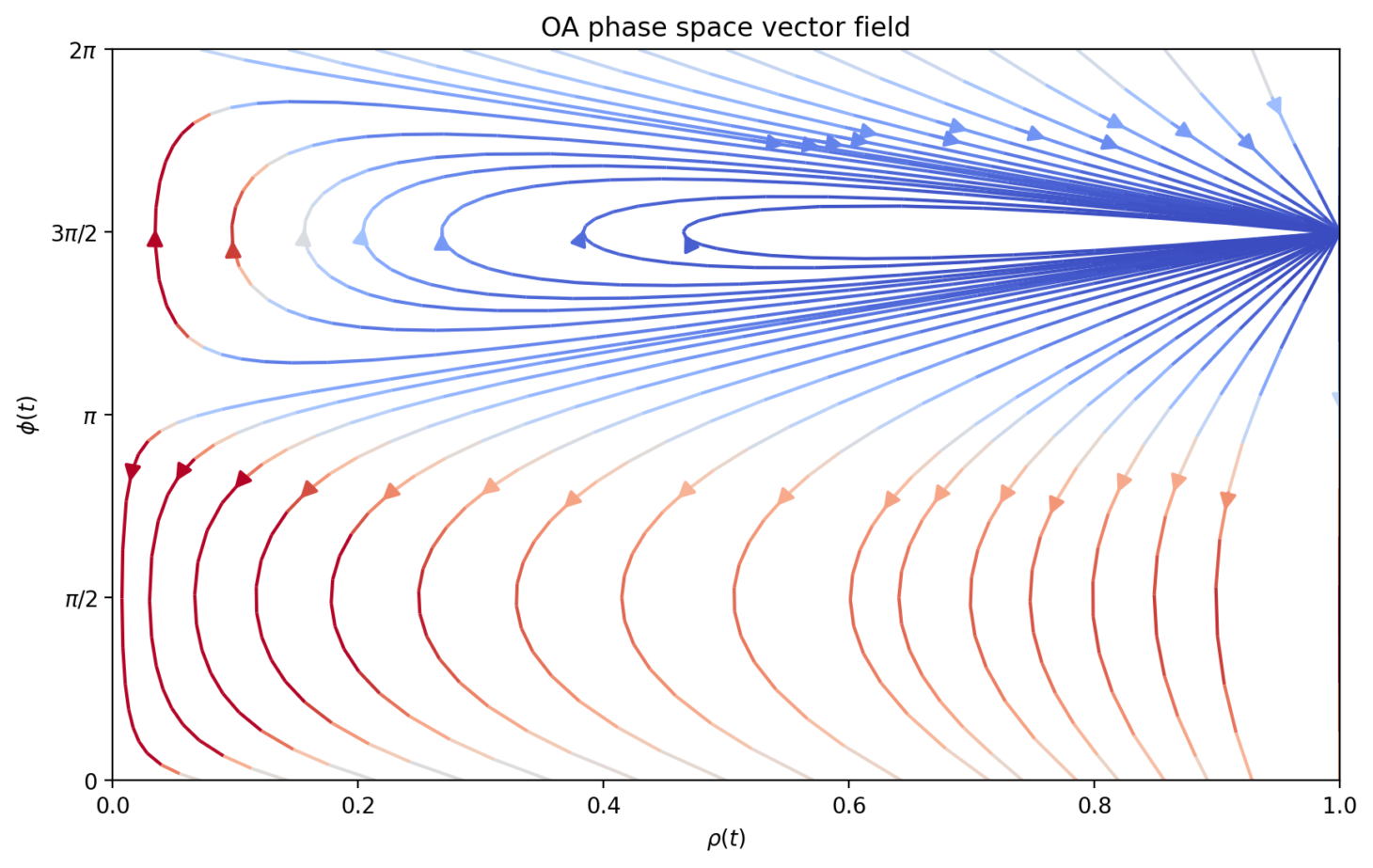}
  \label{fig:case_4_1_b_phase}
\end{subfigure}%
\begin{subfigure}{.33\textwidth}
  \centering
  \includegraphics[trim={0cm 0cm 0cm 1.7cm},clip,width=\textwidth]{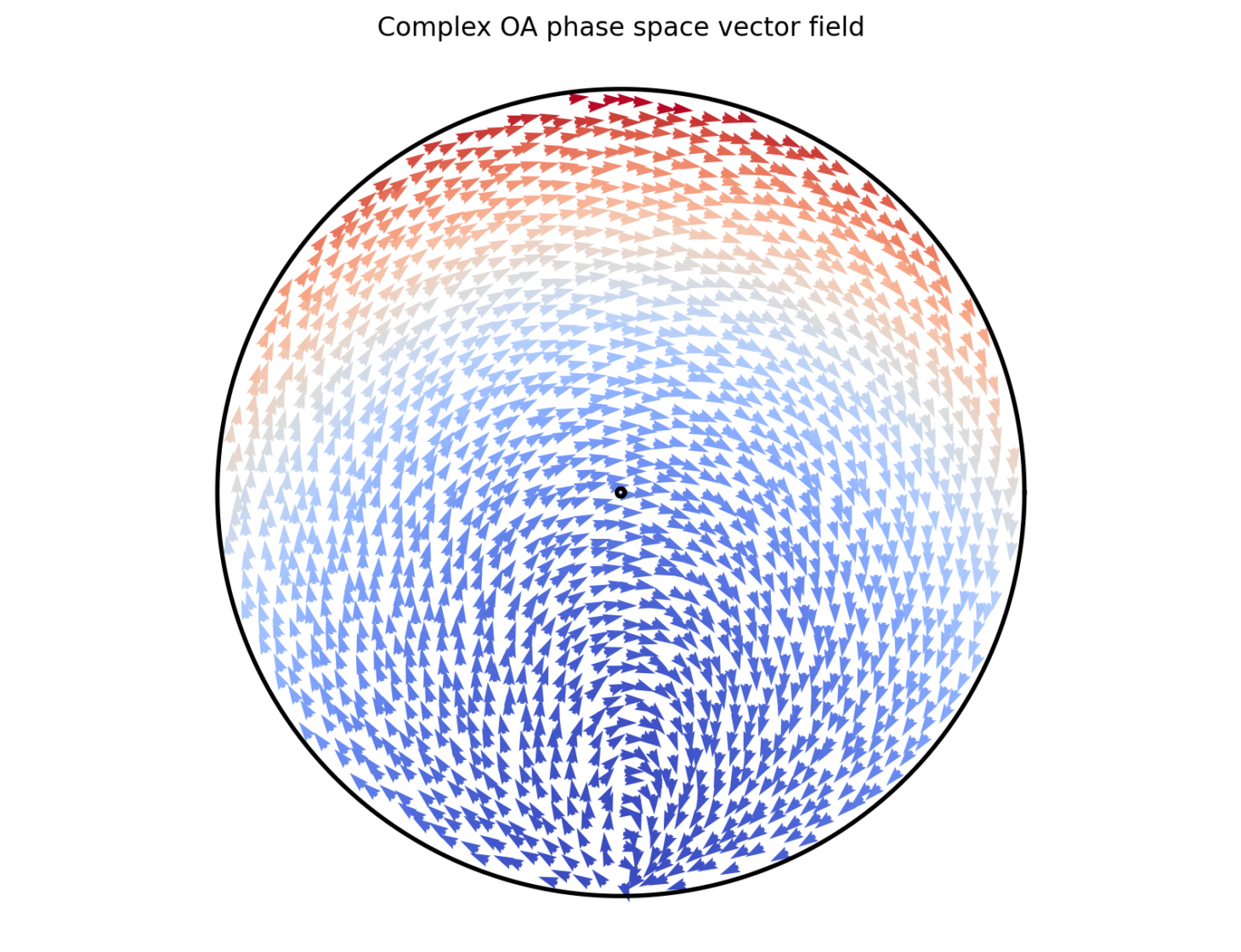}
  \label{fig:case_4_1_b_cphase}
\end{subfigure}%
\begin{subfigure}{.33\textwidth}
  \centering
  \includegraphics[trim={0cm 0cm 0cm 1.6cm},clip,width=\textwidth]{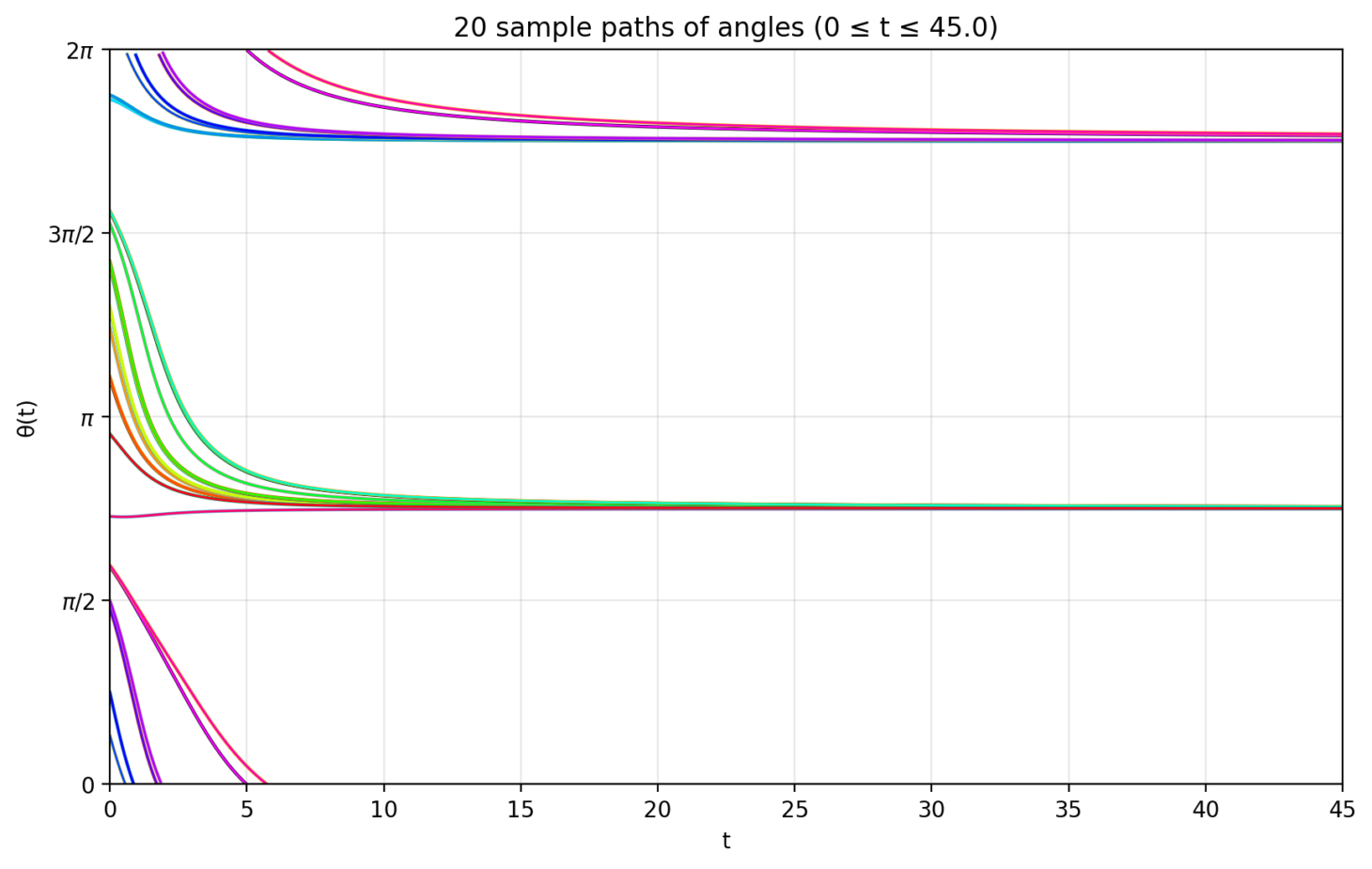}
  \label{fig:case_4_1_b_particle}
\end{subfigure}

\medskip
\begin{subfigure}{.33\textwidth}
  \centering
  \includegraphics[trim={0cm 0cm 0cm 1.6cm},clip,width=\textwidth]{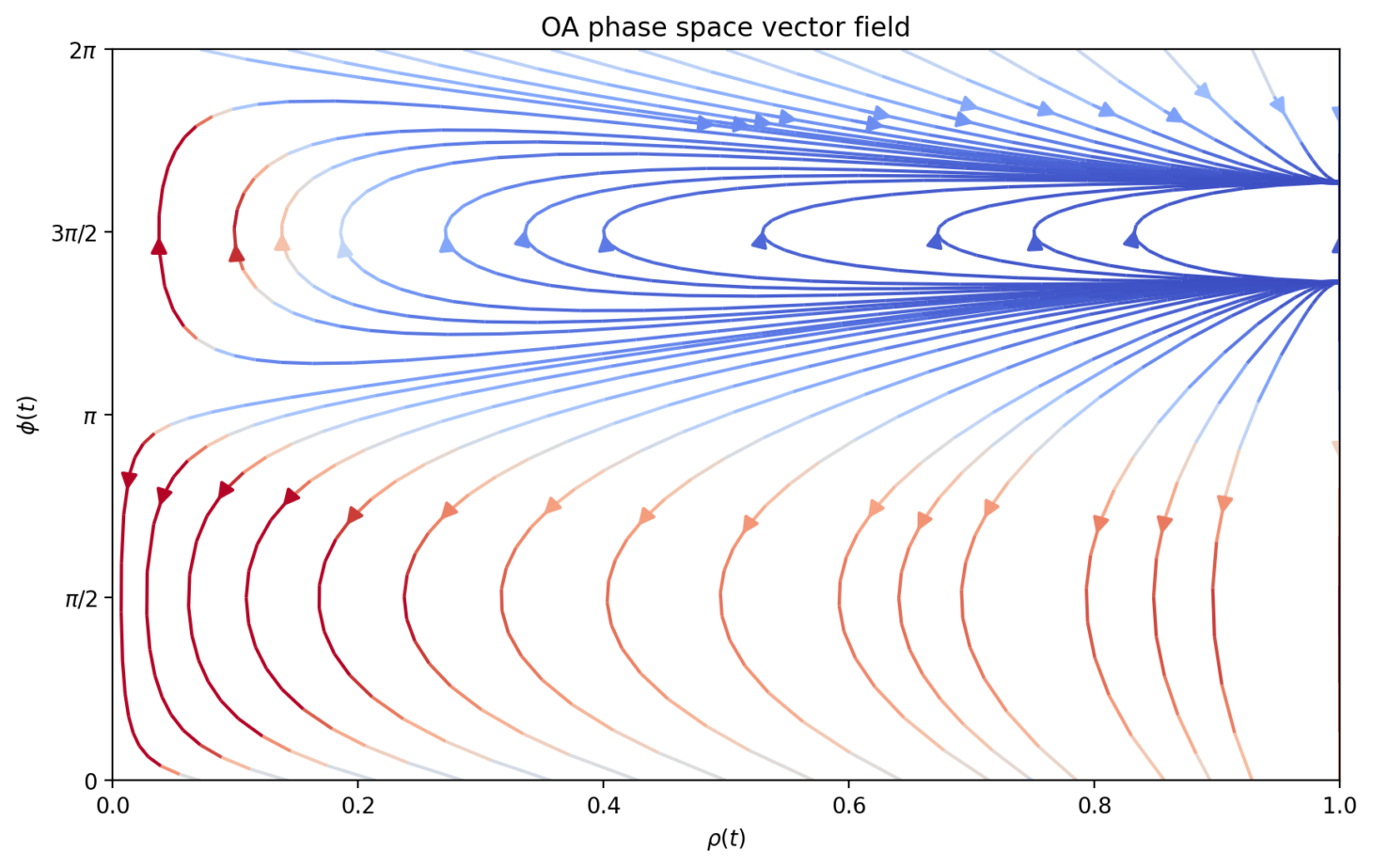}
  \label{fig:case_4_1_c_phase}
\end{subfigure}
\begin{subfigure}{.33\textwidth}
  \centering
  \includegraphics[trim={0cm 0cm 0cm 1.7cm},clip,width=\textwidth]{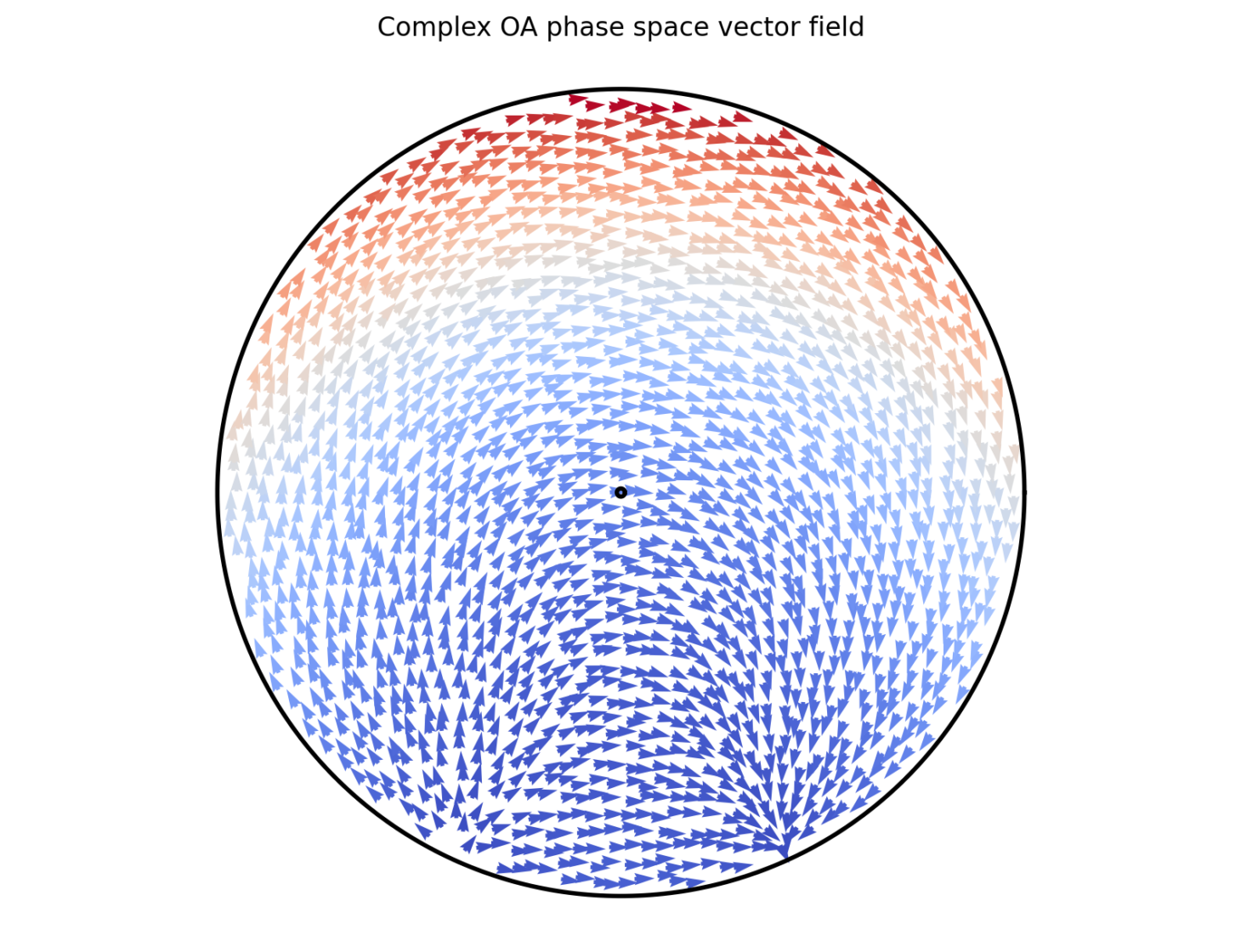}
  \label{fig:case_4_1_c_cphase}
\end{subfigure}%
\begin{subfigure}{.33\textwidth}
  \centering
  \includegraphics[trim={0cm 0cm 0cm 1.6cm},clip,width=\textwidth]{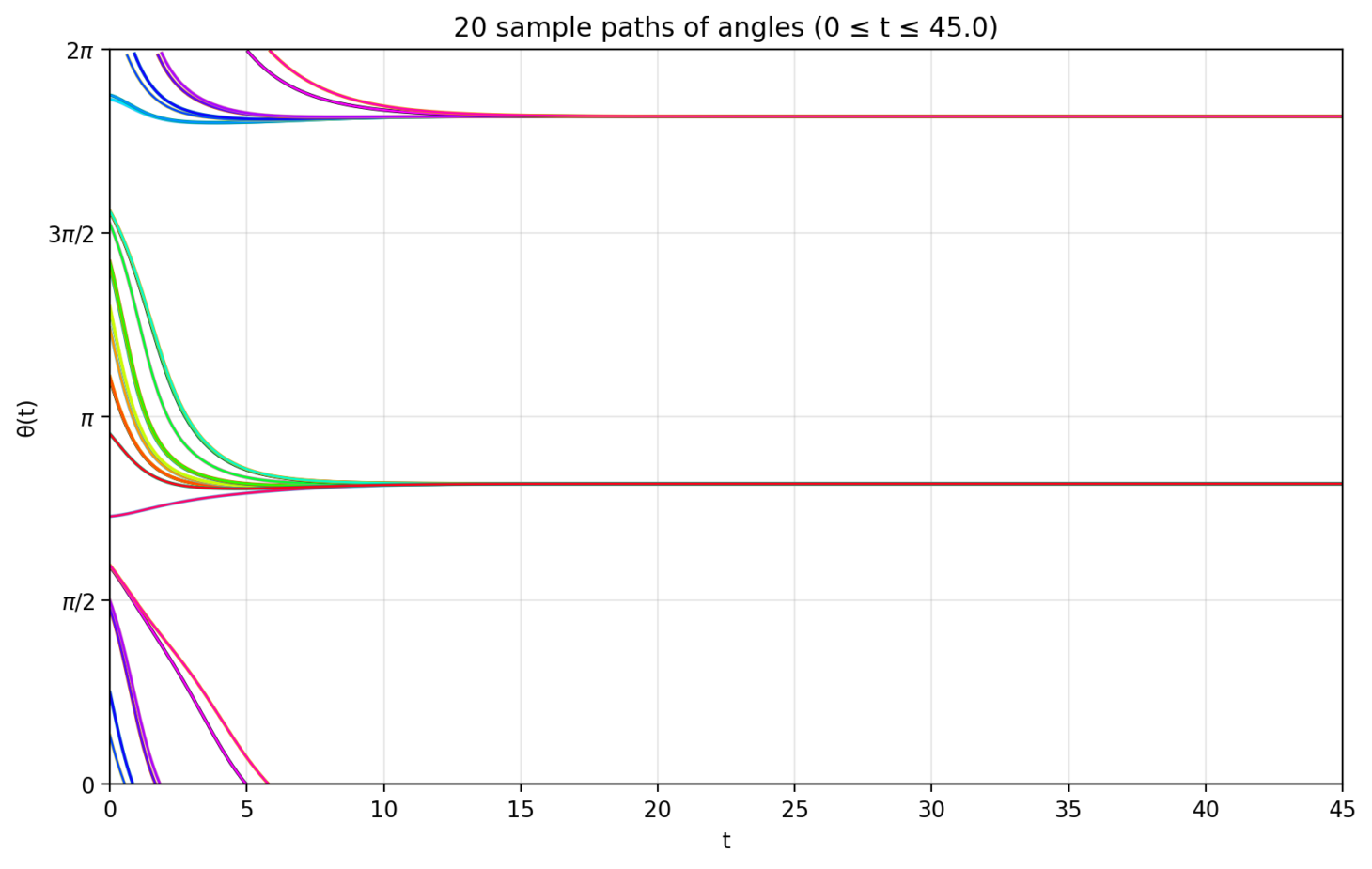}
  \label{fig:case_4_1_c_particle}
\end{subfigure}%
\caption{Figures showing the phase portraits and bifurcation phenomena in the dynamics \eqref{eq:finite_sys} with $A = I$ and $V = \left(\begin{smallmatrix} v_{11} & 1 \\ -1 & -v_{11} \end{smallmatrix}\right)$.
\textbf{Top:} $v_{11} = 0.9$. \textbf{Middle:} $v_{11} = 1.0$. \textbf{Bottom:} $v_{11} = 1.1$. The first column shows the cartesian phase portraits, the middle column shows the complex portraits, and the final column the particle trajectories.  Note that $T=45$ in the trajectories in order to highlight cyclic and slow convergence behavior. Moving from top to bottom, the panels illustrate the transition from a stable limit cycle, to a semistable fixed point, and then to a stable-unstable pair of equilibria. 
Note the slow convergence in the middle panel at the bifurcation point.}
\label{fig:case_4_1}
\end{figure}

Theorem \ref{thm:case_4} states that, when $A = I$ and the matrix $V$ is given by \eqref{eq:matrices_A_V_case_4}, the system \eqref{eq:pc_dyn_case_4} undergoes a saddle-node bifurcation and exhibits qualitatively distinct long-term behavior across different parameter regimes.
In particular, when $v_{11} < v_{12}$, almost any solution $(\rho, \phi)$ to \eqref{eq:pc_dyn_case_4} is periodic and traces out a closed orbit.
Since the distribution $g$ of the double-angle variable $\Xi = 2\Theta$ is always wrapped Cauchy and is completely determined by the order parameter $\CR_2$, it follows that $g$ is likewise periodic and therefore does not converge to a stationary limit.
Moreover, through the relationship between $\CR_2$ and the distribution $f$ of the original particle $\Theta$, this also implies that $f$ continues to oscillate in time, possibly aperiodically.
By contrast, when $v_{11} \geq v_{12}$, $\CR_2(t)$ converges to a boundary equilibrium, which in turn implies that $f$ converges to a stationary distribution supported on two antipodal points.
This bifurcation behavior is illustrated in Figure~\ref{fig:case_4_1}.

To the best of our knowledge, this is the first theoretical characterization of bifurcation phenomena in the transformer dynamics.
This result, together with Theorem \ref{thm:case_3}, goes beyond the clustering and synchronization behaviors that have been the primary focus of the existing literature, and highlights the richness of dynamical behaviors that can already arise in linear transformers in dimension $d=2$.
It also suggests that, for certain choices of the matrices $A$ and $V$, the transformer output may be sensitive even to small perturbations of these parameters.
This observation may be particularly relevant in the context of transformer model quantization \cite{xiao2023smoothquant,hooper2024kvquant}.

\begin{remark}
We note that in the case $v_{11} \geq v_{12}$ we prove convergence to a fixed point despite the fact that the system can be described in Hamiltonian form~\eqref{eq:hamiltonian_form}.  This is not a contradiction, as the Hamiltonian description is only valid on the \textit{open} cylinder $(0,1) \times \BBT^1$, whereas the full dynamics are well defined at $\rho = 1$.  It turns out that the energy degenerates in a convenient way in the limit $\rho \to 1$, where all level sets agree on the same phase $\phi_\infty$.
\end{remark}

\begin{remark}
The proof in Appendix~\ref{proof:case_4} provides a more detailed description of the cyclic behavior in the regime $v_{11} < v_{12}$.
In particular, we characterize the regions where dynamics flow \textit{around} the cylinder (so called `rotations') versus the regions where the dynamics flow in a closed region on the cylinder (so called `librations').  
These can been seen in the top left panel of Figure~\ref{fig:case_4_1}, where rotations are observed as flow lines completing a cycle inside the plot, and librations are flow lines leaving the bottom of the plot and wrapping around to the top. These two behaviors are separated by a special Hamiltonian level set $\{H = -v_{12}/2 \}$.
Level sets with values above this separatrix correspond to rotations, while those below correspond to librations; see Figure \ref{subfig:H_level_set_a_less_than_b} for an illustration.
\end{remark}

\begin{remark}
In typical dynamical systems, slow dynamics occur in parameter regimes around a bifurcation. We have empirically observed such a bottleneck in our system. In particular, we have observed that when $v_{11} \approx v_{12}$ and $v_{11} < v_{12}$ the system spends most of its time close to $(1, \phi_\infty)$ (this is the point where all flow lines coalesce in the middle left figure).
We view this phenomenon as an undesirable form of \textit{metastability}, which was studied in a different context for the self-attention dynamics~\cite{geshkovski2023mathematical}. Generalizing and making this observation rigorous is an interesting question that we leave for future work.
\end{remark}

\section{Long-Time Behavior Beyond the Ott--Antonsen Ansatz: Stability Analysis}\label{sec:stability}
In Section~\ref{sec:case_study}, we analyzed the long-time behavior of order parameter $\CR_2$ under the Ott--Antonsen ansatz, thereby characterizing the corresponding behavior of the particle dynamics \eqref{eq:mf_sys} for different choices of the matrices $A$ and $V$.
However, this analysis relies on the assumption that the double-angle particle $\Xi = 2\Theta$ is initialized within the wrapped Cauchy family \eqref{eq:WC_dis}.  
In this section, we go beyond the OA ansatz by studying the long-time behavior of $\CR_2$ in the WS system \eqref{eq:mf_WS_param}.

More specifically, recall that in the general WS system \eqref{eq:mf_WS_param}, the order parameter $\CR_2$ depends on both WS variables $\alpha$ and $\eta$ through \eqref{eq:R2_push_foward}, which makes its long-time behavior difficult to analyze directly.
On the other hand, as discussed in Section \ref{subsec:WS}, whenever $|\alpha| \rightarrow 1$ we also have $\CR_2 = \alpha$ and thus $|\CR_2| \rightarrow 1$, corresponding to the full synchronization of the particle dynamics.
Crucially, this conclusion holds \emph{independently} of the OA ansatz. 

Motivated by this observation, in Section \ref{subsec:structural-stability} we analyze the structural stability of the $\alpha$-dynamics \eqref{eq:alpha_dyn}.
We then apply this result in Section \ref{subsec:LTB_beyond_OA} to characterize the long-time behavior of $\alpha$ and identify the choices of the matrices $A$ and $V$ for which $|\alpha| \to 1$.
In turn, this yields the convergence $|\CR_2| \rightarrow 1$.

Before proceeding, we recall several equations from Section \ref{sec:2} that will be used later.
We consider the WS system in \eqref{eq:mf_WS_param}, which we rewrite here for convenience:
\begin{subequations}\label{eq:WS_sys_sec_4}
\begin{equation}\label{eq:WSshort-a}
    \Dalpha(t) = F(\alpha(t), \CR_2(t)),
\end{equation}
\begin{equation}\label{eq:WSshort-b}
    \Deta(t) = G(\alpha(t), \CR_2(t)) \, ,
\end{equation}
\end{subequations}
where the functions $F$ and $G$ are defined in \eqref{eq:F_and_G}. We also introduce the polar-coordinate form $\alpha = \gamma e^{i\Phi}$ with $\gamma: [0,\infty) \mapsto [0,1]$ and $\Phi: [0,\infty) \mapsto \BBT^1 $.
Recall that applying the M\"obius transform~\eqref{eq:Mobius_tf} to the double-angle variable $\Xi$ defined in \eqref{eq:xi_mf_sys} gives
\begin{equation}\label{eq:WS-mobius}
e^{i\Xi(t;\varphi)} = M_{\alpha(t),\eta(t)} (e^{i\varphi}) = \frac{\alpha(t)+e^{i(\varphi+\eta(t))}}{1+\overline{\alpha}(t)e^{i(\varphi+\eta(t))}}.
\end{equation}
The constant of motion $\varphi$ is time-independent, and hence its associated distribution $\nu$ is preserved over time. 
Moreover, we have the relation $\Xi = T_{\alpha, \eta} (\varphi)$, where the map $T_{\alpha, \eta}$ is defined in \eqref{eq:map_T}. 
We also recall from \eqref{eq:R2_push_foward} that the 
order parameter can be expressed as the pushforward average of the distribution $\nu$ induced by the WS variables $\alpha$ and $\eta$:
\begin{equation}\label{eq:r2-WS}
\CR_2(t)=\int_0^{2\pi} e^{i\xi}\,(T_{\alpha(t), \eta(t)})_{\sharp} \nu (d\xi)
= \int_0^{2\pi} M_{\alpha(t), \eta(t)} (e^{i\varphi}) \,\nu (d\varphi).
\end{equation}

Equations \eqref{eq:WS_sys_sec_4}--\eqref{eq:r2-WS} provide a closed WS representation of the mean-field dynamics \eqref{eq:mf_sys} in terms of $(\alpha,\eta)$ and the fixed distribution $\nu$ of the constant of motion $\varphi$.

\subsection{Stability of the WS Dynamics Beyond the OA-manifold}\label{subsec:structural-stability}
Recall from Section \ref{subsec:OA} that the $\alpha$-dynamics \eqref{eq:WSshort-a} reduces to the OA dynamics \eqref{eq:dyn_r2} when the distribution of the double-angle variable $\Xi$ is initialized within the wrapped Cauchy family \eqref{eq:WC_dis}. 
This condition on $\Xi$ is equivalent to the distribution $\nu$ being uniform (see Lemma \ref{lem:WS_to_OA}).
In this section, we perform a stability analysis of the $\alpha$-dynamics \eqref{eq:WSshort-a} under perturbations away from the OA regime by quantifying its deviation from the OA dynamics \eqref{eq:dyn_r2} when $\nu$ is not the uniform measure.

We begin by outlining the main idea of the argument.
Our analysis is based on a comparison between the $\alpha$-dynamics \eqref{eq:WSshort-a} with a OA-closure counterpart
\begin{equation}\label{eq:OA_counterpart}
\dot{\alpha}_{\OA}(t)=F(\alpha_{\OA}(t),\alpha_{\OA}(t)) \, .
\end{equation}
To this end, we rewrite \eqref{eq:WSshort-a} in the form of a perturbation of the OA dynamics \eqref{eq:OA_counterpart}:
\begin{equation}\label{eq:alpha_rewrite}
\Dot{\alpha}(t) = F(\alpha(t), \CR_2(t)) = F(\alpha(t), \alpha(t)) + \Delta(t),
\end{equation}
where the perturbation term $\Delta: [0, \infty) \mapsto \BBC$ is defined by
\begin{equation}\label{eq:Delta}
    \Delta(t) := F(\alpha(t), \CR_2(t)) - F(\alpha(t), \alpha(t)) \, .
\end{equation}
From \eqref{eq:F_and_G}, the function $F$ appearing in \eqref{eq:WSshort-a} is smooth and globally Lipschitz on 
$D \times D$, where $D$ is the closed complex unit disc.
We denote the corresponding Lipschitz constant by $L_F$.
Assuming that \eqref{eq:OA_counterpart} and \eqref{eq:alpha_rewrite} share the same initial condition $\alpha(0) = \alpha_{\OA} (0)$,
then a standard Grönwall argument comparing the two equations gives, for every $T > 0$,
\begin{equation*}
\sup_{t \in [0,T]} \left| \alpha(t) - \alpha_{\OA}(t) \right| \leq \int_0^T e^{L_F(T-s)} |\Delta(s)|ds \, .
\end{equation*}
Moreover, again by the Lipchitz continuity of $F$, we have
\begin{equation}\label{eq:Delta_bound_gen}
    |\Delta(t)| = \left| F(\alpha(t), \CR_2(t)) - F(\alpha(t), \alpha(t)) \right| \leq L_F \left| \CR_2 (t) - \alpha(t) \right| \, .
\end{equation}
Therefore, the problem is reduced to estimating the discrepancy between the order parameter $\CR_2$ and the WS parameter $\alpha$ within the WS system \eqref{eq:WS_sys_sec_4}.

In the sequel, we first establish upper bounds for $|\CR_2(t) - \alpha(t)|$ in Proposition \ref{prop:diff_r2_alpha}, and then use these estimates to derive a stability bound for $\sup_{t \in [0,T]} |\alpha(t) - \alpha_{\OA}(t)|$ in Proposition \ref{prop:structural_stability}.

Before stating the main results, we recall the bounded-Lipschitz distance between measures $\mu, \tilde \mu$ over $[0,2\pi)$, which is given by
\begin{equation}\label{eq:BL_metric}
d_{\BL}(\mu,\widetilde{\mu}) = \sup\left\{\left|\int_0^{2\pi}g\, d\mu - \int_0^{2\pi}g\, d\widetilde{\mu} \right|:\ \|g\|_\infty\leq 1,\ {\rm Lip}(g)\leq 1\right\}.
\end{equation}
We also clarify a potential source of confusion concerning the notation.
In the WS system, the order parameter $\CR_2(t)$ does not, in general, evolve according to the OA dynamics \eqref{eq:dyn_r2} studied in Section~\ref{sec:case_study}, which is only valid under the OA ansatz. To avoid ambiguity, we denote by $\alpha_{\OA}$ the variable satisfying the OA dynamics, as we wrote in \eqref{eq:OA_counterpart}, which coincides with \eqref{eq:dyn_r2}, while reserving the notation $\CR_2$ for the actual order parameter.


\begin{proposition}\label{prop:diff_r2_alpha} 
Let $\Unif$ denote the uniform probability measure on $[0,2\pi)$, and let $\nu$ be a probability measure on $[0,2\pi)$.
Let $\alpha(t)=\gamma(t)e^{i\Phi(t)}$ be a solution of the dynamics
\eqref{eq:WSshort-a}, and let $\CR_2(t)$ denote the corresponding order parameter defined by \eqref{eq:r2-WS} with respect to $\nu$.
Then, for every $t \geq 0$, the difference between $\CR_2(t)$ and $\alpha(t)$ admits the following two upper bounds:
\begin{equation}\label{eq:diff_r2_alpha_one_minus_gamma}
|\CR_2(t)-\alpha(t)|\leq  C_{\nu}(\gamma(t)) (1-\gamma(t)),
\end{equation}
where
\begin{equation}\label{eq:C_nu_condition}
    C_{\nu}(\gamma) :=\sup_{\chi\in[0,2\pi)}\int_0^{2\pi} \frac{2}{|1+\gamma e^{i(\varphi+\chi)}|}\,\nu(d\varphi),
\end{equation}
and
\begin{equation}\label{eq:diff_r2_alpha_initialization}
|\CR_2(t)-\alpha(t)| \leq \left( 1 + \frac{2}{1 - \gamma(t)} \right) d_{\BL} (\nu, \Unif) \, .
\end{equation}
\end{proposition}
Before turning to the proof, we briefly discuss the interpretation of Proposition \ref{prop:diff_r2_alpha}.
The proposition shows that the discrepancy between the order parameter $\CR_2$ and the WS parameter $\alpha$ within the WS system \eqref{eq:WS_sys_sec_4} can be small in two different ways, provided we can control $C_\nu(\gamma)$ and $(1 - \gamma)$.
On the one hand, inequality \eqref{eq:diff_r2_alpha_one_minus_gamma} shows that the discrepancy is controlled by the factor $1 - \gamma$, and is therefore small when the modulus $\gamma$ of $\alpha$ is close to $1$, assuming that the prefactor $C_{\nu}(\gamma)$ remains controlled.
On the other hand, inequality \eqref{eq:diff_r2_alpha_initialization} shows that the discrepancy is small when the probability measure $\nu$ is close to the uniform measure $\Unif$ in bounded-Lipschitz distance, as long as $\gamma$ remains uniformly bounded away from $1$.


\begin{proof}[Proof of Proposition \ref{prop:diff_r2_alpha}]
We show that the difference $|\CR_2(t) - \alpha(t)|$ admits two different upper bounds.
For notational simplicity, we suppress the time dependence throughout the proof.

We first verify inequality \eqref{eq:diff_r2_alpha_one_minus_gamma}.
Using the form of $\CR_2$ in \eqref{eq:r2-WS}, we compute
\begin{equation}\label{eq:diff_r2_alpha_one_minus_gamma_computation}
\begin{aligned}
\left| \CR_2 - \alpha \right| & = \left| \int_0^{2\pi} \frac{\alpha + e^{i(\varphi + \eta)}}{1 + \alphaBar e^{i(\varphi + \eta)}} d\nu(\varphi) - \alpha \right| = \left| \int_0^{2\pi} \frac{e^{i(\varphi + \eta)}}{1 + \alphaBar e^{i(\varphi + \eta)}} d\nu (\varphi) \right| (1 - |\alpha|^2) \\
&\leq \int_0^{2\pi} \frac{1}{|1 + \gamma e^{i(\varphi + \eta - \Phi)}|} d\nu (\varphi) (1 - \gamma^2) \leq C_{\nu}(\gamma) (1-\gamma) \, .
\end{aligned}
\end{equation}
Here, the first inequality applies the integral triangle inequality and uses the polar-representation $\alpha = \gamma e^{i\Phi}$, and the second inequality makes use of $(1+\gamma)\leq 2$ and the definition of $C_{\nu}(\gamma)$. 

Next we prove inequality \eqref{eq:diff_r2_alpha_initialization}.
Recall the Möbius transform from \eqref{eq:Mobius_tf}. 
By Lemma \ref{lem:Mobius_tf_properties}, we have
$$ \alpha =\int_0^{2\pi} M_{\alpha,\eta}(e^{i\varphi})\, d\Unif(\varphi) \, .$$
Combining this identity with \eqref{eq:r2-WS}, we find that
\[
\CR_2-\alpha=\int_0^{2\pi}M_{\alpha,\eta}(e^{i\varphi})\,d(\nu-\Unif)(\varphi).
\]
By the definition of $d_{\BL}$ in \eqref{eq:BL_metric} and Lemma \ref{lem:Mobius_tf_properties}, we have
\[
|\CR_2-\alpha|
\le \bigl(\|M_{\alpha,\eta}\|_\infty+\mathrm{Lip}(M_{\alpha,\eta})\bigr)\, d_{\BL}(\nu,\Unif)
\le \Bigl(1+\frac{2}{1-\gamma}\Bigr)d_{\BL}(\nu,\Unif) \, .
\]
\end{proof}

Using the polar-coordinate representations $\alpha = \gamma e^{i\Phi}$ and $\CR_2 = \rho e^{i\phi}$, together with elementary inequalities, we obtain the following corollary as a direct consequence of Proposition \ref{prop:diff_r2_alpha}.
The associated proof is deferred to Appendix \ref{app:stability}.

\begin{corollary}\label{cor:diff_r2_alpha_pc_bound}
Let $\alpha(t)=\gamma(t)e^{i\Phi(t)}$ be a solution of the dynamics
\eqref{eq:WSshort-a}, and let $\CR_2(t) = \rho(t) e^{i\phi(t)}$ be defined by \eqref{eq:r2-WS}.
Then the following inequalities hold
\begin{equation}\label{eq:diff_r2_alpha_pc_bound}
|\rho(t)-\gamma(t)|\leq C_\nu(\gamma)(1-\gamma(t)),\quad |e^{i\Phi(t)} - e^{i\phi(t)}|\leq 2C_\nu(\gamma)\frac{1-\gamma(t)}{\gamma(t)},\quad |\Phi(t)-\phi(t)|\leq 2\pi C_\nu(\gamma)\frac{1-\gamma(t)}{\gamma(t)},
\end{equation}
where $C_{\nu}(\gamma)$ is defined in \eqref{eq:C_nu_condition}.
\end{corollary}


Using the discrepancy estimates between $\CR_2(t)$ and $\alpha(t)$ derived in Proposition \ref{prop:diff_r2_alpha}, together with the discussion at the beginning of this section, we are now in a position to establish the stability of the $\alpha$-dynamics \eqref{eq:WSshort-a} with respect to its OA counterpart \eqref{eq:OA_counterpart}.

\begin{proposition}[Stability of the $\alpha$-dynamics]\label{prop:structural_stability}
Let $\nu$ be a probability measure on $[0,2\pi)$ and $\Unif$ denote the uniform
probability measure on $[0,2\pi)$.
Let $(\alpha, \eta)$ be a solution of the dynamics \eqref{eq:WS_sys_sec_4} with initial data $(\alpha(0), \eta(0))$, and $\alpha_{\OA}$ be a solution of the dynamics \eqref{eq:OA_counterpart} satisfying $\alpha_{\OA}(0)=\alpha(0)$.
Then, for arbitrary $\eta(0) \in \BBT^1$ and every $T > 0$,
\begin{equation*}
\sup_{t\in[0,T]}|\alpha(t)-\alpha_{\OA}(t)|
\;\le\;
L_{F}\,e^{L_{F} T} \int_0^T \min \left\{ C_{\nu} (\gamma(t)) (1 - \gamma(t)), \ \left( 1 + \frac{2}{1 - \gamma(t)} \right)
 d_{\BL}(\nu,\Unif) \right\} dt \, .
\end{equation*}
\end{proposition}
Proposition \ref{prop:structural_stability} follows a standard Grönwall argument applied to the comparison between the two ODEs \eqref{eq:OA_counterpart} and \eqref{eq:alpha_rewrite}, together with the two upper bound estimates on $|\CR_2 - \alpha|$ established in Proposition~\ref{prop:diff_r2_alpha}.
The resulting estimate is therefore in the same spirit as Proposition~\ref{prop:diff_r2_alpha}.
In particular, it states that, on any finite time interval $[0,T]$, the deviation of the $\alpha$-dynamics from its OA counterpart $\alpha_{\OA}$ is controlled in two ways: first, by the factor $1 - \gamma(t)$, which measures the degree of synchronization of the particle dynamics; and second, by the proximity of the measure $\nu$ to the uniform distribution $\Unif$, both up to a multiplicative constant. This estimate will serve as one of the key ingredients in identifying different choices of the matrices $A$ and $V$ such that $|\alpha| \rightarrow 1$.
We develop these applications in Section \ref{subsec:LTB_beyond_OA}.

Although the proof of Proposition \ref{prop:structural_stability} is straightforward, we include it for completeness.

\begin{proof}[Proof of Proposition \ref{prop:structural_stability}]
Let $e(t) = \alpha(t) - \alpha_{\OA}(t)$.
Subtracting \eqref{eq:OA_counterpart} and \eqref{eq:alpha_rewrite}, and using the Lipschitz continuity of $F$, we obtain
\begin{equation*}
|\dot{e}(t)| \leq \left| F(\alpha(t), \alpha(t)) - F(\alpha_{\OA}(t), \alpha_{\OA}(t)) \right| + |\Delta(t)| \leq L_F|e(t)| + |\Delta(t)| \, .
\end{equation*}
Since $\alpha(0) = \alpha_{\OA}(0)$, we have $e(0) = 0$.
Integrating the above inequality over $[0,t]$ with $0 \leq t \leq T$, we therefore obtain
\begin{equation*}
|e(t)| \leq L_F \int_0^t |e(s)|ds + \int_0^t |\Delta(s)| ds \, .
\end{equation*}
Applying Grönwall inequality yields
\begin{equation*}
|e(t)| \leq \int_0^t e^{L_F(t-s)} |\Delta(s)| ds \leq e^{L_F T} \int_0^T |\Delta(s)| ds \, .
\end{equation*}
Using again the Lipschitz continuity of $F$ and the Proposition \ref{prop:diff_r2_alpha}, we obtain that
\begin{equation*}
\begin{aligned}
|\Delta(t)| &= \left| F(\alpha(t), \CR_2(t)) - F(\alpha(t), \alpha(t)) \right| \\
&\leq L_F |\CR_2(t) - \alpha(t)| \leq L_F \min \left\{ C_{\nu} (\gamma(t)) (1 - \gamma(t)),  \left( 1 + \frac{2}{1 - \gamma(t)} \right) d_{\BL} (\nu, \Unif) \right\} \, .
\end{aligned}
\end{equation*}
We conclude the proof by combining the above estimates and taking the supremum over $t \in [0,T]$.
\end{proof}

\subsection{Long-Time Behavior}\label{subsec:LTB_beyond_OA}
In this section, we revisit Cases 1 and 2 studied in Section \ref{sec:case_study} and focus on the matrices $A$ and $V$ for which the associated OA dynamics \eqref{eq:OA_counterpart} converges to full synchronization, namely, $|\alpha_{\OA}| \rightarrow 1$.
Building on the stability of the $\alpha$-dynamics \eqref{eq:alpha_rewrite} in the WS system established in Section \ref{subsec:structural-stability}, we prove that the $\alpha$-dynamics likewise satisfies $|\alpha(t)| \rightarrow 1$ as $t \rightarrow \infty$.
This requires slightly stronger assumptions on $A$ and $V$, together with the assumption that the distribution $\nu$ satisfies an $L^p$ regularity condition and is sufficiently close to the uniform measure $\Unif$. 
The corresponding main results for Cases 1 and 2 are presented in Theorem \ref{thm:case_1stab} and Theorem \ref{thm:case_2stab}, respectively.

Before turning to the case studies, we formally state the $L^p$ regularity assumption on $\nu$ and briefly explain its role.

\begin{assumption}[Bounded $L^p$ density]\label{asm:Lp}
For some $p > 1$ and $M > 0$, 
the probability measure $\nu$ has a density, still denoted by $\nu$, satisfying
\begin{equation*}
\|\nu\|_{L^p([0,2\pi))} \leq M \, .
\end{equation*}
\end{assumption}

\begin{remark}[Relation to the initial condition $f(0)$]
The regularity assumption $\nu \in L^p([0,2\pi))$ is equivalent to the condition $f(0) \in L^p([0,2\pi))$, where $f(0)$ is the initial distribution of the physical variable $\Theta$ in the mean-field linear self-attention model \eqref{eq:mf_sys}.

In the rest of this remark, all quantities are evaluated at $t=0$.
Recall from \eqref{eq:g_and_nu} that the distribution $g$ of the double-angle variable $\Xi = 2\Theta$ is given by $g = (T_{\alpha, \eta})_{\sharp} \nu$, where $T_{\alpha,\eta}$ is defined in \eqref{eq:map_T}.
As discussed in Section \ref{subsec:WS}, whenever $|\alpha| \neq 1$, the map $T_{\alpha, \eta}$ is a diffeomorphism.
It follows that $\nu \in L^p([0,2\pi))$ if and only if $g \in L^p([0,2\pi))$.
By the relationship between $f$ and $g$ given in \eqref{eq:relation_f_and_g}, it is equivalent to require $f(0) \in L^p([0,2\pi))$.
\end{remark}

\begin{remark}[The role of Assumption \ref{asm:Lp}]
Assumption \ref{asm:Lp} is motivated by the upper bound \eqref{eq:diff_r2_alpha_one_minus_gamma} on the discrepancy between the order parameter $\CR_2$ and the WS variable $\alpha$ established in Proposition \ref{prop:diff_r2_alpha}.
In particular, to prove convergence of the $\alpha$-dynamics to full synchronization in Theorems \ref{thm:case_1stab} and \ref{thm:case_2stab}, one needs to ensure that $|\CR_2 - \alpha| \rightarrow 0$ as the modulus $\gamma = |\alpha| \rightarrow 1$.  
In view of \eqref{eq:diff_r2_alpha_one_minus_gamma}, it is enough to impose the following abstract regularity condition on $\nu$:
\begin{equation}
\label{eq:ConsistencyCnu}
        C_{\nu}(\gamma) (1 - \gamma) \rightarrow 0 \qquad \text{as } \gamma \rightarrow 1,
    \end{equation}
where $C_{\nu}(\gamma)$ is defined in \eqref{eq:C_nu_condition}.
Assumption \ref{asm:Lp} provides a sufficient condition that ensures that \eqref{eq:ConsistencyCnu} holds; see Lemma \ref{lem:nu-lp-unified}. 

We note that $\nu \in L^p([0,2\pi))$ does not cover the atomic distributions.
Identifying more general assumptions on $\nu$ that also allow for atomic components requires a further analysis of the WS system \eqref{eq:WS_sys_sec_4}. We leave this question for future work. 
\end{remark}

We are now ready to revisit the case studies and identify the parameter regimes of the matrices $A$ and $V$ for which the $\alpha$-dynamics converges to full synchronization.
Our analysis is carried out using the polar-coordinate representation $\alpha = \gamma e^{i\Phi}$.
In particular, the $\alpha$-dynamics in \eqref{eq:alpha_rewrite} can be rewritten in terms of $(\gamma,\Phi)$ as
\begin{subequations}\label{eq:gamma_Phi_dyn}
\begin{equation}\label{eq:gamma_dyn}
\Dgamma = \re{e^{-i\Phi} F(\alpha, \alpha)} + \Delta_{\gamma},
\end{equation}
\begin{equation}\label{eq:Phi_dyn}
\DPhi = \frac{1}{\gamma} \im{e^{-i\Phi} F(\alpha, \alpha)} + \Delta_{\Phi}
\end{equation}
\end{subequations}
where
\begin{equation}\label{eq:Delta_gamma_Phi}
\begin{aligned}
\Delta_{\gamma}(t) := \re{e^{-i\Phi} \Delta(t)}, \qquad \Delta_{\Phi}(t) &:= \frac{1}{\gamma}\im{e^{-i\Phi} \Delta(t)}
\end{aligned}
\end{equation}
with $\Delta(t)$ defined in \eqref{eq:Delta}.

\subsubsection*{Case 1 Revisited.}
We consider the case where $A$ is a real-valued matrix and $V$ is the identity matrix, as defined in \eqref{eq:matrices_A_V_case_1}.
In this setting, the polar-coordinate system \eqref{eq:gamma_Phi_dyn} simplifies to

\begin{equation}\label{eq:alpha_pc_dyn_case1}
    \begin{aligned}
    \dot{\gamma} &= \frac{1}{4} \left(1 - \gamma^2 \right) \left( \aDPlus \gamma + \aDMinus \cos (\Phi) + \aOPlus \sin (\Phi) \right) + \Delta_\gamma,\\[4pt]
    \dot{\Phi} &= \frac{1}{4} \left(1 - \gamma^2 \right) \left( \aOMinus - \frac{\aDMinus \sin (\Phi) - \aOPlus \cos (\Phi)}{\gamma} \right) + \Delta_{\Phi},
    \end{aligned}
\end{equation}
where the explicit forms of $\DeltaGamma$ and $\DeltaPhi$ are provided in \eqref{eq:Delta_gamma_Phi_case1}.

The long-term behavior of \eqref{eq:alpha_pc_dyn_case1} is stated in the following theorem.
\begin{theorem}\label{thm:case_1stab}
Fix any $\iota > 0$, and consider the dynamical system~\eqref{eq:alpha_pc_dyn_case1} with initial data
\begin{equation}\label{eq:case1_stab_initial_data}
(\gamma(0), \Phi(0)) \in \left\{ (\gamma, \Phi) \in (0,1] \times \BBT^1: |(\gamma, \Phi) - (\rho_{\eq}, \phi_{\eq})| \geq \iota \right\},
\end{equation}
where $(\rho_{\eq}, \phi_{\eq})$ is defined in \eqref{eq:rho*_phi*}.
Assume that the matrix $A$ satisfies 
\begin{equation}\label{eq:case_1_stab_A_condi}
\frac{A + A^{\top}}{2} \succ 0 \, .
\end{equation}
Moreover, assume that the probability measure $\nu$ satisfies Assumption \ref{asm:Lp} and
\begin{equation}\label{eq:case_1_stab_close_to_unif}
 d_{\BL} (\nu, \Unif) < \varepsilon(\iota, p, M, A),
\end{equation}
where $\varepsilon(\iota, p, M, A) > 0$ is a sufficiently small constant depending only on $\iota$, the constants $p$ and $M$ in Assumption \ref{asm:Lp}, and the matrix $A$.
Then the solution of \eqref{eq:alpha_pc_dyn_case1} converges to $(1, \Phi_{\infty})$ as $t \rightarrow \infty$ for some $\Phi_{\infty} \in \BBT^1$.
\end{theorem}

Theorem \ref{thm:case_1stab} shows that the WS variable $\alpha$ converges to a fixed point on the boundary $\partial D$, provided that $V = I$, the symmetric part of $A$ is positive definite, and the distribution $\nu \in L^p([0,2\pi))$ is sufficiently close to the uniform measure $\Unif$ on $[0,2\pi)$. 
As discussed in Section \ref{subsec:WS}, one has $\CR_2 = \alpha$ whenever $|\alpha| = 1$.
Consequently, the order parameter converges to the same boundary fixed point as $t \rightarrow \infty$. 
This implies that the distribution $f$ of the physical variable $\Theta$, induced by the linear self-attention model \eqref{eq:mf_sys}, concentrates at two a pair of fixed antipodal points. 

We also briefly comment on the assumption on the initial condition in \eqref{eq:case1_stab_initial_data}.
This assumption merely excludes the case in which the initial datum $(\gamma(0), \Phi(0))$ coincides with $(\rhoEq, \phi_{\eq})$, the interior equilibrium of the OA dynamics \eqref{eq:OA_counterpart} in this parameter regime; see Figure \ref{fig:case_1_1_a} for an illustration. By Theorem \ref{thm:case_1}, the condition \eqref{eq:case1_stab_initial_data} ensures that the corresponding solution of the OA dynamics \eqref{eq:OA_counterpart} with this initialization converges to full synchronization, namely, $|\alpha_{\OA}| \rightarrow 1$. This convergence is a key ingredient in establishing the corresponding convergence of the WS variable $\alpha$.

\subsubsection*{Case 2 Revisited.}
We revisit now the case where $A$ is the identity matrix and $V$ is a symmetric real-valued matrix, as in \eqref{eq:matrices_A_V_case_2}.
In this setting, the polar-coordinate system \eqref{eq:gamma_Phi_dyn} simplifies to
\begin{equation}\label{eq:alpha_pc_dyn_case2}
\begin{aligned}
\dot{\gamma} &= \frac{1}{4} \left( 1 - \gamma^2 \right) \left( \vDPlus \gamma + \vDMinus \cos (\Phi) + \vOPlus \sin (\Phi) \right) + \DeltaGamma\\[4pt]
\dot{\Phi} &= -\frac{1}{4} \frac{3\gamma^2 + 1}{\gamma} \left( \vDMinus \sin (\Phi) - \vOPlus \cos (\Phi) \right) + \DeltaPhi,
\end{aligned}
\end{equation}
where the explicit forms of $\DeltaGamma$ and $\DeltaPhi$ are provided in \eqref{eq:Delta_gamma_Phi_case2}.

The long-term behavior of \eqref{eq:alpha_pc_dyn_case2} is stated in the following theorem.
\begin{theorem}\label{thm:case_2stab}
Fix any $\iota > 0$, and consider the dynamical system~\eqref{eq:alpha_pc_dyn_case2} with initial data
\begin{equation}\label{eq:case2_stab_initial_data}
(\gamma(0), \Phi(0)) \in \left\{ (\gamma, \Phi) \in (0,1] \times \BBT^1: |\Phi - ( \PhiV + \pi)| \geq \iota \right\},
\end{equation}
where $\PhiV$ is defined in \eqref{eq:case2_PhiV}.
Assume that the matrix $V$ satisfies 
\begin{equation}\label{eq:case_2_stab_V_condi}
\lambdaMax(V) > 0 \, .
\end{equation}
Moreover, assume that the probability measure $\nu$ satisfies Assumption \ref{asm:Lp} and
\begin{equation}\label{eq:case_2_stab_close_to_unif}
 d_{\BL} (\nu, \Unif) < \varepsilon(\iota, p, M, V)
\end{equation}
for a sufficiently small constant $\varepsilon(\iota, p, M, V) > 0$ depending only on $\iota$, the constants $p$ and $M$ in Assumption \ref{asm:Lp}, and on the matrix $V$.
Then the solution of \eqref{eq:alpha_pc_dyn_case2} converges to $(1, \Phi_{\infty})$ as $t \rightarrow \infty$ for some $\Phi_{\infty} \in \BBT^1$.
\end{theorem}

Theorem \ref{thm:case_2stab} establishes a similar qualitative result as Theorem \ref{thm:case_2}: the WS variable $\alpha$ converges globally to a boundary fixed point on $\partial D$, provided that $A = I$ and the largest eigenvalue of the symmetric matrix $V$ is strictly positive. It follows that the order parameter $\CR_2(t)$ also converges to a boundary fixed point, and the associated distribution $f$ converges to a stationary distribution supported on two antipodal points.

As with condition \eqref{eq:case1_stab_initial_data} in Theorem \ref{thm:case_1stab}, the assumption on the initial condition in stated in \eqref{eq:case2_stab_initial_data}  excludes an exceptional equilibrium of the OA dynamics. More precisely, this assumption rules out that $\Phi(0) = \PhiV + \pi$,
which is an equilibrium of the angular dynamics in the OA system \eqref{eq:OA_counterpart} in this parameter regime. By Theorem \ref{thm:case_2}, condition \eqref{eq:case2_stab_initial_data} ensures that the corresponding solution of the OA dynamics converges to full synchronization. 
This convergence again serves as a key ingredient in establishing the corresponding convergence of the WS variable $\alpha$.

\medskip 

The overall proof strategies for Theorems \ref{thm:case_1stab} and \ref{thm:case_2stab} are similar.
We summarize the main ideas below and defer the full details to Appendix \ref{app:stability}.

\begin{proof}[Proof sketch for Theorems \ref{thm:case_1stab} and \ref{thm:case_2stab}]

In Section~\ref{sec:case_study}, we identified parameter regimes for the matrices $A$ and $V$ under which the OA dynamics \eqref{eq:OA_counterpart} converges to a full synchronization state; that is, $|\alpha_{\OA}(t)| \rightarrow 1$ for almost every initial condition. We now explain why, in these regimes, the associated $\alpha$-dynamics \eqref{eq:alpha_rewrite} also converge to full synchronization states, provided that the distribution $\nu$ is close enough to the uniform measure over $[0,2\pi)$. 
The argument consists of two key steps, which show the convergence of the modulus $\gamma(t) \rightarrow 1$.

\begin{itemize}
    \item \textbf{Step 1.}
    We first show that, there exists a threshold value $\gammaTH \in (0,1]$ so that, if $\gamma(t_0) \in [\gammaTH, 1]$, then $\gamma(t)$ converges to $1$ exponentially fast for $t \geq t_0$.
    This follows from the fact that, whenever $\gamma(t) \in [\gammaTH, 1]$, the first term in \eqref{eq:gamma_dyn} is positive and strictly dominate the perturbation term $\Delta_{\gamma}$.

    \item \textbf{Step 2.}
    It remains to show that, for almost every initial condition $\gamma(0) \notin [\gammaTH, 1]$, the dynamics \eqref{eq:gamma_dyn} enters this regime in finite time.
    This follows from the the stability of the $\alpha$-dynamics \eqref{eq:alpha_rewrite} with respect to the perturbation of the OA dynamics \eqref{eq:OA_counterpart}, as established in Proposition~\ref{prop:structural_stability}, together with the asymptotic convergence of the OA dynamics to the fully synchronized state.
    Indeed, since $|\alpha_{\OA} (t)| \rightarrow 1$, then there exists a time $t_0$ so that $|\alpha_{\OA} (t_0)| > \gammaTH$.
    Proposition~\ref{prop:structural_stability} then implies that, if $\nu$ is sufficiently close to the uniform measure $\Unif$, then
    $$\sup_{t \in [0,t_0]} \left| \alpha(t) - \alpha_{\OA}(t) \right|$$
    is sufficiently small, and hence $\gamma(t_0) = |\alpha (t_0)| \geq \gammaTH$.
\end{itemize}

Lastly, by using the convergence of the modulus $\gamma(t)$, together with the properties of the $\Phi$-dynamics, we further show that the angular variable $\Phi(t)$ converges to a fixed point in $\BBT^1$. 
\end{proof}

We note that our proof does not rely solely on a standard Grönwall-type estimate established in Proposition~\ref{prop:structural_stability}, which by itself is insufficient to conclude the asymptotic convergence of the $\alpha$-dynamics \eqref{eq:alpha_rewrite} without stronger assumptions such as $\nu\to \Unif$ in $d_{\BL}$.
Indeed, we additionally need to exploit the structural properties of the dynamics induced by the linear self-attention model for the particular choices of the matrices $A$ and $V$, as reflected in Step 1 of the proof sketch above. This additional structure allows us to extend the Grönwall-type argument to the infinite time horizon.

We close this section by comparing Theorems \ref{thm:case_1stab} and \ref{thm:case_2stab} with their OA counterparts, Theorems \ref{thm:case_1} and \ref{thm:case_2}.
The conditions on the matrices $A$ and $V$ in Theorems \ref{thm:case_1stab} and \ref{thm:case_2stab}, which ensure the desired convergence of the $\alpha$-dynamics \eqref{eq:alpha_rewrite} in the WS system, are slightly stronger than the corresponding conditions used to establish the convergence of the OA dynamics \eqref{eq:OA_counterpart} in Theorems \ref{thm:case_1} and \ref{thm:case_2}.
This is natural since the general $\alpha$-dynamics \eqref{eq:alpha_rewrite} involves an additional perturbation term $\Delta$ that is not present in the OA dynamics.
As the proofs of Theorems \ref{thm:case_1stab} and \ref{thm:case_2stab} will show, establishing convergence of the $\alpha$-dynamics requires controlling $\Delta$ by the first term in \eqref{eq:alpha_rewrite}, and this leads to the stronger conditions on $A$ and $V$ that we impose.

Nevertheless, Theorems \ref{thm:case_1stab} and \ref{thm:case_2stab} characterize the long-term behaviors of the general WS system \eqref{eq:WS_sys_sec_4} for certain choices $A$ and $V$. In turn, they extend the analysis of the mean-field dynamics \eqref{eq:mf_sys} induced by the linear self-attention model beyond the special class of wrapped Cauchy initializations.
Specifically, they allow for initial distributions that are close to, but not necessarily exactly on, the OA manifold.

\section{Experiments}
\label{sec:experiments}
In this section, we present numerical simulations of the continuous-time self-attention dynamics
\begin{equation}
    \dot{x}_k(t) = P^\perp_{x_k(t)}\!\left(\frac{1}{n}\sum_{j=1}^n h\!\left(\beta\langle Ax_k(t), x_j(t)\rangle\right) V x_j(t)\right), \quad k = 1,\dots,n,
\end{equation}
on the unit sphere $\BBS^{d-{1}}$,
considering both linear self-attention \eqref{eq:LSA}, corresponding to $h(y) = y$, and
unnormalized softmax attention ($\operatorname{USA}$), corresponding to $h(y) = \exp(y)$.
The latter model has
been analyzed for some specific choices of the parameter matrices $A$ and $V$ in \cite{geshkovski2023mathematical,geshkovski2024dynamic,chen2025quantitative,abella2024asymptotic,abella2025consensus,altafini2026multistability}.
 
We conduct simulations in higher dimensions $d \gg 2$ under parameter regimes analogous to those analytically identified in Section~\ref{sec:case_study}.
The experimental results suggest that the theoretical characterizations of the long-time behavior of the $2$-dimensional linear self-attention model, presented in Sections~\ref{sec:th_case1}--\ref{sec:th_case4}, persist in higher-dimensional settings. 
In addition, we observe that, within the identified parameter regimes, the USA model exhibits similar qualitative behaviors such as clustering and cyclicality, while also displaying some differences from the LSA model.


Before presenting our experimental results, let us first specify the diagnostics used to quantify clustering and cyclic behaviors, as well as the general experimental setup.

\paragraph{Clustering and cyclicality diagnostics.}
We track three quantities to measure clustering behavior and one quantity to measure cyclic behavior.

For clustering,
we use the mean-squared cosine similarity
\begin{equation}
\widehat{\CR}_2(t) := \frac{1}{\binom{n}{2}} \sum_{j < k} \langle x_j(t),\, x_k(t)\rangle^2 \in [0,1]
\end{equation}
as the primary diagnostic. 
This quantity is equal to $1$ whenever all tokens collapse to a pair of antipodal points
$\{x^*, -x^*\}$, allowing the two clusters to contain different numbers of tokens; this includes the case in which all tokens collapse to a single point $x^*$.
We note that the clustering diagnostic $\widehat{\CR}_2(t)$ is directly related to the second order parameter $\CR^n_2(t)$ defined in \eqref{eq:order_param_finite}.
More precisely, in dimension $d=2$, it holds
\begin{equation*}
|\CR^n_2(t)|^2 = \widehat{\CR}_2(t) +  \frac{n-2}{n} (\widehat{\CR}_2(t) - 1).
\end{equation*}
In particular, $|\CR_2^n(t)| = 1$ if and only if $\widehat{\CR}_2(t) = 1$.

We also track the mean cosine similarity
\begin{equation}
    \widehat{\CR}_1(t) :=  \frac{1}{\binom{n}{2}} \sum_{j < k} \langle x_j(t),\, x_k(t)\rangle \in [0,1],
\end{equation}
which equals 1 if and only if all tokens collapse to a single point $x^*$. 
This diagnostic is particularly useful for distinguishing consensus formation at a single point from antipodal clustering.
In particular, while $\widehat{\CR}_2(t)$ is insensitive to the distinction between $x^*$ and $-x^*$, the quantity $\widehat{\CR}_1(t)$ detects whether the tokens have collapsed to one cluster or split into two antipodal clusters.
Clearly, $\widehat{\CR}_1(t)=1$ implies $\widehat{\CR}_2(t)=1$

In addition to tracking these two scalar clustering diagnostics, we also plot the Gram matrix $G(t)$, whose entries are give by $G_{ij}(t) = \langle x_i(t), x_j(t)\rangle$, at $t=0$ and
$t=T$.
For Gram matrix plots, we sort the tokens according to their locations on $\mathbb{S}^{d-1}$ at time $t = T$.
The formation of two antipodal clusters 
is indicated by a $2\times 2$ block structure, with positive diagonal blocks and negative off-diagonal blocks, and where the block sizes reflect the corresponding cluster sizes.

Finally, to quantify cyclic behavior, we track the coordinate-wise mean of the token representation vectors.
For $l = 1,\dots, d$, writing $x_j = (x_j^{(1)}, x_j^{(2)}, \dots, x_j^{(d)} )^{\top}$, we define
\begin{equation}
m_l (t) := \frac{1}{n} \sum_{j=1}^n x_j^{(l)} (t) \, .
\end{equation}
If the clustering diagnostics indicate consensus while the trajectory $m(t) = (m_1(t), \dots, m_d(t))$ fails to converge, this suggests that the system has formed clusters that continue to rotate on $\mathbb{S}^{d-1}$.

\begin{remark}
To avoid presenting an overwhelming number of figures, we omit certain diagnostic plots when they do not provide substantial additional information or insight.
In particular, we will not present the Gram matrix for the $\operatorname{USA}$ model in regimes where the dynamics cluster to a single point, since in such cases the Gram matrix provides little additional information beyond the scalar diagnostics $\widehat{\CR}_1$ and $\widehat{\CR}_2$.
Moreover, in cases where the dynamics both cluster and converge, we present only the clustering diagnostics and omit the coordinate-wise mean plot, describing the convergence behavior in words instead.
\end{remark}

\paragraph{Experimental setup.}
In all experiments, we take the embedding dimension to be $d=100$ and use $n=200$ tokens. 
The choice $n > d$ is motivated not only by the mean-field nature of our analysis, but also by the fact that large-token regimes are common in modern large language models \cite{gpt,gemini2023,llama,liu2024deepseek}. 
We have observed similar phenomena in other embedding dimensions, but focus on this parameter setting for conciseness.
The initial tokens $\{x_k(0)\}_{k=1}^n \subset \mathbb{S}^{d-1}$
are drawn independently from the uniform distribution on the unit sphere. 


\subsection{Case 1: \texorpdfstring{$V = I$, $A$}{VA} Arbitrary}
\label{sec:exp_case1}
Our theoretical findings from Theorem~\ref{thm:case_1} characterize parameter regimes of the matrix $A$ associated with clustering and non-clustering behaviors for the $2$-dimensional LSA model.  
In particular, if $\tr{A} > 0$ or $\det(A) \leq 0$, then the LSA dynamics converge to a pair of antipodal points.  
On the other hand, if the symmetric part of $A$ satisfies $(A + A^\top)/2 \prec 0$, then the dynamics do not exhibit clustering.  




 
The corresponding numerical results are presented in Figures~\ref{fig:case1_gram_trA_pos_d100}--\ref{fig:case1_ps_d100}.
We observe that the qualitative behavior predicted by Theorem \ref{thm:case_1} for the $2$-dimensional LSA model persists in higher dimensions.
Specifically, when $A$ satisfies \eqref{eq:case1_condition_E1}, the LSA dynamics converge to a pair of antipodal points, as indicated by the clustering diagnostic $\widehat{\CR}_2(t)$ converging to $1$ while $\widehat{\CR}_1(t)$ remains close to $0$; see Figures~\ref{fig:case1_gram_trA_pos_d100} and \ref{fig:case1_gram_trA_detA_neg_d100}.
In contrast, when $(A + A^{\top})/2 \prec 0$, the LSA dynamics do not synchronize, as reflected by $\widehat{\CR}_1(t) = \widehat{\CR}_2(t) = 0$ in Figure~\ref{fig:case1_ps_d100}.
\begin{figure}[!htb]
\centering
\newlength{\subfigimgheight}
\setlength{\subfigimgheight}{0.2\textheight} 

\begin{subfigure}[t]{.48\textwidth}
  \centering
  \begin{minipage}[b][\subfigimgheight][b]{\linewidth}
    \centering    \includegraphics[trim={0cm 0cm 0cm 1cm}, clip,width=\linewidth,height=\subfigimgheight,keepaspectratio]{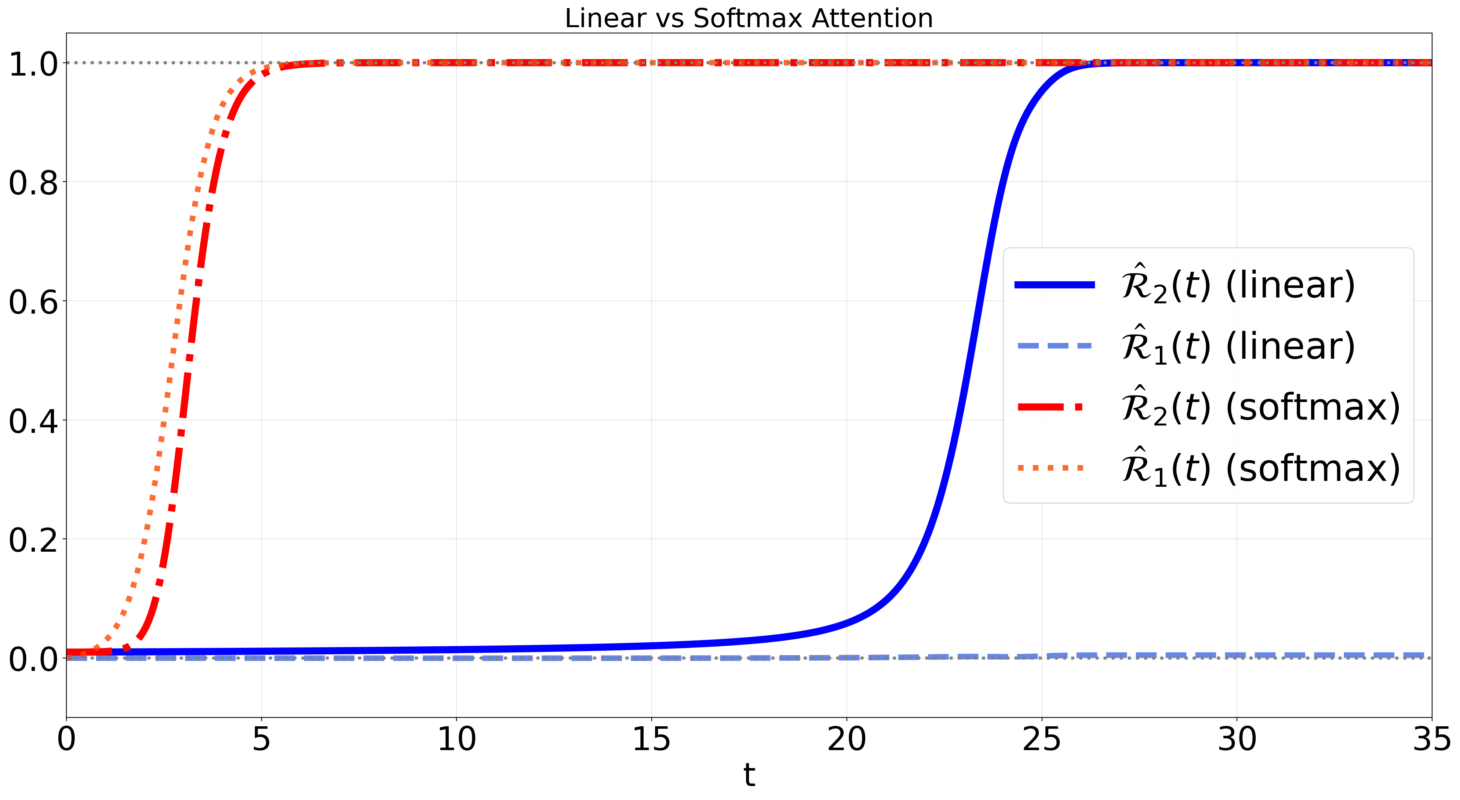}
  \end{minipage}
  \caption{Evolution of the clustering diagnostics $\widehat{\CR}_{1}(t)$ and $\widehat{\CR}_{2}(t)$ for the $\operatorname{LSA}$ (blue) and $\operatorname{USA}$ (red) models.}
\end{subfigure}
\hfill
\begin{subfigure}[t]{.48\textwidth}
  \centering
  \begin{minipage}[b][\subfigimgheight][b]{\linewidth}
    \centering
    \includegraphics[width=\linewidth,height=\subfigimgheight,keepaspectratio]{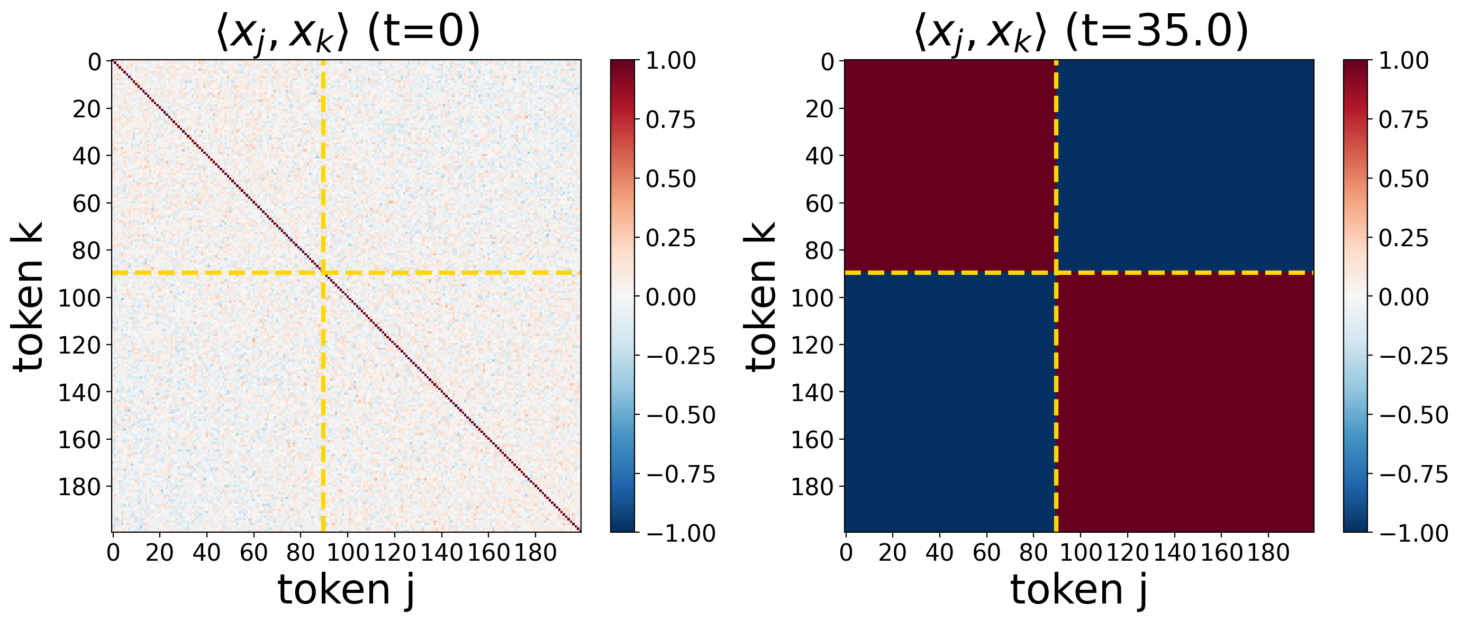}
  \end{minipage}
  \caption{The Gram matrices $G(t)$ for the LSA model at $t=0$ and $t=35$.}
\end{subfigure}%

\caption{
\textbf{Case $\mathbf 1$, parameter regime:} $\tr{A} > 0$ and $V=I$.
For both the LSA and USA models, 
the mean-squared cosine similarity $\widehat{\CR}_2(t)$ converges to $1$. 
However, the mean cosine similarity $\widehat{\CR}_1(t)$ converges to $1$ for the $\operatorname{USA}$ model, shown by the red dotted line, but to $0$ for the LSA model, shown by the blue dashed line. 
This indicates that the USA dynamics clusters at a single point, whereas the LSA dynamics converges to two antipodal clusters.  
The Gram matrices $G$ for the LSA model, shown on the right, suggest that the tokens approximately evenly split between the two clusters.
}
\label{fig:case1_gram_trA_pos_d100}
\end{figure}
\begin{figure}[!htb]
\centering
\newlength{\subfigimgheightA}
\setlength{\subfigimgheightA}{0.2\textheight} 
\begin{subfigure}[t]{.48\textwidth}
  \centering
  \begin{minipage}[b][\subfigimgheightA][b]{\linewidth}
    \centering
    \includegraphics[trim={0cm 0cm 0cm 1cm}, clip,width=\linewidth,height=\subfigimgheightA,keepaspectratio]{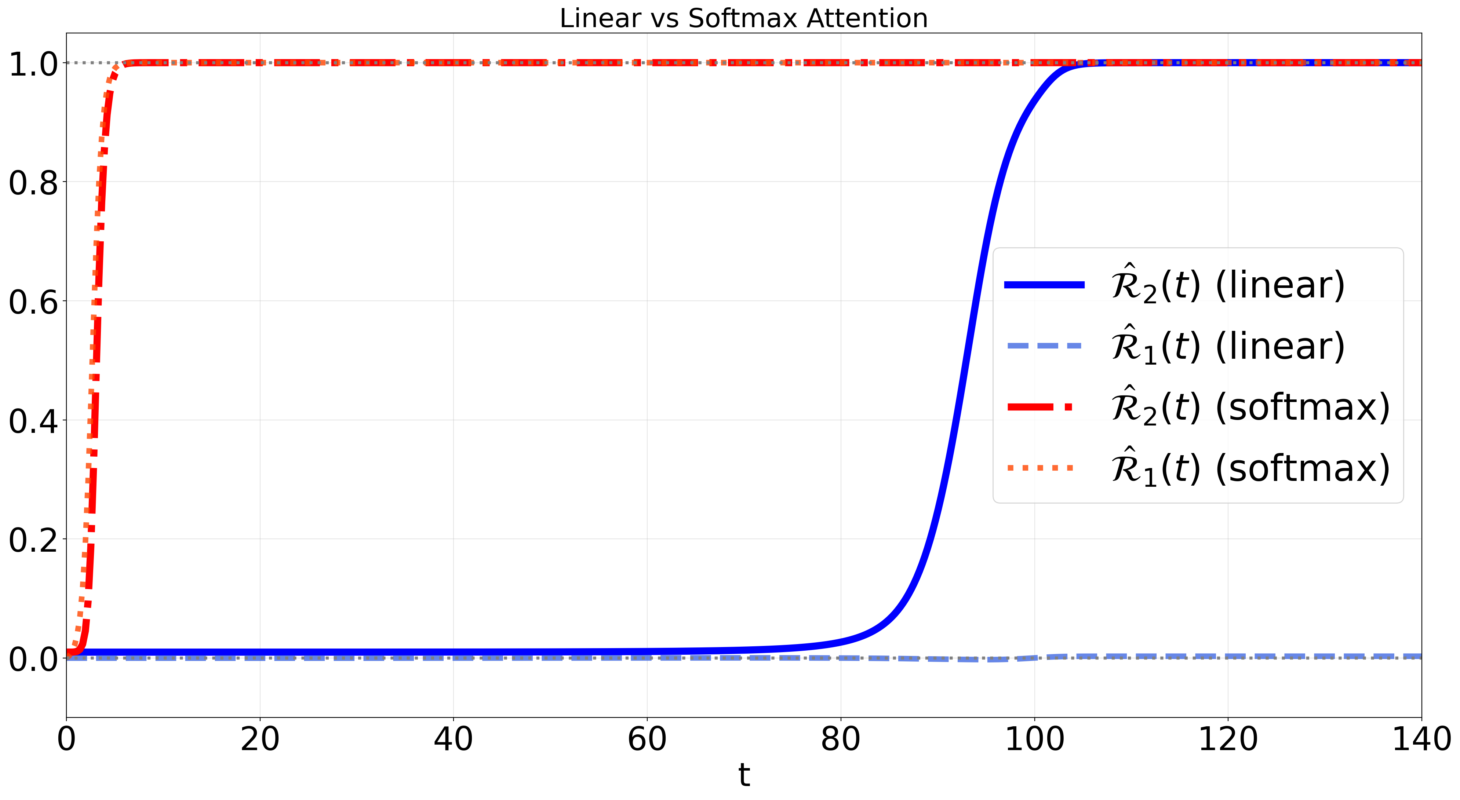}
  \end{minipage}
  \caption{Evolution of the clustering diagnostics $\widehat{\CR}_{1}(t)$ and $\widehat{\CR}_{2}(t)$ for the $\operatorname{LSA}$ (blue) and $\operatorname{USA}$ (red)  models.}
\end{subfigure}%
\hfill
\begin{subfigure}[t]{.48\textwidth}
  \centering
  \begin{minipage}[b][\subfigimgheightA][b]{\linewidth}
    \centering
    \includegraphics[width=\linewidth,height=\subfigimgheightA,keepaspectratio]{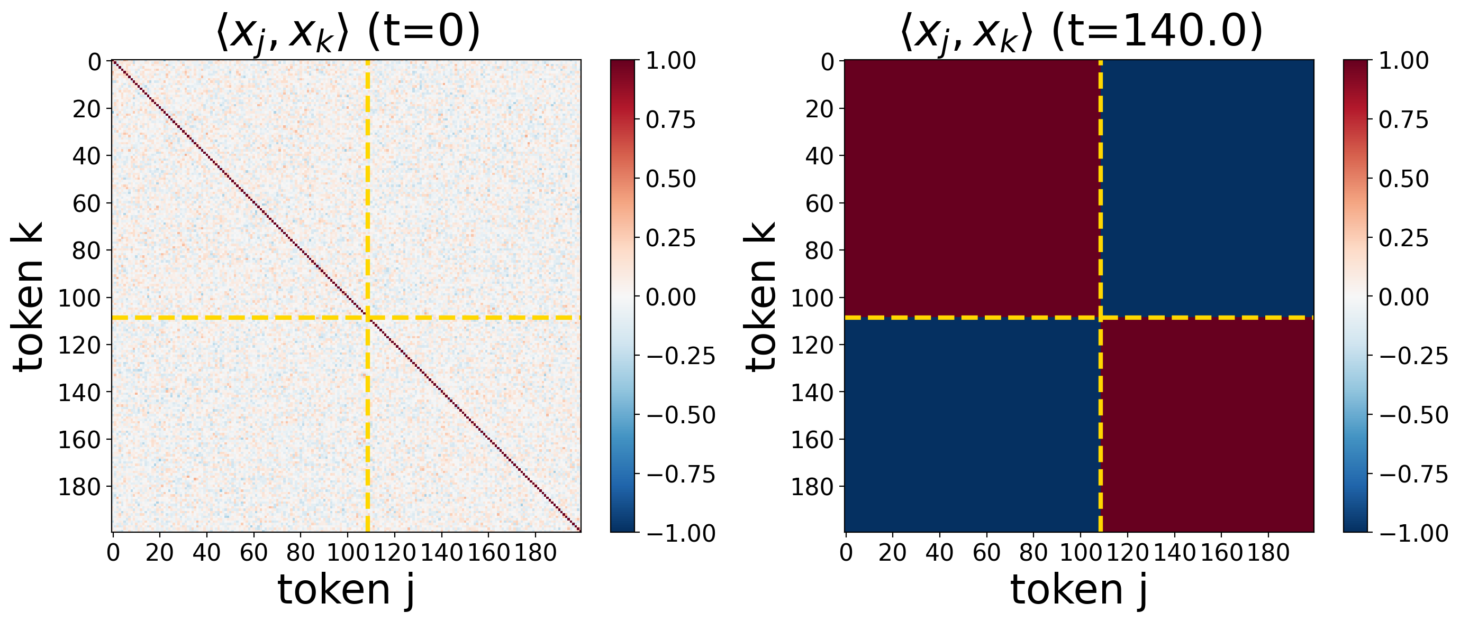}
  \end{minipage}
  \caption{The Gram matrices $G(t)$ for the LSA model at $t=0$ and $t=130$.}
\end{subfigure}%
\caption{
\textbf{Case $\mathbf 1$, parameter regime:} $\mathrm{tr}(A) \leq 0$, $\det(A) \leq 0$ and $V=I$.
The same qualitative behavior as Figure~\ref{fig:case1_gram_trA_pos_d100} is observed: the USA dynamics clusters at a single point, while the LSA dynamics converges to two approximately balanced antipodal clusters.
}
\label{fig:case1_gram_trA_detA_neg_d100}
\end{figure}
\begin{figure}[!htb]
\centering

\newlength{\subfigimgheightB}
\setlength{\subfigimgheightB}{0.2\textheight} 

\begin{subfigure}[t]{.48\textwidth}
  \centering
  \begin{minipage}[b][\subfigimgheightB][b]{\linewidth}
    \centering
    \includegraphics[trim={0cm 0cm 0cm 1cm}, clip,width=\linewidth,height=\subfigimgheightB,keepaspectratio]{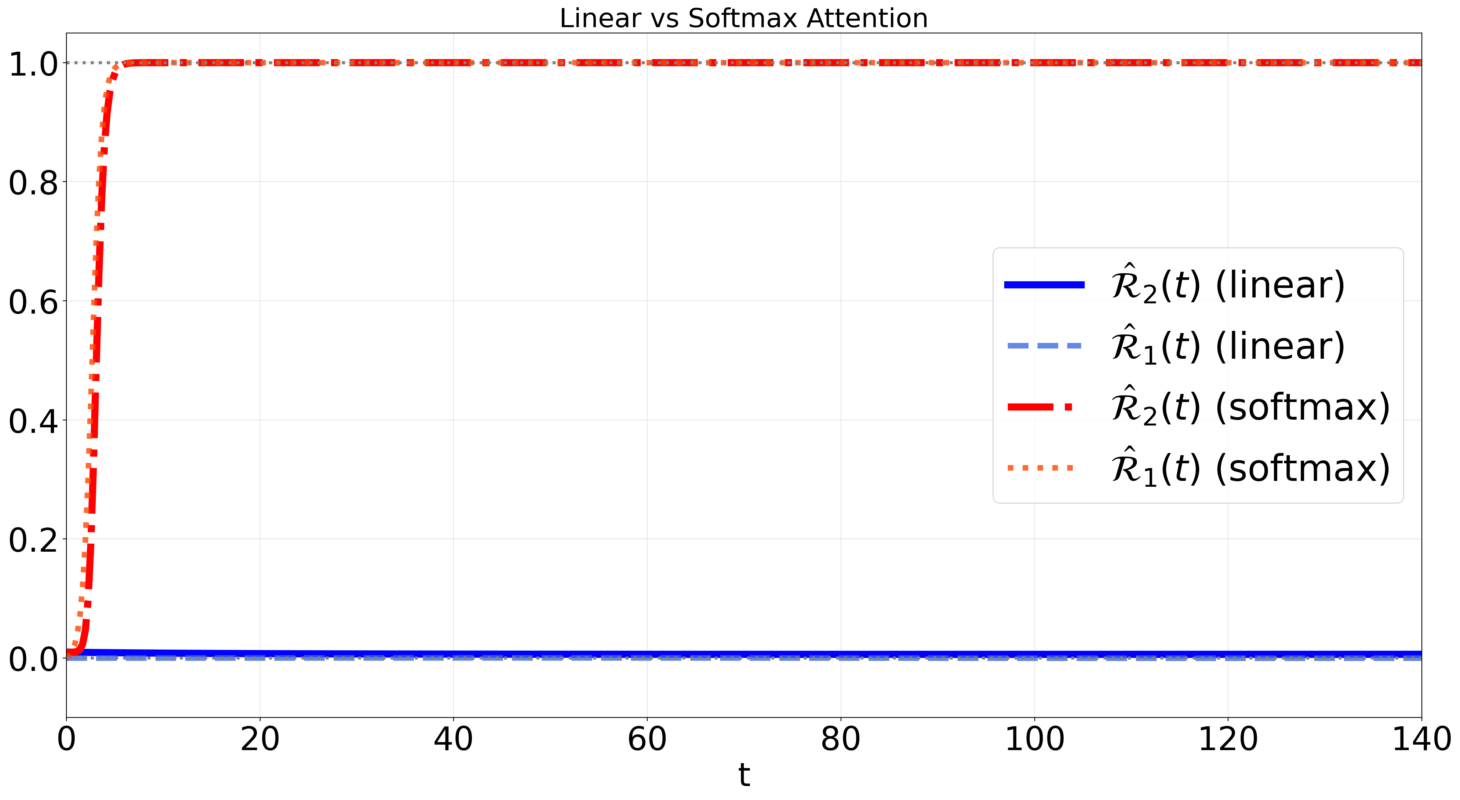}
  \end{minipage}
  \caption{Evolution of the clustering diagnostics $\widehat{\CR}_{1}(t)$ and $\widehat{\CR}_{2}(t)$ for the $\operatorname{LSA}$ (blue) and $\operatorname{USA}$ (red)  models.}
\end{subfigure}%
\hfill
\begin{subfigure}[t]{.48\textwidth}
  \centering
  \begin{minipage}[b][\subfigimgheightB][b]{\linewidth}
    \centering
    \includegraphics[width=\linewidth,height=\subfigimgheightB,keepaspectratio]{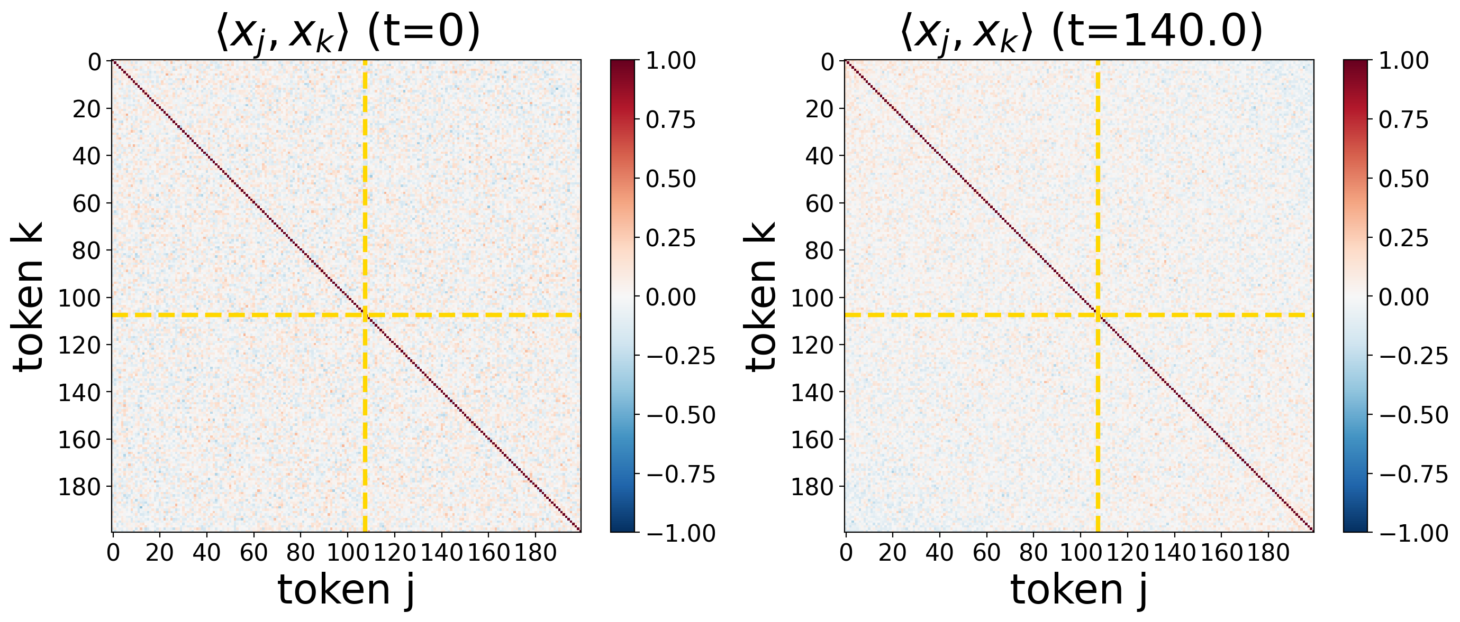}
  \end{minipage}
  \caption{The Gram matrices $G(t)$ for the LSA model at $t=0$ and $t=120$.}
\end{subfigure}%
\caption{
\textbf{Case $\mathbf 1$, parameter regime:} $(A+A^\top)/2 \prec 0$ and $V=I$.
Both clustering diagnostics $\widehat{\CR}_1(t)$ and $\widehat{\CR}_2(t)$ converge to almost $0$ for the LSA model and converge to $1$ for the USA model.
This implies that the LSA model does not synchronize whereas the USA model clusters at a single point.
}
\label{fig:case1_ps_d100}
\end{figure}%
For the USA model, the dynamics cluster to a single point in all three parameter regimes considered here, as indicated by the convergence of both $\widehat{\CR}_1(t)$ and $\widehat{\CR}_2(t)$ to $1$ in Figures~\ref{fig:case1_gram_trA_pos_d100}--\ref{fig:case1_ps_d100}.
Thus, while the USA dynamics exhibits clustering in regimes where LSA forms antipodal clusters, it also synchronizes in the regime where LSA does not cluster.
This highlights a qualitative difference between the two models: the USA dynamics tends to produce single-cluster synchronization, which is consistent with the existing theoretical findings in \cite{geshkovski2023mathematical,burger2025analysis}, whereas the LSA dynamics can either form antipodal clusters or fail to cluster, depending on the spectral structure of the matrix~$A$.

We also observe quantitative differences in convergence speed across models and parameter regimes.
Compared with the regime $\mathrm{tr}(A) > 0$ shown in Figure~\ref{fig:case1_gram_trA_pos_d100}, the convergence of the LSA model appears to be slower when $\tr{A} \leq 0$ and $\det(A) \leq 0$ as shown in Figure~\ref{fig:case1_gram_trA_detA_neg_d100}.
In contrast, the USA model exhibits more uniform convergence speeds across these regimes and generally converges faster than the LSA model.


\subsection{Case 2: \texorpdfstring{$A = I$, $V$}{AV} Symmetric}
\label{sec:exp_case2}

Our theoretical findings from Theorem~\ref{thm:case_2} characterize the role of the dominant eigenvalue $\lambdaMax(V)$ in the $2$-dimensional LSA model. 
In particular, if $\lambdaMax(V) \geq 0$, then the LSA dynamics converge to a pair of antipodal points.
If $\lambdaMax(V) < 0$, then the dynamics converge without forming clusters.

 
The corresponding numerical results are presented in Figures~\ref{fig:case2_eigpos_d100}--\ref{fig:case2_eigneg_d100}.
For the LSA model, the high-dimensional simulations agree with the $2$-dimensional theory: the sign of $\lambdaMax(V)$ determines whether clustering occurs. 
When $\lambda_{\max}(V) \geq 0$, the tokens converge to two antipodal points, as shown in Figures~\ref{fig:case2_eigpos_d100} and \ref{fig:case2_eig_d100}. 
When $\lambda_{\max}(V) < 0$, the tokens converge without cluster formation, as indicated by $\widehat{\CR}_1(t) = \widehat{\CR}_2(t) = 0$ together with the convergence of the coordinate-wise empirical mean of the tokens $m_l(t)$; see Figure~\ref{fig:case2_eigneg_d100}.  

For the USA model, the behavior also depends on $\lambda_{\max}(V)$, but the resulting cluster configuration can differ from that of the LSA model. 
When $\lambda_{\max}(V) > 0$, the USA dynamics converge to a single point, in contrast to the antipodal clustering of the LSA dynamics.
At the boundary regime $\lambda_{\max}(V) = 0$, Figure~\ref{fig:case2_eig_d100} shows that the USA model can also exhibit antipodal clustering similarly to the LSA model.
Lastly, when $\lambdaMax(V) < 0$, both models converge without forming clusters. 


For both the USA and LSA models, the boundary regime $\lambdaMax(V) = 0$ also displays a notable quantitative feature: clustering occurs much slower than in the positive-eigenvalue regime.
In particular, synchronization is rapid when $\lambdaMax(V) > 0$, whereas the boundary case requires orders of magnitude longer to approach the clustering state.


\begin{figure}[!htb]
\centering

\newlength{\subfigimgheightC}
\setlength{\subfigimgheightC}{0.2\textheight} 

\begin{subfigure}[t]{.48\textwidth}
  \centering
  \begin{minipage}[b][\subfigimgheightC][b]{\linewidth}
    \centering
    \includegraphics[trim={0cm 0cm 0cm 1cm}, clip,width=\linewidth,height=\subfigimgheightC,keepaspectratio]{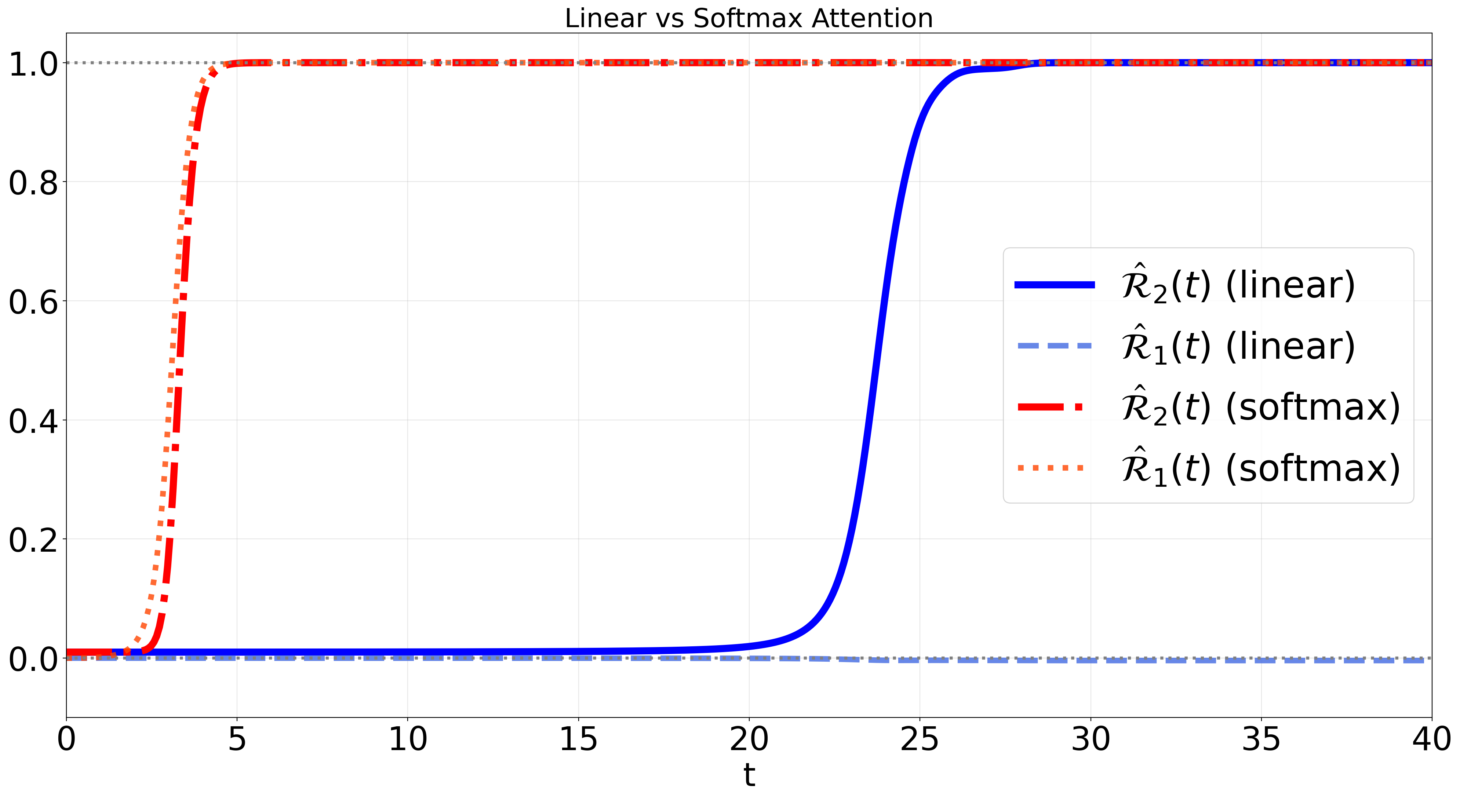}
  \end{minipage}
  \caption{Evolution of the clustering diagnostics $\widehat{\CR}_{1}(t)$ and $\widehat{\CR}_{2}(t)$ for the $\operatorname{LSA}$ (blue) and $\operatorname{USA}$ (red)  models.}
\end{subfigure}%
\hfill
\begin{subfigure}[t]{.48\textwidth}
  \centering
  \begin{minipage}[b][\subfigimgheightC][b]{\linewidth}
    \centering
    \includegraphics[width=\linewidth,height=\subfigimgheightC,keepaspectratio]{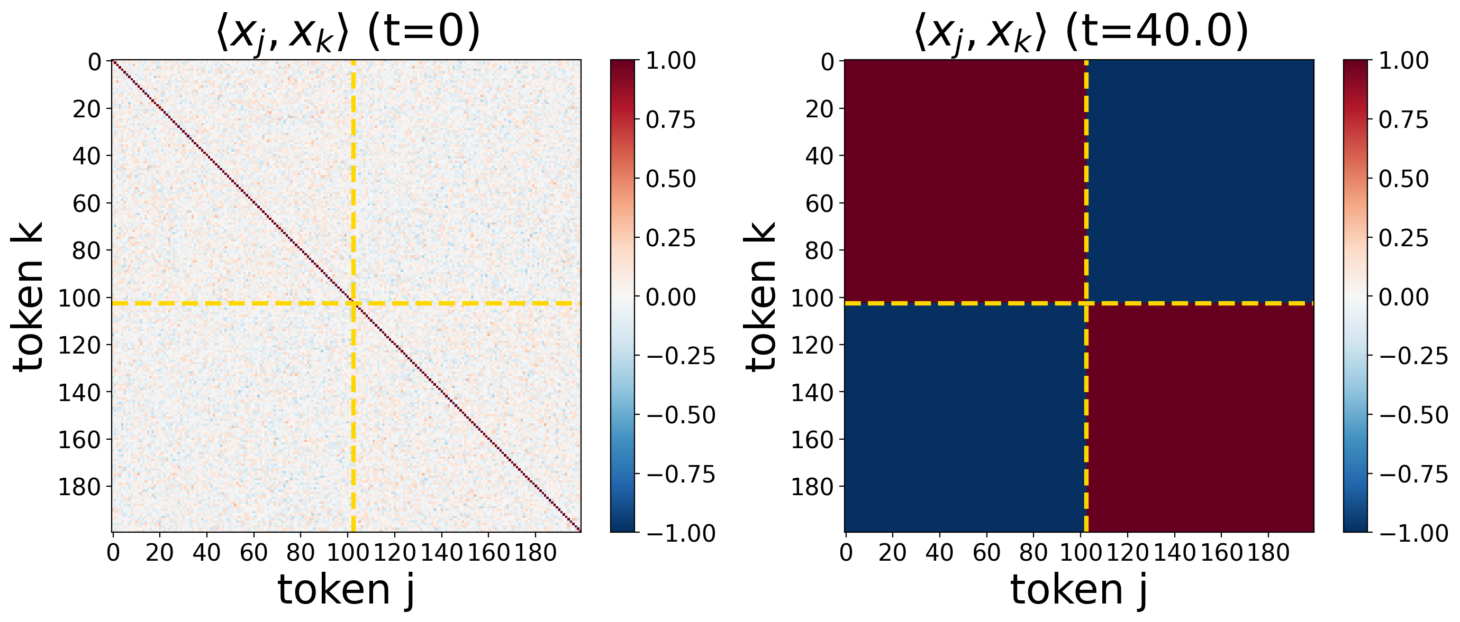}
  \end{minipage}
  \caption{The Gram matrices $G(t)$ for the LSA model at $t=0$ and $t=40$.}
\end{subfigure}
\caption{
\textbf{Case $\mathbf 2$, parameter regime:} $A=I$ and $\lambda_{\max}(V) > 0$.
For both the LSA and USA models, 
the mean-squared cosine similarity $\widehat{\CR}_2(t)$ converges to $1$. 
However, the mean cosine similarity $\widehat{\CR}_1(t)$ converges to $1$ for the $\operatorname{USA}$ model, shown by the red dotted line, but to $0$ for the LSA model, shown by the blue dashed line. 
This indicates that the USA dynamics clusters at a single point, whereas the LSA dynamics converges to two antipodal clusters.  
The Gram matrices $G$ for the LSA model, shown on the right, suggest that the tokens approximately evenly split between the two clusters.
}
\label{fig:case2_eigpos_d100}
\end{figure}


\begin{figure}[!htb]
\centering

\newlength{\subfigimgheightD}
\setlength{\subfigimgheightD}{0.2\textheight} 

\begin{subfigure}[t]{.48\textwidth}
  \centering
  \begin{minipage}[b][\subfigimgheightD][b]{\linewidth}
    \centering
    \includegraphics[trim={0cm 0cm 0cm 1cm}, clip,width=\linewidth,height=\subfigimgheightD,keepaspectratio]{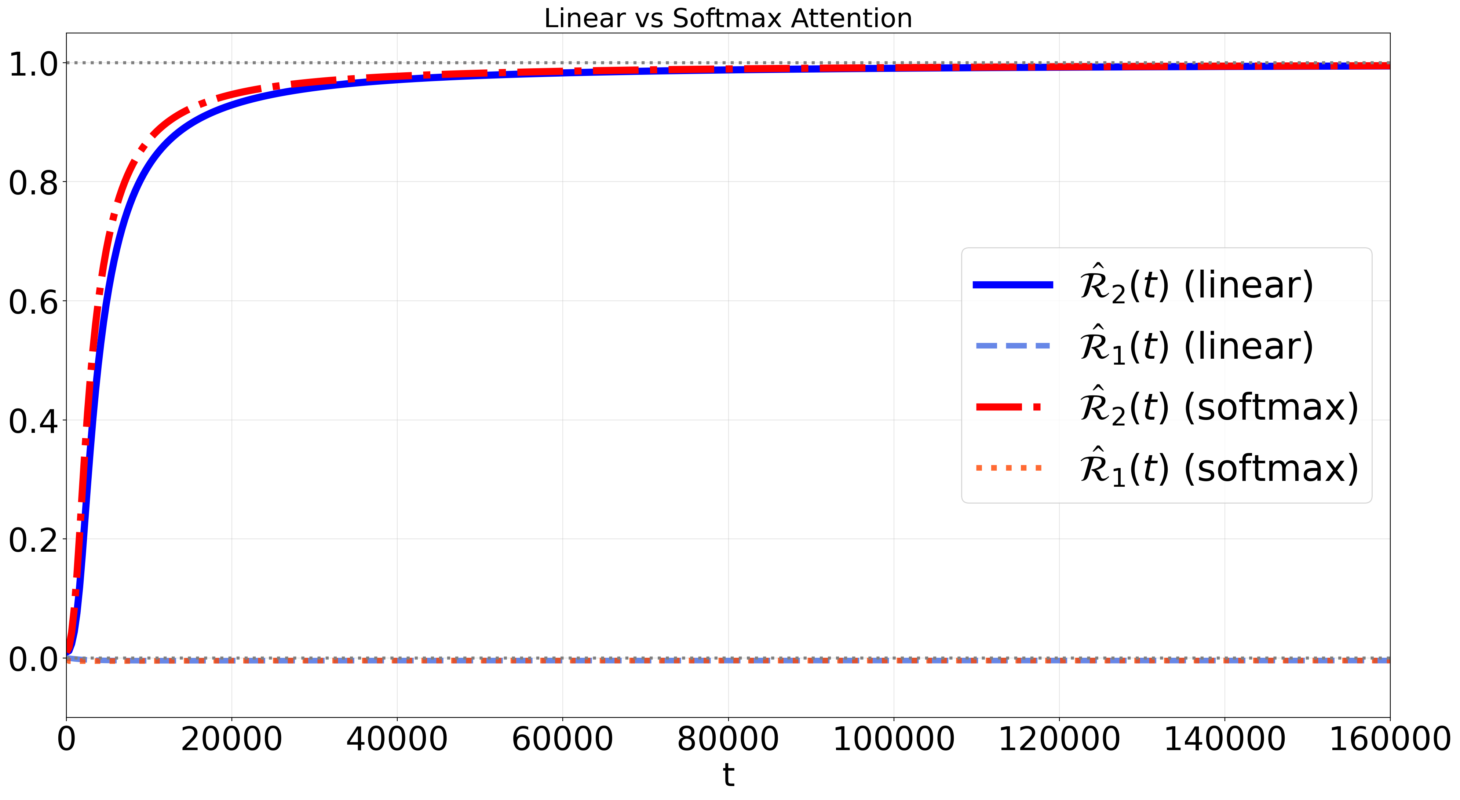}
  \end{minipage}
  \caption{Evolution of the clustering diagnostics $\widehat{\CR}_{1}(t)$ and $\widehat{\CR}_{2}(t)$ for the $\operatorname{LSA}$ (blue) and $\operatorname{USA}$ (red) models.}
\end{subfigure}%
\hfill
\begin{subfigure}[t]{.48\textwidth}
  \centering
  \begin{minipage}[b][\subfigimgheightD][b]{\linewidth}
    \centering
    \includegraphics[width=\linewidth,height=\subfigimgheightD,keepaspectratio]{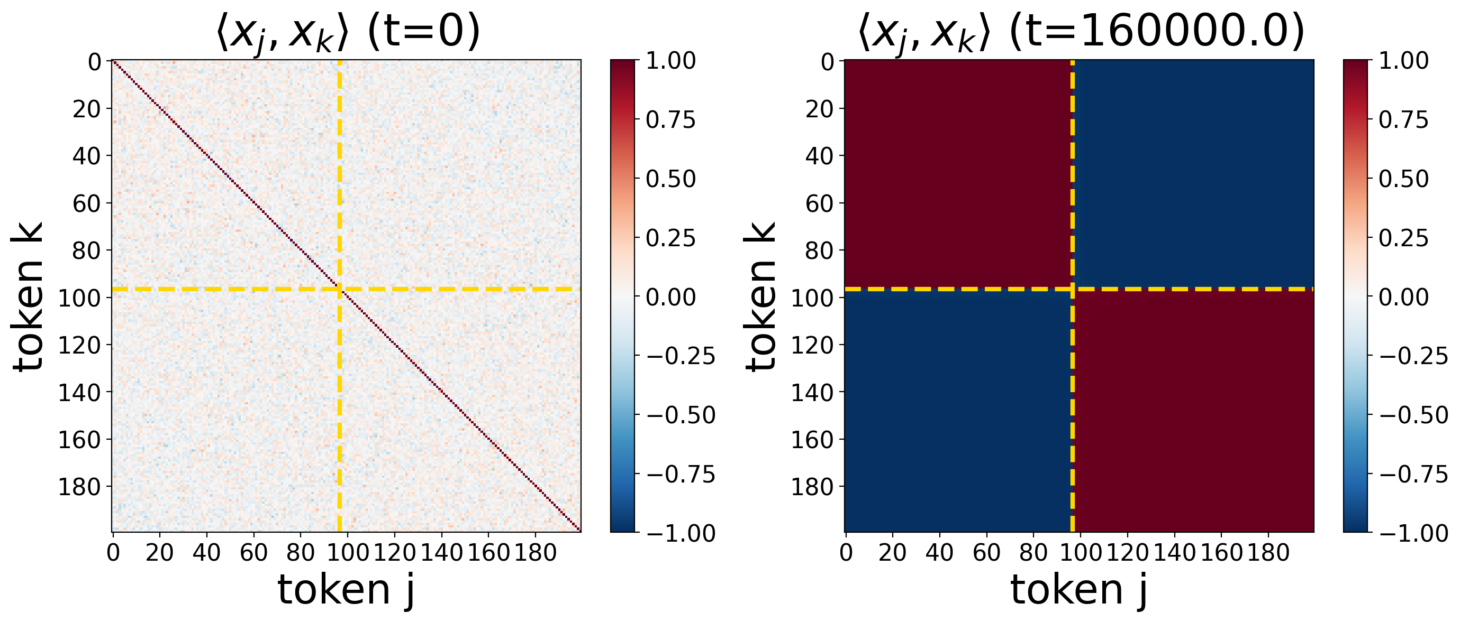}
  \end{minipage}
  \caption{The Gram matrices $G(t)$ for the LSA model at $t=0$ and $t=150{,}000$.}
\end{subfigure}%

\caption{
\textbf{Case $\mathbf 2$, parameter regime:} $A=I$ and $\lambda_{\max}(V) = 0$.
For both the LSA and USA models, the mean-squared cosine similarity $\widehat{\CR}_2(t)$ converges to $1$, while the mean cosine similarity $\widehat{\CR}_1(t)$ remains close to $0$.
This indicates that both dynamics converge to antipodal clustering configurations. 
Compared with the regime $\lambda_{\max}(V) > 0$, the convergence in this boundary regime is significantly slower for both models.
}

\label{fig:case2_eig_d100}
\end{figure}


\begin{figure}[!htb]
\centering

\newlength{\subfigimgheightE}
\setlength{\subfigimgheightE}{0.2\textheight} 

\begin{subfigure}[t]{.48\textwidth}
  \centering
  \begin{minipage}[b][\subfigimgheightE][b]{\linewidth}
    \centering
    \includegraphics[trim={0cm 0cm 0cm 1cm}, clip,width=\linewidth,height=\subfigimgheightE,keepaspectratio]{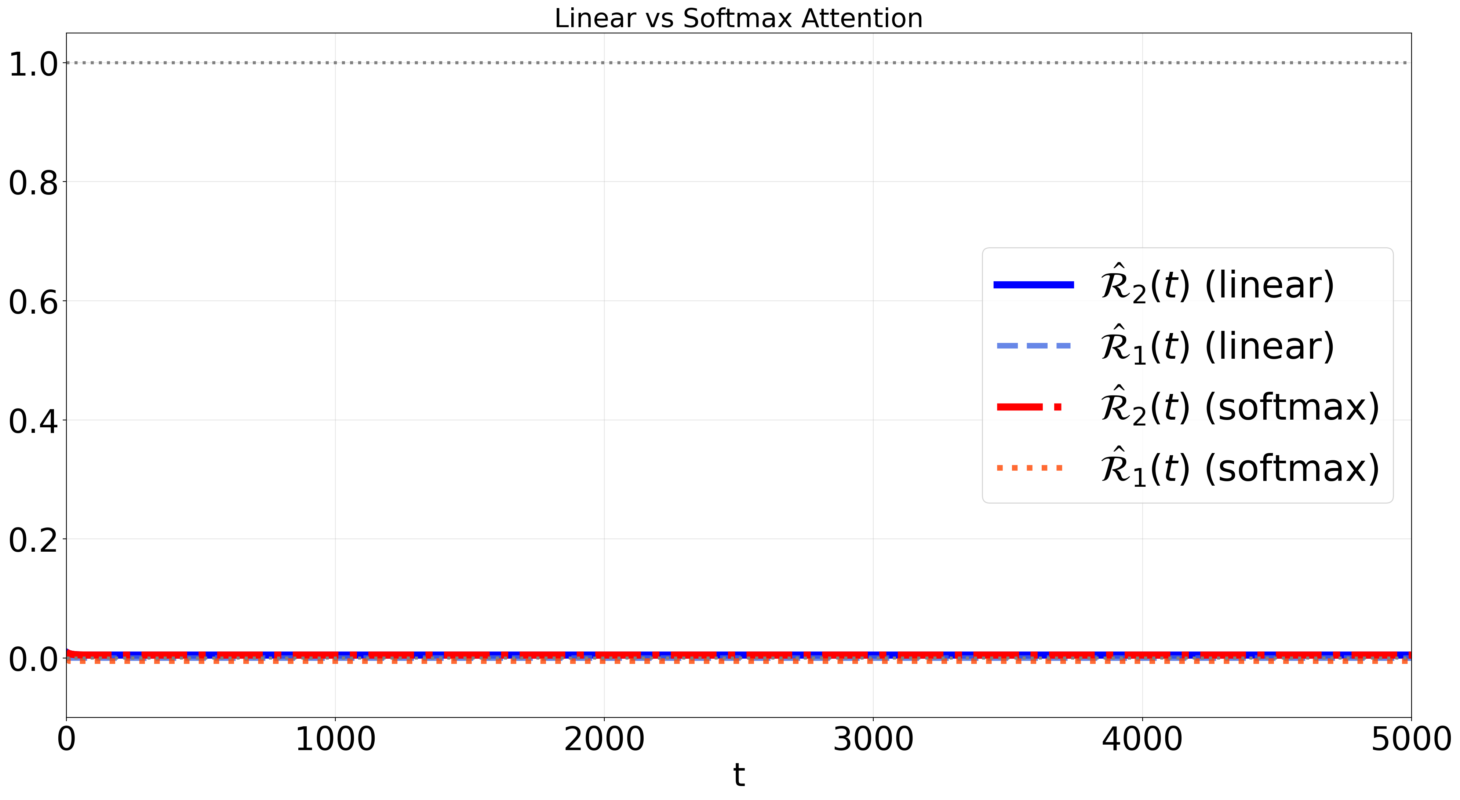}
  \end{minipage}
  \caption{Evolution of the clustering diagnostics $\widehat{\CR}_{1}(t)$ and $\widehat{\CR}_{2}(t)$ for the $\operatorname{LSA}$ (blue) and $\operatorname{USA}$ (red)  models.}
  \label{subfig:case_2_clustering_metrics}
\end{subfigure}
\hfill
\begin{subfigure}[t]{.48\textwidth}
  \centering
  \begin{minipage}[b][\subfigimgheightE][b]{\linewidth}
    \centering
    \includegraphics[width=\linewidth,height=\subfigimgheightE,keepaspectratio]{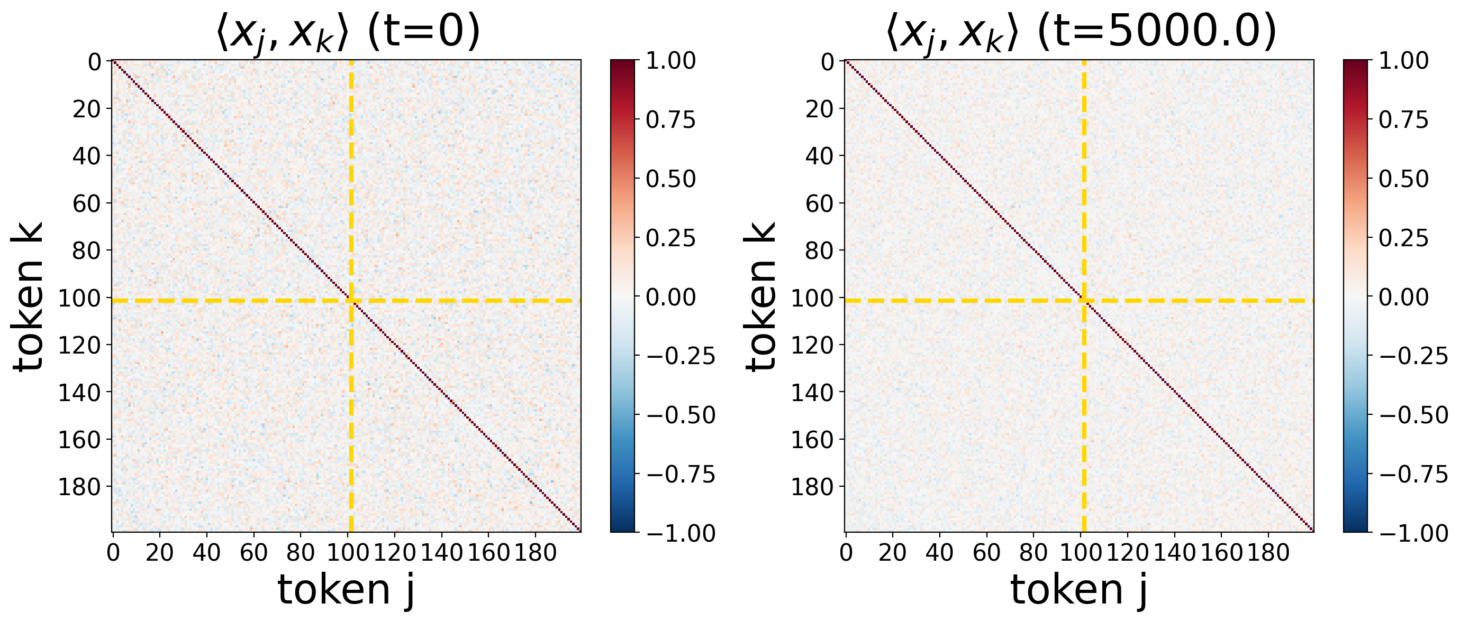}
  \end{minipage}
  \caption{The Gram matrices $G(t)$ for LSA model at $t=0$ and $t=5000$.}
\end{subfigure}
\\[1.0em]
\begin{subfigure}[t]{.48\textwidth}
  \centering
  \begin{minipage}[b][\subfigimgheightE][b]{\linewidth}
    \centering
    \includegraphics[trim={1.2cm 0cm 0cm 1cm}, clip,width=\linewidth,height=\subfigimgheightE,keepaspectratio]{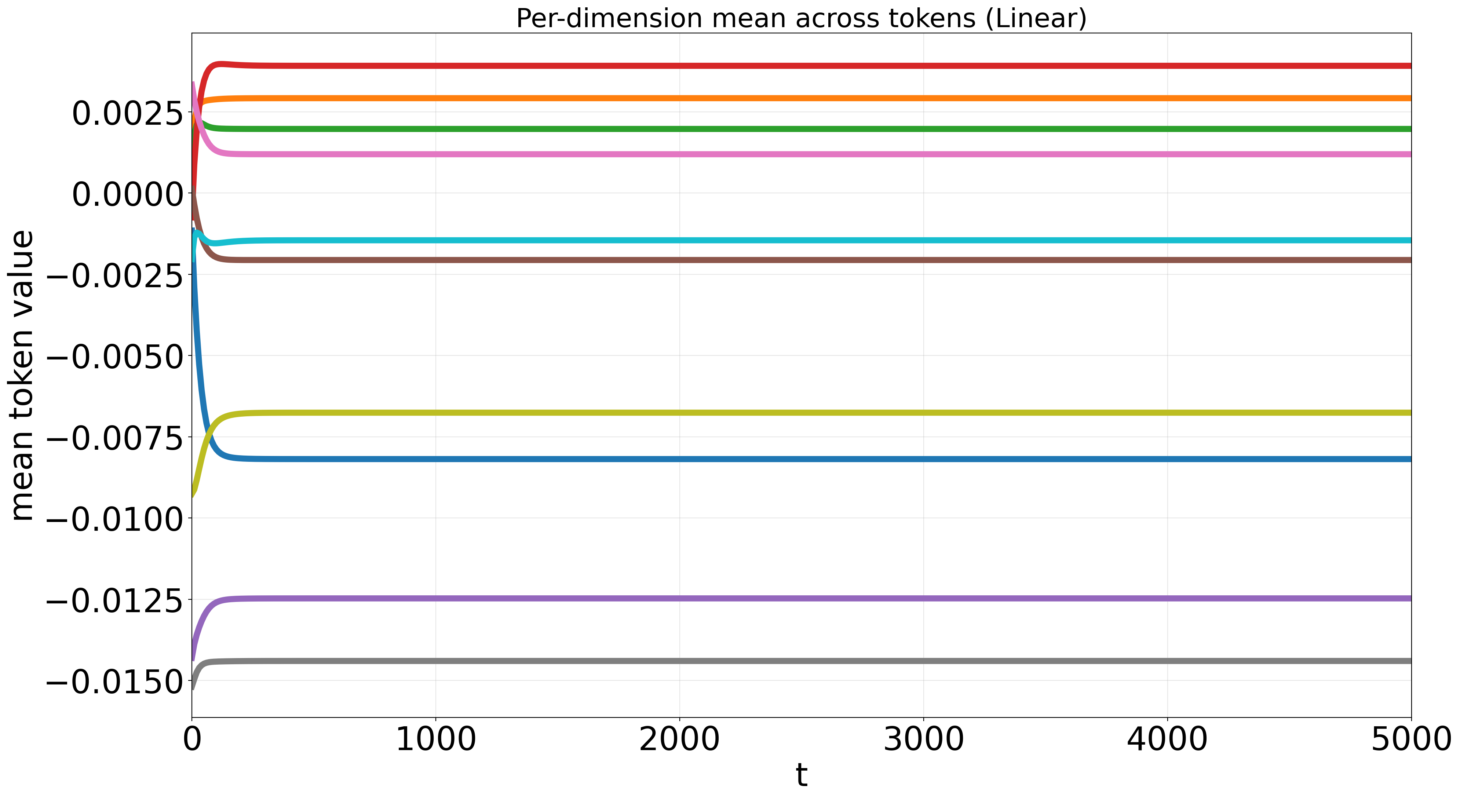}
  \end{minipage}
  \caption{Coordinate-wise mean $m_l(t)$ for the LSA model.}
  \label{subfig:case_2_LSA_per_dim_mean}
\end{subfigure}
\hfill
\begin{subfigure}[t]{.48\textwidth}
  \centering
  \begin{minipage}[b][\subfigimgheightE][b]{\linewidth}
    \centering
    \includegraphics[trim={1.2cm 0cm 0cm 1cm}, clip,width=\linewidth,height=\subfigimgheightE,keepaspectratio]{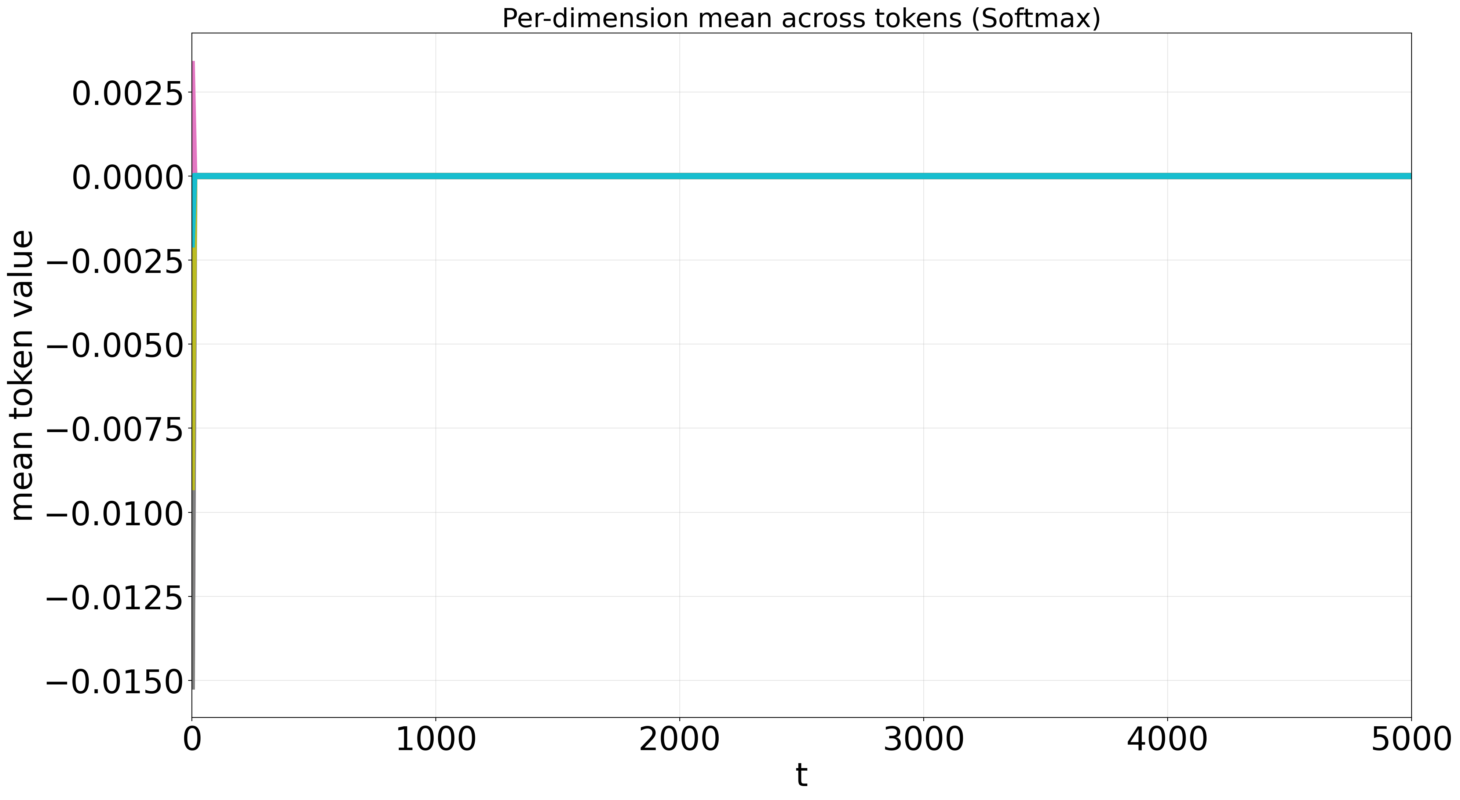}
  \end{minipage}
  \caption{Coordinate-wise mean $m_l(t)$ for the USA model.}
  \label{subfig:case_2_USA_per_dim_mean}
\end{subfigure}
%
\caption{
\textbf{Case $\mathbf 2$, parameter regime:} $A=I$ and $\lambda_{\max}(V) < 0$.
For both the LSA and USA models, the clustering diagnostics remain at $\widehat{\CR}_1(t) = \widehat{\CR}_2(t) = 0$, while the per-dimensional means $m_l(t)$ converge across dimensions.
These observations indicate that the dynamics converge, but without cluster formation.
}
\label{fig:case2_eigneg_d100}
\end{figure}

\subsection{Case 3: Non-commuting \texorpdfstring{$A$}{A} and \texorpdfstring{$V$}{V}}
\label{sec:exp_case3}
Our theoretical findings from Theorem \ref{thm:case_3} identify parameter regimes for the matrices $A$ and $V$ under which the associated $2$-dimensional LSA dynamics form antipodal clusters without converging.
As a representative example from Remark \ref{rem:case3prime}, we take $A = I$ and $V = \left(\begin{smallmatrix} v_{11} & v_{12} \\ -v_{12} & v_{11} \end{smallmatrix}\right)$ with $v_{11} > 0$ and $v_{12} \neq 0$.
In this regime, the LSA dynamics forms two antipodal clusters, but the cluster locations continue to oscillate periodically.

To extend this construction to higher dimensions, we set $A=I$ and $V$ to be a block-diagonal matrix whose blocks are $2\times 2$ blocks of the form $\left(\begin{smallmatrix} v_{11} & v_{12} \\ -v_{12} & v_{11} \end{smallmatrix}\right)$.

The corresponding numerical results are presented in Figure~\ref{fig:case3_d100}.
For the LSA model, the observed behavior of the higher-dimensional dynamics with $d > 2$ agrees with the $2$-dimensional theory from Theorem \ref{thm:case_3}.
Specifically, the tokens form two antipodal clusters, as shown by $\widehat{\CR}_2(t)$ converging to $1$ while $\widehat{\CR}_1(t)$ remains close to $0$ in Figure~\ref{subfig:case_3_clustering_metrics}.
However, the token trajectories do not converge to fixed locations; instead, they oscillate periodically, as indicated by the oscillatory behavior of the per-dimensional means $m_l(t)$ in Figure~\ref{subfig:case_3_LSA_per_dim_mean}.

For the USA model, the dynamics exhibit a similar combination of clustering and non-convergence.
The main difference lies in the cluster configuration: the USA dynamics collapses to a single cluster, whereas the LSA dynamics form two antipodal clusters.
Nevertheless, in both models, the oscillatory behavior of the per-dimensional means indicates that the token dynamics remain time-periodic rather than converging to a stationary configuration.

\begin{figure}[!htb]
\centering

\newlength{\caseThreeTopGraphicHeight}
\setlength{\caseThreeTopGraphicHeight}{0.2\textheight}

\newlength{\caseThreeBottomGraphicHeight}
\setlength{\caseThreeBottomGraphicHeight}{0.2\textheight}

\begin{subfigure}[t]{.48\textwidth}
  \centering
  \parbox[b][\caseThreeTopGraphicHeight][b]{\linewidth}{%
    \includegraphics[trim={0cm 0cm 0cm 1cm}, clip,width=\linewidth]{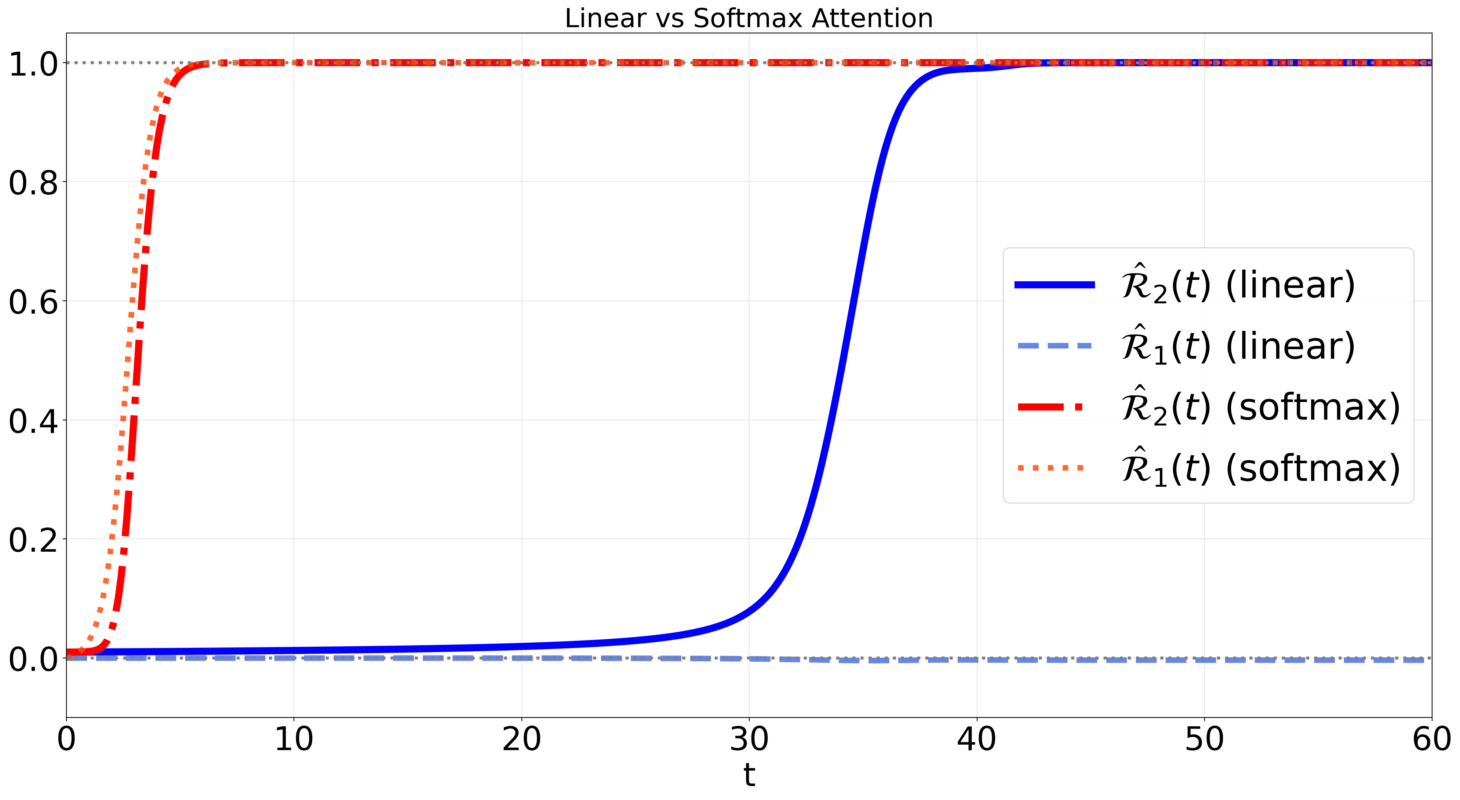}%
  }
  \caption{Evolution of the clustering diagnostics $\widehat{\CR}_{1}(t)$ and $\widehat{\CR}_{2}(t)$ for the $\operatorname{LSA}$ (blue) and $\operatorname{USA}$ (red)  models.}
  \label{subfig:case_3_clustering_metrics}
\end{subfigure}%
\hfill
\begin{subfigure}[t]{.48\textwidth}
  \centering
  \parbox[b][\caseThreeBottomGraphicHeight][b]{\linewidth}{%
    \includegraphics[width=\linewidth]{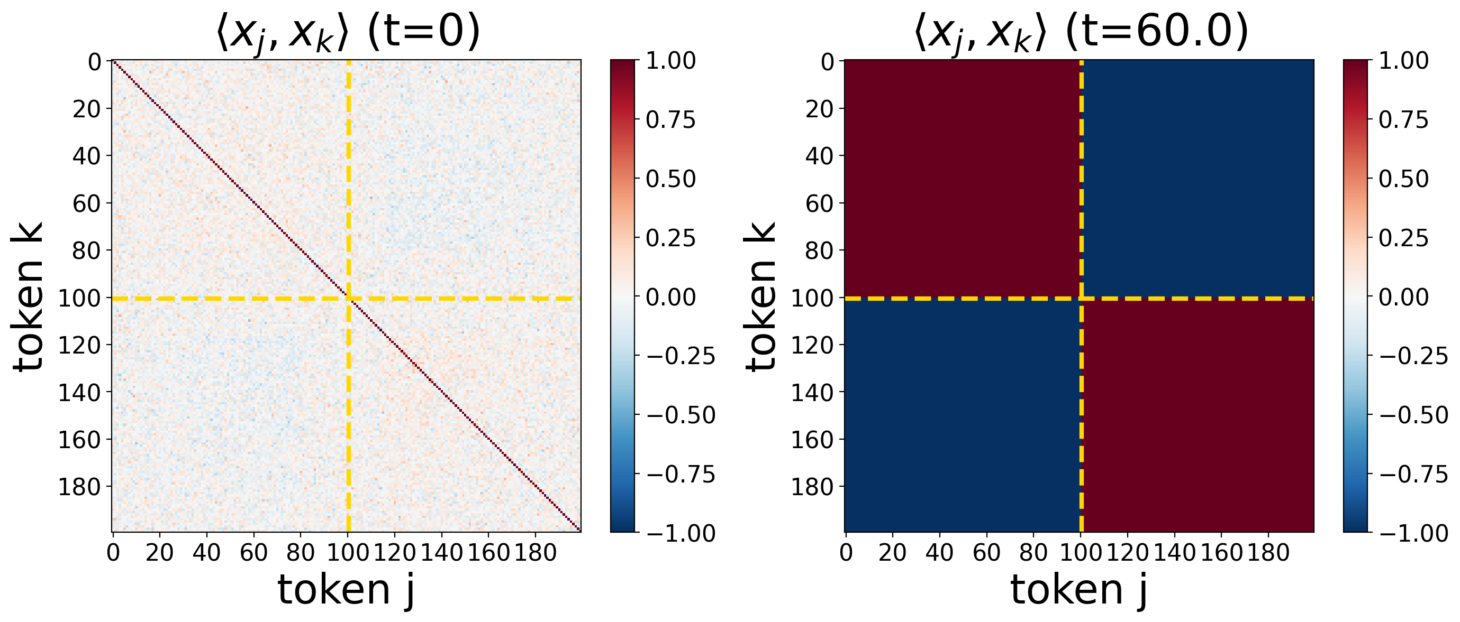}%
  }
  \caption{The Gram matrices $G(t)$ for the LSA model at $t=0$ and $t=60$.}
  \label{subfig:case_3_LSA_gram}
\end{subfigure}%

\begin{subfigure}[t]{.48\textwidth}
  \centering
  \parbox[b][\caseThreeTopGraphicHeight][b]{\linewidth}{%
    \includegraphics[trim={1.2cm 0cm 0cm 1cm}, clip,width=\linewidth]{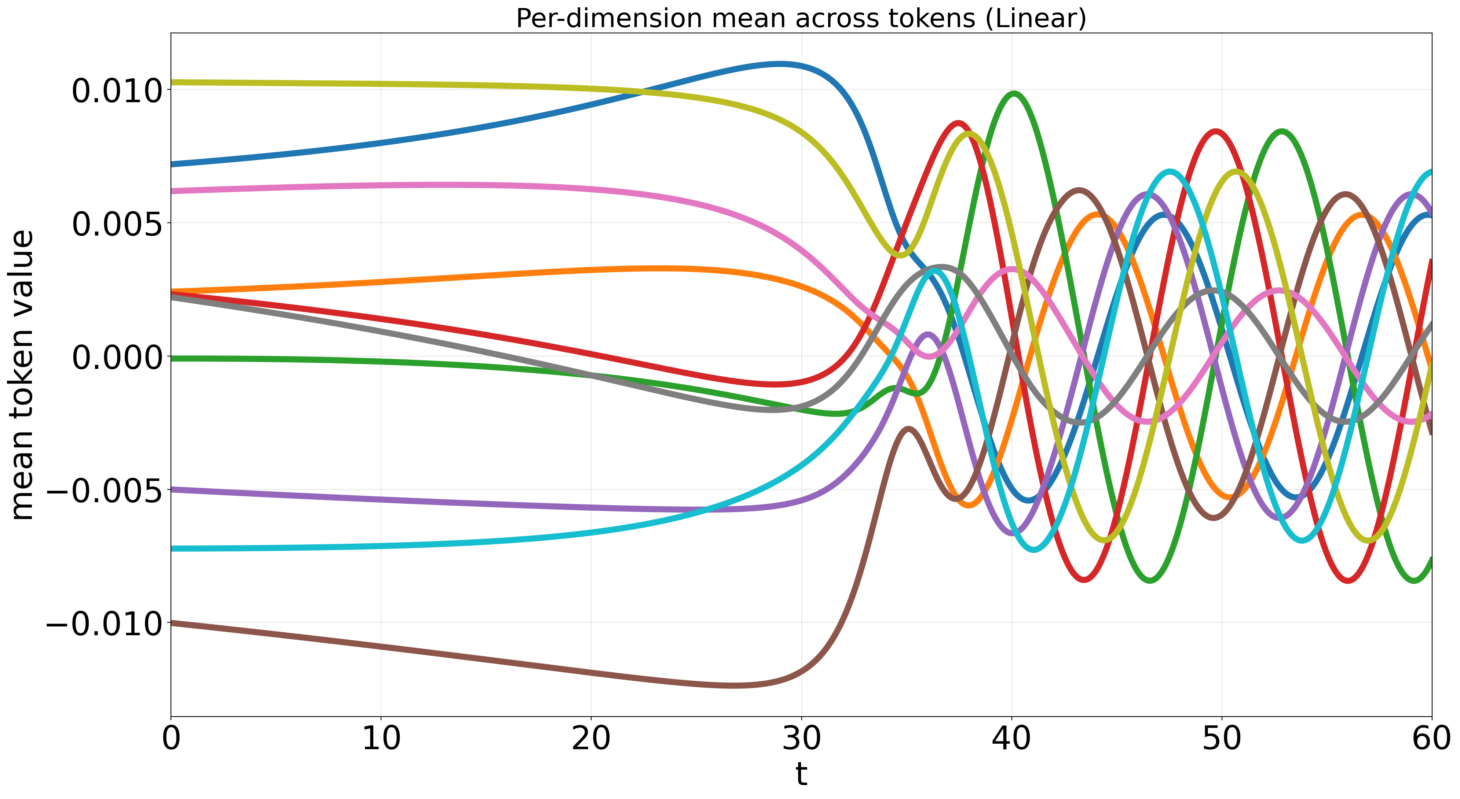}%
  }
  \caption{Coordinate-wise mean $m_l(t)$ for the LSA model.}
  \label{subfig:case_3_LSA_per_dim_mean}
\end{subfigure}%
\hfill
\begin{subfigure}[t]{.48\textwidth}
  \centering
  \parbox[b][\caseThreeTopGraphicHeight][b]{\linewidth}{%
    \includegraphics[trim={1.2cm 0cm 0cm 1cm}, clip,width=\linewidth]{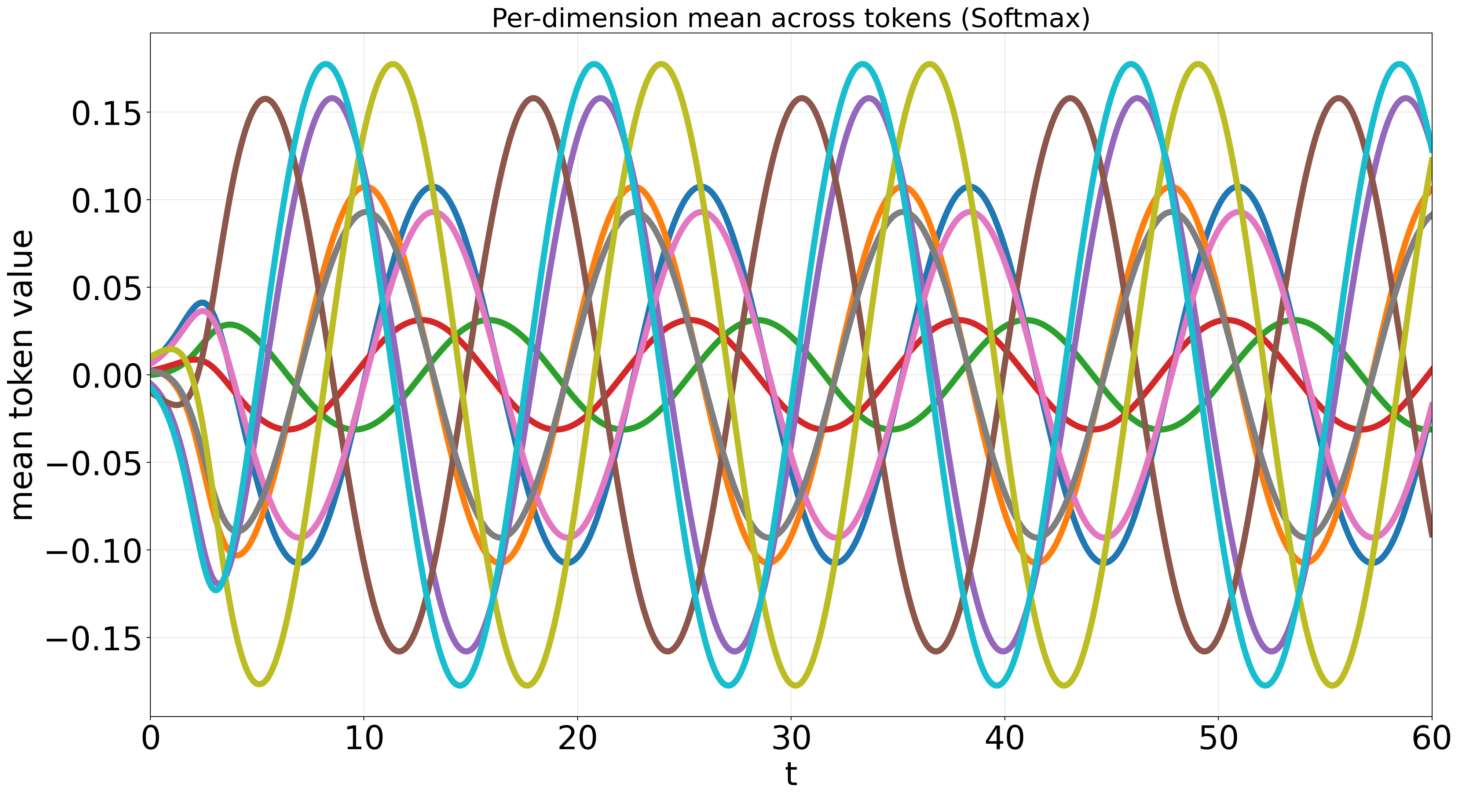}%
  }
  \caption{Coordinate-wise mean $m_l(t)$ for the USA model.}
  \label{subfig:case_3_USA_per_dim_mean}
\end{subfigure}
\caption{
\textbf{Case 3, parameter regime:} $A = I$ and $V = \left(\begin{smallmatrix} v_{11} & v_{12} \\ -v_{12} & v_{11} \end{smallmatrix}\right)$ with $v_{11} > 0$ and $v_{12} \neq 0$.
For both the LSA and USA models, 
the mean-squared cosine similarity $\widehat{\CR}_2(t)$ converges to $1$. 
However, the mean cosine similarity $\widehat{\CR}_1(t)$ converges to $1$ for the $\operatorname{USA}$ model, shown by the red dotted line, but to $0$ for the LSA model, shown by the blue dashed line; see Figure~\ref{subfig:case_3_clustering_metrics}. 
Thus, the USA dynamics clusters at a single point, whereas the LSA dynamics forms two antipodal clusters.
The periodic oscillations of the coordinate-wise means $m_l(t)$, shown in Figures~\ref{subfig:case_3_LSA_per_dim_mean} and \ref{subfig:case_3_USA_per_dim_mean}, indicate that the token dynamics remain oscillatory.
}
\label{fig:case3_d100}
\end{figure}


\subsection{Case 4: Hidden Hamiltonian Dynamics}
\label{sec:exp_case4}
Our theoretical findings from Theorem~\ref{thm:case_4} identify the parameter regimes under which the associated $2$-dimensional LSA model exhibits a ``hidden" Hamiltonian structure after suitable transformations.
Specifically, we take $A = I$ and $V = \left(\begin{smallmatrix} v_{11} & v_{12} \\ -v_{12} & -v_{11} \end{smallmatrix}\right)$, which depends on two parameters.
When $v_{11} < v_{12}$, the LSA dynamics exhibits persistent periodic oscillations.
In contrast, when $v_{11} \geq v_{12}$, the LSA dynamics converge to a pair of antipodal points.

To extend this construction to higher dimensions, we take $V$ to be a block-diagonal matrix whose blocks are $2\times 2$ blocks of the form
$\left(\begin{smallmatrix} v_{11} & v_{12} \\ -v_{12} & -v_{11} \end{smallmatrix}\right)$.

The corresponding numerical results are presented in Figures~\ref{fig:case4_agb_d100}--\ref{fig:case4_alb_d100}.
For the LSA model, the long-time behavior is governed by the relation between $v_{11}$ and $v_{12}$, in agreement with the $2$-dimensional theory from Theorem \ref{thm:case_4}. 
When $v_{11} \geq v_{12}$, the tokens cluster into a pair of antipodal points, as indicated
by $\widehat{\CR}_2(t)$ converging to $1$ while $\widehat{\CR}_1(t)$ remains close to $0$ in Figures~\ref{fig:case4_agb_d100} and \ref{fig:case4_aeb_d100}. 
When $v_{11} < v_{12}$, the token trajectories are periodic and do not form clusters. 
In this regime, the clustering diagnostics remain at $\widehat{\CR}_1(t) = \widehat{\CR}_2(t) = 0$, as shown in Figure~\ref{subfig:case_4_clustering_diagnostics}, while the per-dimensional means $m_l(t)$ exhibit sustained periodic oscillations, as shown in Figure~\ref{subfig:case_4_per_dim_mean_LSA}.
    
For the USA model, the same parameter threshold separates clustering from oscillatory non-clustering behavior. 
When $v_{11} \geq v_{12}$, the USA model converges to a single point, as indicated by $\widehat{\CR}_1(t) = \widehat{\CR}_2(t) = 1$ in Figures~\ref{fig:case4_agb_d100} and \ref{fig:case4_aeb_d100}.
This differs from the antipodal clustering observed in the LSA model. 
When $v_{11} < v_{12}$, the USA dynamics also fails to synchronize and instead exhibits persistent periodic oscillations, as shown by the oscillatory behavior of $m_l(t)$ in Figure~\ref{subfig:case_4_per_dim_mean_USA}.

Finally, we observe a quantitative slowdown near the transition.
In the regime $v_{11} > v_{12}$, as shown in Figure~\ref{fig:case4_agb_d100}, convergence is relatively rapid for both models. 
However, at the bifurcation point $v_{11} = v_{12}$, shown in Figure~\ref{fig:case4_aeb_d100}, convergence is substantially slower and takes orders of magnitude longer to stabilize. 
Consistent with the previous cases, the USA model generally converges faster than the LSA model in the regimes where clustering occurs.


\begin{figure}[!htb]
\centering

\newlength{\subfigimgheightF}
\setlength{\subfigimgheightF}{0.2\textheight} 

\begin{subfigure}[t]{.48\textwidth}
  \centering
  \begin{minipage}[b][\subfigimgheightF][b]{\linewidth}
    \centering
    \includegraphics[trim={0cm 0cm 0cm 1cm}, clip,width=\linewidth,height=\subfigimgheightF,keepaspectratio]{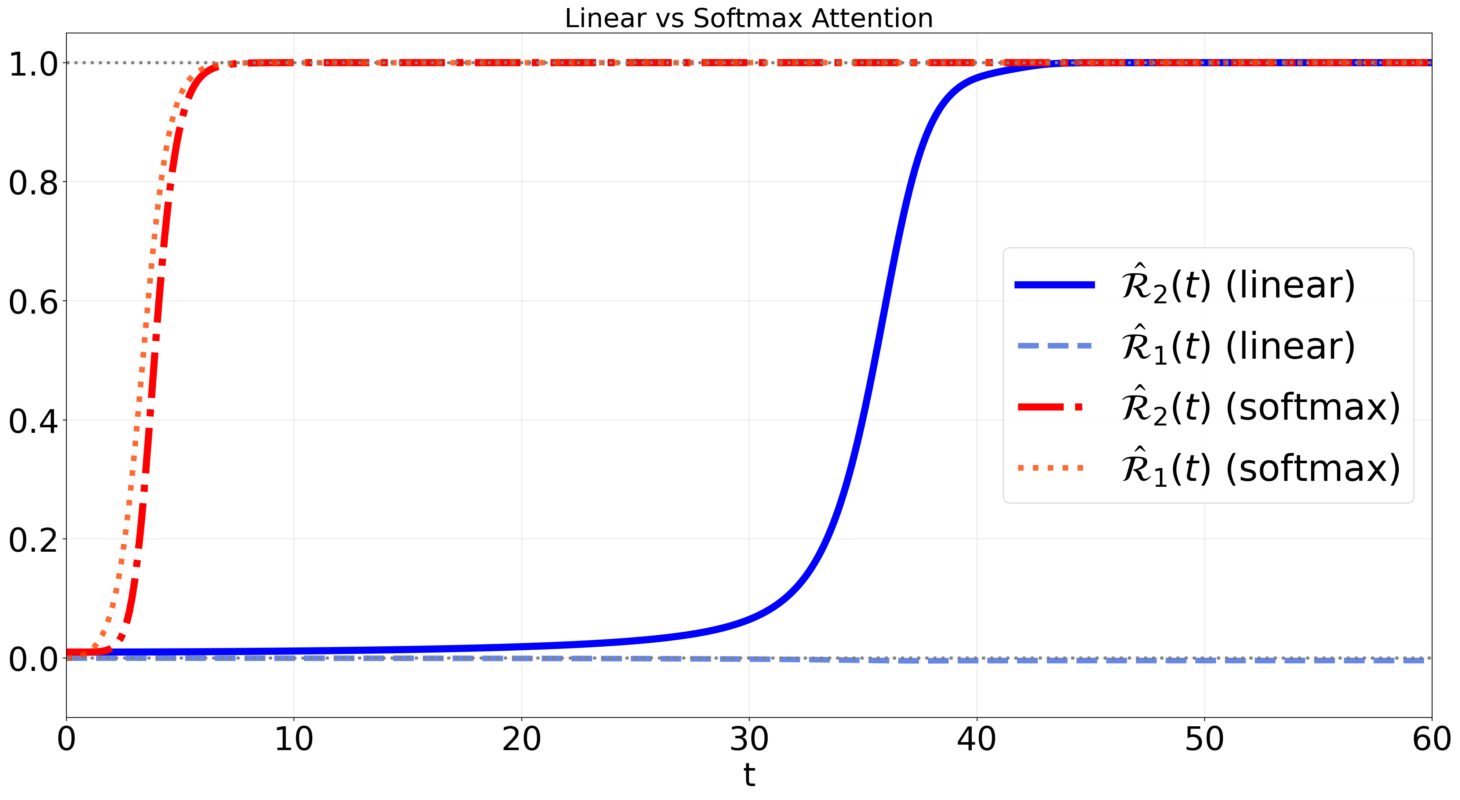}
  \end{minipage}
  \caption{Evolution of the clustering diagnostics $\widehat{\CR}_{1}(t)$ and $\widehat{\CR}_{2}(t)$ for the $\operatorname{LSA}$ (blue) and $\operatorname{USA}$ (red)  models.}
\end{subfigure}%
\hfill
\begin{subfigure}[t]{.48\textwidth}
  \centering
  \begin{minipage}[b][\subfigimgheightF][b]{\linewidth}
    \centering
    \includegraphics[width=\linewidth,height=\subfigimgheightF,keepaspectratio]{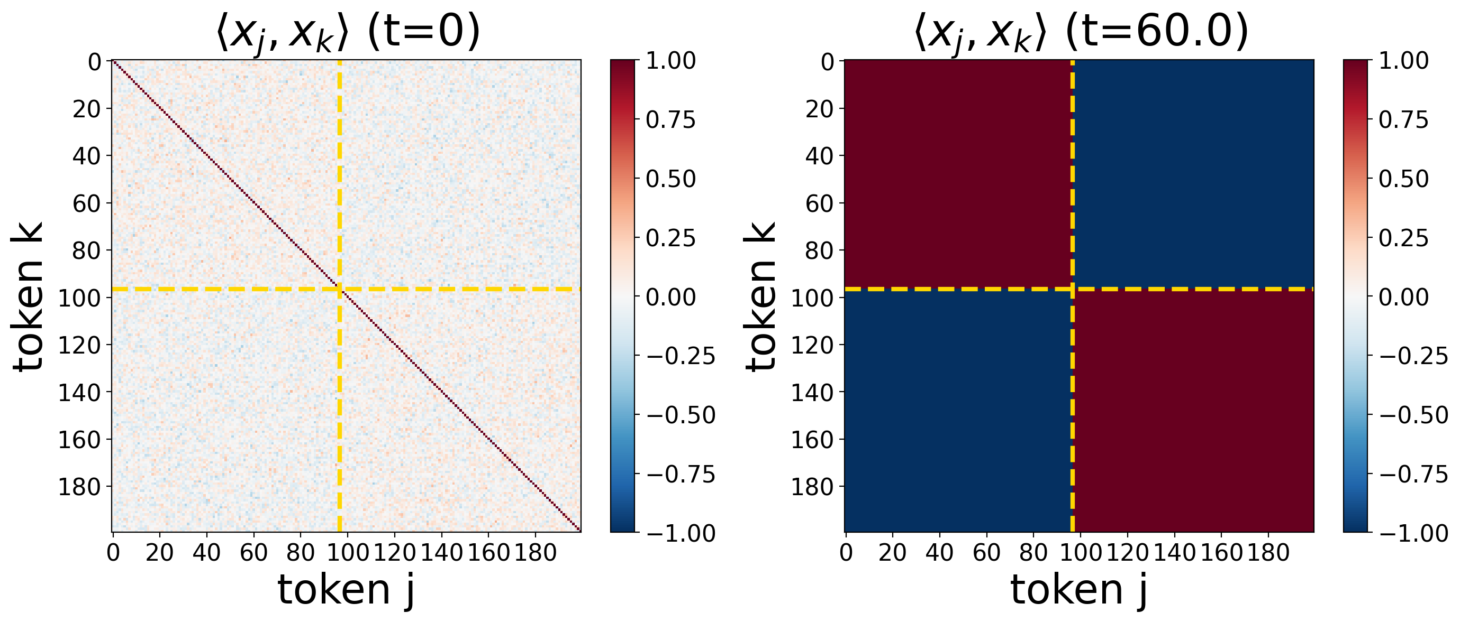}
  \end{minipage}
  \caption{The Gram matrices $G(t)$ for the LSA model at $t=0$ and $t=45$.}
\end{subfigure}

\caption{
\textbf{Case $\mathbf 4$, parameter regime:} $v_{11} > v_{12}$.
For both the LSA and USA models, 
the mean-squared cosine similarity $\widehat{\CR}_2(t)$ converges to $1$. 
However, the mean cosine similarity $\widehat{\CR}_1(t)$ converges to $1$ for the $\operatorname{USA}$ model, shown by the red dotted line, but to $0$ for the LSA model, shown by the blue dashed line. 
This indicates that the USA dynamics clusters at a single point, whereas the LSA dynamics converges to two antipodal clusters.  
The Gram matrices $G$ for the LSA model, shown on the right, suggest that the tokens are not entirely evenly split between the two clusters.}
\label{fig:case4_agb_d100}
\end{figure}


\begin{figure}[!htb]
\centering

\newlength{\subfigimgheightG}
\setlength{\subfigimgheightG}{0.2\textheight} 

\begin{subfigure}[t]{.48\textwidth}
  \centering
  \begin{minipage}[b][\subfigimgheightG][b]{\linewidth}
    \centering
    \includegraphics[trim={0cm 0cm 0cm 1cm}, clip,width=\linewidth,height=\subfigimgheightG,keepaspectratio]{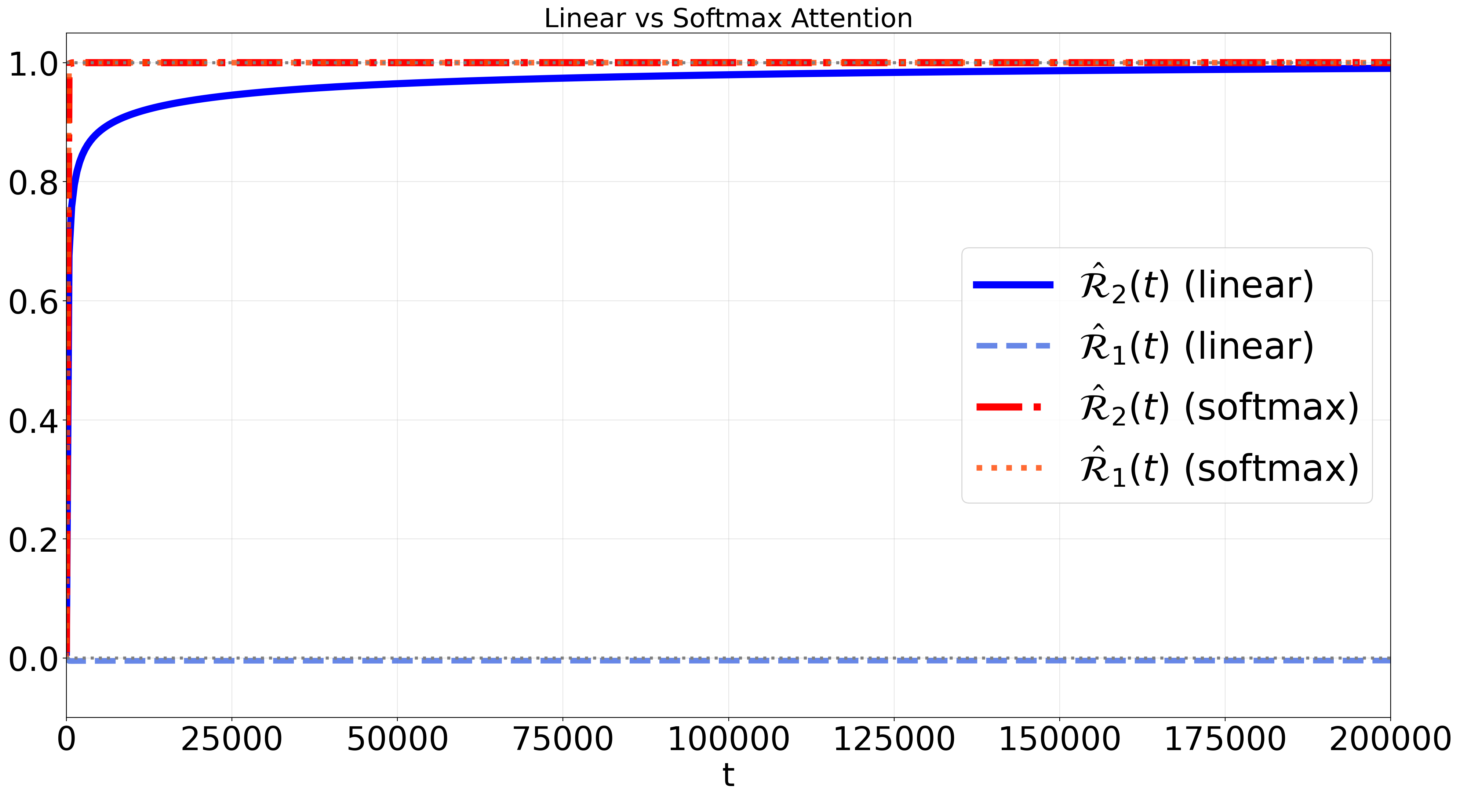}
  \end{minipage}
  \caption{Evolution of the clustering diagnostics $\widehat{\CR}_{1}(t)$ and $\widehat{\CR}_{2}(t)$ for the $\operatorname{LSA}$ (blue) and $\operatorname{USA}$ (red)  models.}
\end{subfigure}%
\hfill
\begin{subfigure}[t]{.48\textwidth}
  \centering
  \begin{minipage}[b][\subfigimgheightG][b]{\linewidth}
    \centering
    \includegraphics[width=\linewidth,height=\subfigimgheightG,keepaspectratio]{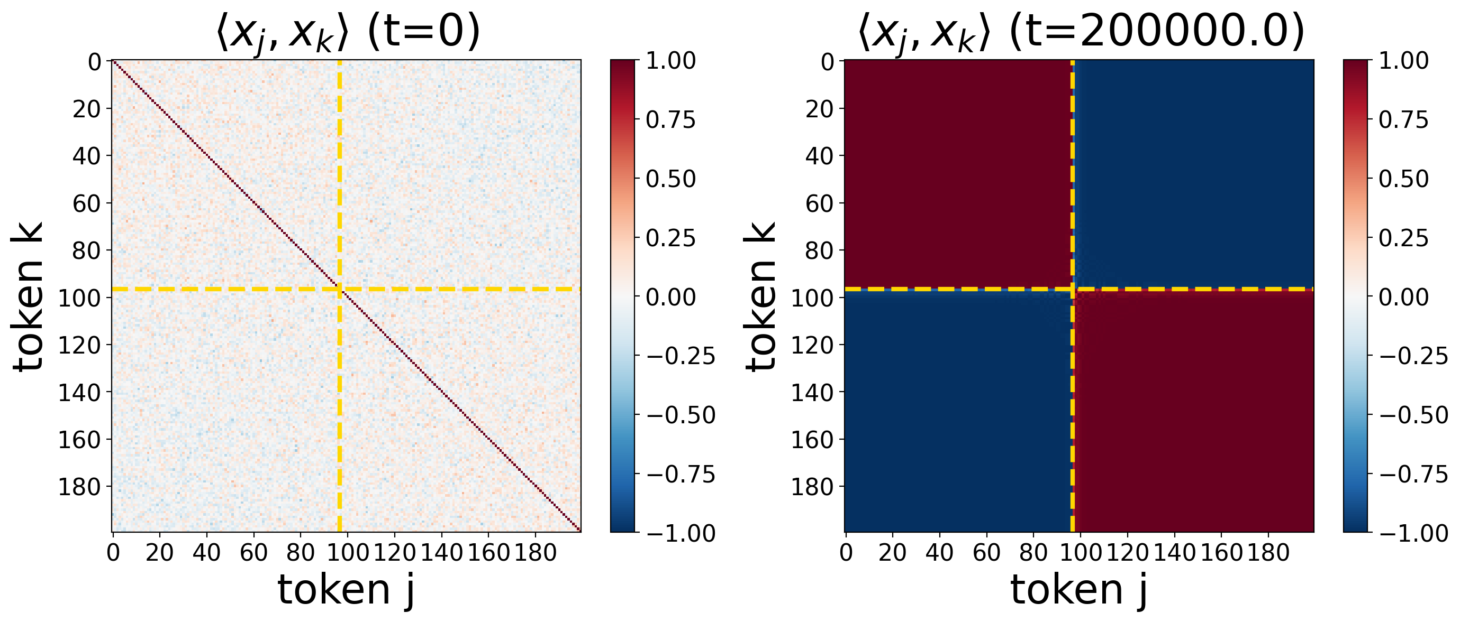}
  \end{minipage}
  \caption{The Gram matrices $G(t)$ for the LSA model at $t=0$ and $t=200{,}000$.}
\end{subfigure}

\caption{
\textbf{Case 4, parameter regime:} $v_{11} = v_{12}$.
The qualitative behavior is the same as in the regime $v_{11} > v_{12}$: for both the LSA and USA models, $\widehat{\CR}_2(t)$ converges to $1$, while $\widehat{\CR}_1(t)$ converges to $1$ for the $\operatorname{USA}$ model and to $0$ for the LSA model.
This indicates that the USA dynamics clusters at a single point, whereas the LSA dynamics converges to two antipodal clusters.
However, convergence is notably slower in this regime, taking approximately $t=150{,}000$ steps.
}

\label{fig:case4_aeb_d100}
\end{figure}


\begin{figure}[!htb]
\centering

\newlength{\subfigimgheightH}
\setlength{\subfigimgheightH}{0.2\textheight} 

\begin{subfigure}[t]{.48\textwidth}
  \centering
  \begin{minipage}[b][\subfigimgheightH][b]{\linewidth}
    \centering
    \includegraphics[trim={0cm 0cm 0cm 1cm}, clip,width=\linewidth,height=\subfigimgheightH,keepaspectratio]{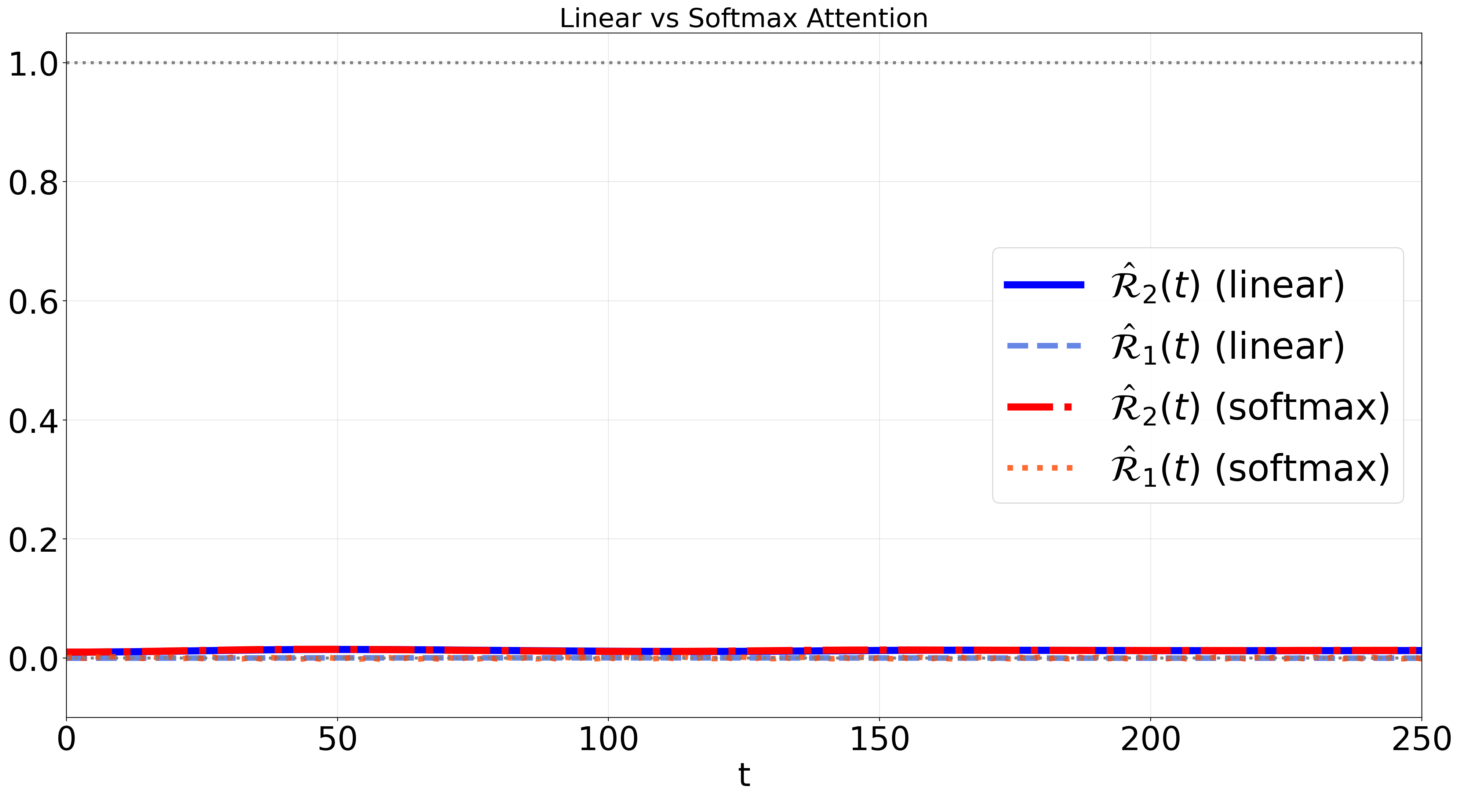}
  \end{minipage}
  \caption{Evolution of the clustering diagnostics $\widehat{\CR}_{1}(t)$ and $\widehat{\CR}_{2}(t)$ for the $\operatorname{LSA}$ (blue) and $\operatorname{USA}$ (red)  models.}
  \label{subfig:case_4_clustering_diagnostics}
\end{subfigure}
\hfill
\begin{subfigure}[t]{.48\textwidth}
  \centering
  \begin{minipage}[b][\subfigimgheightH][b]{\linewidth}
    \centering
    \includegraphics[width=\linewidth,height=\subfigimgheightH,keepaspectratio]{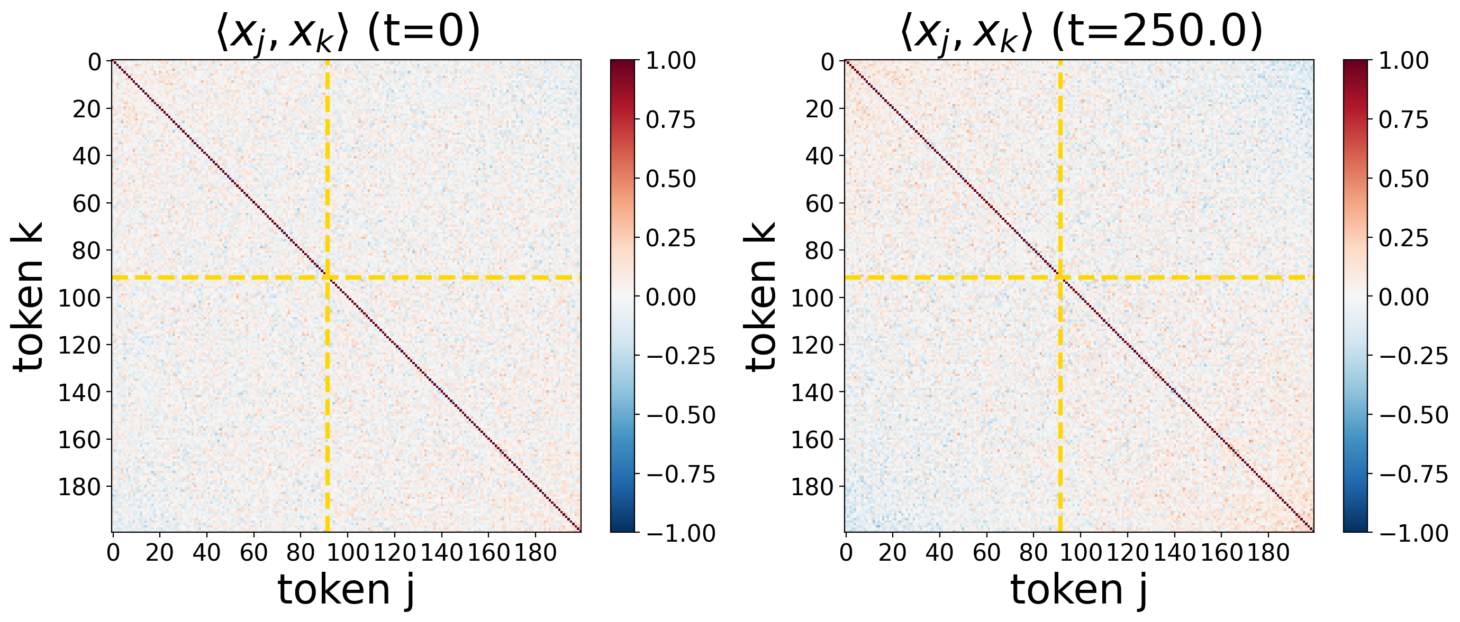}
  \end{minipage}
  \caption{Gram matrices $G(t)$ for the LSA model at $t=0$ and $t=250$.}
\end{subfigure}%
\\[1.0em]
\begin{subfigure}[t]{.48\textwidth}
  \centering
  \begin{minipage}[b][\subfigimgheightH][b]{\linewidth}
    \centering
    \includegraphics[trim={1.2cm 0cm 0cm 1cm}, clip,width=\linewidth,height=\subfigimgheightH,keepaspectratio]{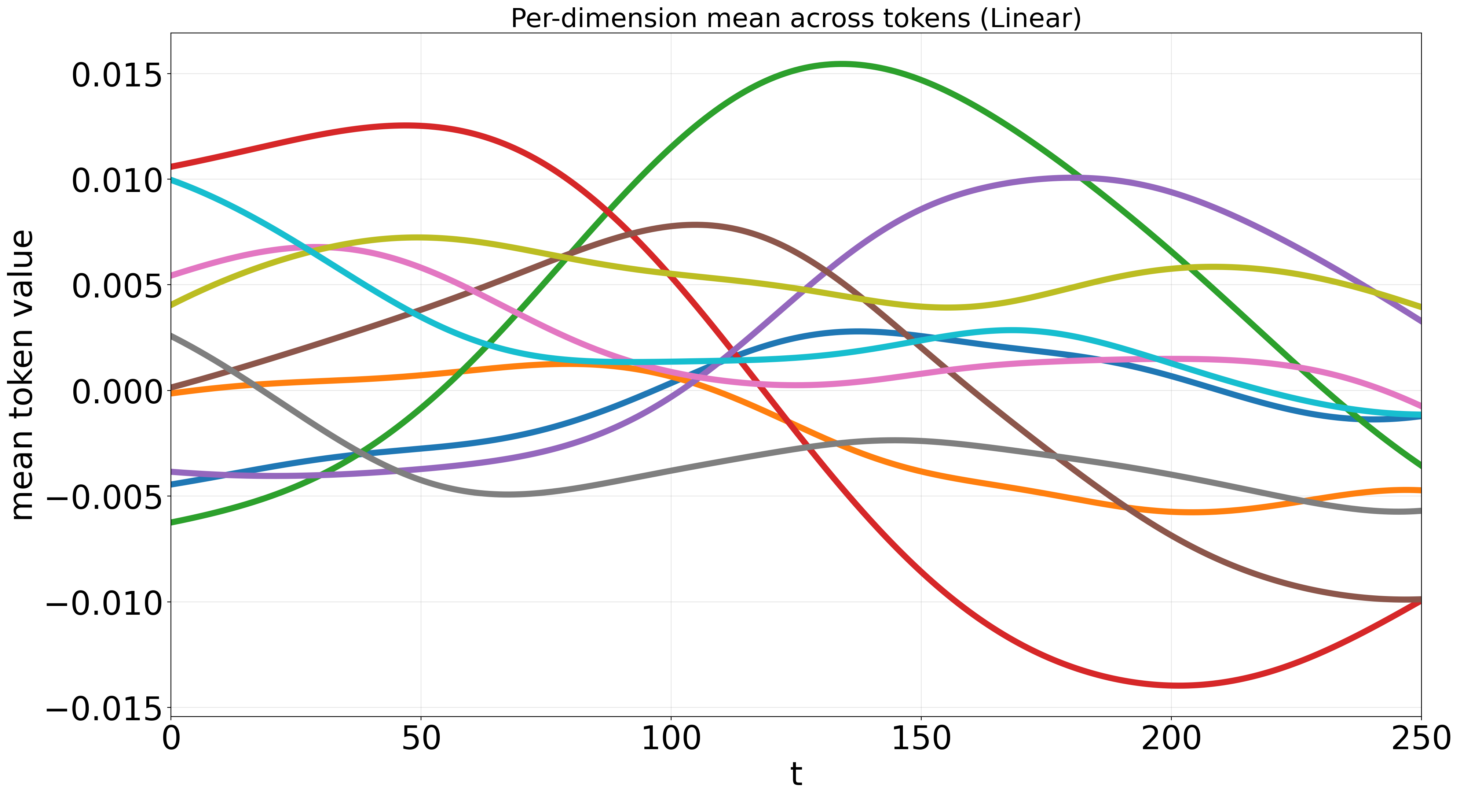}
  \end{minipage}
  \caption{Coordinate-wise mean $m_l(t)$ for the LSA model.}
  \label{subfig:case_4_per_dim_mean_LSA}
\end{subfigure}
\hfill
\begin{subfigure}[t]{.48\textwidth}
  \centering
  \begin{minipage}[b][\subfigimgheightH][b]{\linewidth}
    \centering
    \includegraphics[trim={1.2cm 0cm 0cm 1cm}, clip,width=\linewidth,height=\subfigimgheightH,keepaspectratio]{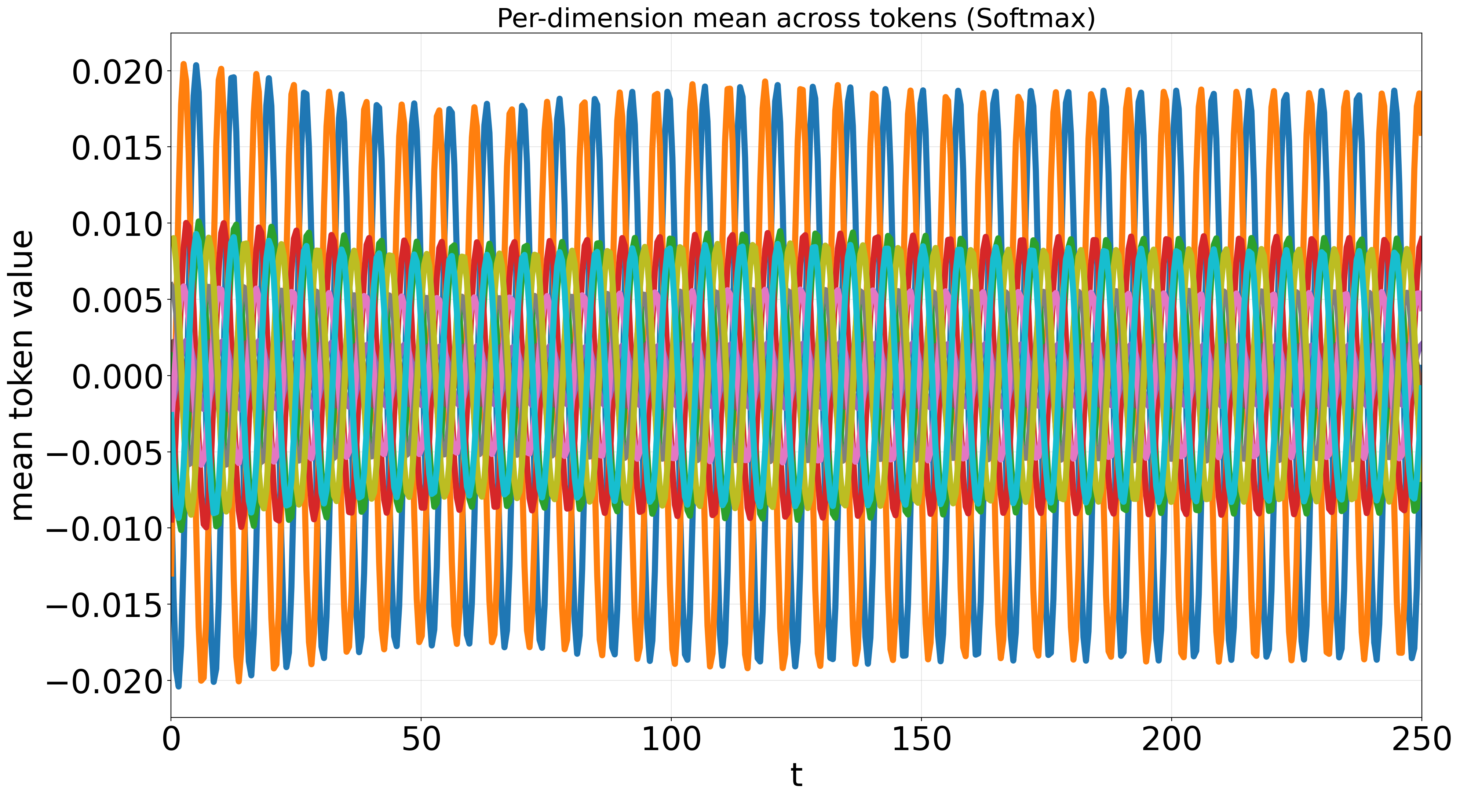}
  \end{minipage}
  \caption{Coordinate-wise mean $m_l(t)$ for the USA model.}
  \label{subfig:case_4_per_dim_mean_USA}
\end{subfigure}

\caption{
\textbf{Case 4, parameter regime:} $v_{11} < v_{12}$.
For both the LSA and USA models, the clustering diagnostics remain at $\widehat{\CR}_1(t) = \widehat{\CR}_2(t) = 0$, while the per-dimensional means $m_l(t)$ oscillate periodically.
This indicates persistent oscillatory dynamics without cluster formation.
}
\label{fig:case4_alb_d100}
\end{figure}

\if{false}
\begin{figure}[H]
\centering
\begin{subfigure}{.3\textwidth}
  \centering
  \includegraphics[width=\linewidth]{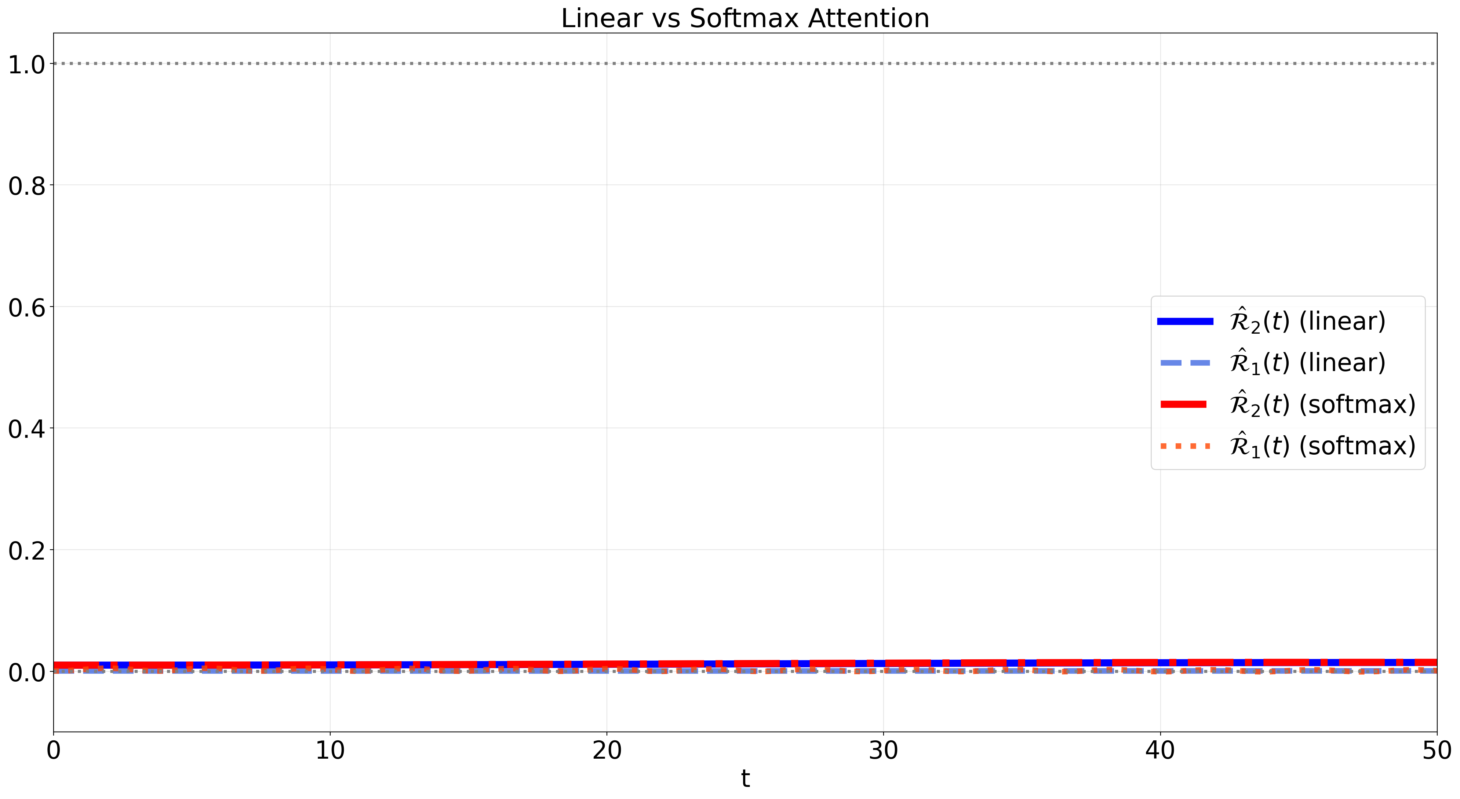}
  \caption{Linear and unnormalized softmax attention: clustering metric over time.}
\end{subfigure}
~
\begin{subfigure}{.3\textwidth}
  \centering
  \includegraphics[width=\linewidth]{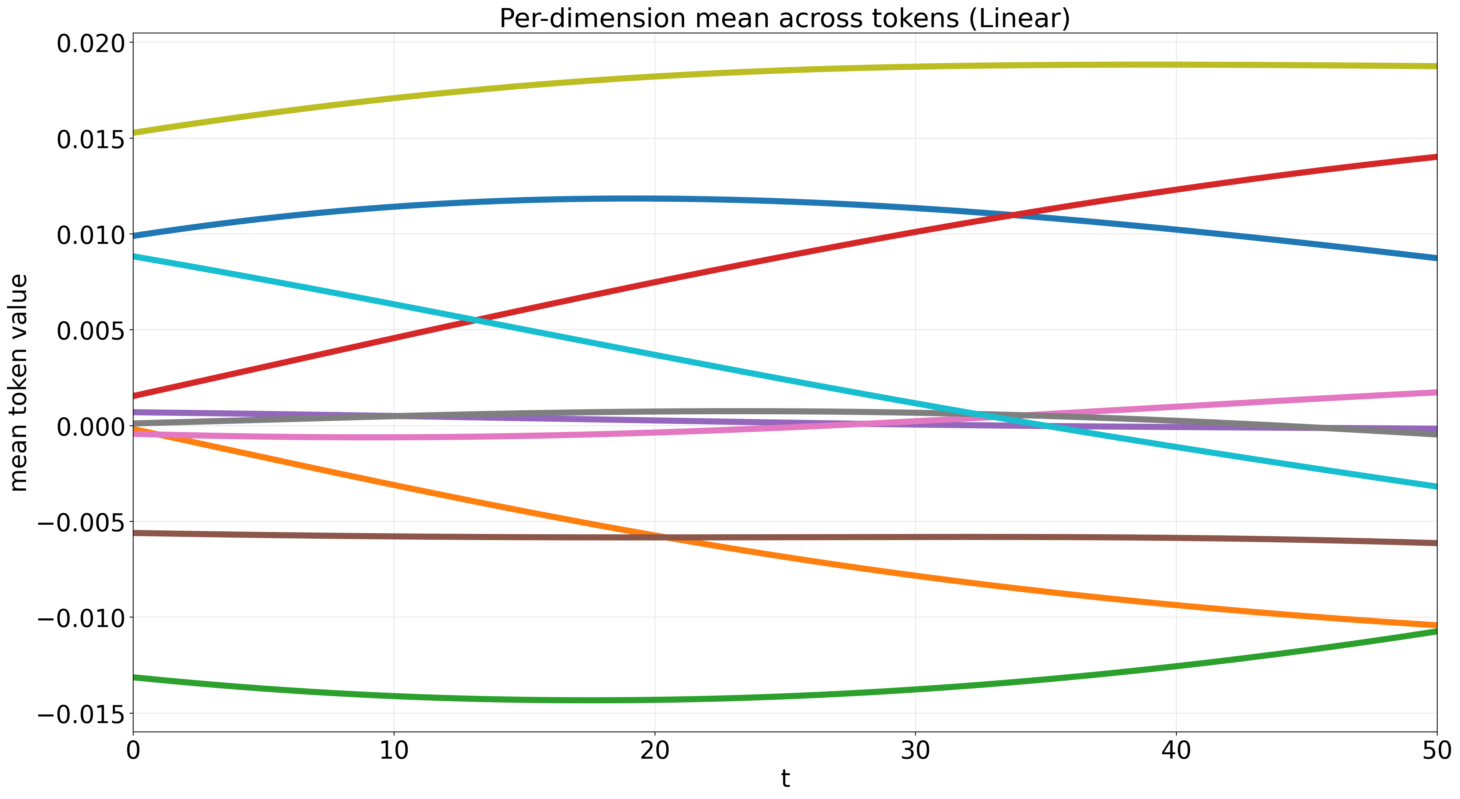}
  \caption{Per-dimension mean across tokens for linear attention figures.}
\end{subfigure}
~
\begin{subfigure}{.3\textwidth}
  \centering
  \includegraphics[width=\linewidth]{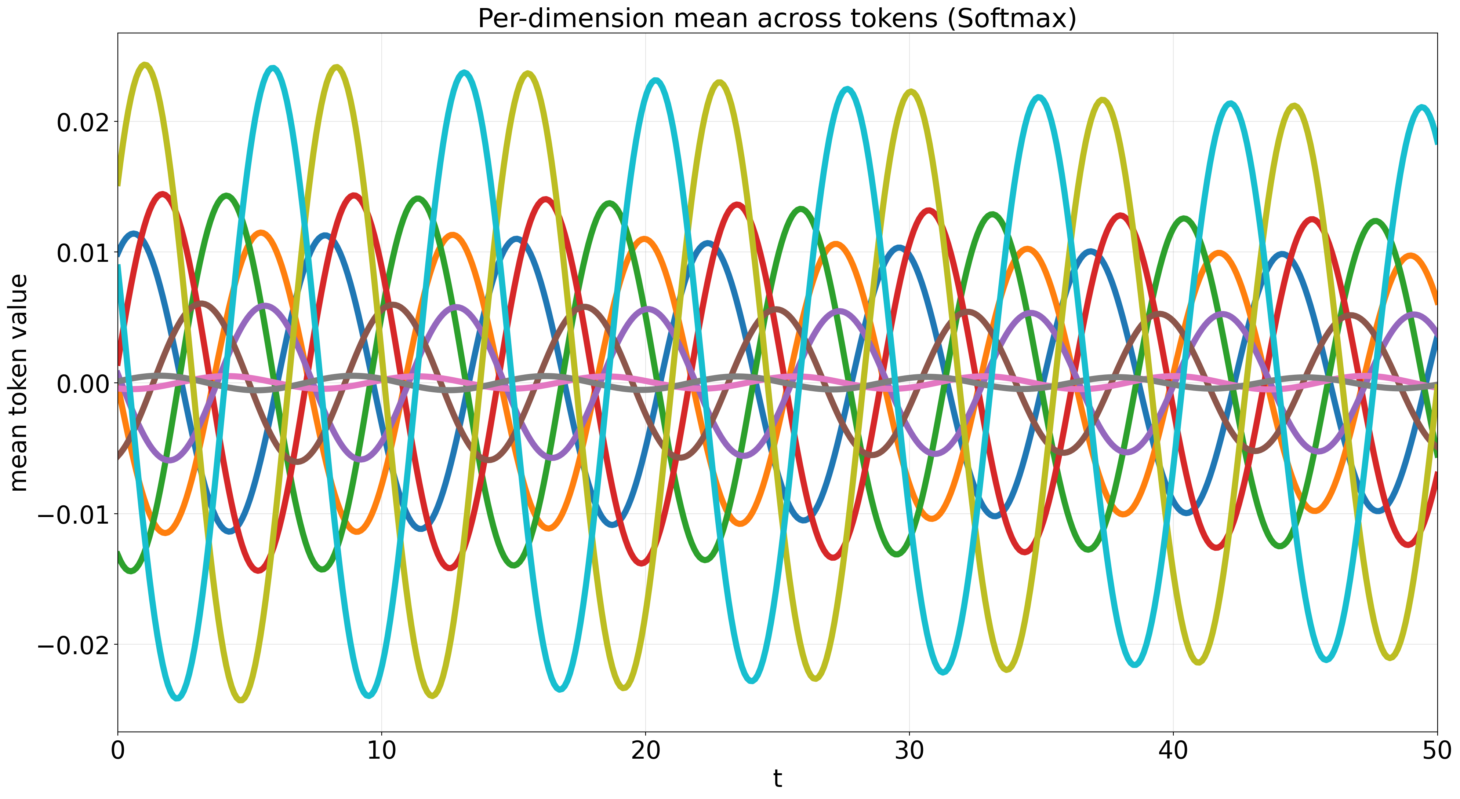}
  \caption{Per-dimension mean across tokens for unnormalized softmax attention.}
\end{subfigure}
\caption{
\textbf{Case 4, oscillatory regime} ($A = I$, $V$ Hamiltonian, $v_{11}= \frac{1}{2}, v_{22}=1$, $d=100$, $n=200$).
As expected we do not observe clustering in either model.  Oscillations are clearly visible in the per-dimension mean plots,.  Compare this to Figure~\ref{fig:xxx}, where there is no clustering but apparent convergence to a stable configuration.
}
\label{fig:case4_alb_d100}
\end{figure}
\fi 

\section{Conclusion}
\label{sec:Conclusions}
In this paper, we develop an interacting-particle-system-based approach to analyze the inference-time behavior of deep linear encoder-only transformers.
Our main results demonstrate that such dynamics exhibit a rich variety of dynamical behaviors beyond token synchronization and clustering. 
To reveal these behaviors, we establish a connection between two-dimensional linear self-attention dynamics and generalized Kuramoto models with pure second-harmonic coupling. 
This connection allows us to apply dimension-reduction techniques, including the Watanabe--Strogatz transformation and the Ott--Antonsen ansatz, to derive analytically tractable reduced dynamics. The resulting framework identifies the mechanisms underlying several distinct long-time behaviors, including clustering with convergence, clustering without convergence, oscillations, and bifurcations. 
In particular, our theory characterizes how specific choices of the key, query, and value matrices can lead to fundamentally different long-time behaviors, while our stability analysis shows that, under some parameter regimes, these behaviors persist beyond initializations exactly on the OA manifold.
Our numerical experiments confirm these theoretical observations, and further suggest that the mechanisms identified in dimension $d=2$ continue to appear in higher-dimensional linear self-attention dynamics and may also be relevant to softmax self-attention.


\paragraph{Open questions.}
Let us conclude by discussing several open questions and research directions that may broaden the scope of the presented analytical framework, deepen the connection between transformer dynamics and other models of collective dynamics, as well as contribute to a more comprehensive understanding of transformer architectures.
\begin{itemize}
    \item 
    A first natural question is whether the analytical framework developed here for dimension $d=2$ can be extended to higher-dimensional settings.
    Our analysis in dimension $d=2$ builds on the fact that the linear self-attention model \eqref{eq:LSA} possesses a low-dimensional structure, revealed through the WS transformation and the OA ansatz, that is substantially more amenable to analysis.
    It is therefore natural to ask whether also the higher-dimensional LSA model admits an analogous reduced description governing the evolution of the particle dynamics.
    
    Higher-dimensional generalizations of the classical Kuramoto model \eqref{eq:related_kuramoto} and the corresponding extensions of the WS transformation and the OA ansatz have been studied in~\cite{Lohe_2018,chandra2019complexity,Lipton_2021,park2022watanabe,lohe2025exact}.
    We leave the investigation of analogous reduced structures for higher-dimensional LSA models to future work.

    \item A second important question is whether our analytical framework can be extended to the transformer dynamics \eqref{eq:transformer} in case of more general attention kernels $h$.
    As shown in Theorem~\ref{thm:finite_sys}, the linear kernel $h(y) = y$ gives rise to a Kuramoto model with \emph{pure} second-harmonic coupling, placing the dynamics within a class for which powerful dimension-reduction techniques are available.
    
    By contrast, the nonlinear attention kernel $h(y) = \exp(y)$, corresponding to softmax self-attention \cite{vaswani2017attention}, produces dynamics with \emph{mixed-order} harmonics coupling.
    In this setting, the WS transformation and the OA ansatz are no longer directly applicable~\cite{Gong_2019}.
    It would therefore be interesting to develop alternative dimension-reduction techniques or identify approximate low-dimensional structures for such models~\cite{tonjes2020low,cestnik2022hierarchy,tokunaga2026low}.

    A related direction is to investigate regimes in which one harmonic dominates the remaining coupling terms.
    Under suitably prepared initializations, one may ask whether the resulting mixed-harmonic dynamics remain close to those of the corresponding pure-harmonic model over a finite time interval.
    A rigorous approximation result of this type could help explain the metastable clustering observed in the softmax self-attention model~\cite{geshkovski2024dynamic,bruno2024emergence,bruno2025multiscale}.
    In particular, the dominant harmonic may initially govern the dynamics and induce a transient multi-cluster state resembling that of the corresponding pure-harmonic model, before the cumulative influence of the remaining harmonics becomes significant.

    \item A third direction is to extend the present framework to linear transformer architectures incorporating additional model components. 
    For instance, the fully connected layer may contribute an active-rotator term, leading to an active-rotator model that generalizes the classical Kuramoto dynamics~\cite{shinomoto1986phase,giacomin2012transitions}, or may be viewed as an external control input~\cite{Caponigro2015}. Similarly, causal attention dynamics \cite{karagodin2024causal,karagodin2024clustering} is related to Kuramoto-type model with a directed and nonuniform interaction structure~\cite{ji2014low,rodrigues2016kuramoto}, while time-dependent parameter matrices lead naturally to Kuramoto models with time-varying coefficients~\cite{petkoski2012kuramoto,pietras2016ott}. 
    Investigating how these additional structures interact with the long-time behaviors and parameter regimes identified in the present work would constitute an important step toward understanding more realistic transformer architectures.


    \item A fourth direction is to clarify the possible connection between transformer dynamics and Vicsek-type models of collective alignment~\cite{vicsek1995novel,degond2008continuum,degond2013macroscopic}. 
    In two dimensions, continuous-time Vicsek-type orientation dynamics can be viewed as a spatially coupled form of Kuramoto alignment, in which the interaction network is determined by the evolving particle configuration. 
    The state-dependent interaction structure of the transformer dynamics \eqref{eq:transformer} is therefore reminiscent of such models, although general attention weights and value transformations need not represent standard attractive alignment. 
    It would be interesting to identify choices of attention kernels and parameter regimes for which transformer dynamics admit a precise interpretation as a Vicsek-type model, and to investigate whether analytical tools developed for collective alignment systems can be adapted to this setting.
\end{itemize}

\section*{Acknowledgements}
SL and NGT were supported by NSF-DMS grant 2236447. They also gratefully acknowledge support from the IFDS at UW-Madison through NSF TRIPODS
grant 2023239. JP was supported by the Polish National Science Centre’s Grant No. 2025/58/E/ST1/00482 (Sonata Bis).

For the purpose of Open Access, the authors have applied a CC BY public copyright license to any Author Accepted Manuscript (AAM) version arising from this submission. 

\bibliographystyle{abbrv}
\bibliography{biblio}

\newpage
\appendix
\clearpage
\phantomsection
\addcontentsline{toc}{section}{Appendix}
\addtocontents{toc}{\protect\setcounter{tocdepth}{0}}

\etocdepthtag.toc{appendix}

\begingroup
\etocsettagdepth{main}{none}
\etocsettagdepth{appendix}{subsection}
\etocsettocstyle{\section*{Appendix Contents}}{}
\tableofcontents
\endgroup

\section{Collected Notation and Quantities}
In Table~\ref{tab:notation}, we collect some notation used throughout the manuscript.
\begin{table*}[!htb]
\centering
\caption{Notation.}
\label{tab:notation}
{\small \renewcommand{\arraystretch}{1.15} \begin{tabularx}{\textwidth}{@{} l X @{\hspace{1.5em}} l X @{}}
\toprule
\textbf{Symbol} & \textbf{Meaning} 
& \textbf{Symbol} & \textbf{Meaning}\\
\midrule
$\BBS^1$ & Unit circle in $\mathbb{R}^2$.
& $\mathbb{T}^1$ & One-dimensional torus $\mathbb{R}/2\pi\mathbb{Z}$. \\

$x_k(t)$ & Representation of the $k$-th token.
& $A,V$ & Parameter matrices. \\

$\P_x^\perp(y)$ & Projection of $y$ onto the tangent space at $x$.
& $\theta_k(t)$ & Angular variable associated with $x_k(t)$. \\

$\Theta(t)$ & Mean-field angular variable.
& $f(t,\theta)$ & Law of $\Theta(t)$. \\

$\xi_k(t), \Xi(t)$ & Double-angle variables $2\theta_k(t)$ and $2\Theta(t)$.
& $g(t,\xi)$ & Law of $\Xi(t)$. \\

$\mathcal{R}_2^n(t), \CR_2(t)$ & Finite-particle and mean-field second-order parameters.
& $\rho(t), \phi(t)$ & Writing $\CR_2(t) = \rho(t) e^{i\phi(t)}$. \\

$\alpha(t),\eta(t)$ & Watanabe--Strogatz variables.
& $\gamma(t),\Phi(t)$ & Writing $\alpha(t) = \gamma(t) e^{i \Phi(t)}$. \\

$\alpha_{\OA}$ & Ott--Antonsen closure counterpart of $\alpha$~\eqref{eq:OA_counterpart}.
& $M_{\alpha,\eta}(\omega)$ & M\"{o}bius transform~\eqref{eq:Mobius_tf}. \\

$T_{\alpha,\eta}$ & Map induced by $M_{\alpha,\eta}$~\eqref{eq:map_T}.
& $\varphi_k,\varphi$ & Finite-particle and mean-field constants of motion. \\

$\nu$ & Law of $\varphi$.
& $\WC(\xi;\alpha)$ & Wrapped Cauchy distribution with parameter $\alpha$. \\
\bottomrule
\end{tabularx}
}
\end{table*}

We now collect some notation used in Section~\ref{sec:2}. 
Recall the parameter matrices $A$ and $V$ in dimension $d=2$,
\begin{equation}
A := \left(\begin{array}{ll}
    a_{11} & a_{12} \\
    a_{21} & a_{22} 
\end{array} \right)
    \qquad
    \text{and}
    \qquad
V := \left( \begin{array}{ll}
    v_{11} & v_{12}  \\
    v_{21} & v_{22} 
\end{array} \right).
\end{equation}
We define the following real-valued constants depending on the matrices $A$ and $V$:
\begin{equation}\label{eq:w_i}
    \begin{aligned}
        w_0 &:= v_{11}a_{21} +v_{12}a_{22} - v_{21}a_{11} - v_{22}a_{12}, \qquad 
        &&w_1 := v_{11}a_{11}+v_{12}a_{12} - v_{21}a_{21} - v_{22}a_{22},\\[4pt]
        w_2 &:= -v_{11}a_{21}- v_{12}a_{22} - v_{21}a_{11} - v_{22}a_{12}, \qquad
        &&w_3 := v_{11}a_{21} - v_{12}a_{22}-v_{21}a_{11} + v_{22}a_{12},\\[4pt]
        w_4 &:= v_{11}a_{22}+v_{12}a_{21}-v_{21}a_{12} - v_{22}a_{11}, \qquad
        &&w_5 := v_{11}a_{11}- v_{12}a_{12} - v_{21}a_{21} + v_{22}a_{22},\\[4pt]
        w_6 &:= v_{11}a_{12}+v_{12}a_{11}-v_{21}a_{22}-v_{22}a_{21}, \qquad
        &&w_7 := v_{11}a_{21} -v_{12}a_{22} + v_{21}a_{11}- v_{22}a_{12},\\[4pt]
        w_8 &:= v_{11}a_{22} + v_{12}a_{21} + v_{21}a_{12} + v_{22}a_{11}.
    \end{aligned}
\end{equation}
This enables us to compactly write the following complex-valued constants,
\begin{equation}\label{eq:coeff_b}
    \begin{gathered}
    b_1 := \frac{1}{16}(iw_5+w_6+w_7-iw_8), \qquad b_2 := \frac{1}{16}(iw_5-w_6+w_7+iw_8), \qquad b_3 := \frac{1}{8}(iw_1-w_2),\\
    b_4 := \frac{1}{8}(-w_3+iw_4), \qquad b_5 := -\frac{1}{4}w_0.
    \end{gathered}
\end{equation}
The functions $B$ and $C$, as introduced in Theorem \ref{thm:finite_sys}, are defined as,
\begin{equation}
    \label{eq:coeff_B_C}
         B(r)
        := b_1 r + b_2 \bar{r} + b_3
        \qquad
        \text{and}
        \qquad
        C(r)
        := b_4 r + \overline{b_4} \bar{r} + b_5.
    \end{equation}
Finally, the definitions of $F$ and $G$ are derived from $B$ and $C$ as in~\eqref{eq:F_and_G}
\begin{equation}
F(\alpha, r) := 2i \left( B(r) \alpha^2 + C (r) \alpha + \overline{B} (r)\right), \qquad G(\eta, r) := 2\left(2 \re{B(r) \eta} + C(r) \right),
\end{equation}

\section{Proof Details for Section~\ref{sec:2}}

In this appendix, we provide the technical proofs for the results presented in Section~\ref{sec:2}.

\subsection{Proof of Lemma \ref{lem:model_gen_h_WGF}}
\begin{lemma}\label{lem:model_gen_h_WGF}
Consider the model
\begin{equation}\label{eq:model_with_gen_h}
    \Dot{x}_k(t) = \P^{\perp}_{x_k(t)} \left( \frac{1}{n} \sum_{j=1}^n h(\beta \langle x_k(t), A x_j(t) \rangle) V x_j(t) \right), \qquad k=1, \dots, n \, .
\end{equation}
Suppose that the attention kernel $h$ admits an anti-derivative $H$, and that the matrices $A$ and $V$ satisfy
\begin{equation*}
    A^{\top} = A, \quad V = cA,
\end{equation*}
for some constant $c \in \R$.
Then \eqref{eq:model_with_gen_h} can be written as a Wasserstein gradient flow associated with the energy functional
\begin{equation*}
    E[\mu] := -\frac{c}{2\beta}\int_{\BBS^{d-1}} \int_{\BBS^{d-1}} H(\beta \langle x, A y \rangle) d\mu(x) d\mu(y) \, .
\end{equation*}
\end{lemma}

\begin{proof}
Let
\[
    W(x,y):=H(\beta\langle x,Ay\rangle).
\]
Since $A^\top=A$, one has $W(x,y)=W(y,x)$, which fits the classical framework of gradient flows of symmetric interaction energies. Hence, for
\[
    E[\mu]
    =
    -\frac{c}{2\beta}
    \iint W(x,y)\,d\mu(x)d\mu(y),
\]
the standard first-variation formula for symmetric interaction energies gives
\[
    \frac{\delta E}{\delta\mu}(x)
    =
    -\frac{c}{\beta}
    \int W(x,y)\,d\mu(y).
\]
Therefore, since $H'=h$,
\[
    \nabla_{\mathbb S^{d-1}}
    \frac{\delta E}{\delta\mu}(x)
    =
    -cP_x^\perp
    \int h(\beta\langle x,Ay\rangle)Ay\,d\mu(y).
\]
Consequently, the characteristic velocity of the Wasserstein gradient flow is
\[
    v(x)
    =
    -\nabla_{\mathbb S^{d-1}}
    \frac{\delta E}{\delta\mu}(x)
    =
    cP_x^\perp
    \int h(\beta\langle x,Ay\rangle)Ay\,d\mu(y)=P_x^\perp
    \int h(\beta\langle x,Ay\rangle)Vy\,d\mu(y).
\]
Taking $\mu_t^n=\frac1n\sum_{j=1}^n\delta_{x_j(t)}$ yields exactly \eqref{eq:model_with_gen_h}.
\end{proof}

\subsection{Proof of Theorem \ref{thm:finite_sys}}\label{app:finite_sys}

\begin{proof}[Proof of Theorem \ref{thm:finite_sys}]
For the matrices $A$ and $V$ as defined in \eqref{eq:matrices_A_V_general}, and
for each $k \in \{1,\dots,n\}$ we can first rewrite the dynamics~\eqref{eq:LSA:1d} satisfied by the angle $\theta_k(t)$ as
\begin{equation}\label{eq:dyn_theta_i_origin}
\begin{aligned}
    \Dot{\theta}_k &= -\frac{1}{n} \sum_{j=1}^n  \left(a_{11} \cos \theta_k \cos\theta_j + a_{12} \cos \theta_k \sin \theta_j + a_{21} \sin \theta_k\cos \theta_j + a_{22} \sin \theta_k \sin \theta_j\right) \\
    &\qquad \qquad \cdot \left( v_{11} \sin \theta_k \cos \theta_j + v_{12} \sin \theta_k \sin \theta_j - v_{21} \cos \theta_k \cos \theta_j - v_{22} \cos \theta_k \sin \theta_j \right).
\end{aligned}
\end{equation}
To obtain Equation~\eqref{eq:dyn_theta_i_origin}, recall that $x_k = (\cos\theta_k,\sin\theta_k)^\top$ and analogously for $x_j$.
Hence,
    \begin{equation*}
        \langle x_k, A x_j \rangle = a_{11} \cos \theta_k \cos\theta_j + a_{12} \cos \theta_k \sin \theta_j + a_{21} \sin \theta_k \cos \theta_j + a_{22} \sin \theta_k \sin \theta_j.
    \end{equation*}
    Since $V$ is real-valued, we can compute
    \begin{equation*}
    \begin{aligned}
        \frac{1}{\sin \theta_k} \langle V x_j, e_1-\langle x_k, e_1\rangle x_k \rangle
        &= v_{11} \sin \theta_k \cos \theta_j + v_{12} \sin \theta_k \sin \theta_j - v_{21} \cos \theta_k \cos \theta_j - v_{22} \cos \theta_k \sin \theta_j,
    \end{aligned}
    \end{equation*}
    which concludes the computations for Equation~\eqref{eq:dyn_theta_i_origin}.

Next we reorganize the different summands in~\eqref{eq:dyn_theta_i_origin},
which requires us to use the quantities defined in~\eqref{eq:w_i}.   After expanding the product on the right-hand side of~\eqref{eq:dyn_theta_i_origin} by using the trigonometric identities
$\sin \theta \cos \theta = \frac{1}{2}\sin 2\theta$, $(\sin \theta)^2 = \frac{1}{2}(1-\cos 2\theta)$, and $(\cos \theta)^2 = \frac{1}{2}(1+\cos 2\theta)$,
we arrive for~\eqref{eq:dyn_theta_i_origin} at
    \begin{align}
        \label{eq:dyn_theta_i_origin_expanded}
        \Dot{\theta}_k
        &=-\frac{1}{n} \sum_{j=1}^n \bigg(\Big(v_{11}a_{11} \sin \theta_k \cos \theta_k (\cos\theta_j)^2 + v_{11}a_{12} \sin \theta_k \cos \theta_k\cos \theta_j \sin \theta_j \notag\\
        &\qquad \qquad \qquad + v_{11}a_{21} (\sin \theta_k)^2 (\cos \theta_j)^2 + v_{11}a_{22} (\sin \theta_k)^2 \cos \theta_j \sin \theta_j \Big) \notag\\
        &\qquad \qquad + \Big(v_{12}a_{11} \cos \theta_k\sin \theta_k \sin \theta_j \cos\theta_j + v_{12}a_{12} \cos \theta_k \sin \theta_k (\sin \theta_j)^2 \notag \\
        &\qquad \qquad \qquad + v_{12}a_{21} (\sin \theta_k)^2 \sin \theta_j\cos \theta_j + v_{12}a_{22} (\sin \theta_k)^2 (\sin \theta_j)^2\Big) \notag \\
        &\qquad \qquad - \Big(v_{21}a_{11} (\cos \theta_k)^2 (\cos \theta_j)^2 + v_{21}a_{12} (\cos \theta_k)^2 \cos \theta_j\sin \theta_j \notag\\
        &\qquad \qquad \qquad + v_{21}a_{21} \sin \theta_k\cos \theta_k (\cos \theta_j)^2 + v_{21}a_{22} \sin \theta_k\cos \theta_k \cos \theta_j \sin \theta_j\Big) \notag\\
        &\qquad \qquad - \Big(v_{22}a_{11} (\cos \theta_k)^2 \sin \theta_j \cos\theta_j  + v_{22}a_{12} (\cos \theta_k)^2 (\sin \theta_j)^2 \notag \\
        &\qquad \qquad \qquad + v_{22}a_{21} \sin \theta_k\cos \theta_k \sin \theta_j\cos \theta_j + v_{22}a_{22} \sin \theta_k \cos \theta_k (\sin\theta_j)^2\Big) \bigg) \notag \\
        &=-\frac{1}{4}\Big(v_{11}a_{21} +v_{12}a_{22} - v_{21}a_{11} - v_{22}a_{12} \Big)
        \\ 
        &\quad\,-\frac{1}{4}\Big((v_{11}a_{11}+v_{12}a_{12} - v_{21}a_{21} - v_{22}a_{22}) \sin 2 \theta_k  + (-v_{11}a_{21}- v_{12}a_{22} - v_{21}a_{11} - v_{22}a_{12}) \cos 2\theta_k\Big) \notag \\
        &\quad-\frac{1}{4n} \sum_{j=1}^n \Big((v_{11}a_{21} - v_{12}a_{22}-v_{21}a_{11} + v_{22}a_{12}) \cos 2\theta_j + (v_{11}a_{22}+v_{12}a_{21}-v_{21}a_{12} - v_{22}a_{11}) \sin 2\theta_j \Big) \notag\\
        &\quad-\frac{1}{4n} \sum_{j=1}^n \left((v_{11}a_{11}- v_{12}a_{12} - v_{21}a_{21} + v_{22}a_{22}) \sin 2 \theta_k \cos 2\theta_j\right) \notag\\
        &\qquad \qquad + \left((v_{11}a_{12}+v_{12}a_{11}-v_{21}a_{22}-v_{22}a_{21}) \sin 2 \theta_k \sin 2\theta_j\right) \notag\\
        &\qquad \qquad - \left((v_{11}a_{21} -v_{12}a_{22} + v_{21}a_{11}- v_{22}a_{12}) \cos 2\theta_k \cos 2\theta_j\right) \notag \\
        &\qquad \qquad - \left((v_{11}a_{22} + v_{12}a_{21} + v_{21}a_{12} + v_{22}a_{11}) \cos 2\theta_k \sin 2\theta_j \right) \notag\\
        &= \bigg(-\frac{1}{4}\left(w_0 + w_1 \sin 2 \theta_k + w_2 \cos 2\theta_k\right) -\frac{1}{4n} \sum_{j=1}^n \left(w_3 \cos 2\theta_j + w_4 \sin 2\theta_j\right) \bigg) \notag\\
        &\quad-\frac{1}{4n} \sum_{j=1}^n \left(w_5 \sin 2 \theta_k \cos 2\theta_j + w_6 \sin 2 \theta_k \sin 2\theta_j - w_7 \cos 2\theta_k \cos 2\theta_j - w_8 \cos 2\theta_k \sin 2\theta_j \right) := T_1 + T_2. \notag
    \end{align}
    Here, $\{w_i\}_{i=0}^8$ are defined in~\eqref{eq:w_i}. Thanks to Euler's identity, we observe that
    \begin{equation*}\label{eq:order_para_finite_sum}
        \RE{\CR_2^n}
        = \frac{1}{n} \sum_{j=1}^n \cos 2\theta_j
        \qquad\text{and}\qquad
        \IM{\CR_2^n}
        = \frac{1}{n} \sum_{j=1}^n \sin 2\theta_j.
    \end{equation*}
    We now simplify each of the terms $T_1$ and $T_2$.
    
    \textbf{Term $T_1$.}
    For term $T_1$, we have
    \begin{equation*}
    \begin{split}
        T_1
        &= -\frac{1}{4}\left(w_0 + w_1 \sin 2 \theta_k + w_2 \cos 2\theta_k\right) -\frac{1}{4} \left(w_3 \left(\frac{1}{n} \sum_{j=1}^n\cos 2\theta_j\right) + w_4 \left(\frac{1}{n} \sum_{j=1}^n\sin 2\theta_j\right)\right)\\
        &= -\frac{1}{4}\left(w_0 + w_1 \im{e^{2i\theta_k}} + w_2 \re{e^{2i\theta_k}}\right) -\frac{1}{4} \left(w_3 \RE{\CR_2^n} + w_4 \IM{\CR_2^n}\right).
    \end{split}
    \end{equation*}
    Recalling that $\re{z} = \frac{1}{2}(z+\bar{z})$ and $\im{z} = \frac{1}{2i}(z-\bar{z})$,
    we obtain
    \begin{equation*}
    \begin{split}
        T_1
        &= -\frac{1}{4}\left(w_0 + w_1 \frac{e^{2i\theta_k}-e^{-2i\theta_k}}{2i} + w_2 \frac{e^{2i\theta_k}+e^{-2i\theta_k}}{2}\right) -\frac{1}{4} \left(w_3 \frac{\CR_2^n+\overline{\CR_2^n}}{2} + w_4 \frac{\CR_2^n-\overline{\CR_2^n}}{2i}\right) \\
        &= b_3e^{2i\theta_k} + \overline{b_3}e^{-2i\theta_k} + C(\CR_2^n),
    \end{split}
    \end{equation*}
    with $b_3$ and $C(r)$ as in~\eqref{eq:coeff_b} and~\eqref{eq:coeff_B_C}.
    
    \textbf{Term $T_2$.}
    For term $T_2$, we have
    \begin{equation*}
    \begin{split}
        T_2
        &= -\frac{1}{4n} \sum_{j=1}^n \left(w_5 \sin 2 \theta_k \cos 2\theta_j + w_6 \sin 2 \theta_k \sin 2\theta_j - w_7 \cos 2\theta_k \cos 2\theta_j - w_8 \cos 2\theta_k \sin 2\theta_j \right)\\
        &= -\frac{1}{4} \left( w_5 \im{e^{2i\theta_k}} \re{\CR_2^n} + w_6 \im{e^{2i\theta_k}} \im{\CR_2^n} - w_7 \re{e^{2i\theta_k}} \re{\CR_2^n} -w_8 \re{e^{2i\theta_k}} \im{\CR_2^n} \right) .
    \end{split}
    \end{equation*}
    By applying the real and complex identities, we obtain
    \begin{equation*}
    \begin{aligned}
    T_2 &= -\frac{1}{4} \bigg( w_5 \frac{e^{2i\theta_k} - e^{-2i\theta_k}}{2i} \frac{\CR_2^n + \overline{\CR_2^n}}{2} + w_6 \frac{e^{2i\theta_k} - e^{-2i\theta_k}}{2i} \frac{\CR_2^n - \overline{\CR_2^n}}{2i}\\
    &\qquad \quad \; \; - w_7 \frac{e^{2i\theta_k} + e^{-2i\theta_k}}{2} \frac{\CR_2^n + \overline{\CR_2^n}}{2} - w_8 \frac{e^{2i\theta_k} + e^{-2i\theta_k}}{2} \frac{\CR_2^n - \overline{\CR_2^n}}{2i}\bigg)\\
    &= \frac{1}{16} e^{2i\theta_k} \left( iw_5 \left( \CR_2^n + \overline{\CR_2^n} \right) + w_6 \left( \CR_2^n - \overline{\CR_2^n} \right) + w_7 \left( \CR_2^n + \overline{\CR_2^n} \right) - iw_8 \left( \CR_2^n - \overline{\CR_2^n} \right)\right)\\
    &\quad + \frac{1}{16} e^{-2i \theta_k} \left( iw_5 \left( \CR_2^n + \overline{\CR_2^n} \right) - w_6 \left( \CR_2^n - \overline{\CR_2^n} \right) - w_7 \left( \CR_2^n + \overline{\CR_2^n} \right) + iw_8 \left( \CR_2^n - \overline{\CR_2^n} \right)\right)\\
    &= B_2(\CR_2^n)e^{2i\theta_k} + \overline{B_2}(\CR_2^n)e^{-2i\theta_k} 
    \end{aligned}
    \end{equation*}
    with
    \begin{equation*}
        B_2(r)
        := b_1 r + b_2 \bar{r}.
    \end{equation*}
    \textbf{Conclusion.} Combining now the expressions for $T_1$ and $T_2$
    and inserting them back into \eqref{eq:dyn_theta_i_origin_expanded} yields for \eqref{eq:dyn_theta_i_origin} the equivalent formulation
    \begin{equation*}
        \dot \theta_k
        = B(\CR_2^n)e^{2i\theta_k} + \overline{B}(\CR_2^n)e^{-2i\theta_k} + C(\CR_2^n),
    \end{equation*}
    where $B(r)$ and $C(r)$ are defined in~\eqref{eq:coeff_B_C}.
\end{proof}

\subsection{Proof Details for Section \ref{subsec:WS_OA}}
\label{app:WS_OA}

\begin{proof}[Proof of Proposition \ref{prop:WS_transform}]
    The proof follows computations similar to those presented in \cite{pikovsky2015dynamics,Gong_2019}.
    For notational simplicity, let us abbreviate $B(\CR_2^n)$ and $C(\CR_2^n)$ as $B$ and $C$, respectively.
    
    Denoting $z_k := e^{i\xi_k}$, we compute with chain rule and using the dynamics \eqref{eq:xi_finite_sys} that
    \begin{equation}\label{eq:riccati_eq}
    \begin{aligned}
    \frac{d}{dt} z_k
    &= i \Dxi_k z_k = 2i \left( B z_k^2 + C z_k + \overline{B} \right)
    \end{aligned}
    \end{equation}
    for all $k \in \{1,\dots,n\}$.
    Observe that each $z_k$ solves a Riccati equation (with the same coefficients for all $k$) and satisfies the constraint $|z_k(t)| = 1$ for all $t \geq 0$.
    Hence, $z_k$ can be represented as
    \begin{equation}\label{eq:theta_vartheta_tf}
    e^{i\xi_k} = z_k = \frac{\alpha^n + e^{i\vartheta_k}}{1 + \overline{\alpha^n} e^{i\vartheta_k}} 
    \end{equation}
    with $\alpha^n: [0, \infty) \mapsto \BBC$ satisfying $|\alpha^n| < 1$ and $\vartheta_k \in \BBT$ for $k \in \{1,\dots,n\}$.
    For notational simplicity, we shorthand $\alpha^n$ by $\alpha$, and denote $u_k := e^{i \vartheta_k}$.
    
    Expanding the left-hand side of~\eqref{eq:riccati_eq}, by plugging in the transformation \eqref{eq:theta_vartheta_tf}, we get
    \begin{equation*}
    \begin{split}
    \frac{d}{dt} z_k
    &= \frac{d}{dt} \left( \frac{\alpha + u_k}{1 + \alphaBar u_k }\right) = \frac{\Dalpha + i \Dvartheta_k u_k}{1 + \alphaBar u_k} - \frac{\left(\alpha + u_k \right) \left( \DalphaBar u_k + i \Dvartheta_k u_k \alphaBar \right)}{\left( 1 + \alphaBar u_k \right)^2}\\
    &= \frac{1}{\left( 1 + \alphaBar u_k \right)^2} \left( \Dalpha + \left( \Dalpha \alphaBar - \alpha \DalphaBar + i \Dvartheta_k (1- |\alpha|^2)  \right) u_k - \DalphaBar u_k^2 \right).
    \end{split}
    \end{equation*}
    Likewise, the right-hand side of \eqref{eq:riccati_eq} becomes
    \begin{equation*}
    \begin{split}
    2i \left( B z_k^2 + C z_k + \overline{B} \right)
    &= 2i \left( B \left( \frac{\alpha + u_k}{1 + \alphaBar u_k} \right)^2 + C \frac{\alpha + u_k}{1 + \alphaBar u_k} + \overline{B} \right)\\
    &= \frac{2i}{\left( 1 + \overline{\alpha} u_k \right)^2} \bigg( \left( B \alpha^2 +  C \alpha + \overline{B} \right) + \left(4 \re{B \alpha} + C \left( 1 + |\alpha|^2 \right) \right) u_k + \left( B + C \alphaBar  + \overline{B} \alphaBar^2 \right) u_k^2 \bigg) .
    \end{split}
    \end{equation*}
    By combining the constant terms as well as the coefficients of $u_k, u_k^2$, we obtain the condition
    \begin{equation}\label{eq:LHS_eq_RHS}
    K_1 + K_2 u_k - \overline{K_1} u_k^2 = 0
    \end{equation}
    for all $k \in \{1,\dots,n\}$, where the coefficients are defined as
    \begin{equation*}
    K_1 := \Dalpha - 2i \left( B \alpha^2 + C \alpha + \overline{B} \right), \quad K_2 :=  \Dalpha \alphaBar - \alpha \DalphaBar + i \Dvartheta_k (1 - |\alpha|^2) - 2i \left( 4\re{B\alpha} + C(1 + |\alpha|^2) \right) .
    \end{equation*}

    We now choose $\alpha$ and $\vartheta_k$ so that the identity \eqref{eq:LHS_eq_RHS} holds for all $u_k$, for which it suffices to set the coefficients to zero, i.e., $K_1 = K_2 = 0$.
    
    From $K_1 = 0$, we obtain the dynamics that $\alpha$ needs to satisfy.
    Specifically,
    \begin{equation}
        \label{eq:proof:prop:WS_transform:aux18}
    \Dalpha = 2i \left( B \alpha^2 + C \alpha + \overline{B} \right) .
    \end{equation}
    From $K_2 = 0$, we deduce using \eqref{eq:proof:prop:WS_transform:aux18} that
    \begin{equation}
        \label{eq:proof:prop:WS_transform:aux20}
    \Dvartheta_k = 2 \left( 2\re{B \alpha} + C  \right)
    \end{equation}
    for all $k \in \{1,\dots,n\}$.
    Noticing that the right-hand side of \eqref{eq:proof:prop:WS_transform:aux20} is independent from the particle index $k$ implies that all angles $\{\vartheta_k\}_{k=1}^n$ rotate at the same speed.
    Therefore, we can introduce a new time-dependent but particle-independent parameter $\eta$ which has the same rotational speed as all $\{\vartheta_k\}$'s, i.e., $\Dvartheta_k=\Deta$, and define
    \begin{equation*}
        \eta(t) = \vartheta_k (t) - \varphi_k
    \end{equation*}
    for all $k \in \{1,\dots,n\}$,
    where $\{\varphi_k\}_{k=1}^n$ are the constants of motion.
    
    We thus arrive at the family of Möbius transformations as defined in \eqref{eq:Mobius_tf}, mapping constants $\varphi_k$ to phases $\xi_k$ according to
    \begin{equation*}
    M_{\alpha, \eta}: \varphi_k \mapsto \xi_k, \qquad e^{i\xi_k} = M_{\alpha,\eta} (e^{i\varphi_k}) = \frac{\alpha + e^{i(\varphi_k +\eta)}}{ 1 + \alphaBar e^{i(\varphi_k +\eta)}},
    \end{equation*}
    with the WS variables $\alpha$ and $\eta$ following the coupled ODE system given by
    \begin{equation*}
    \dot{\alpha} = 2 i \left( B \alpha^2 + C  \alpha + \overline{B} \right), \qquad
    \dot{\eta} = 2 \left(2 \re{B \alpha} + C \right) .
    \end{equation*}
    Moreover, when the parameter $\alpha$ has modulus less than one, the Möbius transformation is a one-to-one mapping and thus admits a corresponding inverse transformation, which can be explicitly written as in \eqref{eq:inverse_M}.
\end{proof}


\begin{proof}[Proof of Lemma~\ref{lem:WS_to_OA}]
We fix $t \geq 0$ and suppress the time dependence for notational simplicity.
Let $\alpha \in \C $ be such that $|\alpha|<1$.
Since $T_{\alpha, \eta}$ is an one-to-one map with $|\alpha| < 1$, it suffices to prove that the pushforward of the uniform distribution under $T_{\alpha,\eta}$ is the wrapped Cauchy distribution $\WC(\dummy;\alpha)$. 

Suppose $\nu = \mathrm{Unif}([0,2\pi))$, and let $g(\xi)$ denote the density of the distribution $T_{\alpha, \eta \sharp} \left(\mathrm{Unif}([0,2\pi)) \right)$.
Then, using the definitions of pushforward and
the of the map~$T_{\alpha, \eta}$, we obtain, for all $m \in \mathbb{Z}$,
\begin{equation*}
\begin{split}
    \int_{0}^{2\pi} g(\xi) e^{im \xi} d\xi
    &= \frac{1}{2\pi}\int_0^{2\pi} e^{im T_{\alpha, \eta}(\varphi)} d \varphi
    = \frac{1}{2\pi}\int_0^{2\pi} (e^{i T_{\alpha, \eta}(\varphi)})^m d \varphi
    =\frac{1}{2\pi}\int_0^{2\pi} \left( \frac{\alpha + e^{i(\varphi + \eta)}}{1 + \bar{\alpha} e^{i( \varphi + \eta)}} \right)^m d \varphi \\
    &= \frac{-i}{2\pi} \oint_{\S^1} \left( \frac{\alpha + z}{1 + \bar{\alpha} z} \right)^m \frac{1}{z} d z
    = \alpha^m,
\end{split}  
\end{equation*}
where we wrote the second to last integral as a line integral in the complex plane (with the change of variables $z=e^{i(\varphi + \eta)}$), and used the residue theorem to obtain the last equality. The above allows us to write $g$ in Fourier basis as
\begin{equation*}
\begin{split}
    g(\xi)
    & = \frac{1}{2\pi} \sum_{m=0}^{\infty} \left(\alpha^m e^{-im\xi}  + \bar{\alpha}^m e^{i m \xi }\right)
    = \frac{1}{2\pi} \bigg( 1+ 2\sum_{m=1}^\infty \mathrm{Re}\left( \bar{\alpha}^m e^{im\xi} \right)  \bigg) 
    = \frac{1}{2\pi} \mathrm{Re}\left( \frac{1+ \bar{\alpha} e^{i \xi} }{1-\bar{\alpha}e^{i\xi} } \right) = \frac{1}{2\pi} \frac{1 - |\alpha|^2}{|e^{i\xi} - \alpha|^2},
\end{split}
\end{equation*}
which is precisely the formula for the density of $\WC(\dummy; \alpha)$ stated in \eqref{eq:WC_dis}.
Note that to get the third equality we used the formula for the geometric series.
\end{proof}

\section{Proof Details for Section~\ref{sec:case_study}}\label{app:case_study}

In this appendix, we provide the technical details for the proofs of Section~\ref{sec:case_study}.

\subsection{Proofs of Main Results}

\subsubsection*{Case 1.}

\begin{proof}[Proof of Theorem \ref{thm:case_1}]
Recall from \eqref{eq:OAcase1} that the dynamics for the order parameter $\CR_2$ are given by
\begin{equation}\label{eq:case1_CR2_dyn_app}
\Dot{\CR}_2 = \frac{1}{4} \left( 1 - |\CR_2|^2 \right) \left( \left(\aDPlus + i\aOMinus \right) \CR_2 + \left( \aDMinus + i\aOPlus \right) \right).
\end{equation}
For notational convenience, we set
\begin{equation}\label{eq:case1_convenience_notation}
z(t) := \CR_2(t), \qquad a:=\aDPlus,\qquad b:=\aOMinus,\qquad c:=\aDMinus+i\aOPlus,\qquad R:=|c|, \qquad \PhiA := \Arg(c) .
\end{equation}
We set $\PhiA = 0$ if $R = 0$.
We also denote the closed unit disk and its boundary by
\[
    D := \{w : |w| \leq 1\}, \qquad \partial D := \{ w : |w| = 1 \}.
\]
We first consider the general case where $a \neq 0$ or $b \neq 0$. 
The equation of $z$ can be written as
\begin{equation}
\label{eqn:case1_complex}
\dot z=\frac{1}{4} (a+ib) \bigl(1-|z|^2\bigr) (z - z_\mathrm{eq}),\qquad z(0)=z_0, 
\end{equation}
where 
\begin{equation}\label{eq:zeq}
z_\mathrm{eq} := -\frac {c}{a + ib} .
\end{equation}

\paragraph{Two important quantities and associated observations.}
It is straightforward to verify that the equilibrium set of the dynamics~\eqref{eqn:case1_complex} is $\partial D \cup \{ z_{\mathrm{eq}} \}$.
To understand the asymptotic behavior of $z(t)$, we examine how the trajectory $z(t)$ moves relative to the interior equilibrium $\zeq$, and how it may interact with the boundary equilibrium set $\partial D$.

To this end, we introduce two quantities associated with the displacement $z(t) - \zeq$.
The first is the energy function
\begin{equation}\label{eq:case1_energy}
    \CE(t) = |z - z_\mathrm{eq}|^2,
\end{equation}
which measures the squared distance between $z(t)$ and $\zeq$.
A direct computation gives
\begin{equation}
    \label{eq:case1_lyapunov_derivative}
    \dot{\CE} = \frac 1 2 a(1 - |z|^2) \CE .
\end{equation}  
Consequently, integrating above equation yields
\begin{equation}\label{eq:energy_CE}
\CE(t) = \CE(0) \exp \left\{ \frac{a}{2} \int_0^t (1 - |z(s)|^2) \, ds \right \} \, .
\end{equation}

An important observation is that the sign of $a$ determines the monotonicity of the energy function $\CE(t)$, and hence describes how the trajectory $z(t)$ moves relative to $\zeq$.
More precisely,
\begin{itemize}
    \item when $a>0$, $\dot{\CE} > 0$, and thus the trajectory $z(t)$ is pushed away from $z_\mathrm{eq}$;

    \item when $a = 0$, $\dot{\CE} = 0$, and thus the distance between $z(t)$ and $z_\mathrm{eq}$ is preserved;

    \item when $a < 0$, $\dot{\CE} < 0$, and thus the trajectory $z(t)$ is drawn toward $z_\mathrm{eq}$.
\end{itemize}
Thus, in each of these three regimes, the long-term behavior of $z(t)$ is determined by how the trajectory moves relative to $\zeq$ and whether it reaches the boundary $\partial D$.
Roughly speaking, $z(t)$ converges to a boundary equilibrium point if the trajectory  is pushed away from $\zeq$, or if, while cycling around or moving toward $\zeq$, it reaches the boundary $\partial D$.
By contrast, if the trajectory remains entirely inside the unit disk $D$, then $z(t)$ either keeps cycling around $\zeq$ or converges to it; see Figure \ref{fig:z_trajectory} as an illustrative example.
A more detailed and rigorous characterization is provided in Claims $1,2$, and $3$ below.

\begin{figure}[!htb]
\centering
\newlength{\subfigimgheightL}
\setlength{\subfigimgheightL}{0.2\textheight}

\begin{subfigure}[t]{.33\textwidth}
  \centering
  \includegraphics[
    width=\linewidth,
    height=\subfigimgheightL,
    keepaspectratio,
    valign=b
  ]{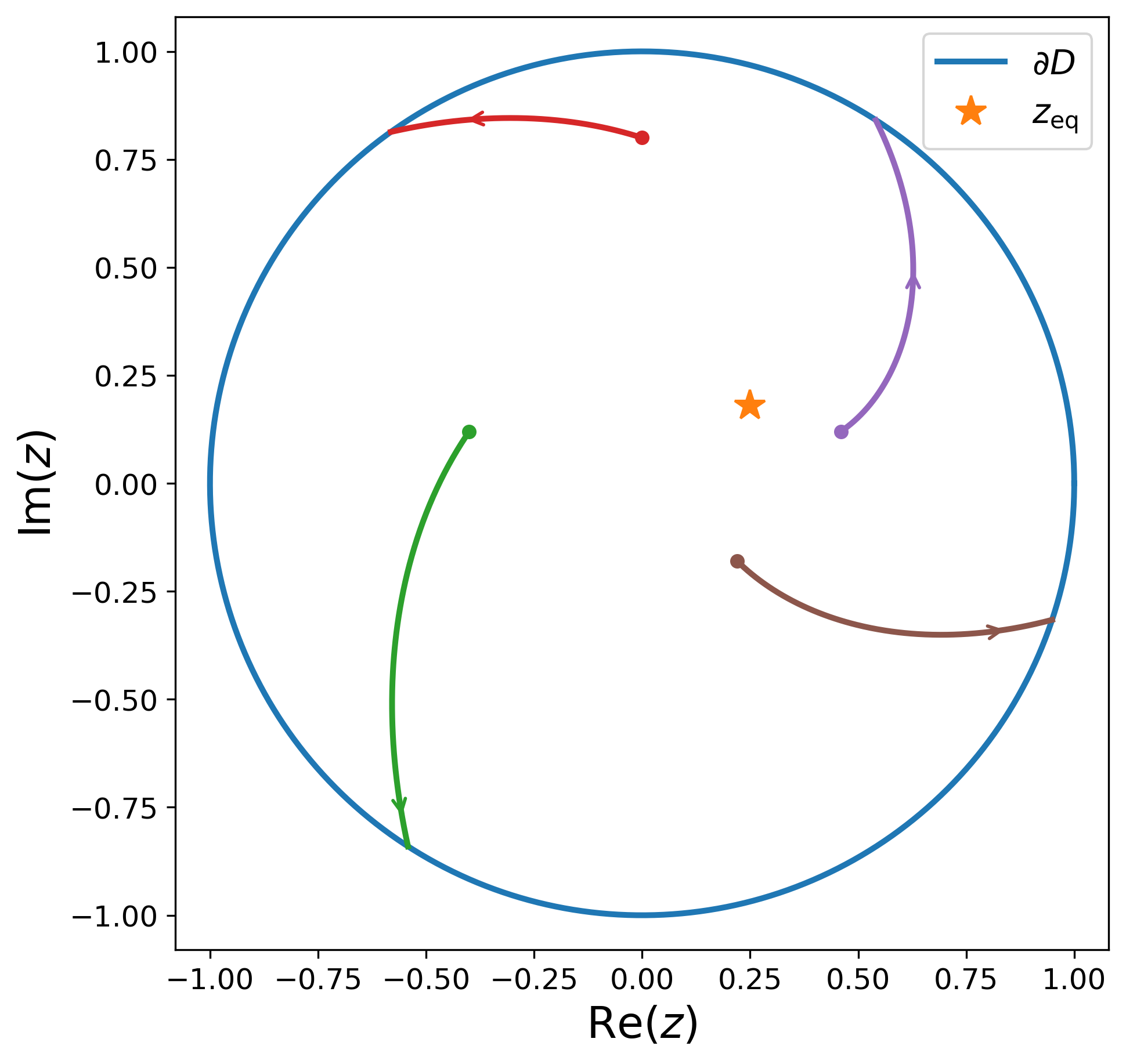}
  \caption{$z(t)$ is pushed away from $\zeq$.}
  \label{subfig:case1_a_positive}
\end{subfigure}%
\hfill
\begin{subfigure}[t]{.33\textwidth}
  \centering
  \includegraphics[
    width=\linewidth,
    height=\subfigimgheightL,
    keepaspectratio,
    valign=b
  ]{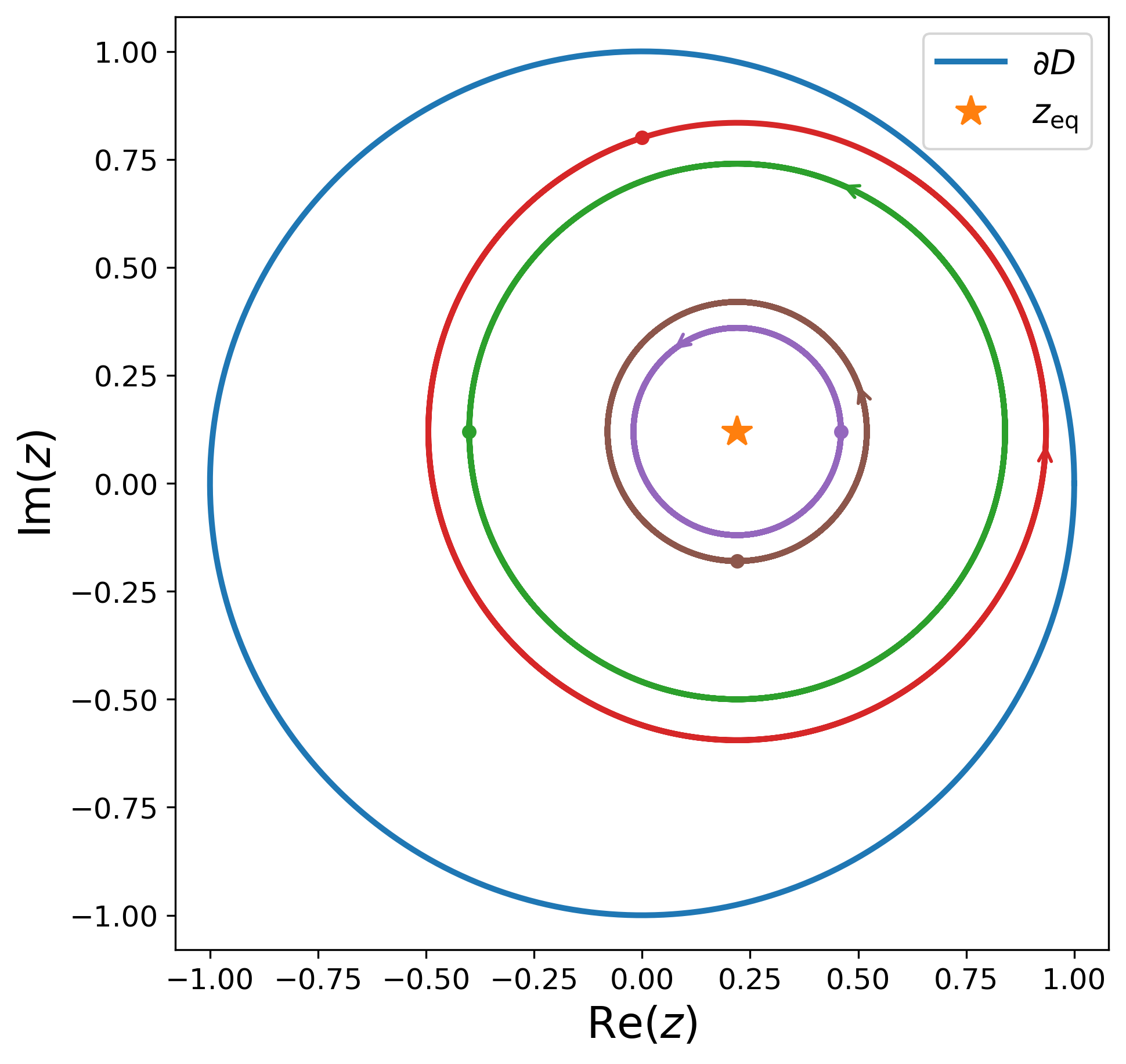}
  \caption{$z(t)$ cycles around $\zeq$.}
  \label{subfig:b}
\end{subfigure}%
\hfill
\begin{subfigure}[t]{.33\textwidth}
  \centering
  \includegraphics[
    width=\linewidth,
    height=\subfigimgheightL,
    keepaspectratio,
    valign=b
  ]{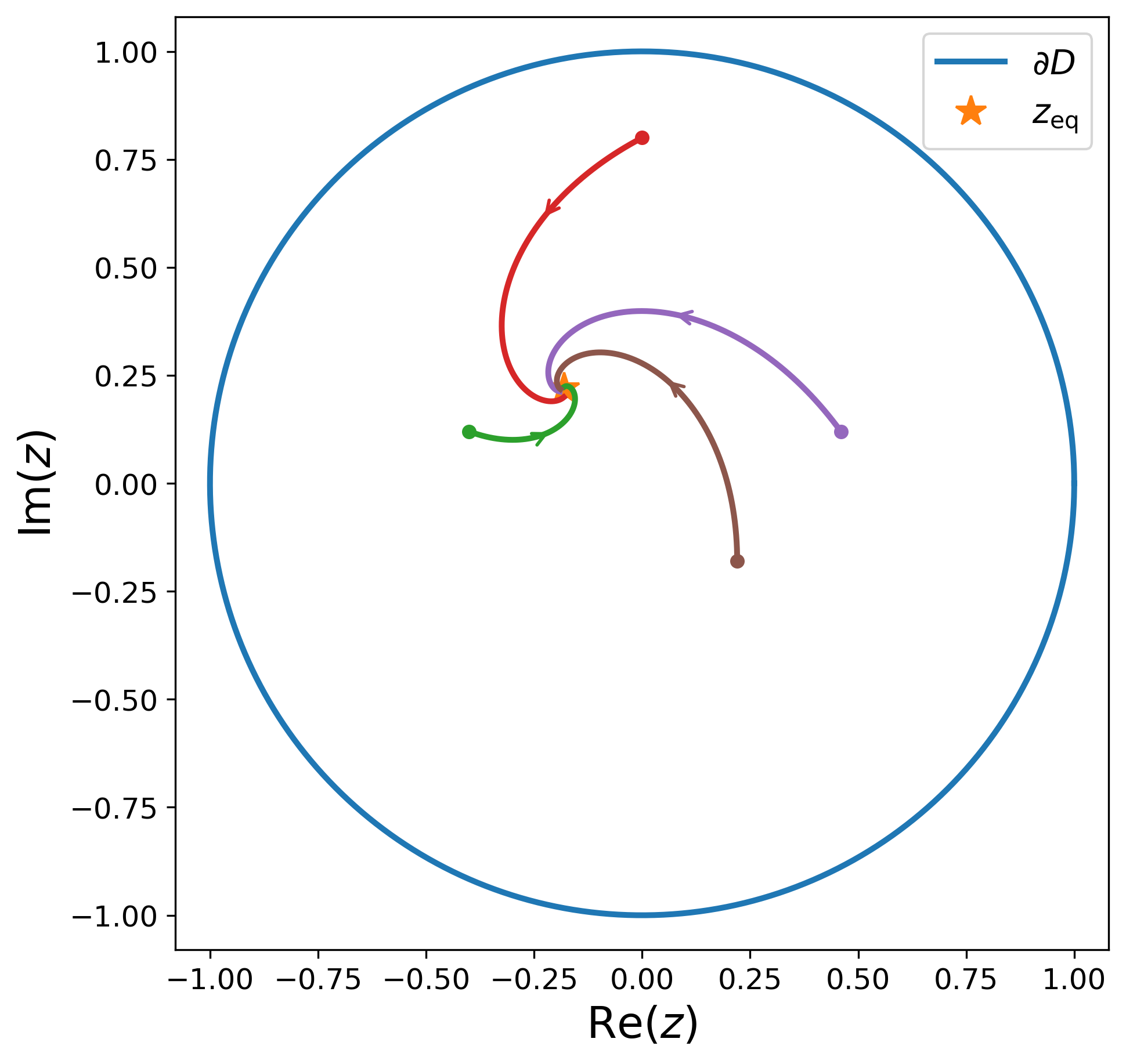}
  \caption{$z(t)$ moves toward to $\zeq$.}
  \label{subfig:a}
\end{subfigure}

\caption{The visualization of the trajectory $z(t)$ relative to the interior equilibrium $\zeq$ and the boundary equilibrium set $\partial D$.
The solid dot is the initialization point, while the arrow indicates the moving direction of $z(t)$.
}
    \label{fig:z_trajectory}
\end{figure}

The second important quantity is the angular motion of the displacement $z(t) - z_\mathrm{eq}$, that is,
\begin{equation}\label{eq:z_polar}
     \zeta(t) := \Arg(z(t) - \zeq) \, .
\end{equation} 
This quantity describes how the trajectory $z(t)$ rotates around $\zeq$, and will be useful for 
determining whether the angular component of the trajectory converges.

A direct computation shows that the angle $\zeta$ satisfies
\begin{equation}
    \label{eqn:case1-angle-dynamics}
        \dot{\zeta} = \frac{b} 4 (1 - |z|^2).
\end{equation}
Integrating in time yields
\begin{equation}
    \label{eqn:case1-angle-integral}
    \zeta(t) = \zeta(0) + \frac{b}{4}\int_0^t (1-|z(s)|^2)\,ds.
\end{equation}
Thus, the convergence of $\zeta$ is determined by the following condition:
\begin{equation}\label{eq:case1-finite-integral}
    \lim_{t \rightarrow \infty} \int_0^{t} (1 - |z(s)|^2) \, ds = C < \infty \, .
\end{equation}

We now formalize the above observations by analyzing the three parameter regimes: $a > 0$, $a = 0 \ \& \ b \neq 0$, and $a < 0$.

    \textbf{Claim 1 (regime $a>0$)}: when $a > 0$, for any initial condition $z_0 \neq z_\mathrm{eq}$ and $z_0 \in D$, the solution of \eqref{eqn:case1_complex} converges to a point $z_* \in \partial D$.

\begin{proof}[Proof of Claim 1.]
    Since $\partial D$ is an equilibrium set and $z_0 \in D$, we have $z(t) \in D$ and an estimate
    \begin{equation*}
    0 \leq \CE(t) \leq 2\left(|z|^2 + |\zeq|^2\right) \leq 2 + \frac{2R}{a^2 + b^2} < \infty, \qquad \text{for all } t \geq 0 .
    \end{equation*}
    On the other hand, from \eqref{eq:energy_CE}, $a > 0$ and $\CE(t) < \infty$, we know that the condition \eqref{eq:case1-finite-integral} must hold.
    Together with \eqref{eqn:case1-angle-integral}, this implies the convergence of the angle $\zeta$ to some $\zeta_{\infty} \in \BBT^1$. 
    By \eqref{eq:z_polar}, we thus know that $\Arg(z)$ also converges to a fixed value in $\BBT^1$.
    
    To conclude that $|z| \to 1$, we note that from the dynamics of $z$ in \eqref{eqn:case1_complex}, we know that $|z(t)|$ is Lipchitz, and thus the integrand $1 - |z(t)|^2$ is Lipchitz.  
    Together with \eqref{eq:case1-finite-integral} and Barb\u{a}lat's lemma \cite{Barbalat1959}, this implies that $1 - |z(t)|^2 \to 0$.

    Combining the convergence of both modulus and argument of $z$, we conclude the proof.
\end{proof}

Here, we record an additional quantitative convergence estimate in the sub-regime $a > R$, which will be used in the proof of Theorem \ref{thm:case_1stab}.
We defer the proof until after the proof of Theorem \ref{thm:case_1}.

\begin{proposition}\label{prop:case1_quantitative}
Let $z(t)$ be the solution to \eqref{eqn:case1_complex} with initial data $z(0) \neq \zeq$, and let $z_* \in \partial D$ denote the limiting point, so that $\lim_{t \rightarrow \infty} z(t) = z_*$.
For any prescribed $\epsilon > 0$, define
\begin{equation*}
    T_{\epsilon} := \inf \left\{ t \geq 0: |z(t) - z_*| \leq \epsilon \right\} \, .
\end{equation*}
Suppose that $a > R$. 
Then there exist constants $C_1, C_2 > 0$, depending only on the matrix $A$, such that
\begin{equation*}
T_{\epsilon} \leq C_1 \log \frac{1}{|z_0 - \zeq|} + C_2 \log \frac{1}{\epsilon} \, .
\end{equation*}
Here, $\zeq$ is defined in~\eqref{eq:zeq}.
\end{proposition}

    \textbf{Claim 2 (regime $a = 0$, $b \neq 0$):} suppose $a = 0$, then  
    \[
        z(t)=z_{\mathrm{eq}}+|z_0-z_{\mathrm{eq}}|e^{i\zeta(t)}
        \qquad \text{for all } t\ge0,
    \]
    Furthermore, if 
    \begin{itemize}
        \item $|b| > R$ and $|z_0 - z_\mathrm{eq}| < \mathrm{dist}(z_\mathrm{eq}, \partial D)$, 
        then $z(t)$ keeps circling within $D \setminus \partial D$ for all $t \geq 0$.
        
        \item $0 < |b| \leq R$ or $|z_0 - z_\mathrm{eq}| \geq \mathrm{dist}(z_\mathrm{eq}, \partial D)$, then $\lim_{t \to \infty} z(t) = z_* \in \partial D$.
    \end{itemize}

    \begin{proof}[Proof of Claim 2.]
        Setting $a = 0$ in \eqref{eq:case1_lyapunov_derivative}, we obtain $\dot{\CE} = 0$ and hence the dynamics~\eqref{eqn:case1_complex} of $z$ remain on a circle $C_{\zeq}$ centered around $z_\mathrm{eq}$, that is,
        \begin{equation}\label{eq:circle_Cz}
             z(t) \in C_{\zeq} := \{ z: |z - z_{\mathrm{eq}}| = |z_0 - z_{\mathrm{eq}}| \}, \qquad \text{ for all } t \geq 0.
        \end{equation}
        Together with the polar representation~\eqref{eq:z_polar}, we have
        \begin{equation*}
        z(t) - \zeq = |z(t) - \zeq| e^{i\zeta(t)} = |z_0 - \zeq| e^{i\zeta(t)} \qquad \Longrightarrow \qquad z(t) = \zeq + |z_0 - \zeq| e^{i\zeta(t)} .
        \end{equation*}
        
        Next, we discuss the long-time behavior of $z$ under different parameter regimes of $A$ and initial condition of $z_0$.
        
        \begin{itemize}
            \item \textbf{Parameter regime:} $|b| > R$.
            
            Recall from \eqref{eq:zeq} that $\zeq = -\frac{c}{a + i b}$. 
            Together with $a=0$ and $|b| > R$, we know that $\zeq \in D \setminus \partial D$.

            In this parameter regime, we further consider two conditions on the initial data $z_0$.

            \begin{itemize}
                \item $|z_0 - z_{\mathrm{eq}}| < \mathrm{dist}(z_{\mathrm{eq}}, \partial D)$.

                Since $\zeq \in D$, we have
                \begin{equation*}
                    |z_0 - z_{\mathrm{eq}}| < \mathrm{dist}(z_{\mathrm{eq}}, \partial D) = 1 - |\zeq| .
                \end{equation*}
                Therefore, we know that there exists $\delta \in (0,1)$ such that

                \[
            |z(t)| - |z_\mathrm{eq}| \leq |z(t) - z_\mathrm{eq}| \leq 1 - \delta - |z_\mathrm{eq}|,
        \]
        and thus
        \[
            |z(t)| < 1 - \delta, \qquad \text{ for all } t \geq 0.
        \]

        Therefore, from the dynamics~\eqref{eqn:case1-angle-dynamics} and $\delta \in (0,1)$, we have for $b \neq 0$
        \[
            |\dot{\zeta}| = \frac{|b|}{4} (1 - |z|^2) \geq \frac{|b|}{4} (2 \delta - \delta^2) > 0,
        \]
        Thus, the angular velocity is bounded from away from zero for all time and in turn implies that $\zeta(t)$ does not converge.
        Together with the fact that $z(t)$ always lies on the circle $C_{\zeq}$ as defined in \eqref{eq:circle_Cz},
        and the initial condition $|z_0 - z_{\mathrm{eq}}| < \mathrm{dist}(z_{\mathrm{eq}}, \partial D)$,
        we know that the the circle $C_{\zeq}$ strictly contained in the unit disc $D$. 
        Therefore the trajectory of $z$ rotates around $z_{\mathrm{eq}}$ within $D \setminus \partial D$ perpetually. 

        \item $|z_0 - \zeq| \geq \mathrm{dist} (\zeq, \partial D)$.

        In this case, the circle $C_{\zeq}$ intersects with the boundary of the unit disc $\partial D$.
        Recall the dynamics of angle $\zeta$ in \eqref{eqn:case1-angle-dynamics}
        and $b \neq 0$,
        thus we know that $\zeta$ is strictly monotone when $z \in D \setminus \partial D$.
        Together with $z(0) \in D \setminus \partial D$,
        this implies that the trajectory of $z(t)$ must converge to one of the points in $C_{\zeq} \cap \partial D$.
        Equivalently, we have $\lim_{t \to \infty} z(t) = z_* \in \partial D$.
            
        \end{itemize}
        \item \textbf{Parameter regime:} $0 < |b| \leq R$.
        
        Recall from \eqref{eq:zeq} that $\zeq = - \frac{c}{a + ib}$.
        Together with $a = 0$ and $0 < |b| \leq R$, we know that $\zeq \in (D \setminus \partial D)^c$.
        In this case, the circle $C_{\zeq}$ intersects with the boundary of the unit disc $\partial D$.
        Recalling the dynamics of angle $\zeta$ in \eqref{eqn:case1-angle-dynamics}
        and $b \neq 0$, thus we know that $\zeta$ is strictly monotone when $z \in D \setminus \partial D$.
        Together with $z(0) \in D \setminus \partial D$,
        this implies that the trajectory of $z(t)$ must converge to one of the points in $C_{\zeq} \cap \partial D$.
        Equivalently, we have $\lim_{t \to \infty} z(t) = z_* \in \partial D$.
        \end{itemize}
\end{proof}
    
\textbf{Claim 3 (regime $a < 0$):} For $a < 0$ and $z_0 \in D \setminus \partial D$, 
    \begin{itemize}
    \item if $a^2+b^2 \leq R^2$, then $\lim_{t\to\infty} z(t) = z_* \in \partial D$.
    
    \item if $a < -R$ or ($a^2+b^2 > R^2$ $\&$ $|z_0-z_{\mathrm{eq}}| < \operatorname{dist}(z_{\mathrm{eq}},\partial D)$), then $\lim_{t\to\infty} z(t) = z_{\mathrm{eq}}$.
    \end{itemize}
    
    \begin{proof}[Proof of Claim 3.]
        When $a<0$, recall from \eqref{eq:case1_lyapunov_derivative} that $\dot{\CE} \leq 0$, and thus $\CE$ is a Lyaponov function.  
        Since $D$ is a positive invariant compact set of the dynamics~\eqref{eqn:case1_complex} and the zero derivative set of $\CE$ coincides with the equilibrium set of \eqref{eqn:case1_complex}:
        \begin{equation*}
        \left\{ z \in D: \Dot{z} = 0 \right\} = \left\{ z \in D: \Dot{\CE} = 0 \right\} = \partial D \cup \left\{ \zeq \right\},
        \end{equation*}
        then the LaSalle invariance Principle implies that for all $z_0 \in D$, $z(t)$ converges to $\partial D \cup \left\{ \zeq \right\}$ as $t \rightarrow \infty$.
        Next, we discuss three different parameter regimes, respectively.

        \begin{itemize}
            \item \textbf{Parameter regime: $a^2 + b^2 \leq R^2$.}

            In this case, recall the definition of $\zeq$ from \eqref{eq:zeq}, we thus know that $\zeq = - \frac{c}{a + i b} \in \left( D \setminus \partial D \right)^c$.
            This implies that for all $z_0 \in D$, $z(t)$ only converges to $\partial D$, that is, $|z(t)| \rightarrow 1$ as $t \rightarrow \infty$.
            
            Next, we show that $\Arg(z)$ also converges to a fixed point in $\BBT^1$.
            We discuss it in two sub-cases.
            \begin{itemize}
            \item[1.] $z \rightarrow \zeq$.
            In this case, we by default obtain that $\Arg(z) \rightarrow \Arg(\zeq) \in \BBT^1$ as $t \rightarrow \infty$.

            \item[2.] $z \nrightarrow \zeq$.
            In this case, from \eqref{eq:energy_CE}, we know that the energy function $\CE(t) > 0$ for all $t \geq 0$.
            Together with $a < 0$, this implies the condition \eqref{eq:case1-finite-integral} holds.
            Then again from the representation~\eqref{eqn:case1-angle-integral} of $\zeta(t)$, we know that $\zeta(t) \rightarrow \zeta_{\infty} \in \BBT^1$.
            From the relationship of $z(t)$ and $\zeta(t)$ in \eqref{eq:z_polar}, we thus conclude that $\Arg(z)$ converges in $\BBT^1$ as $t \rightarrow \infty$.
        \end{itemize}
        Based on the above discussions, we can conclude that when $a < 0$ and $a^2 + b^2 \leq R^2$, the dynamics of $z(t)$ satisfies $\lim_{t \rightarrow \infty} z(t) = z_* \in \partial D$.

        \item \textbf{Parameter regime: $a < -R$.}
        
        From the LaSalle invariance principle, we have already shown that $z(t) \rightarrow \partial D \cup \{\zeq\}$.
        We now show that $z(t)$ cannot converge to $\partial D$, that is, $|z(t)| \nrightarrow 1$, and therefore conclude that $z(t) \rightarrow \zeq$.
        To this end, it is convenient to work in polar coordinates.

        Plugging the polar representation of the order parameter $z(t) = \CR_2(t) = \rho(t) e^{i\phi(t)}$ in \eqref{eqn:case1_complex}, together with $\PhiA = \arg(c)$, we obtain the dynamics of $\rho$ as
        \begin{equation*}
        \Drho = \frac{1}{4} (1 - \rho^2) (a \rho + R \cos (\phi - \PhiA)) \,  .
        \end{equation*}
        Now fix any $\rho \in (\rhoTH, 1)$, where $\rhoTH := \frac{R}{|a|} \in (0,1)$.
        Since $a < 0$, we have
        \begin{equation*}
        a \rho + R \cos (\phi - \PhiA) \leq a \rho + R < a \rhoTH + R = 0 .
        \end{equation*}
        Moreover, for $\rho \in (\rhoTH,1)$, we have $1 - \rho^2 > 0$. 
        Hence $\Drho < 0$ whenever $\rho \in (\rhoTH,1)$.
        This shows that the trajectory of $\rho(t)$ cannot approach $\rho = 1$ within the interval $(0,1)$. Equivalently, for every initial condition $z_0 \notin \partial D$, we have $z(t) \nrightarrow \partial D$.
        Combining with the conclusion from the LaSalle invariance principle, we therefore obtain $\lim_{t \rightarrow \infty} z(t) = \zeq$.

        \item \textbf{Parameter regime: $a^2 + b^2 > R^2$ $\&$ $|z_0 - \zeq| < \mathrm{dist} (\zeq, \partial D)$.}

        Since $a^2 + b^2 > R^2$, it follows from \eqref{eq:zeq} that $\zeq = -\frac{c}{a + ib} \in D \setminus \partial D$.
        Together with the initial condition on $z_0$: $|z_0 - \zeq| < \mathrm{dist} (\zeq, \partial D)$, this implies that there exists a $\delta > 0$ such that
        \begin{equation*}
        |z_0 - \zeq| \leq \mathrm{dist} (\zeq, \partial D) - \delta = 1 - |\zeq| - \delta .
        \end{equation*}
        On the other hand, by $a<0$ and \eqref{eq:case1_lyapunov_derivative}, the energy function $\CE(t) = |z(t) - \zeq|^2$ is non-increasing.
        Therefore, for all $t \geq 0$,
        \begin{equation*}
        |z(t) - \zeq| \leq |z_0 - \zeq| \leq 1 - |\zeq| - \delta \,.
        \end{equation*}
        Using the inverse triangle inequality we obtain 
        \begin{equation*}
            |z(t)| \leq |\zeq| + |z(t) - \zeq| \leq 1 - \delta < 1, \qquad \text{for all } t \geq 0 \, .
        \end{equation*}
        Hence $z(t)$ cannot converge to the boundary $\partial D$.
        Combining this with the convergence conclusion from the LaSalle invariance principle, we deduce that $\lim_{t \rightarrow \infty} z(t) = \zeq$. 
        \end{itemize}


        
    \end{proof}

Lastly, we consider a boundary case $a = b = 0$, where dynamics \eqref{eq:case1_CR2_dyn_app} of $z$ take the form
\begin{equation}\label{eq:case1_special}
\dot{z} = \frac{c}{4} \left(1 - |z|^2 \right), \qquad z(0) = z_0 .
\end{equation}
Note that $R = |c| > 0$ unless $A$ is the zero matrix, which corresponds to a trivial situation.

\textbf{Claim 4 (regime $a = 0$ and $b = 0$):} For any initial condition $z_0 \in D$, the solution of~\eqref{eq:case1_special} converges to a boundary point $z_* \in \partial D$.

\begin{proof}[Proof of Claim 4.]
If $z_0 \in \partial D$, then $|z_0| = 1$, and hence $\dot{z} = 0$.
Thus the claim is immediate.
It remains to consider $z_0 \in D \setminus \partial D$.

Rotating the complex plane by $e^{-i\PhiA}$, set $\wz := e^{-i\PhiA} z$.
Then
\begin{equation*}
    \dot{\wz} = \frac{R}{4} (1 - |\wz|^2) \, .
\end{equation*}
Write $\wz(t) = x(t) + i y(t)$, and then the dynamics of $x$ and $y$ are given by
\begin{equation*}
\Dot{x} = \frac{R}{4} \left(1 - x^2 - y^2 \right), \qquad \Dot{y} = 0 \, .
\end{equation*}
with initial condition $x(0) = x_0 := \re{\wz_0}$ and $y(0) = y_0 := \im{\wz_0}$.
Thus $y(t) \equiv y_0$, and $x$ satisfies
\begin{equation*}
\Dot{x} = \frac{R}{4} \left( \alpha^2 - x^2 \right), \qquad \alpha := \sqrt{1 - y_0^2} \, .
\end{equation*}

Since $\wz_0 \in D \setminus \partial D$, we have $x_0 \in (-\alpha, \alpha)$. 
Moreover, when $x(t) \in (-\alpha, \alpha)$,
the condition $R > 0$ implies that $\Dot{x} > 0$.
Thus $x(t)$ is strictly increasing until it reaches the equilibrium $\alpha$.
Therefore, 
\begin{equation*}
\lim_{t \rightarrow \infty} z(t) = \lim_{t \rightarrow \infty} e^{i\PhiA} \wz(t) =  \lim_{t \rightarrow \infty} e^{i\PhiA} \left(  x(t) + i y_0 \right) = e^{i\PhiA} \left( \alpha + i y_0 \right) =: z_*,
\end{equation*}
and by construction $|z_*| = 1$, so $z_* \in \partial D$.
\end{proof}

\paragraph{Conclusion for the order parameter $\CR_2$.}
Recall that $\CR_2(t) = z(t)$.
We first summarize the long-term behavior of  $\CR_2(t)$ under different parameter regimes of the matrix $A$, by simply combining the four claims established above.
We include only those regimes where the asymptotic behavior holds for all initial conditions, and omit cases that require additional assumptions on the initial data.
\begin{itemize}
    \item If the matrix $A$ satisfies $a > 0$ or $a \leq 0 \ \& \ a^2 + b^2 \leq R^2$,
    then for every initial condition $\CR_2(0) \neq \zeq$, the solution of \eqref{eqn:case1_complex} satisfies $\lim_{t \rightarrow \infty} \CR_2 (t) =  z_* \in \partial D$.

\item If the matrix $A$ satisfies $a < - R$,
then for every initial condition $\CR_2(0) \in D \setminus \partial D$, the solution of \eqref{eqn:case1_complex} satisfies $\lim_{t \rightarrow \infty} \CR_2(t) = \zeq$, where $\zeq \in D \setminus \partial D$ is defined in \eqref{eq:zeq}.
\end{itemize}

Using the notation $a = \aDPlus = \tr{A}, \ b = \aOMinus, \ R = |c| = \sqrt{(\aDMinus)^2 + (\aOPlus)^2}$, together with basic algebraic identities, the above conditions on $A$ can be equivalently expressed as the following final results:
\begin{itemize}
    \item If the matrix $A$ satisfies:
    \begin{equation*}
        \tr{A} > 0 \quad \text{or} \quad  \det(A) \leq 0,
    \end{equation*}
    then the order parameter $\CR_2(t) = \rho(t) e^{i\phi(t)}$ satisfies $\rho(t) \rightarrow 1$ and $\phi(t) \rightarrow \phi_{\infty}$ for some $\phi_{\infty} \in \BBT^1$ as $t \rightarrow \infty$.
    

    \item If the matrix $A$ satisfies:
    \begin{equation*}
        \frac{A + A^{\top}}{2} \prec 0,
    \end{equation*}
    then the order parameter $\CR_2(t) = \rho(t) e^{i\phi(t)}$ satisfies
    $\rho(t) \rightarrow \rho_{\mathrm{eq}}$ and $\phi(t) \rightarrow \phi_{\mathrm{eq}}$ as $t \rightarrow \infty$, where
    \begin{equation}\label{eq:rho*_phi*}
    \rho_{\mathrm{eq}} := \frac{\sqrt{(\aDMinus)^2 + (\aOPlus)^2}}{\sqrt{(\aDPlus)^2 + (\aOMinus)^2}} \in (0,1), \qquad \phi_{\mathrm{eq}} := \pi + \Arg (\aDMinus + i \aOPlus) - \Arg (\aDPlus + i \aOMinus) \ \ (\mathrm{mod} \ 2\pi) .
    \end{equation}
\end{itemize}
\end{proof}

\begin{proof}[Proof of Proposition \ref{prop:case1_quantitative}]
For notation simplicity, we use the notation introduced in the proof of Theorem \ref{thm:case_1}; see the shorthand notation in~\eqref{eq:case1_convenience_notation}, the dynamics for $z(t)$ in~\eqref{eqn:case1_complex}, and the definition of $\zeq$ in~\eqref{eq:zeq}.
By Theorem \ref{thm:case_1}, when $a > R > 0$, for any initial condition $z(0) \neq \zeq$, the solution $z(t)$ converges to a boundary equilibrium $z_* \in \partial D$.

Denote $q(t) := 1 - |z(t)|^2$. 
Using \eqref{eqn:case1_complex}, a straightforward calculation gives
\begin{equation}\label{eq:case1_qt}
    \dot{q} (t) = - \kappa(t) q(t), \qquad \text{where } \; \kappa(t) := \frac{1}{2} \left(a |z(t)|^2 + \re{\overline{z}(t) c} \right) \, .
\end{equation}
Furthermore, by the elementary inequality, we obtain $\kappa(t) \geq \frac{1}{2} |z(t)| \left( a|z(t)| - R \right)$.
Thus, if we choose any number $r \in (R/a, 1)$, then whenever $|z(t)| \geq r$, we have
\begin{equation}\label{eq:case1_kappa_ineq}
\kappa(t) \geq \mu_r := \frac{1}{2} r(ar - R) > 0 \, .
\end{equation}
Define
\begin{equation}\label{eq:case1_Tr}
    T_r := \inf \left\{ t \geq 0: |z(t)| \geq r \right\} \, .
\end{equation}
Then for $t \geq T_r$, applying the inequality \eqref{eq:case1_kappa_ineq} in \eqref{eq:case1_qt} and then integrating the differential inequality over time gives
\begin{equation*}
q(t) \leq q(T_r) \exp \left( - \mu_r (t - T_r) \right) \leq (1 - r^2) \exp \left( - \mu_r (t - T_r) \right) \, .
\end{equation*}
On the other hand, from \eqref{eqn:case1_complex}, we have
\begin{equation*}
    |\dot{z} (t)| \leq \frac{1}{4} \left( 1 - |z(t)|^2 \right) \left( \sqrt{a^2 + b^2} + R \right) = K q(t), \qquad \text{where } K:= \frac{\sqrt{a^2 + b^2} + R}{4} \, .
\end{equation*}
Then we can compute that, for $t \geq T_r$,
\begin{equation*}
    |z(t) - z_*| \leq \int_t^{\infty} |\dot{z}(s)|ds \leq K \int_t^{\infty} q(s) ds \leq \frac{K}{\mu_r} (1-r^2) \exp \left( - \mu_r (t - T_r) \right) \, .
\end{equation*}
Letting the right-hand side of above inequality less or equal to $\epsilon$ gives
\begin{equation*}
    t \geq T_r + \frac{1}{\mu_r} \log \frac{K(1-r^2)}{\mu_r \epsilon} =: T_{\epsilon} \, .
\end{equation*}
Thus, it remains to estimate $T_r$ as defined in \eqref{eq:case1_Tr}.

Denote $w(t) := z(t) - \zeq$.
Using \eqref{eqn:case1_complex}, a direct computation gives
\begin{equation*}
\frac{d}{dt} |w(t)| = \frac{a}{4} (1 - |z(t)|^2) |w(t)| \, .
\end{equation*}
For $t < T_r$, we have $|z(t)| < r$, and therefore $\frac{d}{dt} |w(t)| \geq \frac{a}{4} (1 - r^2) |w(t)|$.
Integrating it over time gives
\begin{equation}\label{eq:case1_wt_ineq}
|w(t)| \geq |w(0)| \exp \left( \frac{a}{4} (1 - r^2) t \right), \qquad t < T_r \, .
\end{equation}
On the other hand, since $|z(t)| < r$, we have $|w(t)| = |z(t) - \zeq| \leq |z(t)| + |\zeq| < r + |\zeq|$.
Together with \eqref{eq:case1_wt_ineq}, this implies that
\begin{equation*}
    T_r \leq \frac{4}{a(1-r^2)} \log \left( \frac{r + |\zeq|}{|z(0) - \zeq|} \right) \, .
\end{equation*}
Therefore, we obtain the final estimate on $T_{\epsilon}$:
\begin{equation*}
\begin{aligned}
T_{\epsilon} &\leq \frac{4}{a(1-r^2)} \log \left( \frac{r + |\zeq|}{|z(0) - \zeq|} \right) + \frac{1}{\mu_r} \log \frac{K(1-r^2)}{\mu_r \epsilon} \qquad \text{for any } r \in (R/a, 1),\\
&\leq C_1 \log \frac{1}{|z(0) - \zeq|} + C_2 \log \frac{1}{\epsilon},
\end{aligned}
\end{equation*}
where the constants $C_1, C_2 > 0$ only depend on the matrix $A$.
\end{proof}

\subsubsection*{Case 2.}
\label{proof:case_2}
\begin{proof}[Proof of Theorem \ref{thm:case_2}]
We begin with a change of variables.
Denote $R := \sqrt{(\vDMinus)^2 + (\vOPlus)^2}$.
In the following, we assume $R > 0$; the case $R = 0$ will be treated separately.
Define the angle $\phiV$ by
\begin{equation}\label{eq:case2_phiV}
    \cos \phiV = \frac{\vDMinus}{R},
    \qquad
    \sin \phiV = \frac{\vOPlus}{R},
\end{equation}
Then, by the trigonometric identities, we have
\begin{equation*}
    \vDMinus \cos \phi + \vOPlus \sin \phi
    = R \cos (\phi - \phiV), \qquad \vDMinus \sin \phi - \vOPlus \cos \phi
    = R \sin (\phi - \phiV) \, . 
\end{equation*}
For notational convenience, we set
\begin{equation}\label{eq:case2_notation}
    \psi(t) := \phi (t) - \phiV, \qquad v := \vDPlus, \qquad \kappa(t) := \frac{1}{4}R \frac{3\rho(t)^2 + 1}{\rho(t)} \, .
\end{equation}
With these definitions, the original system \eqref{eq:pc_dyn_case_2} becomes
\begin{equation}\label{eq:case2_pc_dyn_app}
    \Drho = \frac{1}{4} (1 - \rho^2) (v \rho + R \cos \psi), \qquad  \Dpsi = -\kappa \sin \psi.
\end{equation}
The equilibria of \eqref{eq:case2_pc_dyn_app} are easily identified:
\begin{equation*}
\begin{aligned}
&\{ (\rho, \psi) = (1,0) \}, \qquad &&\{ (\rho, \psi) = (1,\pi) \},\\[5pt]
&\{ (\rho, \psi) = (\rhoEq,0) \} \text{ (if $v + R < 0$)}, \qquad &&\{ (\rho, \psi) = (-\rhoEq, \pi) \} \text{ (if $v - R > 0$)},
\end{aligned}
\end{equation*}
where we denote
\begin{equation}\label{eq:case2_rho_equi}
\rhoEq := - \frac{R}{v} \, .
\end{equation}

We now analyze the long-time behavior.
We first exclude a special initialization $\psi(0) = \pi$. 
This case will be discussed at the end.

We first show that $\rho(t) \in (0,1]$ for all $t \geq 0$, which in turn ensures that $\kappa(t)$ stays positive and does not diverge.
The upper bound $\rho(t) \leq 1$ follows from the invariance of the boundary $\rho = 1$ to the $\rho$-dynamics \eqref{eq:case2_pc_dyn_app} and the assumption $\rho(0) \leq 1$.
To prove $\rho(t) > 0$, we momentarily convert to Cartesian coordinates $x = \rho \cos \psi$ and $y = \rho \sin \psi$.
A direct computation gives
\[
\dot{x} = \frac{1}{4}\Big[v(1 - x^2 - y^2)x + R(1 - x^2 + 3y^2)\Big], \qquad \dot{y} = \frac{1}{4}\Big[v(1 - x^2 - y^2) - 4Rx\Big] y \, .
\]
The Cartesian system is smooth and is globally well-posed on the invariant unit disk $x^2 + y^2 \leq 1$.
Integrating the $y$-equation gives
\begin{equation*}
    y(t) = y(0) \exp \left\{ \frac {1} 4 \int_0^t \left[ v(1 - \rho(s)^2) - 4Rx(s) \right] ds \right\}.
\end{equation*}
If $y(0) \neq 0$, then the apriori estimates $\rho(t) \leq 1$ and $|x(t)| \leq 1$ ensure that the exponent in the $y$-formulation is finite on every finite time interval, which implies $y(t) \neq 0$ for every finite $t \geq 0$.
If $y(0) = 0$, then $y(t) = 0$ for all $t\geq 0$, and the $x$-equation reduces to
\begin{equation*}
    \dot{x} = \frac{1}{4} (1 - x^2) (vx + R), \quad x(0) = x_0 \, .
\end{equation*}
Since $\rho(0) > 0$ and $\psi(0) \neq \pi$, we have $x(0) > 0$.
Moreover, note that $\dot{x}|_{x=0} = \frac{R}{4} > 0$, we thus know that $x(t)$ cannot reach zero.

Combining the above discussions, we show that $\rho(t) = \sqrt{x^2(t) + y^2(t)} \in (0,1]$ for all $t \geq 0$, and hence $\kappa(t)$ is well-defined and positive. 
More precisely, $\kappa(t) = \frac{R}{4} (3\rho(t) + 1/\rho(t)) \geq \frac{R}{4}$.

Next, we analysis the $\psi$-dynamics. 
We first observe that $\psi = 0$ is an invariance state of the $\psi$-dynamics in \eqref{eq:case2_pc_dyn_app}.
Hence, if $\psi(0) = 0$, then $\psi(t) = 0$ for all $t \geq 0$.
Thus, in the remaining discussion, we restrict our attention to the case $\psi(0) \neq 0, \pi$.

Introduce the half-angle variable
\[
u(t) := \tan (\psi(t) / 2) \, .
\]
Using the identities
\begin{equation*}
\sin \psi = \frac{2u}{1 + u^2}, \qquad  \Dot{u} = \frac 1 2 (1 + u^2) \Dpsi,
\end{equation*}
we find that $u(t)$ solves the linear equation:
\begin{equation*}
    \Dot{u} = - \kappa(t) u \, .
\end{equation*}
Therefore,
\begin{equation}
    u(t) = u(0) \exp \left\{ -\int_0^t \kappa(s) \, ds \right\}.
\end{equation}
Since $\kappa(t) \geq \frac {R} 4$, together with $|\psi(t)| = 2 |\arctan (u(t))| \leq 2|u(t)|$ gives
\begin{equation}\label{eq:case2_psi_conv}
    |\psi(t)| \leq 2|u(0)| \exp \left\{ -\int_0^t \kappa(s) \, ds \right\} \leq 2 |u(0)| \exp \left( - \frac{R}{4} t \right) \, .
\end{equation}
Since $\psi(0) \neq 0, \pi$, which implies $|u(0)| < \infty$, we therefore conclude from the above inequality that $\psi(t)$ converges to $0$ exponentially fast as $t \rightarrow \infty$.

Next, we turn to the $\rho$-dynamics.
Define
\begin{equation}\label{eq:case2_func_G_and_E}
G(\rho) := \frac{1}{4} (1 - \rho^2) (v\rho + R), \qquad E(t) := \frac{R}{4} (1 - \rho^2) (1 - \cos \psi) \, .
\end{equation}
Then the $\rho$-equation in \eqref{eq:case2_pc_dyn_app} can be written as
\begin{equation}\label{eq:case2_rho_dyn_pt_form}
\Drho = G(\rho) - E(t), \quad \rho(0) = \rho_0 \, .
\end{equation}
Since $\rho(t) \leq 1$, $\cos \psi(t) \leq 1$, and $R > 0$, we know $E(t)$ is non-negative.
Moreover, by the exponential convergence of $\psi(t)$ to $0$, established in \eqref{eq:case2_psi_conv}, $E(t) \rightarrow 0$ exponentially fast as $t \rightarrow \infty$.
Thus the long-time behavior of $\rho(t)$ is governed, up to an exponentially decaying perturbation, by the vector field $G$.
We make this observation mathematically precise in the following lemma.

\begin{lemma}
\label{lem:pt_lem}
Let $a < b$.
Suppose that $\varrho:[0,\infty)\to[a,b]$ solves
\[
\dot{\varrho}(t)=g(\varrho(t))+e(t),
\]
where $g:[a,b]\to\mathbb{R}$ is Lipschitz, and $e: [0, \infty) \mapsto \R$ is continuous and satisfies
\[
\lim_{t\to\infty} e(t)=0 \, .
\]
Assume that there exists $\varrho_*\in[a,b]$ such that
\[
g(\varrho_*)=0,
\]
and
\begin{equation*}
(\varrho - \varrho_*) g (\varrho) < 0 \qquad \text{for every } \varrho \in [a,b] \setminus \{\varrho_*\} \, .
\end{equation*}
Then
\[
\lim_{t\to\infty}\varrho(t)=\varrho_* \, .
\]
\end{lemma}
\begin{proof}
Fix $\varepsilon>0$ so that the set
\begin{equation*}
K_{\varepsilon} := \left\{ \varrho \in [a,b]: |\varrho - \varrho_*| \geq \varepsilon \right\}
\end{equation*}
is nonempty.
By the continuity of $g$ and the strict sign assumption, we have
\[
m_\varepsilon
:= \min_{\varrho \in K_{\varepsilon}} |g(\varrho)| > 0 \, .
\]
Since $e(t)\to0$, there exists $T_\varepsilon>0$ such that
\[
|e(t)|\leq \frac{m_\varepsilon}{2}
\qquad\text{for all }t\geq T_\varepsilon.
\]
Therefore, for $t\geq T_\varepsilon$, if
\[
\varrho(t)\leq \varrho_*-\varepsilon,
\]
then
\[
\dot{\varrho}(t)
=
g(\varrho(t))+e(t)
\geq
m_\varepsilon-|e(t)|
\geq
\frac{m_\varepsilon}{2}>0.
\]
Similarly, if
\[
\varrho(t)\geq \varrho_*+\varepsilon,
\]
then
\[
\dot{\varrho}(t)
=
g(\varrho(t))+e(t)
\leq
-m_\varepsilon+|e(t)|
\leq
-\frac{m_\varepsilon}{2}<0 \, .
\]
Define 
\begin{equation*}
    I_{\varepsilon} := [a, b] \cap [\varrho_* - \varepsilon, \varrho_* + \varepsilon] \, . 
\end{equation*}
We claim that the solution enters $I_{\varepsilon}$ in finite time after $T_{\varepsilon}$.

Indeed, if the solution stays below
$\varrho_*-\varepsilon$, it would increase at speed at least $m_\varepsilon/2$, while if the solution stays above
$\varrho_*+\varepsilon$, it would decrease at speed at least $m_\varepsilon/2$.
Moreover, once the solution has entered $I_{\varepsilon}$ after time $T_\varepsilon$, it cannot leave:
at the left endpoint $\varrho_*-\varepsilon$, whenever this point belongs to $[a,b]$, the derivative is strictly positive, and at the right endpoint
$\varrho_*+\varepsilon$, whenever this point belongs to $[a,b]$, the derivative is strictly negative. 
Hence there exists
$\widetilde{T}_\varepsilon\geq T_\varepsilon$ such that
\[
|\varrho(t)-\varrho_*|\leq\varepsilon
\qquad\text{for all }t\geq \widetilde{T}_\varepsilon.
\]
Since $\varepsilon>0$ is arbitrary, we conclude that
\[
\lim_{t\to\infty}\varrho(t)=\varrho_*.
\]
\end{proof}

With the above lemma, we are ready to discuss the asymptotic behavior of $\rho$-dynamics \eqref{eq:case2_rho_dyn_pt_form} in two parameter regimes.

\vspace{0.8em}
\textbf{Claim 1.} If $v + R \geq 0$, the solution $\rho(t)$ of \eqref{eq:case2_rho_dyn_pt_form} satisfies $\lim_{t \rightarrow \infty} \rho(t) = 1$. 
\begin{proof}[Proof of Claim 1.]
We focus on the non-trivial case $\rho(0) < 1$.
For any $\rho \in [0,1)$, our assumption implies that $v\rho + R > 0$, and $G(\rho) > 0$, as $G$ defined in \eqref{eq:case2_func_G_and_E}.
Then the $\rho$-dynamics in \eqref{eq:case2_rho_dyn_pt_form} satisfies the assumptions in Lemma \ref{lem:pt_lem} with choosing $a=0$, $b=1$, and $\varrho_* = 1$, and thus we conclude the convergence $\lim_{t \rightarrow \infty} \rho(t) = 1$.
\end{proof}

Here, we record an additional quantitative estimate in the sub-regime $v + R > 0$, which will be used in the proof of Theorem \ref{thm:case_2stab}.
We defer the proof until after the proof of Theorem \ref{thm:case_2}.

\begin{proposition}\label{prop:case2_quantitative}
Assume that $R > 0$ and $v + R > 0$. 
Let $(\rho(t), \psi(t))$ be the solution to \eqref{eq:case2_pc_dyn_app} with initial data $\rho(0) \in (0,1]$ and $\psi(0) \in \BBT^1 \setminus \{\pi\}$.
Then there exist constant $0 < \epsilon_* \leq \frac{1}{2\pi}$ and $C > 0$, depending  only on $V$, such that for every $0 < \epsilon \leq \epsilon_*$, if
\begin{equation}
    T_{\epsilon} := \inf \left\{ t \geq 0: \rho(t) \geq 1 - \epsilon, |\psi(t)| \leq \epsilon \right\},
\end{equation}
then
\begin{equation*}
T_{\epsilon} \leq C \log \frac{1}{\epsilon |\psi(0) - \pi|} \, .
\end{equation*}
Here $\psi(0)$ is represented in $[0,2\pi)$.
\end{proposition}

\textbf{Claim 2.} If $v + R < 0$ and $\rho(0) \neq 1$, the solution $\rho(t)$ of \eqref{eq:case2_rho_dyn_pt_form} satisfies $\lim_{t \rightarrow \infty} \rho(t) = \rhoEq$, with $\rhoEq$ defined in \eqref{eq:case2_rho_equi}.
\begin{proof}[Proof of Claim 2]
Since $v + R < 0$, we have
\[
    v \rho + R 
    \begin{cases}
        > 0, \quad \rho < \rho_{\mathrm{eq}}, \\
        = 0, \quad \rho = \rho_{\mathrm{eq}}, \\
        < 0, \quad \rho > \rho_{\mathrm{eq}},
    \end{cases}
    \text{ and therefore }
    G(\rho) 
    \begin{cases}
        > 0, \quad \rho < \rho_{\mathrm{eq}}, \\
        = 0, \quad \rho = \rho_{\mathrm{eq}}, \\
        < 0, \quad \rho > \rho_{\mathrm{eq}} 
    \end{cases} \, .
\]
Then the $\rho$-dynamics in \eqref{eq:case2_rho_dyn_pt_form} satisfies the assumptions in Lemma \ref{lem:pt_lem} with choosing $a=0$, $b=1$, and $\varrho_* = \rhoEq$, and thus we conclude the convergence $\lim_{t \rightarrow \infty} \rho(t) = \rhoEq$.
\end{proof}

Finally, we consider the boundary case $R = 0$.
In this regime, we restrict our attention to $v \neq 0$ since $v = 0$ and $R = 0$ corresponds to the trivial case where the matrix $V$ is the zero matrix.

When $R = 0$, the dynamics \eqref{eq:case2_pc_dyn_app} simplifies to 
\begin{equation*}
\Drho = \frac{1}{4} (1 - \rho^2) v \rho, \qquad \Dpsi = 0 \, .
\end{equation*}
The angular motion $\psi$ is thus a constant over time.
For the $\rho$-equation, we can easily identify the two fixed points: $\rho = 0$ and $\rho = 1$.
The long-time behavior is also straightforward.
If $v > 0$, Lemma \ref{lem:pt_lem}, applied with $\varrho(t) = \rho(t)$, $\varrho_* = 1$, and $e(t) = 0$, gives that for any initial condition $\rho(0) \in (0,1]$, $\lim_{t \rightarrow \infty} \rho(t) = 1$.
If $v < 0$, we set $\varrho(t) = -\rho(t)$, $\varrho_* = 0$, and it shows that for any initial condition $\rho(0) \in [0,1)$, $\lim_{t \rightarrow \infty} \rho(t) = 0$.

\paragraph{A note on the measure-zero initialization set $\{(\rho(0), \psi(0)): \psi(0) = \pi\}$.}
Since $\psi = \pi$ is a fixed point of the $\psi$-dynamics in \eqref{eq:case2_pc_dyn_app}, we have $\psi(t) = \psi(0)$ for all $t \geq 0$.
In this case, the $\rho$-equation reduces to
\[
\Drho = \frac{1}{4}(1-\rho^2) (v\rho - R) \, .
\]
We note that, in both parameter regimes $v + R \geq 0$ and $v+R < 0$, the limiting value of $\rho(t)$ may depend on the initial value $\rho(0)$.
Since this exceptional set has measure zero and its initial-condition-dependent behavior is not the main focus of this paper, we omit the detailed discussion.

\paragraph{Conclusion for the order parameter $\CR_2$.}
Combining the preceding arguments with the relation $\psi(t) = \phi(t) - \phiV$, we obtain the following convergence results.
\begin{itemize}
    \item If $v + R \geq 0$, then for any initial condition satisfying $\psi(0) \neq \pi$, the solution $(\rho, \phi)$ of \eqref{eq:pc_dyn_case_2} converges to $(1, \phiV)$.

    \item If $v + R < 0$, then for any initial condition satisfying $\rho(0) \neq 1$ and $\psi(0) \neq \pi$, the solution $(\rho, \phi)$ of \eqref{eq:pc_dyn_case_2} converges to $(\rhoEq, \phiV)$, where $\rhoEq$ is defined in \eqref{eq:case2_rho_equi}.
\end{itemize}

It remains only to translate these conditions into the parameter regimes stated in Theorem \ref{thm:case_2}.
This follows from the definitions $v = \vDPlus$, $R = \sqrt{(\vDMinus)^2 + (\vOPlus)^2}$, together with elementary algebraic identities.
For the reader's convenience, we also record the expression
\begin{equation}\label{eq:case2_rhoEq}
\rhoEq := - \frac{\sqrt{(\vDMinus)^2 + (\vOPlus)^2}}{\vDPlus} \, .
\end{equation}
\end{proof}

\begin{proof}[Proof of Proposition \ref{prop:case2_quantitative}]
For notation simplicity, we use the notation introduced in the proof of Theorem \ref{thm:case_2}; see the shorthand notation in \eqref{eq:case2_notation}, and the dynamics for $(\rho(t), \psi(t))$ in \eqref{eq:case2_pc_dyn_app}.

From the $\psi$-dynamics, we know that if $\psi(0) = 0$, then $\psi(t) = 0$ for all $t \geq 0$, satisfying $|\psi(t)| \leq \epsilon$.
Next, we consider the initial data $\psi(0) \neq \{0, \pi\}$.
By the inequalities \eqref{eq:case2_psi_conv} and $| \tan \psi| \leq \frac{1}{|\psi - \pi/2|}$ for $\psi \in [0,\pi]$, we have
\begin{equation*}
|\psi(t)| \leq 2 |\tan (\psi(0)/2)| \exp \left( -\frac{R}{4} t \right) \leq \frac{4}{|\psi(0) - \pi|} \exp \left( -\frac{R}{4} t \right) \, .
\end{equation*}
By letting the right-hand side of the above inequality less or equal to $\epsilon$, we get
\begin{equation*}
    t \geq \frac{4}{R} \log \frac{4}{\epsilon |\psi(0) - \pi|} =: T_{\psi, \epsilon} \, .
\end{equation*}

Next, we turn to the $\rho$-dynamics in \eqref{eq:case2_pc_dyn_app}.
Define
\begin{equation*}
\chi := \min_{\rho \in [0,1]} (v \rho + R) = \min \left\{ R, v+R \right\} \, .
\end{equation*}
Since $R > 0$ and $v+R > 0$, we have $\chi > 0$.
Choose $0 < \epsilon_* \leq 1/(2\pi) $, depending only on the matrix $V$, such that $R(1 - \cos \epsilon_*) \leq \frac{\chi}{2}$.
Then for every $0 < \epsilon \leq \epsilon_*$, whenever $|\psi| \leq \epsilon$, we have
\begin{equation*}
    v \rho + R \cos \psi = (v\rho + R) - R(1 - \cos\psi) \geq \frac{\chi}{2} \, .
\end{equation*}
Consequently, after time $T_{\psi, \epsilon}$, the $\rho$-dynamics satisfies
\begin{equation*}
\Drho = \frac{1}{4} (1 - \rho^2) (v\rho + R \cos\psi) \geq \frac{\chi}{8} (1-\rho^2) \geq \frac{\chi}{8} (1-\rho) \, .
\end{equation*}
Equivalently, $q(t) := 1 - \rho(t)$ satisfies $\dot{q}(t) \leq -\frac{\chi}{8} q(t)$ for $t \geq T_{\psi, \epsilon}$.
Integrating it over time gives
\begin{equation*}
q(t) \leq q(T_{\psi, \epsilon}) \exp \left( -\frac{\chi}{8} (t - T_{\psi,\epsilon}) \right) \leq \exp \left( -\frac{\chi}{8} (t - T_{\psi,\epsilon}) \right) \, .
\end{equation*}
Hence $\rho(t) \geq 1 - \epsilon$ if 
\begin{equation*}
    t \geq T_{\psi, \epsilon} + \frac{8}{\chi} \log \frac{1}{\epsilon} \, .
\end{equation*}
Since $|\psi(t)|$ continues to decrease after $T_{\psi, \epsilon}$, we also have $|\psi(t)| \leq \epsilon$.
Therefore,
\begin{equation*}
T_{\epsilon} = T_{\psi, \epsilon} + \frac{8}{\chi} \log \frac{1}{\epsilon} \leq \frac{4}{R} \log \frac{4}{\epsilon |\psi(0) - \pi|} + \frac{8}{\chi} \log \frac{1}{\epsilon} \leq C  \log \frac{1}{\epsilon |\psi(0) - \pi|},
\end{equation*}
where constant $C > 0$ only depends on the matrix $V$.
\end{proof}

\subsubsection*{Case 3.}
\label{proof:case_3}
\begin{proof}[Proof of Theorem \ref{thm:case_3}]
We start from the $(\rho, \phi)$-dynamics \eqref{eq:pc_dyn_general} associated with arbitrary matrices $A$ and $V$ defined in \eqref{eq:matrices_A_V_general}.
By Lemma \ref{lem:pc_dyn_general_expand}, this system admits the equivalent representation
\begin{equation}\label{eq:case3_pc_dyn_gen}
\begin{split}
        \Drho = 2 (1-\rho^2) \left( S_1(\phi) \rho + S_2(\phi) \right), \qquad
        \Dphi = 2 \left( S_3(\phi) (\rho^2+1) +  S_4(\phi) \frac{\rho^2 + 1}{\rho} + 2S_5(\phi) \rho + \re{b_5} \right),
    \end{split}
\end{equation}
where the constant $b_5$ and the functions $S_j$ are defined in \eqref{eq:coeff_b} and \eqref{eq:func_Sj}, respectively.

Our goal is to identify a family of matrices $A$ and $V$, together with explicit parameter regimes, such that the corresponding solution of \eqref{eq:case3_pc_dyn_gen} satisfies, as $t \rightarrow \infty$,
\[
\rho(t) \rightarrow 1 \qquad \text{while} \qquad \phi(t) \text{ does not converge} \, .
\]

\paragraph{$\rho$-dynamics.}
We first analyze the $\rho$-dynamics in \eqref{eq:case3_pc_dyn_gen}. 
Since $\rho \in (0,1)$, we obtain the following sufficient condition for $\rho(t) \rightarrow 1$: 
\begin{equation}\label{eq:suffi_cond_rho_converge}
S_1 (\phi) > 0, \qquad \text{and} \qquad S_2(\phi) \geq 0, \qquad \text{for all } \phi \in \BBT^1 .
\end{equation}
Indeed, under this condition, 
\begin{equation*}
\Drho = 2 (1-\rho^2) \left(S_1(\phi)\rho + S_2(\phi) \right) \geq 0
\end{equation*}
with equality only holding at $\rho = 1$.
Therefore, by Lemma \ref{lem:gamma-integral}, $\rho(t) \rightarrow 1$ as $t \rightarrow \infty$ for every $\rho(0) \in (0,1]$.

We next translate the condition \eqref{eq:suffi_cond_rho_converge} into explicit constraints on the matrices $A$ and $V$.
From \eqref{eq:func_Sj}, we recall
\begin{equation*}
S_1(\phi) = \im{b_1 e^{i2\phi} + b_2}, \qquad S_2(\phi) = \im{b_3 e^{i\phi}},
\end{equation*}
where $b_1, b_2$, and $b_3$ are defined in \eqref{eq:coeff_b}.
Hence, the conditions $S_1(\phi) > 0$ and $S_2(\phi) = 0$ for all $\phi \in [0, 2\pi)$ is equivalent to
\begin{equation*}
\im{b_2} > |b_1| \qquad \text{and} \qquad b_3 = 0 .
\end{equation*}
Using the definitions of $b_1, b_2$, and $b_3$, these conditions can be rewritten in terms of the coefficients $w_i$ introduced in \eqref{eq:w_i} as follows: 
\begin{itemize}
    \item[(i).] $w_1 = w_2 = 0$, and
    \item[(ii).] $w_5 + w_8 > \sqrt{(w_6 + w_7)^2 + (w_5 - w_8)^2}$.
\end{itemize}
Since translating these conditions into explicit constraints on general matrices $A$ and $V$ is algebraically involved, in the following we restrict attention to the important special case where $A$ is diagonal:
\begin{equation*}
A := \left(\begin{array}{ll}
    a_{11} & 0 \\
    0 & a_{22}
\end{array} \right), \qquad a_{11}, a_{22} \neq 0 .
\end{equation*}
We now characterize the family of matrices $V$ satisfying the above two conditions in this setting.

Under this assumption on $A$, condition (i) becomes
\begin{equation*}
w_1 = a_{11} v_{11} - a_{22} v_{22}, \qquad w_2 = - (a_{22} v_{12} + a_{11} v_{21}) .
\end{equation*}
Hence, $w_1 = w_2 = 0$ is equivalent to
\begin{equation*}
v_{22} = \frac{a_{11}}{a_{22}} v_{11}, \quad v_{21} = - \frac{a_{22}}{a_{11}} v_{12} .
\end{equation*}
Accordingly, the corresponding admissible family of matrices $V$ can be written as
\begin{equation*}
V := \left( \begin{array}{cc}
    v_{11} & v_{12}  \\[4pt]
    -\frac{a_{22}}{a_{11}} v_{12} & \frac{a_{11}}{a_{22}} v_{11} 
\end{array} \right), \qquad  v_{11}, v_{12} \in \R .
\end{equation*}

We next consider condition (ii).
For the matrices $A$ and $V$ defined above, one can verify that
\begin{equation*}
    w_5 + w_8 > \sqrt{(w_6 + w_7)^2 + (w_5 - w_8)^2} \qquad
    \Longleftrightarrow \qquad \frac{v_{11}}{a_{22}} > 0, \quad  8 v_{11}^2  (a_{11}^2 + a_{22}^2) \frac{a_{11}}{a_{22}} > (a_{11} - a_{22})^4 \frac{v_{12}^2}{a_{11}^2} .
\end{equation*}

Therefore, for the family of matrices 
\begin{equation}\label{eq:case3_matrices_A_V_in_proof}
A := \left(\begin{array}{ll}
    a_{11} & 0 \\
    0 & a_{22}
\end{array} \right), \qquad  V := \left( \begin{array}{cc}
    v_{11} & v_{12}  \\[4pt]
    -\frac{a_{22}}{a_{11}} v_{12} & \frac{a_{11}}{a_{22}} v_{11} 
\end{array} \right),
\end{equation}
with the entries $a_{11}, a_{22} \neq 0, v_{11}, v_{12} \in \R$ satisfying
\begin{equation}\label{eq:suffi_cond_rho_converge_A_V}
\frac{v_{11}}{a_{22}} > 0, \qquad 8 \frac{a_{11}}{a_{22}} (a_{11}^2 + a_{22}^2) v_{11}^2  >  \frac{(a_{11} - a_{22})^4}{a_{11}^2} v_{12}^2,
\end{equation}
the $\rho$-equation satisfies $\rho(t) \rightarrow 1$ as $t \rightarrow \infty$.

\paragraph{$\phi$-dynamics.}
We now turn to the $\phi$-dynamics in \eqref{eq:case3_pc_dyn_in_proof}.
For the matrices $A$ and $V$ above, the system simplifies to
\begin{equation}\label{eq:case3_pc_dyn_in_proof}
\Drho = \rho (1 - \rho^2) H(\phi), \qquad \Dphi = K(\phi) + (\rho-1) Q(\rho, \phi) .
\end{equation}
Here, the functions $H$, $K$, and $Q$ are given by
\begin{equation}\label{eq:case3_func_H}
H(\phi) := \frac{v_{11}}{a_{22}} (\aDPlus)^2 + (\aDMinus)^2 \left( -\frac{v_{11}}{a_{22}} \cos (2\phi) + \frac{v_{12}}{a_{11}} \sin (2\phi) \right),
\end{equation}
\begin{equation}\label{eq:case3_func_K}
K(\phi) := \frac{1}{2} \left( \aDPlus + \aDMinus \cos \phi \right) P(\phi),
\end{equation}
where
\begin{equation}
    P(\phi) :=\left( \frac{v_{11}}{a_{22}} \aDMinus \sin \phi + \frac{v_{12}}{a_{11}} \aDMinus \cos \phi - \frac{v_{12}}{a_{11}} \aDPlus \right),
\end{equation}
and
\begin{equation}\label{eq:case3_func_Q}
Q(\rho, \phi) := \frac{1}{2} \frac{v_{11}}{a_{22}} \aDPlus \aDMinus \sin \phi + \frac{1}{8} (\rho + 1) \left( - \frac{v_{12}}{a_{11}} (\aDPlus)^2 + \frac{v_{12}}{a_{11}} (\aDMinus)^2 \cos (2\phi) + \frac{v_{11}}{a_{22}} (\aDMinus)^2 \sin (2\phi) \right) .
\end{equation}

To show that $\phi(t)$ does not converge, it suffices to prove that there exists $T \geq 0$ such that 
\[
|\Dphi(t)| > 0, \qquad \forall t \geq T .
\]
We now identify an additional sufficient condition which, together with the previous condition on $A$ and V, guarantees this property.

Although the expressions for $K$ and $Q$ are somewhat involved, the key observation is simple.
Under the previous condition \eqref{eq:suffi_cond_rho_converge_A_V}, we already know that $\rho(t) \rightarrow 1$ as $t \rightarrow \infty$.
Moreover, since $\rho \in [0,1]$ and $a_{11}, a_{22} \neq 0$, there exists a constant $C < \infty$ such that 
\begin{equation*}
|Q(\rho, \phi)| \leq C, \qquad \forall (\rho, \phi) \in [0,1] \times \BBT^1 .
\end{equation*}
Therefore, for every $\varepsilon > 0$, there exists $T_{\varepsilon} < \infty$ such that 
\begin{equation*}
|(\rho - 1) Q(\rho, \phi)| \leq \varepsilon, \qquad \forall t \geq T_{\varepsilon} .
\end{equation*}
Thus, in order to ensure $|\Dphi| > 0$ for all sufficiently large $t$, it remains to impose a condition guaranteeing that $|K(\phi)| > 0$ for all $\phi \in \BBT^1$.

Note that the previous condition \eqref{eq:suffi_cond_rho_converge_A_V} implies that $a_{11}$ and $a_{22}$ have the same sign.
Hence,
\begin{equation*}
\begin{aligned}
\left|\aDPlus + \aDMinus \cos \phi \right| &= \left| (a_{11} + a_{22}) + (a_{11} - a_{22}) \cos \phi \right| = \left| 2 \left(a_{11} \cos^2 \frac{\phi}{2} + a_{22} \sin^2 \frac{\phi}{2} \right) \right|\\
&\geq 2 \min \left\{ |a_{11}|, |a_{22}| \right\} =: c_0 > 0, \qquad \forall \phi \in \BBT^1 .
\end{aligned}
\end{equation*}
Therefore, $|K(\phi)| > 0$ for all $\phi \in \BBT^1$ if and only if $|P(\phi)| > 0$ for all $\phi \in \BBT^1$.
From the definition of $P$, a sufficient condition for $|P(\phi)| > 0$ for all $\phi \in \BBT^1$ is
\begin{equation}\label{eq:suffi_cond_phi_not_converge}
\left| \frac{v_{12}}{a_{11}} \aDPlus \right| > |\aDMinus| \sqrt{\left( \frac{v_{11}}{a_{22}} \right)^2 + \left( \frac{v_{12}}{a_{11}} \right)^2} .
\end{equation}
Indeed, under this condition,
\begin{equation*}
|P(\phi)| = \left| \frac{v_{11}}{a_{22}} \aDMinus \sin \phi + \frac{v_{12}}{a_{11}} \aDMinus \cos \phi - \frac{v_{12}}{a_{11}} \aDPlus \right| \geq \left|\frac{v_{12}}{a_{11}} \aDPlus \right| - |\aDMinus| \sqrt{\left( \frac{v_{11}}{a_{22}} \right)^2 + \left( \frac{v_{12}}{a_{11}} \right)^2} =: c_1 > 0 .
\end{equation*}
It follows that
\begin{equation*}
|K(\phi)| \geq \frac{1}{2} c_0 c_1 > 0, \qquad \forall \phi \in \BBT^1 .
\end{equation*}

Combining the sufficient conditions above, we conclude that $\rho(t) \rightarrow 1$ as $t \rightarrow$, while $\phi(t)$ cannot converge.
This completes the proof.
\end{proof}

\subsubsection*{Case 4.}
\begin{proof}[Proof of Theorem~\ref{thm:case_4}]
\label{proof:case_4}
For convenience we relabel $v_{11} = a, v_{12} = b$.

\textbf{Step 1: Characterization of the fixed points.}
By letting $\Drho = \Dphi = 0$, we obtain the following fixed points of the dynamics \eqref{eq:pc_dyn_case_4}:
\begin{itemize}
    \item If $a \geq b$, there are two fixed points:
    \begin{equation}\label{eq:case4_E1_E2}
    E_1 := \left\{ (\rho, \phi) = (1, \pi + \arcsin(b/a)) \right\}, \qquad E_2 := \left\{ (\rho, \phi) = (1, 2\pi - \arcsin(b/a)) \right\} \, .
    \end{equation}
    When $a=b$, these two fixed points coincide, namely $E_1 = E_2 = \{(\rho, \phi) = (1, 3\pi/2)\}$.
    
    \item If $a < b$, there is a unique fixed point:
    \begin{equation}\label{eq:case4_E3}
    E_3 := \left\{(\rho,\phi)=\left(\tanh\!\left(\frac{1}{3}\operatorname{arctanh}\!\left(\frac{a}{b}\right)\right),\frac{3\pi}{2}\right)\right\} \, . 
    \end{equation}
\end{itemize}
We now provide the detailed derivation. First, consider the boundary case $\rho=1$.
In this case, we automatically have $\Drho = 0$, while $\Dphi = 0$ if and only if $b + a \sin(\phi) = 0$.
Since $a, b > 0$, this equation has no solutions when $a < b$, has the unique solution at $\phi = \frac{3\pi}{2}$ when $a = b$, and has two solutions when $a > b$.  
In the latter case, the two fixed points are $E_1$ and $E_2$ as defined in \eqref{eq:case4_E1_E2}.

Next, consider the interior case $0 < \rho < 1$.
From $\Drho = 0$, we obtain $\phi \in \{\frac{\pi}{2}, \frac{3\pi}{2}\}$.  
Imposing $\Dphi = 0$ then gives the following constraints on $\rho$:
\begin{equation*}
    \begin{cases}
        f_1(\rho) := -b \rho^3 - 3a \rho^2 - 3b\rho - a = 0, \qquad \text{if } \phi = \frac{\pi}{2},\\[4pt]
        f_2(\rho) := -b \rho^3 + 3a \rho^2 - 3b\rho + a = 0, \qquad \text{if } \phi = \frac{3\pi}{2} \, .
    \end{cases}
\end{equation*}
The equation $f_1(\rho) = 0$ has no solution in $[0,1]$ for any $a,b > 0$,  since $f_1(0) = -a$ and $f'(\rho) < 0$ for $\rho \in [0,1]$. 

It remains to analyze $f_2(\rho)=0$. Rewriting this equation gives
\[
\frac{a}{b}
=
\frac{\rho^3+3\rho}{3\rho^2+1}
=: q(\rho).
\]
The function $q$ is monotone increasing on $[0,1]$, since $q'(\rho)
=
\frac{3(1-\rho^2)^2}{(3\rho^2+1)^2}
\geq 0$,
and it satisfies $q(0)=0$ and $q(1)=1$. Therefore, if $a>b$, then $a/b>1$ and there is no solution in $[0,1]$. 
If $a=b$, the unique solution is $\rho=1$. 
If $a<b$, then $a/b\in(0,1)$, and hence there is a unique solution $\rho\in(0,1)$.
Using the identity
\[
\tanh(3u)=\frac{\tanh^3 u+3\tanh u}{3\tanh^2 u+1},
\]
this solution is given explicitly as in \eqref{eq:case4_E3}.

Next, we perform the linear stability analysis on these fixed points. 
The Jacobian matrix \eqref{eq:pc_dyn_case_4} is
\begin{equation}
    \mathcal{J} = \begin{bmatrix}
         - 2a\rho \cos(\phi) &-a(1-\rho^2)\sin(\phi)\\
         -2b\rho - 3a\sin(\phi) + \frac{a}{\rho^2} \sin(\phi) & -3a\rho\cos(\phi) - \frac{a}{\rho}\cos(\phi) 
    \end{bmatrix} \, .
\end{equation}
When $a > b$, evaluating $\mathcal{J}$ at the two boundary fixed points shows that both $E_1$ and $E_2$ are hyperbolic.
By the Hartman--Grobman Theorem, $E_1$ is unstable, while $E_2$ is asymptotically stable.
When $a = b$, the unique equilibrium is $E_3={(1,3\pi/2)}$. 
In this case, evaluating the Jacobian gives $\mathcal{J} = 0$, so $E_3$ is a degenerate equilibrium.
We postpone the characterization of its behavior to the energy analysis below.
Finally, when $a < b$, the two eigenvalues of $\mathcal{J}$ evaluated on $E_3$ are pure imaginary, and thus $E_3$ is a linear center.
In the next step, we perform an energy analysis and show that, in this case, $E_3$ is in fact a stable nonlinear center.




\textbf{Step 2: Energy Analysis.} 
\begin{equation}\label{eq:case_4_energy}
H(\rho, \phi) = \frac{-a \rho\sin(\phi) - \frac{b}{2}(\rho^2 + 1)}{(1-\rho^2)^2} \, .
\end{equation}
We first show that $H$ is a first integral for~\eqref{eq:pc_dyn_case_4}.

\begin{lemma} \label{lemma:case_4_energy}
    If $(\rho, \phi)$ follow the dynamics~\eqref{eq:pc_dyn_case_4} and $0 < \rho < 1$, then $\dot H(\rho, \phi) = 0$.
\end{lemma}
\begin{proof}
    We can first 
    compute that
    \begin{equation}\label{eq:case4_nabla_H}
    \nabla H (\rho, \phi) = 
    \begin{pmatrix}
            \frac{\partial H}{\partial \rho} \\[5pt] \frac{\partial H}{\partial \phi}
    \end{pmatrix}
    = \begin{pmatrix}
            -\frac{a(1 + 3\rho^2) \sin \phi + b \rho (\rho^2 + 3)}{(1-\rho^2)^3} \\[5pt] - \frac{a \rho \cos \phi}{(1-\rho^2)^2}
    \end{pmatrix} \, .
    \end{equation} 
    Denote $\mu(\rho) = \frac{2\rho}{(1- \rho^2)^3}$ and $J = \begin{pmatrix}
        0 & -1 \\ 1& 0
    \end{pmatrix}$, then
    we have
    \begin{equation}\label{eq:case4_pc_vec_and_nabla_H}
        \begin{pmatrix}
            \dot \rho \\ \dot \phi
        \end{pmatrix} = \frac{1}{\mu(\rho)} J \nabla H \, .
    \end{equation}
    Therefore,
    \begin{equation*}
        \dot H = \nabla H \cdot \begin{pmatrix}
            \dot \rho \\ \dot \phi
        \end{pmatrix} = \frac{1}{\mu(\rho)} \nabla H \cdot J \nabla H = 0,
    \end{equation*}
    where the last equality follows from the skew-symmetry of $J$. 
    This proves that $H$ is conserved along trajectories in the interior region $0<\rho<1$.
\end{proof}
We now analyze the level set of the Hamiltonian $H$, since every trajectory $(\rho(t), \phi(t)) \in (0,1) \times \BBT^1$ evolves along a level set of $H$.  
For notation simplicity, we write $\{H(\rho, \phi) = c\}$, or simply $\{H = c\}$, to denote the level set of $H$ at value $c \in \R$ within the open cylinder $(0,1) \times \BBT^1$.

In the following, we first show in Lemma \ref{lem:case4_H_level_set} that for almost every $c \in \R$, the level set $\{H=c\}$ is a \emph{connected} one-dimensional manifold whenever it is nonempty.
Building upon this geometric property, we discuss in two regimes.

\begin{itemize}
    \item In the parameter regime $a < b$, we show that $\{H=c\}$, whenever nonempty, is a \emph{compact} subset of the open cylinder $(0,1) \times \BBT^1$; see Figure \ref{subfig:H_level_set_a_less_than_b} for a visual illustration.
    
    Consequently, any trajectory $(\rho(t), \phi(t))$ of the dynamics \eqref{eq:pc_dyn_case_4} that evolves along such a level set forms a closed periodic orbit.

    \item In the parameter regime $a \geq b$, we show that $\{H=c\}$, whenever nonempty, has closure intersecting the boundary $\{\rho=1\}$; see Figures \ref{subfig:H_level_set_a_equal_to_b} and \ref{subfig:H_level_set_a_greater_than_b} for a visual illustration.
    In particular, such a level set is \emph{not compact} as a subset of the open cylinder $(0,1) \times \BBT^1$. 
    
    Consequently, no trajectory $(\rho(t), \phi(t))$ of the dynamics \eqref{eq:pc_dyn_case_4} evolving along such a level set can form a nontrivial periodic orbit in the open cylinder $(0,1) \times \BBT^1$.
    Moreover, Step $1$ shows that the closed cylinder $[0,1] \times \BBT^1$ is a compact positive invariant set for the dynamics, and it contains an unstable fixed point $E_1$ and an asymptotically stable fixed point $E_2$.
    Then the Poincaré--Bendixson theorem \cite{wiggins2003introduction} implies that, for almost every initial condition $(\rho(0), \phi(0)) \in [0,1] \times \BBT^1$, the solution converges to $E_2$. 
\end{itemize}

\begin{figure}[!htb]
\centering
\newlength{\subfigimgheightK}
\setlength{\subfigimgheightK}{0.2\textheight}

\begin{subfigure}[t]{.32\textwidth}
  \centering
  \includegraphics[
    width=\linewidth,
    height=\subfigimgheightK,
    keepaspectratio,
    valign=b
  ]{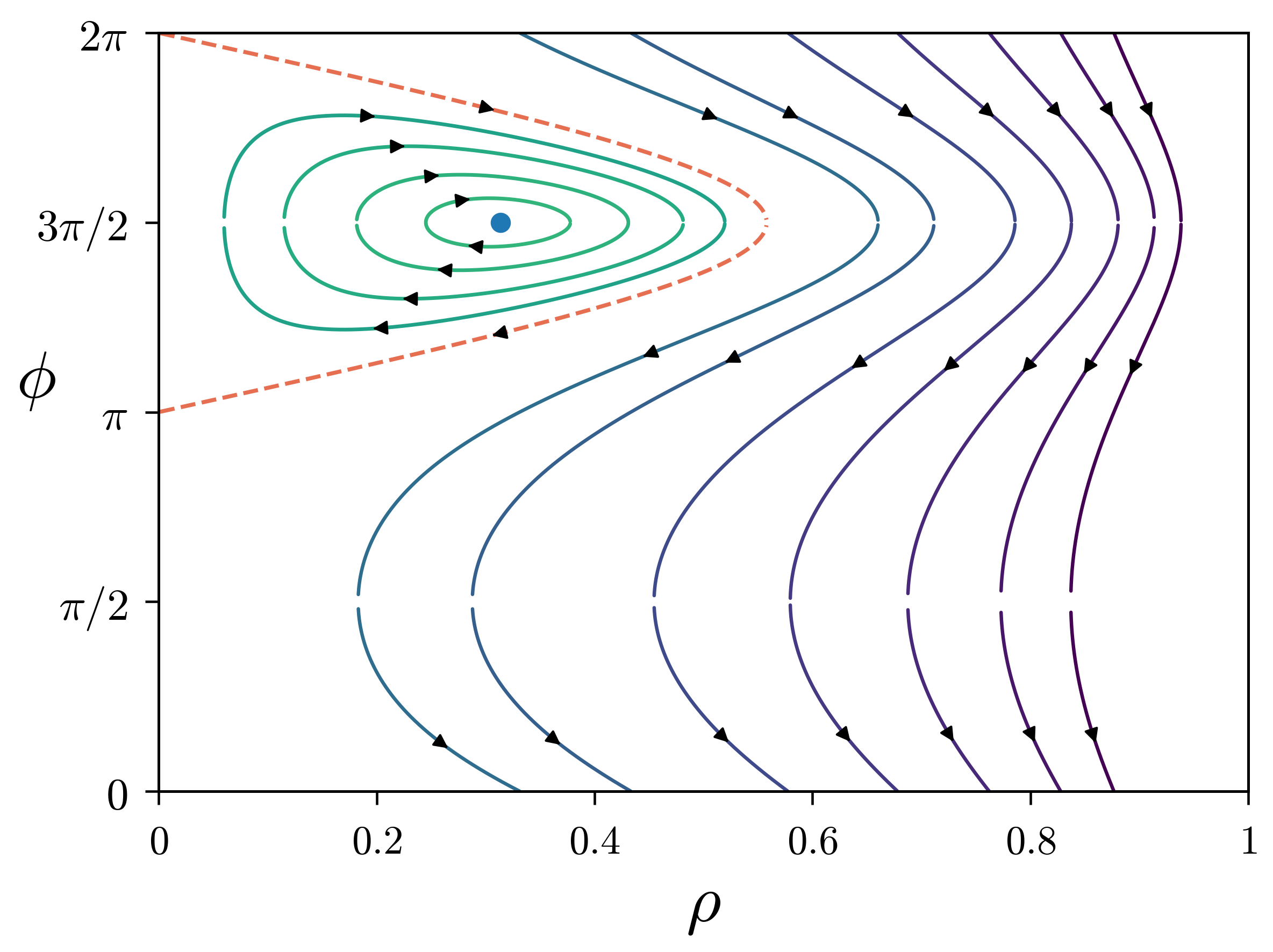}
  \caption{Regime $a < b$.}
  \label{subfig:H_level_set_a_less_than_b}
\end{subfigure}%
\hfill
\begin{subfigure}[t]{.32\textwidth}
  \centering
  \includegraphics[
    width=\linewidth,
    height=\subfigimgheightK,
    keepaspectratio,
    valign=b
  ]{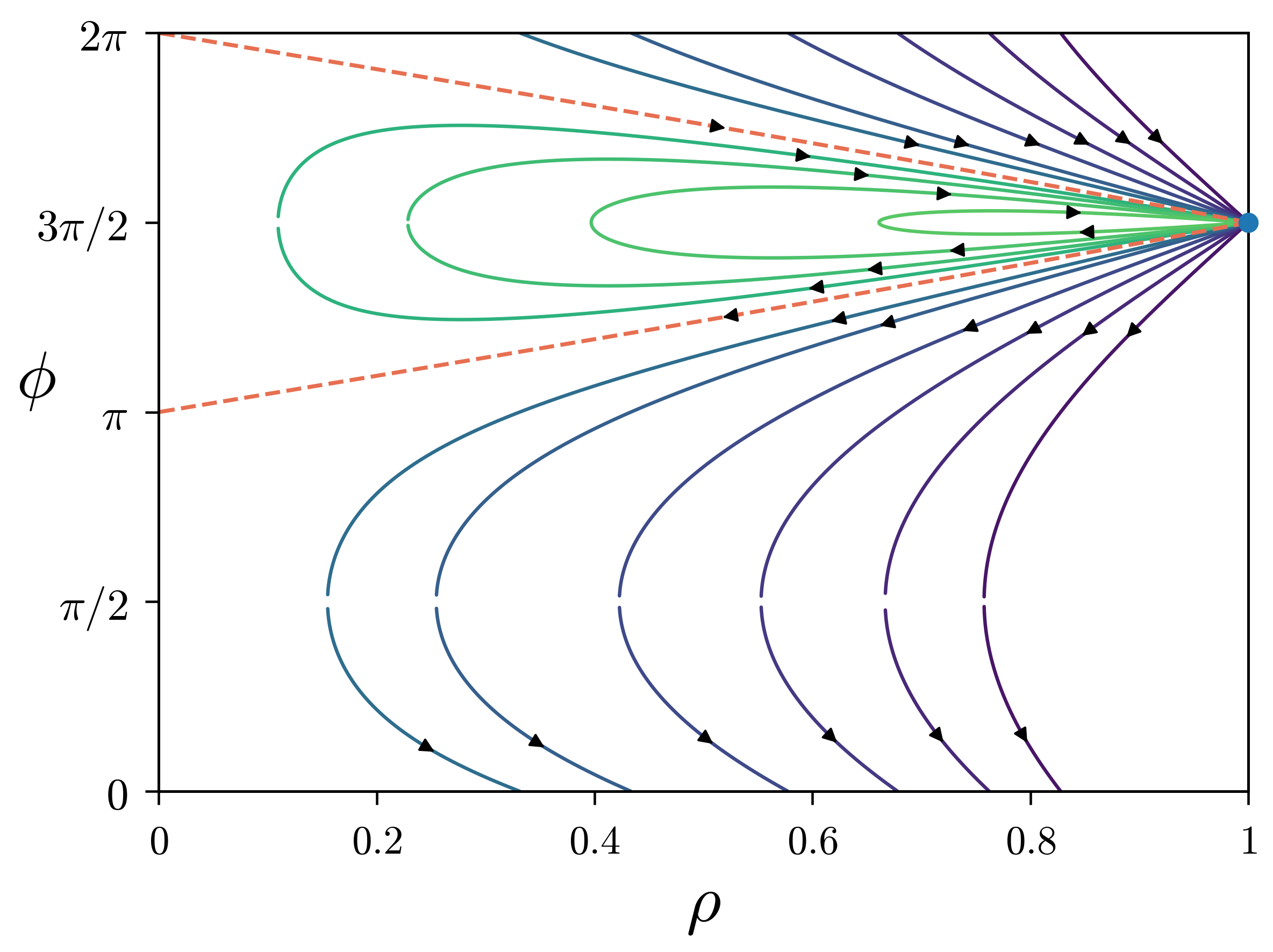}
  \caption{Regime $a = b$.}
  \label{subfig:H_level_set_a_equal_to_b}
\end{subfigure}%
\hfill
\begin{subfigure}[t]{.32\textwidth}
  \centering
  \includegraphics[
    width=\linewidth,
    height=\subfigimgheightK,
    keepaspectratio,
    valign=b
  ]{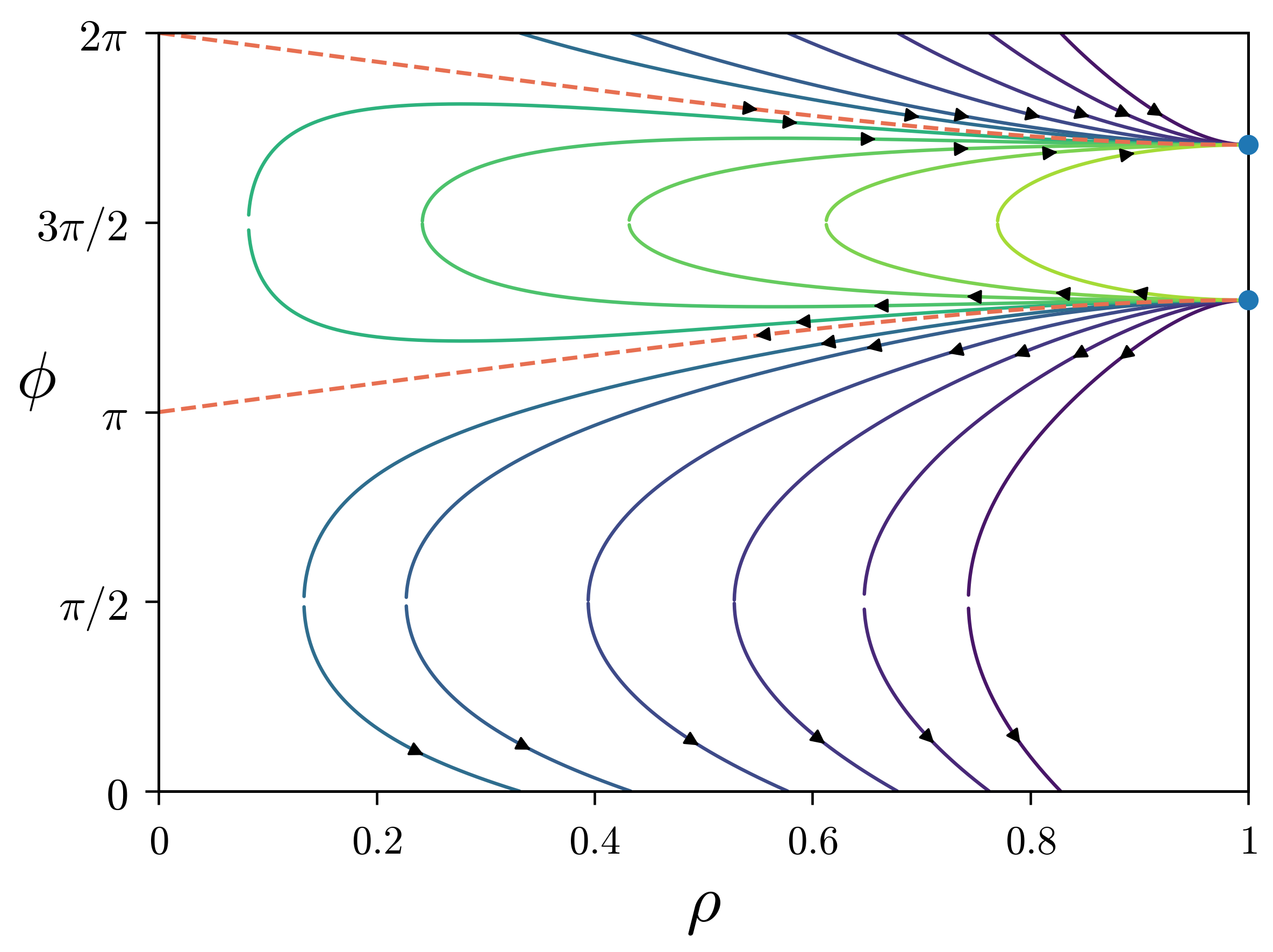}
  \caption{Regime $a > b$.}
  \label{subfig:H_level_set_a_greater_than_b}
\end{subfigure}

\caption{Level sets of the Hamiltonian $H(\rho,\phi)$ for different parameter regimes. 
The dashed orange curve indicates the special level set $\{H=-b/2\}$, while the colored solid curves represent other level sets. 
The blue dots denote the fixed points of the dynamics~\eqref{eq:pc_dyn_case_4}, introduced in~\eqref{eq:case4_E1_E2} and~\eqref{eq:case4_E3}. 
The arrows indicate the direction of the flow.}
\label{fig:H_level_set}
\end{figure}

We introduce some notation that will be used throughout the analysis and record a useful geometric observation.
Define 
\begin{equation}\label{eq:case4_func_W}
W(\rho) = W(\rho;a,b,c) := -\frac{c(1-\rho^2)^2 + \frac{b}{2}(\rho^2 + 1)}{a\rho} \, .
\end{equation}  
Then the level set of $H$ at value $c \in \R$ can be written as
\begin{equation}\label{eq:case4_H_level_set_parametrization}
\left\{ H(\rho, \phi) = c \right\} = \left\{ (\rho, \phi) \in (0,1) \times \BBT^1: \sin \phi = W(\rho) \right\} \, .
\end{equation}
Define 
\begin{equation}\label{eq:case4_set_I}
    I := \left\{ \rho \in (0,1): |W(\rho)| \leq 1 \right\} \, .
\end{equation}
Then the level set can be decomposed into two branches,
\begin{equation}\label{eq:case4_Gamma_1_2}
\left\{ H(\rho, \phi) = c \right\} = \left\{ (\rho, \arcsin (W(\rho))): \rho \in I \right\} \cup \left\{ (\rho, \pi - \arcsin (W(\rho))) : \rho \in I \right\} =: \Gamma_1 \cup \Gamma_2  \, .
\end{equation}
There are two special level sets that we exclude from the analysis below.
First, for any $a, b > 0$, there is a unique level set whose closure intersects the boundary $\{\rho = 0\}$, that is, $\{H = -b/2\}$;
see the dashed orange curve in Figure~\ref{fig:H_level_set}.
Second, when $a < b$, Step 1 shows that the dynamics \eqref{eq:pc_dyn_case_4} has a unique fixed point $E_3$, as defined in \eqref{eq:case4_E3}.
Together with the relation \eqref{eq:case4_pc_vec_and_nabla_H}, this implies that $\nabla H(\rho, \phi) = 0$ only at $E_3$.
We denote the corresponding critical value as $c^*$ and the associated level set by $\{H = c^*\}$.

The long-time behavior of \eqref{eq:pc_dyn_case_4} initialized on either of these exceptional level sets depends sensitively on the precise initial position and is not the focus of the present analysis. 
We therefore omit a detailed study of trajectories lying on these two level sets.

Finally, we record a geometric property of the remaining level sets in the following lemma.

\begin{lemma}\label{lem:case4_H_level_set}
For any $a, b >0$, and $c \in \R \setminus \{-b/2, c^*\}$, the level set $\{H(\rho, \phi) = c\}$ is either empty or a connected one-dimensional manifold.
\end{lemma}

\begin{proof}
Fix $c \in \R \setminus \{-b/2, c^*\}$.
Recall the decomposition of the level set $\{H=c\}$ into the two branches $\Gamma_1$ and $\Gamma_2$ over the set $I$ defined in \eqref{eq:case4_set_I} and \eqref{eq:case4_Gamma_1_2}. 
If $I = \emptyset$, then both $\Gamma_1$ and $\Gamma_2$ are empty, and therefore $\{H=c\}$ is empty.
Next, we focus on $I \neq \emptyset$.
In this case, we will show that both $\Gamma_1$ and $\Gamma_2$ are connected, and $\Gamma_1 \cap \Gamma_2 \neq \emptyset$, which implies that $\{H=c\}$ is connected.
Moreover, since $\nabla H(\rho, \phi) \neq 0$ for every $(\rho, \phi) \in \{H = c\}$, therefore by the regular level set theorem, $\{H=c\}$ is a connected one-dimensional manifold.

We first justify that both $\Gamma_1$ and $\Gamma_2$ are connected.
Since $W$ defined in \eqref{eq:case4_func_W} is continuous, it is enough to show that $I$ is connected.
A direct calculation gives
\[W'(\rho) = \frac{(1-\rho^2)(b + 2c(1 + 3\rho^2))}{2a\rho^2} \, .\]
Since $0 < \rho < 1$, the sign of $W'(\rho)$ is determined by the sign of 
\begin{equation}\label{eq:case4_small_w}
w(\rho) := b + 2c(1 + 3\rho^2) \, .
\end{equation}
The function $w$ is affine in $\rho^2$, and therefore has at most one zero in $(0,1)$.  
Hence $W$ is either strictly monotone, or it has exactly one critical point.
We now split the argument into two cases.

First, suppose $c \leq -\frac{b}{2}$.
Then, for every $\rho \in (0,1)$, $w(\rho) = b + 2c(1 + 3\rho^2) < 0$.
Therefore, $W$ is strictly decreasing on $(0,1)$.
Since $I$ is the preimage of the interval $[-1, 1]$ under a strictly monotone continuous function and is nonempty, then $I$ is an interval and therefore connected.

Next, suppose $c > -\frac{b}{2}$.
If $W$ is monotone, then again $I$ is an interval.
Otherwise, $W$ has a unique critical point, denoted as $\rho_* \in (0,1)$.
At this point,
\[
w(\rho_*) = b + 2c \left(1 + 3 \rho_*^2 \right) = 0 \, .
\]
Substituting this relation into the definition of $W$, we obtain
\begin{equation*}
W (\rho_*) = - \frac{b \rho_* (3 + \rho_*^2)}{a (1 + 3\rho_*^2)} < 0 \, .
\end{equation*}
Since $\rho_*$ is the unique maximum point of $W$, then we know $W(\rho) \leq W(\rho_*) < 0$ for all $\rho \in (0,1)$.
Consequently, 
\begin{equation*}
I = \left\{ \rho \in (0,1): |W(\rho)| \leq 1 \right\} = \left\{ \rho \in (0,1): W(\rho) \geq - 1 \right\} \, .
\end{equation*}
Together with $W$ increasing first and then decreasing, this implies the superlevel set $\{W \geq -1\}$ is an interval.
Therefore, in all cases, $I$ is connected, and consequently both $\Gamma_1$ and $\Gamma_2$ are connected.

It remains to show that $\Gamma_1 \cap \Gamma_2 \neq \emptyset$.
By the definitions of $I$ in \eqref{eq:case4_set_I} and the branch decomposition in \eqref{eq:case4_Gamma_1_2}, it suffices to show that there exists $\kappa \in (0,1)$ such that $|W(\kappa)| = 1$.
We split the argument into two regimes.

First, consider the regime $a < b$.
Let $\kappa_R := \sup I$.
Since $W$ extends continuously to $\rho=1$ and $W(1) = -b/a < - 1$, there exists a neighborhood of $\rho=1$ on which $W < -1$.
Together with $I \neq \emptyset$, this implies that $\kappa_R < 1$.
Moreover, because $c \neq -b/2$, we have $|W(\rho)| \to \infty$ as $\rho \to 0^+$, so $I$ is bounded away from $\rho=0$.
Hence $\kappa_R \in (0,1)$.
By the continuity of $W$ and the definition of $I$, we get $\kappa_R \in I$ and $W(\kappa_R) = -1$.
Therefore, the two branches coincide at the point $(\kappa_R, 3\pi/2)$, and hence $\Gamma_1 \cap \Gamma_2 \neq \emptyset$.

Next, consider the regime $a \geq b$.
Let $\kappa_L := \inf I$.
Since $c \neq -b/2$, we have
\[
\lim_{\rho\to0^+} W(\rho)=-\infty
\quad\text{if } c>-b/2,
\qquad
\lim_{\rho\to0^+} W(\rho)=+\infty
\quad\text{if } c<-b/2 \, .
\]
Thus $I$ is bounded away from $\rho=0$, and so $\kappa_L > 0$.
Moreover, we have $W(1) = -b/a \geq -1$.
Then by the continuity of $W$ and the definition of $I$, we know that $\kappa_L \in I$ and $|W(\kappa_L)| = 1$.
More precisely,
\[
W(\kappa_L)=-1
\quad\text{if } c>-b/2,
\qquad
W(\kappa_L)=1
\quad\text{if } c<-b/2.
\]
Consequently, the two branches coincide at $(\kappa_L,3\pi/2)$ when $c > -b/2$, and at $(\kappa_L,\pi/2)$ when $c < -b/2$. 
This proves that $\Gamma_1 \cap \Gamma_2 \neq \emptyset$, completing the proof.
\end{proof}

Now, we are ready to study the structure of the level sets of $H$ based on the relation between the parameters $a$ and $b$.

\textbf{Regime $\mathbf{a < b}$.}
Fix $c \in \R \setminus \{-b/2, c^*\}$. 
By Lemma~\ref{lem:case4_H_level_set}, the level set $\{H=c\}$ is a connected one-dimensional manifold whenever it is nonempty.
Thus it remains to show that $\{H = c\}$ is compact as a subset of the open cylinder $(0,1) \times \BBT^1$. 

To this end, we first note that, by the definition of $H$ in \eqref{eq:case_4_energy}, for every $\phi \in \BBT^1$, we have $H(\rho, \phi) \to -\infty$ as $\rho \to 1^-$.  
Together with the continuity of $H$, this implies that, for any finite $c \in \R$, the level set $\{H = c\}$ stays bounded away from the boundary $\{\rho = 1\}$.  
Combining this with the fact that $\{H = -b/2\}$ is the only level set whose closure intersects with the boundary $\{\rho=0\}$, we see that, for every $c \neq -\frac{b}{2}$, there exist $0 < \rho_-^c < \rho_+^c < 1$ such that $\{H = c\} \subset K_c$ with $K_c := [\rho_-^c, \rho_+^c] \times \BBT^1$.   
Since $H$ is continuous, $\{H = c\}$ is closed in $K_c$.  
Because $K_c$ is compact, it follows that $\{H = c\}$ is also compact. 
We thus conclude that, for every $c \in \R \setminus \{-b/2, c^*\}$, the level set $\{H=c\}$ is a compact, connected one-dimensional manifold whenever it is nonempty.

\begin{remark}
We record a qualitative consequence of the above analysis in the regime $a < b$. 
Since distinct level sets are disjoint, and since no level set other than $\{H = -\frac{b}{2}\}$ can approach the boundary $\{\rho = 0\}$, the exceptional curve $\{H = -\frac{b}{2}\}$ separates the cylinder $(0,1) \times \BBT^1$ into two regions.  
The region $\{H > -\frac{b}{2}\}$ consists of closed loops on the cylinder, corresponding to libration in the dynamics~\eqref{eq:pc_dyn_case_4}.  
By contrast, the region $\{H < -\frac{b}{2}\}$ consists of loops that wrap around the cylinder, corresponding to rotation.
This is also illustrated in Figure \ref{subfig:H_level_set_a_less_than_b}.
\end{remark}

\textbf{Regime $\mathbf{a \geq b}$.} 
Fix $c \in \R \setminus \{-b/2, c^*\}$.
By Lemma~\ref{lem:case4_H_level_set}, the level set $\{H=c\}$, whenever nonempty, is a connect one-dimensional manifold.
It remains to show that the closure of $\{H=c\}$ intersects the boundary $\{\rho=1\}$.
In fact, the interaction points are precisely the boundary fixed points $E_1$ and $E_2$ defined in \eqref{eq:case4_E1_E2}; see Figures \ref{subfig:H_level_set_a_equal_to_b} and \ref{subfig:H_level_set_a_greater_than_b}.

Recall the definition of set $I$ in \eqref{eq:case4_set_I} and the branch decomposition of $\{H = c\}$ in \eqref{eq:case4_Gamma_1_2}. 
Together with the continuity of function $W$ (see \eqref{eq:case4_func_W}), it suffices to show that the closure of $I$ reaches $\rho=1$ whenever $I \neq \emptyset$.

If $a > b$, this is immediate.
Indeed, $W$ is continuous, and $W(1) = -\frac{b}{a} > -1$.
Hence $|W(\rho)| \leq 1$ for all $\rho$ sufficiently close to $1$, and therefore the closure of $I$ reaches $\rho=1$.

It remains to consider the borderline case $a = b$.
In this case, $W(1) = -1$.
Moreover, a direct calculation gives
\begin{equation}\label{eq:case4_W_rho_plus_1}
W(\rho) + 1 = - \frac{(1-\rho)^2}{a \rho} \left( \frac{a}{2} + c(1 + \rho)^2 \right) \, .
\end{equation}
Since $I \neq \emptyset$, there exists some $\rho \in (0,1)$ such that $|W(\rho)| \leq 1$ and hence $W(\rho) + 1 \geq 0$.
By \eqref{eq:case4_W_rho_plus_1}, $W(\rho) + 1 \geq 0$ can occur only if $a/2 + c(1 + \rho)^2 \leq 0$ for some $\rho \in (0,1)$, which forces $c < -a/8$.
Consequently, for all $\rho$ sufficiently close to $1$, we have $a/2 + c(1+\rho)^2 < 0$ and therefore $W(\rho) + 1 > 0$.
Together with $W(1) = -1$ and the definition of $I$, this implies that the closure of $I$ reaches $\rho=1$, completing the proof.



\end{proof}

\subsection{Auxiliary Lemmas}
\begin{lemma}\label{lem:pc_dyn_general_expand}
The polar-coordinate system \eqref{eq:pc_dyn_general} admits the equivalent representation
\begin{equation*}
    \begin{split}
        \Drho = 2(1-\rho^2) \left( S_1(\phi) \rho + S_2(\phi) \right), \qquad
        \Dphi = 2 \left( S_3(\phi) (\rho^2+1) + S_4(\phi) \frac{\rho^2 + 1}{\rho} + 2S_5(\phi) \rho + \re{b_5} \right),
    \end{split}
    \end{equation*}
where both functions $S_{j}$ constant $b_5$ depend on the matrices $A$ and $V$.
The explicit form of $S_j$ and $b_5$ are provided in \eqref{eq:func_Sj} and \eqref{eq:coeff_b}, respectively.
\end{lemma}

\begin{proof}
    For notational convenience we drop the time dependence on all variables.
    
    Substituting $\CR_2 = \rho e^{i\phi}$ and $\overline{\CR_2} = \rho e^{-i\phi}$ into the ODE \eqref{eq:dyn_r2} satisfied by $\CR_2$ gives
    \begin{equation}\label{eq:pc_transform}
        \frac{d}{dt} \CR_2 = \frac{d}{dt} (\rho e^{i\phi}) = e^{i\phi} \left( \Drho + i \rho \Dphi \right) .
    \end{equation}
    By using the definitions of $B$ and $C$ from \eqref{eq:coeff_B_C}, we can expand the left-hand side of the above equation:
    \begin{equation*}
    \begin{aligned}
        \mathrm{LHS} &= \frac{d}{dt} \CR_2 = 2i \left( B(\CR_2) (\CR_2)^2 + C(\CR_2) \CR_2 + \overline{B}(\CR_2) \right) \\
        &= 2i \left( b_1 \CR_2^3 + b_2 \CR_2^2\overline{\CR_2} + b_3\CR_2^2 + b_4 \CR_2^2 + \overline{b_4} \overline{\CR_2}\CR_2 + b_5 \CR_2 + \overline{b_1} \overline{\CR_2} + \overline{b_2} \CR_2 + \overline{b_3} \right) \\
        &= 2i \left( b_1 \rho^3 e^{i3\phi} + b_2 \rho^3 e^{i\phi} + (b_3 + b_4) \rho^2 e^{i2\phi} + \overline{b_4} \rho^2 + (\overline{b_2} + b_5) \rho e^{i\phi} + \overline{b_1} \rho e^{-i\phi} + \overline{b_3} \right)\\
        &= e^{i\phi} \left(2 \left( ib_1e^{i2\phi} \rho^3 + i \overline{b_1} e^{-i2\phi} \rho + i(b_3 + b_4) e^{i\phi} \rho^2 + i (\overline{b_3} + \overline{b_4} \rho^2) e^{-i\phi} + i \left( b_2 \rho^3 + (\overline{b_2} + b_5) \rho \right) \right) \right) = e^{i\phi} G(\rho, \phi),
    \end{aligned}
    \end{equation*}
    where we defined $G$ implicitly in the last step.
   
    Comparing the two expressions for the left-hand and right-hand sides of~\eqref{eq:pc_transform} results in the complex-valued identity
    \begin{equation}
    \label{eq:complex_iden}
        \Drho + i \rho \Dphi = G(\rho, \phi).
    \end{equation}
    Recalling that $\rho$ and $\phi$ are both real-valued,
    we separate the real and imaginary parts of the function $G$.
    In particular, letting
    \begin{equation*}
        G(\rho, \phi)
        = G_R (\rho, \phi) + i G_I (\rho, \phi)
        \qquad
        \text{with}
        \qquad
        G_R := \re{G}
        \text{ and }
        G_I := \im{G},
    \end{equation*}
    we can match the real and imaginary parts of \eqref{eq:complex_iden} to obtain the real-valued coupled ODE system
    \begin{equation}
        \label{eq:aux_coupled_dyn}
        \Drho = G_R(\rho, \phi), \qquad
        \Dphi = \frac{1}{\rho} G_I(\rho, \phi),
    \end{equation}
    and it remains to write down $G_R$ as well as $G_I$ explicitly.
    
    To do so, we introduce some notation and set
    \begin{equation*}
    \begin{aligned}
        G(\rho, \phi)
        &= 2 \left( ib_1e^{i2\phi} \rho^3 + i \overline{b_1} e^{-i2\phi} \rho + i(b_3 + b_4) e^{i\phi} \rho^2 + i (\overline{b_3} + \overline{b_4} \rho^2) e^{-i\phi} + i \left( b_2 \rho^3 + (\overline{b_2} + b_5) \rho \right) \right)\\
        &= 2 \left( A_1 + A_2 + A_3 + A_4 + A_5 \right).
    \end{aligned}
    \end{equation*}
    We can then compute for each term $A_j$ the real and imaginary part.
    \begin{itemize}
        \item Term $A_1$:
        \begin{equation*}
        \begin{aligned}
        \re{A_1} = \re{ib_1 e^{i2\phi} \rho^3} = - \im{b_1 e^{i2\phi}} \rho^3, \qquad 
        \im{A_1} = \im{ib_1 e^{i2\phi} \rho^3} = \re{b_1 e^{i2\phi}} \rho^3.
        \end{aligned}
        \end{equation*}
    
        \item Term $A_2$:
        \begin{equation*}
        \begin{aligned}
        \re{A_2} = \re{i\overline{b_1} e^{-i2\phi} \rho} = \im{b_1 e^{i2\phi}} \rho, \qquad \im{A_2} = \im{i\overline{b_1} e^{-i2\phi} \rho} = \re{b_1 e^{i2\phi}} \rho .
        \end{aligned}
        \end{equation*}
    
        \item Term $A_3$:
        \begin{equation*}
        \begin{aligned}
        &\re{A_3} = \re{i(b_3 + b_4) e^{i\phi} \rho^2} = - \im{(b_3 + b_4) e^{i\phi}} \rho^2,\\[5pt]
        &\im{A_3} = \im{i(b_3 + b_4) e^{i\phi} \rho^2} = \re{(b_3 + b_4) e^{i\phi}} \rho^2 .
        \end{aligned}
        \end{equation*}
    
        \item Term $A_4$:
        \begin{equation*}
        \begin{aligned}
        &\re{A_4} = \re{i (\overline{b_3} + \overline{b_4} \rho^2) e^{-i\phi}} = \im{b_3 e^{i\phi}} + \im{b_4 e^{i\phi}} \rho^2,\\[5pt]
        &\im{A_4} = \im{i (\overline{b_3} + \overline{b_4} \rho^2) e^{-i\phi}} = \re{b_3 e^{i\phi}} + \re{b_4 e^{i\phi}} \rho^2 .
        \end{aligned}
        \end{equation*}
    
        \item Term $A_5$:
        \begin{equation*}
        \begin{aligned}
        &\re{A_5} = \re{i \left( b_2 \rho^3 + (\overline{b_2} + b_5) \rho \right)} = -\im{b_2} \rho^3 + \im{b_2} \rho,\\[5pt]
        &\im{A_5} = \im{i \left( b_2 \rho^3 + (\overline{b_2} + b_5) \rho \right)} = \re{b_2} \rho^3 +  (\re{b_2} + \re{b_5}) \rho .
        \end{aligned}
        \end{equation*}
    \end{itemize}
    For the last term, recall that $b_5$ is real-valued.
    
    With those equalities we have 
    \begin{equation*}
    \begin{aligned}
    G_R(\rho, \phi) &= 2 \left( \re{A_1} + \re{A_2} + \re{A_3} + \re{A_4} + \re{A_5} \right)\\
    &= 2 (1-\rho^2)\left( \im{b_1 e^{i2\phi} + b_2} \rho  + \im{b_3 e^{i\phi}} \right),
    \end{aligned}
    \end{equation*}
    and
    \begin{equation*}
    \begin{aligned}
    G_I(\rho, \phi) &= 2 \left( \im{A_1} + \im{A_2} + \im{A_3} + \im{A_4} + \im{A_5} \right)\\
    &= 2 \left( (\rho^2 + 1) \left(  \re{b_1 e^{i2\phi} + b_2} \rho + \re{b_3 e^{i\phi}} \right) + \rho \left(2 \re{b_4 e^{i\phi}} \rho + \re{b_5} \right) \right) .
    \end{aligned}
    \end{equation*}
    By introducing the notation
    \begin{equation}
    \label{eq:func_Sj}
    \begin{aligned}
    &\qquad \qquad S_1(\phi) := \im{b_1 e^{i2\phi} + b_2}, \qquad S_2(\phi) := \im{b_3 e^{i\phi}},\\[5pt]
    &S_3(\phi) := \re{b_1 e^{i2\phi} + b_2}, \qquad S_4(\phi) := \re{b_3 e^{i\phi}}, \qquad S_5(\phi) := \re{b_4 e^{i\phi}},
    \end{aligned}
    \end{equation}
    we eventually obtain the coupled ODE system
    \begin{equation*}
    \begin{split}
        \Drho &= 2 (1-\rho^2) \left( S_1(\phi) \rho + S_2(\phi) \right)\\
        \Dphi &= 2 \left( S_3(\phi) (\rho^2+1) +  S_4(\phi) \frac{\rho^2 + 1}{\rho} + 2 S_5(\phi) \rho + \re{b_5} \right),
    \end{split}
    \end{equation*}
    which concludes the proof.
\end{proof}

\begin{lemma}\label{lem:gamma-integral}
Let $U:[0,\infty)\to[0,1)$ be an absolutely continuous function satisfying
\begin{equation}\label{eq:gamma-ode}
\dot U(t)=(1-U(t)^2)\,S(t)
\end{equation}
for some locally integrable function $S:[0,\infty)\to\mathbb R$.
Assume that
\begin{equation}\label{eq:S-integral}
\lim_{t\to\infty}\int_0^t S(\tau)\,d\tau=\infty .
\end{equation}
Then
\[
\lim_{t\to\infty}U(t)=1 .
\]
\end{lemma}

\begin{proof}
For $U(t)\in[0,1)$ define
\[
u(t):=\operatorname{artanh}(U(t))
=\frac12\log\frac{1+U(t)}{1-U(t)} .
\]
Since $\frac{d}{dU}\operatorname{artanh}(U)=(1-U^2)^{-1}$, the chain rule yields
\[
\dot u(t)=\frac{\dot U(t)}{1-U(t)^2}=S(t)
\quad\text{for a.e.\ }t\ge0 .
\]
Integrating in time, we obtain
\[
u(t)=u(0)+\int_0^t S(\tau)\,d\tau .
\]
By assumption \eqref{eq:S-integral}, the right-hand side diverges to $\infty$ as $t\to\infty$.
Therefore $u(t)\to\infty$, and since $U(t)=\tanh(u(t))$, it follows that
\[
U(t)\xrightarrow[t\to\infty]{}1 .
\]
\end{proof}

\section{Proof Details for Section \ref{sec:stability}}\label{app:stability}
In this appendix, we provide the technical details for the proofs of Section~\ref{sec:stability}.

\subsection{Proof of Main Results}

\begin{proof}[Proof of Corollary \ref{cor:diff_r2_alpha_pc_bound}]
We prove the three inequalities in \eqref{eq:diff_r2_alpha_pc_bound} one by one using elementary inequalities together with the estimate established in Proposition \ref{prop:diff_r2_alpha}, namely,
\begin{equation*}
|\CR_2 - \alpha| \leq C_{\nu}(\gamma) (1 - \gamma) \, .
\end{equation*}
We recall the polar-representation $\alpha = \gamma e^{i\Phi}$ and $\CR_2 = \rho e^{i\phi}$.

We begin with the first inequality. By the reverse triangle inequality,
\begin{equation*}
|\rho - \gamma| = ||\CR_2|-|\alpha||\leq |\CR_2-\alpha| \leq C_{\nu}(\gamma) (1 - \gamma) \, . 
\end{equation*}
 
Next, for the second inequality, we write
\[
\CR_2-\alpha = \rho e^{i\phi}-\gamma e^{i\Phi} = (\rho-\gamma)e^{i\phi} + \gamma\bigl(e^{i\phi}-e^{i\Phi}\bigr) \, .
\]
Rearranging terms and applying the triangle inequality gives
\[
\gamma |e^{i\phi}-e^{i\Phi}| \leq |\CR_2-\alpha|+|\rho-\gamma| \leq 2C_\nu(\gamma)(1-\gamma),
\]
and therefore
\begin{equation*}
    |e^{i\phi}-e^{i\Phi}| \leq 2C_\nu(\gamma) \frac{(1-\gamma)}{\gamma} \, .
\end{equation*}

Finally, we prove the third inequality. Choosing the representation of $\Phi - \phi$ modulo $2\pi$ so that $|\Phi-\phi|\leq \pi$, we have
\[
|\Phi-\phi|
=
2\left|\frac{\Phi-\phi}{2}\right|
\leq
2\pi \left|\sin\frac{\Phi-\phi}{2}\right| = \pi |e^{i\phi}-e^{i\Phi}|
\leq
2\pi C_\nu(\gamma) \frac{1-\gamma}{\gamma}.
\]
This completes the proof.
\end{proof}

\begin{proof}[Proof of Theorem \ref{thm:case_1stab}]
We start the proof with a change of variables.
Define a scalar
\begin{equation*}
    R := \sqrt{(\aDMinus)^2 + (\aOPlus)^2},
\end{equation*}
and an angle $\PhiA$ satisfying
\begin{equation*}
\aDMinus = R \cos \PhiA, \qquad \aOPlus = R \sin \PhiA \, .
\end{equation*}
When $R=0$, we set $\PhiA = 0$.
Then we can compute using trigonometric identities that
\begin{equation*}
\begin{aligned}
&\aDMinus \cos \Phi + \aOPlus \sin \Phi = R \cos \PhiA  \cos \Phi + R \sin \PhiA \sin \Phi = R \cos (\Phi - \PhiA),\\[4pt]
&\aDMinus \sin \Phi - \aOPlus \cos \Phi = R \cos \PhiA \sin \Phi - R \sin \PhiA \cos \Phi = R \sin (\Phi - \PhiA) \, .
\end{aligned}
\end{equation*}
Therefore, the original system \eqref{eq:alpha_pc_dyn_case1} can be written as
\begin{equation}\label{eq:case1_stab_pc_dyn_app}
\begin{aligned}
\dot{\gamma} &= \frac{1}{4}(1 - \gamma^2) \left( \aDPlus \gamma + R \cos (\Phi - \PhiA) \right) + \DeltaGamma,\\
\dot{\Phi} &= \frac{1}{4} (1 - \gamma^2) \left( \aOMinus - \frac{R}{\gamma} \sin (\Phi - \PhiA) \right) + \DeltaPhi \, .
\end{aligned}
\end{equation}

By Lemma \ref{lem:stability_case1_case2}, the error terms $\DeltaGamma$ and $\DeltaPhi$ satisfy
\begin{equation*}
|\DeltaGamma| \leq C_{A,V} C_{\nu}(\gamma) \frac{1 - \gamma}{\gamma} (1 - \gamma^2), \qquad |\DeltaPhi| \leq C_{A, V} C_{\nu} (\gamma) \frac{1 - \gamma}{\gamma},
\end{equation*}
where $C_{A,V}$ is a constant that depends on $A$ and $V$.

Next, we focus on the convergence of the modulus $\gamma$. 
First, we show that there exists a threshold value $\gammaTH \in [0,1]$ so that, if $\gamma(t_0) \in [\gammaTH, 1]$ for some $t_0$, then $\gamma(t)$ converges to $1$ exponentially fast for $t \geq t_0$.
In particular, since the matrix $A$ satisfies condition \eqref{eq:case_1_stab_A_condi}, together with the definitions of $\aDPlus$ and $R$, there exists a constant $\delta > 0$ so that 
\begin{equation*}
   \lambda_{\text{min}} \left(A + A^{\top} \right) = \aDPlus - R \geq \delta \, .
\end{equation*}
This in turn implies that
\begin{equation*}
\aDPlus \gamma + R \cos (\Phi - \PhiA) \geq (R + \delta) \gamma - R = \delta \gamma - R (1 - \gamma) \, .
\end{equation*}
We infer that
\begin{equation}\label{eq:case_1stabgrowth}
     \dot{\gamma}\geq \frac{1}{4}(1-\gamma^2)(\delta\gamma - R(1-\gamma) -C_{A,V} C_\nu(\gamma)\frac{1-\gamma}{\gamma})\geq \frac{1}{8}(1-\gamma^2)\delta \gammaTH,
\end{equation}
for all $\gamma \in [\gammaTH, 1]$, provided that $\gammaTH$ is sufficiently close to $1$.
This follows from Assumption \ref{asm:Lp} and Lemma \ref{lem:nu-lp-unified}, which imply that $C_\nu(\gamma)(1-\gamma) \rightarrow 0$ as $\gamma\to1$.
We note that $\gammaTH$ depends on the constant $M$ which controls the $L^p$-norm of $\nu$. 

From the above argument, we know that if there exists a time $0 < t_0 < \infty$ so that $\gamma(t_0) \in [\gammaTH, 1]$, then $\gamma(t)$ remains in this interval and increases monotonically toward the boundary value $1$.
In particular, inequality \eqref{eq:case_1stabgrowth} continues to hold for all $t \geq t_0$.
Moreover, 
\begin{equation*}
\frac{d}{dt} (1 - \gamma) \leq -\frac{\delta}{8} \gammaTH (1 - \gamma^2) \leq - c (1 - \gamma), \qquad \text{where } c := \frac{\delta}{8} \gammaTH (1 + \gammaTH) > 0 \, .
\end{equation*}
Then Grönwall's inequality gives the exponential convergence of $\gamma(t)$ to $1$, namely, for $t \geq t_0$,
\begin{equation}\label{eq:case1_stab_gamma_exp_conv}
    1 - \gamma(t) \leq (1 - \gammaTH) \exp \left( - c (t - t_0) \right) \, .
\end{equation}
Therefore, it remains to show that for any given initial data $\gamma(0)$ satisfying our assumptions, there exists $t_0 \in (0,\infty) $ such that $\gamma(t_0) \geq \gammaTH$.

Let $\alpha_{\OA} = \gamma_{\OA} e^{i\Phi_{\OA}}$ be the solution of the OA system \eqref{eq:pc_dyn_case_1}, with initial data $\alpha(0)$ satisfying~\eqref{eq:case1_stab_initial_data}.
By Theorem~\ref{thm:case_1} and Proposition~\ref{prop:case1_quantitative}, under the condition~\eqref{eq:case_1_stab_A_condi} on the matrix $A$
there exists a time $t_0 > 0$, depending on $\iota$ in~\eqref{eq:case1_stab_initial_data}, such that $\gamma_{\OA} (t_0) \geq \frac{1 + \gammaTH}{2}$.
Then, we conclude from Lemma \ref{lem:stability_estimate_control} and Corollary \ref{cor:diff_r2_alpha_pc_bound} that there exists a constant $\varepsilon(\iota, p, M, A) > 0$, depending on $\iota$, the matrix $A$, and $p, M$ through $\gammaTH$, that is sufficiently small so that the closeness condition~\eqref{eq:case_1_stab_close_to_unif} implies
\begin{equation*}
    |\gamma(t_0) - \gamma_{\OA}(t_0)| \leq \frac{1 - \gammaTH}{2} \, .
\end{equation*}
The reverse triangle inequality then implies
\begin{equation*}
\gamma(t_0) \geq \gamma_{\OA} (t_0) - \frac{1 - \gammaTH}{2} \geq \gammaTH \,,
\end{equation*}
which is what we needed to conclude.

Finally, we establish the convergence of the angular variable $\Phi$. 
From the $\Phi$-equation in \eqref{eq:case1_stab_pc_dyn_app}, Lemma \ref{lem:stability_case1_case2}, and the fact that $\gamma(t) \in [\gammaTH, 1]$, for all $t \geq t_0$, we obtain
\begin{equation*}
|\DPhi| = \left| \frac{1}{4} (1 - \gamma^2) \left( \aOMinus - \frac{R}{\gamma} \sin (\Phi - \PhiA) \right) + \DeltaPhi \right| \leq C (1 - \gamma) + C_{A,V} C_{\nu}(\gamma) \frac{1 - \gamma}{\gamma} \leq C \left( 1 + C_{\nu}(\gamma) \right) (1 - \gamma) \, .
\end{equation*}
Here, the constant $C > 0$ may change from one inequality to the next.
By Assumption \ref{asm:Lp}, Lemma \ref{lem:nu-lp-unified}, and the exponential convergence of $\gamma(t)$ to $1$ established in \eqref{eq:case1_stab_gamma_exp_conv}, the right-hand side of above inequality is integrable over $[t_0, \infty)$, and hence $|\DPhi| \in L^1 ([t_0, \infty))$.
This implies that $\Phi(t)$ converges to some $\Phi_{\infty} \in \BBT^1$ as $t \rightarrow \infty$.

\end{proof}

\begin{proof}[Proof of Theorem \ref{thm:case_2stab}]
We start the proof with a change of variables.
Define a scalar
\begin{equation*}
    R := \sqrt{(\vDMinus)^2 + (\vOPlus)^2},
\end{equation*}
and an angle $\PhiV$ satisfying
\begin{equation}\label{eq:case2_PhiV}
\vDMinus = R \cos \PhiV, \qquad \vOPlus = R \sin \PhiV \, .
\end{equation}
When $R=0$, we set $\PhiV = 0$.
Then we can compute using trigonometric identities that
\begin{equation*}
\begin{aligned}
&\vDMinus \cos \Phi + \vOPlus \sin \Phi = R \cos \PhiV  \cos \Phi + R \sin \PhiV \sin \Phi = R \cos (\Phi - \PhiV),\\[4pt]
&\vDMinus \sin \Phi - \vOPlus \cos \Phi = R \cos \PhiV \sin \Phi - R \sin \PhiV \cos \Phi = R \sin (\Phi - \PhiV) \, .
\end{aligned}
\end{equation*}
Therefore, the original system \eqref{eq:alpha_pc_dyn_case2} can be written as
\begin{equation*}
\begin{aligned}
\Dgamma(t) &= \frac{1}{4} (1 - \gamma(t)^2) \left( \vDPlus \gamma(t) + R \cos (\Phi(t) - \PhiV) \right) + \DeltaGamma(t), \\
\DPhi(t) &= - \frac{1}{4} \frac{3\gamma(t)^2 + 1}{\gamma(t)} R \sin (\Phi(t) - \PhiV) + \DeltaPhi(t) \, .
\end{aligned}
\end{equation*}

\textit{Claim:} There exist constants $\gammaTH \in [0,1]$ and $\PhiTH > 0$ so that, if $(\gamma(t_0), \Phi(t_0))$ satisfies  
\begin{equation}\label{eq:case2_TH}
\gamma(t_0) \in [\gammaTH, 1] \quad \text{and} \quad |\Phi(t_0) - \PhiV| \leq \PhiTH
\end{equation}
for some $t_0>0$, then $(\gamma(t), \Phi(t)) \rightarrow (1, \PhiV)$ as $t \rightarrow \infty$.\\

Assuming momentarily that the above claim has been proved, it would remain to show that for any initial condition $(\gamma(0), \Phi(0))$ satisfying our assumptions there exists a time $t_0 > 0$ such that $(\gamma(t_0), \Phi(t_0))$ satisfies condition \eqref{eq:case2_TH}. 
To verify the latter assertion, let $\alpha_{\OA}(t) = \gamma_{\OA}(t) e^{i \Phi_{\OA}(t)}$ be the solution to the OA system \eqref{eq:pc_dyn_case_2}, with initial data $\alpha_{\OA}(0) = \alpha(0)$ satisfying \eqref{eq:case2_stab_initial_data}. By Theorem \ref{thm:case_2} and Proposition \ref{prop:case2_quantitative}, under the condition \eqref{eq:case_2_stab_V_condi} on the matrix $V$, there exists a time $t_0 > 0$, depending on $\iota$ in \eqref{eq:case2_stab_initial_data}, such that $\gamma_{\OA}(t_0) \geq \frac{1 + \gammaTH}{2}$ and $|\Phi_{\OA}(t_0) - \PhiV| \leq \frac{\PhiTH}{2}$. 
Moreover, by Lemma \ref{lem:stability_estimate_control} and Corollary \ref{cor:diff_r2_alpha_pc_bound}, there exists a constant $\varepsilon(\iota, p, M, V) > 0$, depending on $\iota$, the matrix $V$, and $p, M$ through $\gammaTH$, sufficiently small so that the closeness condition \eqref{eq:case_2_stab_close_to_unif} implies
\begin{equation*}
|\gamma(t_0) - \gamma_{\OA}(t_0)| \leq \frac{1 - \gammaTH}{2} \quad \text{and} \quad |\Phi(t_0) - \Phi_{\OA}(t_0)| \leq \frac{\PhiTH}{2} \, .
\end{equation*}
Therefore, by the triangle inequality, we have  
\begin{equation*}
\gamma(t_0) \geq \gamma_{\OA} - \frac{1 - \gammaTH}{2} \geq \gammaTH, \quad \text{and} \quad |\Phi(t_0) - \PhiV| \leq |\Phi(t_0) - \Phi_{\OA}(t_0)| + |\Phi_{\OA}(t_0) - \PhiV| \leq \PhiTH \, .
\end{equation*}

\begin{proof}[Proof of the claim]
We first show the convergence $\gamma(t) \rightarrow 1$, and then show that $\Phi(t) \rightarrow \PhiV$.

\paragraph{Step 1:} $\gamma(t) \rightarrow 1$ as $t \rightarrow \infty$.

Since the matrix $V$ satisfies condition \eqref{eq:case_2_stab_V_condi}, there exists a constant $\delta > 0$ so that $\lambdaMax(V) \geq \delta$.
Together with Lemma \ref{lem:stability_case1_case2}, we can compute that
\begin{equation}\label{eq:gamma_aux_ineq}
\begin{aligned}
\Dgamma(t) &= \frac{1}{4} (1 - \gamma(t)^2) \left( \vDPlus \gamma(t) + R \cos (\Phi(t) - \PhiV) \right) + \DeltaGamma(t)\\
&\geq \frac{1}{4} (1 - \gamma(t)^2) \left( \frac{\delta}{2} \gamma(t) + R \left( \cos (\Phi(t) - \PhiV) - \gamma(t) \right) \right) + \DeltaGamma(t)\\
&\geq \frac{1}{8} (1 - \gamma(t)^2) \delta \left( 1 - \frac{2}{k} \right)\gamma(t) - C_{A, V} C_{\nu} (\gamma(t)) \frac{1 - \gamma(t)}{\gamma(t)} (1 - \gamma(t)^2),
\end{aligned}
\end{equation}
provided that 
\begin{equation}\label{eq:case_2_stab_gamma_Phi_condi_origin}
R \left( \cos \left( \Phi(t) - \PhiV \right) - \gamma(t) \right) \geq - \frac{\delta}{k} \gamma(t)
\end{equation}
with any $k > 0$ to be specified later.
Here, the first inequality uses the condition \eqref{eq:case_2_stab_V_condi} on $V$, while the second inequality makes use of the condition \eqref{eq:case_2_stab_gamma_Phi_condi_origin} and Lemma \ref{lem:stability_case1_case2}.

Furthermore, since $\nu$ satisfies Assumption \ref{asm:Lp}, there exists a constant $\BargammaTH \in [0,1]$, depending on constant $M$, so that
\begin{equation*}
C_{A,V} C_{\nu}(\gamma) \frac{1 - \gamma}{\gamma} \leq \frac{\delta}{16} \left(1 - \frac{2}{k} \right) \gamma 
\end{equation*}
for all $\gamma \in [\BargammaTH, 1]$, provided $k > 2$.
The above inequality, together with \eqref{eq:gamma_aux_ineq}, implies that, within the interval $[\BargammaTH, 1]$, we have
\begin{equation}\label{eq:gamma_growth_ineq}
\Dgamma (t) \geq \frac{\delta}{16} \left( 1 - \frac{2}{k} \right) \gamma(t) (1 - \gamma(t)^2) \, .
\end{equation}
Then by Lemma \ref{lem:gamma-integral}, we know that if $\gamma(t)$ enters the interval $[\BargammaTH,1]$, then $\gamma(t) \rightarrow 1$ as $t \rightarrow \infty$, provided the condition \eqref{eq:case_2_stab_gamma_Phi_condi_origin} holds.
We note that when $R = 0$, condition \eqref{eq:case_2_stab_gamma_Phi_condi_origin} is automatically satisfies whenever $k > 2$.
Therefore, for the remaining proof in Step 1, it suffices to consider the case $R > 0$.

Let $k > 2$ be chosen such that $0 < 1 - \frac{\delta}{kR} < 1$. 
Next, we show that, for a time $t_0 > 0$, if $(\gamma(t_0), \Phi(t_0))$ satisfies the conditions:
\begin{subequations}\label{eq:case2_claim_condi}
\begin{equation}\label{eq:case2_claim_condi_1}
\gamma(t_0) \in [\gammaTH, 1],
\end{equation}
\begin{equation}\label{eq:case2_claim_condi_2}
\left| \Phi(t_0) - \PhiV \right| \leq \PhiTH := \arccos \left( 1 - \frac{\delta}{k R} \right),
\end{equation}
\end{subequations}
where $\gammaTH \geq \BargammaTH$ and $\gammaTH$ satisfies
\begin{equation}\label{eq:case2_claim_condi_3}
\left( \frac{4 C_{A,V}}{R} \frac{S_p(\gammaTH)}{\gammaTH (3\gammaTH^2 + 1)} \right)^2 \leq 1 - \left( 1 - \frac{\delta}{kR} \right)^2,
\end{equation}
then the condition \eqref{eq:case_2_stab_gamma_Phi_condi_origin} holds for all $t \geq t_0$.
Here, the function $S_p: [0,1] \mapsto \R$ is defined by
\begin{equation}\label{eq:Sp_gamma}
S_p (\gamma) := \left\{ \begin{array}{ll}
    C_p M (1 - \gamma)^{1 - 1/p}, & 1 < p < \infty, \\[5pt]
    C_{\infty} M (1-\gamma) \log \frac{2}{1 - \gamma}, & p = \infty,
\end{array} \right.
\end{equation}
where the constants $C_p$, for $1 < p \leq \infty$, are given in Lemma \ref{lem:nu-lp-unified}, and $M$ is introduced in Assumption \ref{asm:Lp}.

We first introduce some notation for later use.
Denote $\WPhi(t) := \Phi(t) - \PhiV$.
Then the angular dynamics of $\WPhi$ can be written as
\begin{equation}\label{eq:case2_angular_eq_aux}
\Dot{\WPhi} (t) = -\frac{1}{4} \frac{3\gamma(t)^2 + 1}{\gamma(t)} R \sin \WPhi (t) + \DeltaPhi (t) \, .
\end{equation}
By Lemma \ref{lem:stability_case1_case2}, Lemma \ref{lem:nu-lp-unified}, and the definition of function $S_p$ in \eqref{eq:Sp_gamma}, the perturbation satisfies
\begin{equation}\label{eq:case2_deltaPhi_bound}
|\DeltaPhi (t)| \leq C_{A, V} C_{\nu} (\gamma(t)) \frac{1 - \gamma(t)}{\gamma(t)} \leq \frac{C_{A,V}}{\gamma(t)} S_p(\gamma(t)) = \frac{1}{4} \frac{3\gamma(t)^2 + 1}{\gamma(t)} R Z(\gamma(t)),
\end{equation}
where we denote
\begin{equation}\label{eq:Z_gamma}
Z(\gamma) := \frac{4 C_{A,V}}{R} \frac{S_p(\gamma)}{3 \gamma^2 + 1} \, .
\end{equation}
Lastly, we recall the constant
\begin{equation}\label{eq:const_chi}
    \PhiTH = \arccos \left(1 - \frac{\delta}{kR} \right) \in (0, \pi/2) \, .
\end{equation}

Next, we prove that the interval $[-\PhiTH, \PhiTH]$ is forward invariant for $\WPhi(t)$; see the definition of constant $\PhiTH$ in \eqref{eq:const_chi}.
That is, we prove that if $\WPhi(t)$ satisfies the initial condition $| \WPhi (t_0) | \leq \PhiTH$, then
\begin{equation*}
 | \WPhi (t) | \leq \PhiTH \quad \text{for all } t \geq t_0 \, .
\end{equation*}
Once this is proved, the condition \eqref{eq:case_2_stab_gamma_Phi_condi_origin} follows immediately.
Indeed, if $|\WPhi(t)| \leq \PhiTH$, together with $\gamma(t) \in [0,1]$, we have
\begin{equation*}
\cos \WPhi (t) \geq \cos \PhiTH = 1 - \frac{\delta}{kR} \geq \left( 1 - \frac{\delta}{kR} \right) \gamma(t) \, .
\end{equation*}
Then, first subtracting $\gamma(t)$ and then multiplying $R$ on both slides to the above inequality gives the condition \eqref{eq:case_2_stab_gamma_Phi_condi_origin}.

Thus, it remains to prove the forward invariance of $[-\PhiTH, \PhiTH]$ with respect to the angular dynamics $\WPhi$ in \eqref{eq:case2_angular_eq_aux}.
Suppose, for contradiction, that $\WPhi(t)$ exits the interval $[-\PhiTH, \PhiTH]$.
Let $T > t_0$ be the first exit time.
Then
\begin{equation*}
|\WPhi (t)| \leq \PhiTH \quad \text{for all } t \in [t_0, T],
\end{equation*}
and at the exit time $T$, one has either
\begin{equation*}
\WPhi (T) = \PhiTH \quad \text{or} \quad \WPhi(T) = -\PhiTH \, .
\end{equation*}

For every $t \in [t_0, T]$, since $|\WPhi (t)| \leq \PhiTH$, the previous argument implies the condition \eqref{eq:case_2_stab_gamma_Phi_condi_origin} holds.
Together with the differential inequality \eqref{eq:gamma_growth_ineq} and $\gamma(t_0) \in [\gammaTH,1]$, we know that $\gamma(t)$ is monotonic increasing on $[t_0, T]$.
Moreover, by the definitions of $S_p$ in \eqref{eq:Sp_gamma} and $Z(\gamma)$ in \eqref{eq:Z_gamma}, the function $Z(\gamma)$ is nonincreasing. 
As a consequence,
\begin{equation*}
Z(\gamma(t)) \leq Z (\gamma(t_0)) \quad \text{for all } t \in [t_0, T] \, .
\end{equation*}
Therefore by the initial condition \eqref{eq:case2_claim_condi_3}, together with $Z(\gamma) \geq 0$ and $\sin \PhiTH \geq 0$, we have
\begin{equation}\label{eq:case2_Z_ineq}
Z(\gamma(t)) \leq Z (\gamma(t_0)) \leq \sin \PhiTH \quad \text{for all } t \in [t_0, T] \, .
\end{equation}

We now check the direction of the vector field for $\WPhi$-dynamics at the boundary point $\WPhi(T) = \PhiTH$ and $\WPhi(T) = -\PhiTH$.

First suppose that $\WPhi(T) = \PhiTH$.
Then by the angular dynamics \eqref{eq:case2_angular_eq_aux} and the perturbation bound \eqref{eq:case2_deltaPhi_bound}, we get
\begin{equation*}
\Dot{\WPhi}(T) = -\frac{1}{4} \frac{3\gamma(T)^2 + 1}{\gamma(T)} R \sin \PhiTH + \DeltaPhi(T) \leq -\frac{1}{4} \frac{3\gamma(T)^2 + 1}{\gamma(T)} R \left( \sin \PhiTH - Z (\gamma(T)) \right) \leq 0,
\end{equation*}
where the last inequality uses \eqref{eq:case2_Z_ineq}.
Thus, at the right endpoint $\WPhi = \PhiTH$, the vector field points inward or is tangent.
This implies that the solution $\WPhi(t)$ cannot exit the interval through the right endpoints.

With similar argument, we can show that the solution $\WPhi(t)$ cannot exit the interval $[-\PhiTH, \PhiTH]$ through the left endpoints either.
Indeed, suppose that $\WPhi(T) = -\PhiTH$, again using the angular dynamics \eqref{eq:case2_angular_eq_aux} and the perturbation bound \eqref{eq:case2_deltaPhi_bound}, we obtain
\begin{equation*}
\Dot{\WPhi}(T) = \frac{1}{4} \frac{3\gamma(T)^2 + 1}{\gamma(T)} R \sin \PhiTH + \DeltaPhi(T) \geq \frac{1}{4} \frac{3\gamma(T)^2 + 1}{\gamma(T)} R \left( \sin \PhiTH - Z (\gamma(T)) \right) \geq 0 \, .
\end{equation*}

This shows that the interval $[-\PhiTH, \PhiTH]$ is forward-invariance with respect to the $\WPhi$-dynamics on $[t_0, T]$, which contradicts the definition of $T$ as the first exit time.
Therefore, we conclude that $|\WPhi(t)| \leq \PhiTH$ for all $t \geq t_0$ and then \eqref{eq:case_2_stab_gamma_Phi_condi_origin} holds for all $t \geq t_0$.
Hence, $\gamma(t) \rightarrow 1$ follows immediately.

We further show that $\gamma(t)$ converges to $1$ exponentially fast for $t \geq t_0$.
Indeed, since $\gamma(t) \in [\gammaTH,1]$ for all $t \geq t_0$, inequality \eqref{eq:gamma_growth_ineq} gives
\begin{equation*}
1 - \Dgamma (t) \leq -\frac{\delta}{16} \left( 1 - \frac{2}{k} \right) \gamma(t) \left( 1 - \gamma(t)^2 \right) \leq -c \left( 1 - \gamma(t) \right), \qquad \text{where } c := \frac{\delta}{16} \left( 1 - \frac{2}{k} \right) \gammaTH (1 + \gammaTH) \, .
\end{equation*}
Then Grönwall's inequality gives
\begin{equation}\label{eq:case2_stab_gamma_exp_conv}
    1 - \gamma(t) \leq (1 - \gammaTH) \exp \left( - c (t - t_0) \right), \qquad t \geq t_0 \, .
\end{equation}

\paragraph{Step 2:} $\Phi(t) \rightarrow \PhiV$ as $t \rightarrow \infty$.

We begin by recalling several formulas and conclusions established above.
In particular, recall the angular dynamics \eqref{eq:case2_angular_eq_aux} of $\WPhi(t) = \Phi(t) - \PhiV$:
\begin{equation*}
\dot \WPhi(t) = -c(t) \sin \WPhi(t) + \DeltaPhi(t),
\end{equation*}
where we define
\begin{equation}\label{eq:case2_ct}
    c(t) := \frac{1}{4} \frac{3\gamma(t)^2 + 1}{\gamma(t)} R \, .
\end{equation}
From the proof in Step 1, if $(\gamma(t_0), \Phi(t_0))$ satisfies condition \eqref{eq:case2_claim_condi}, then
\begin{equation*}
\gamma(t) \in [\gammaTH, 1] \quad \text{and} \quad |\WPhi(t)| \leq \PhiTH \in (0, \pi/2) \quad \text{for all } t \geq t_0 \, .
\end{equation*}
Moreover, $\gamma(t)$ converges to $1$ exponentially fast as $t \rightarrow \infty$ .
We now establish the convergence of the angular variable $\Phi(t)$ by considering the cases $R=0$ and $R>0$ separately.

\textbf{Regime $R = 0$.}
In this case, the $\Phi$-equation in \eqref{eq:case2_pc_dyn_app}, together with Lemma \ref{lem:stability_case1_case2}, gives
\begin{equation*}
|\DPhi| = |\DeltaPhi| \leq C_{A,V} C_{\nu} (\gamma) \frac{1 - \gamma}{\gamma} \leq C C_{\nu} (\gamma) (1 - \gamma),
\end{equation*}
where $C > 0$ is a fixed constant.
By Assumption \ref{asm:Lp}, Lemma \ref{lem:nu-lp-unified}, and the exponential convergence of $\gamma(t)$ to $1$ established in \eqref{eq:case2_stab_gamma_exp_conv}, the right-hand side of above inequality is integrable over $[t_0, \infty)$, and hence $|\dot{\Phi}| \in L^1 ([t_0, \infty))$.
This implies that $\Phi(t)$ converges to some $\Phi_{\infty} \in \BBT^1$ as $t \rightarrow \infty$.

\textbf{Regime $R > 0$.}
In this regime, the coefficient $c(t)$, as defined in \eqref{eq:case2_ct}, is uniformly bounded below by a positive constant,
namely,
\begin{equation*}
c(t) \geq c_* := \frac{R}{4} \min_{\gamma \in [\gammaTH, 1]} \frac{3\gamma^2 + 1}{\gamma} > 0 \quad \text{for all } t \geq t_0 \, .
\end{equation*}
Next, recall from \eqref{eq:case2_deltaPhi_bound}, the perturbation $\DeltaPhi(t)$ satisfies
\begin{equation*}
|\DeltaPhi(t)| \leq C_{A,V} C_{\nu} (\gamma(t)) \frac{1 - \gamma(t)}{\gamma(t)} \, .
\end{equation*}
By Assumption \ref{asm:Lp}, Lemma \ref{lem:nu-lp-unified}, and $\gamma(t) \rightarrow 1$, we obtain
\begin{equation*}
    \DeltaPhi(t) \rightarrow 0 \quad \text{as } t \rightarrow \infty \, .
\end{equation*}

Now, we are ready to prove that $\WPhi(t) = \Phi(t) - \PhiV \rightarrow 0$.
Consider the Lyapunov function
\begin{equation*}
    L(t) = \WPhi(t)^2 \, .
\end{equation*}
Using the angular dynamics of $\WPhi$, we get
\begin{equation*}
\Dot{L}(t) = 2 \WPhi(t) \Dot{\WPhi} (t) = -2c(t) \WPhi(t) \sin \WPhi(t) + 2 \WPhi(t) \DeltaPhi (t) \, .
\end{equation*}
Since $|\WPhi(t)| \leq \PhiTH < \pi/2$, we have the elementary inequality
\begin{equation*}
\WPhi \sin \WPhi \geq c_{\mathrm{th}} \WPhi^2 \quad \text{for all } |\WPhi| \leq \PhiTH,
\end{equation*}
where
\begin{equation*}
    c_{\mathrm{th}} := \frac{\sin \PhiTH}{\PhiTH} > 0 \, .
\end{equation*}
Therefore, we have the following differential inequality
\begin{equation*}
\begin{aligned}
\Dot{L}(t) &\leq -2 c_* c_{\mathrm{th}} L(t) + 2 |\WPhi(t)| |\DeltaPhi(t)| \\
&\leq - C L(t) + \frac{1}{C} |\DeltaPhi(t)|^2,
\end{aligned}
\end{equation*}
where the last inequality applies Young's inequality and uses the notation $C := c_* c_{\mathrm{th}}$.
By Grönwall's inequality, for $t \geq t_0$, 
\begin{equation*}
L(t) \leq e^{-C(t-t_0)} L(t_0) + \frac{1}{C} \int_{t_0}^t e^{-C(t-s)} |\DeltaPhi (s)|^2 ds \, .
\end{equation*}
On the right-hand side of above inequality, the first term clearly converges to $0$ as $t \rightarrow \infty$.
For the second term, since $|\DeltaPhi(s)|^2 \rightarrow 0$ as $s \rightarrow \infty$, we also have
\begin{equation*}
    \lim_{t \rightarrow \infty} \int_{t_0}^t e^{-C(t-s)} |\DeltaPhi (s)|^2 ds = 0 \, .
\end{equation*}
Therefore, $L(t) = \WPhi(t)^2 \rightarrow 0$, and thus $\WPhi(t) = \Phi(t) - \PhiV \rightarrow 0$.
\end{proof}

\end{proof}

\subsection{Auxiliary Lemmas}

\begin{lemma}[Properties of the Möbius Transformation]\label{lem:Mobius_tf_properties}
Let $M_{\alpha, \eta}$ be the Möbius transformation defined in \eqref{eq:Mobius_tf} considered as a mapping from the unit circle into itself.
It has the following properties.
\begin{itemize}
    \item[(i).] 
    $\alpha = \int_0^{2\pi} M_{\alpha, \eta} (e^{i\varphi}) d\Unif(\varphi)$, where $\Unif$ denotes the uniform probability measure on $[0, 2\pi)$.
    \item[(ii).] $\mathrm{Lip} (M_{\alpha, \eta}) \leq \frac{2}{1 - \gamma}$, where $\gamma=|\alpha|$.
\end{itemize}
\end{lemma}

\begin{proof}
\begin{itemize}
    \item[(i).]
    If $|\alpha| = 1$, then $M_{\alpha, \eta}(u) = \alpha $ for all $u \in \S^1$, so the identity is clear in this case.  If, on the other hand, $|\alpha| < 1$, for every $u \in \BBS^1$ we may write
    \begin{equation*}
    \frac{1}{1 + \alphaBar e^{i\eta} u} = \sum_{k=0}^{\infty} (-\alphaBar e^{i\eta} u)^k \, .
    \end{equation*}
    Then, we by the definition of $M_{\alpha, \eta}$ in \eqref{eq:Mobius_tf}, we can directly compute that
    \begin{equation*}
    \begin{aligned}
    \int_0^{2\pi} M_{\alpha,\eta} (e^{i\varphi}) d \Unif(\varphi) &= \int_0^{2\pi} \frac{\alpha + e^{i\eta} e^{i\varphi}}{1 + \alphaBar e^{i\eta} e^{i\varphi}} d\Unif(\varphi) \\
    &= \int_0^{2\pi} (\alpha + e^{i\eta} e^{i\varphi}) \sum_{k=0}^{\infty} (-\alphaBar e^{i\eta} e^{i\varphi})^k d\Unif(\varphi) = \alpha \, .
    \end{aligned}
    \end{equation*}
    The last equality uses that all terms with nonzero Fourier frequency vanish after integration.

    \item[(ii).] We start by writing $M_{\alpha, \eta}$ as
    \begin{equation*}
    M_{\alpha, \eta} (e^{i\varphi}) = E_{\alpha} (e^{i(\varphi + \eta)}),
    \end{equation*}
    where, for $u \in \BBS^1$,
    \begin{equation*}
       E_{\alpha}(u) := \frac{\alpha + u}{1 + \alphaBar u} \, . 
    \end{equation*}
    A direct computation gives
    \[
    \frac{dE_{\alpha}}{du}=\frac{(1+\overline{\alpha} u)-\overline{\alpha}(\alpha+u)}{(1+\overline{\alpha} u)^2}
    =\frac{1-|\alpha|^2}{(1+\overline{\alpha} u)^2}.
    \]
    Since $|u|=1$ we have $|1+\overline{\alpha} u|\ge 1-|\alpha|$, and consequently
    \[
    \left|\frac{dE_{\alpha}}{du}\right|
    \le \frac{1-|\alpha|^2}{(1-|\alpha|)^2}
    \le \frac{1+|\alpha|}{1-|\alpha|}
    \le \frac{2}{1-|\alpha|}.
    \]
    Moreover, 
    \begin{equation*}
        \left|\frac{d}{d\varphi} e^{i(\varphi + \eta)} \right| = |i e^{i(\varphi + \eta)}| = 1
    \end{equation*}
    Therefore, by the chain rule, 
    \[
    |\partial_\varphi M_{\alpha,\eta}(\varphi)|
    =\left|\frac{dE_{\alpha}}{du}\right|\left|\frac{d}{d\varphi} e^{i(\varphi + \eta)}\right|
    \le \frac{2}{1-|\alpha|} = \frac{2}{1 - \gamma} \, .
    \]
    This gives 
    $$\mathrm{Lip}(M_{\alpha,\eta})\le \frac{2}{1 - \gamma} \, .$$
\end{itemize}
\end{proof}

\begin{lemma}\label{lem:stability_case1_case2}
Let $\alpha=\gamma e^{i\Phi}$ be a solution to \eqref{eq:WSshort-a}. 
\begin{itemize}
    \item \textbf{Case 1.}
    For the matrices $A$ and $V$ given in \eqref{eq:matrices_A_V_case_1}, the $(\gamma, \Phi)$-dynamics is
    \begin{equation}\label{eq:error_case1}
    \begin{aligned}
    \dot{\gamma} &= \frac{1}{4} \left(1 - \gamma^2 \right) \left( \aDPlus \gamma + \aDMinus \cos (\Phi) + \aOPlus \sin (\Phi) \right) + \Delta_\gamma,\\[4pt]
    \dot{\Phi} &= \frac{1}{4} \left(1 - \gamma^2 \right) \left( \aOMinus - \frac{\aDMinus \sin (\Phi) - \aOPlus \cos (\Phi)}{\gamma} \right) + \Delta_{\Phi} \, .
    \end{aligned}
    \end{equation}

    \item \textbf{Case 2.}
    For the matrices $A$ and $V$ given in \eqref{eq:matrices_A_V_case_2}, the $(\gamma, \Phi)$-dynamics is
    \begin{equation}\label{eq:error_case2}
    \begin{aligned}
    \dot{\gamma} &= \frac{1}{4} \left( 1 - \gamma^2 \right) \left( \vDPlus \gamma + \vDMinus \cos (\Phi) + \vOPlus \sin (\Phi) \right) + \DeltaGamma\\[4pt]
    \dot{\Phi} &= -\frac{1}{4} \frac{3\gamma^2 + 1}{\gamma} \left( \vDMinus \sin (\Phi) - \vOPlus \cos (\Phi) \right) + \DeltaPhi\, .
    \end{aligned}
    \end{equation}
\end{itemize}
Here, \(\Delta_\gamma\) and \(\Delta_\Phi\) denote the corresponding error terms in each case.
Moreover, within each case, there exists a constant $C_{A,V} > 0$ depending on the matrices $A$ and $V$ such that
\begin{equation*}
\quad |\Delta_\gamma|\leq (1-\gamma^2) C_{A,V} C_\nu(\gamma)\frac{1-\gamma}{\gamma},\quad |\Delta_\Phi|\leq C_{A,V} C_\nu(\gamma)\frac{1-\gamma}{\gamma},
\end{equation*}
where $C_{\nu}$ is defined in \eqref{eq:C_nu_condition}.
\end{lemma}

\begin{proof}
For notation simplicity, we suppress the time dependence throughout the proof.

Recall from \eqref{eq:gamma_Phi_dyn} that the $(\gamma, \Phi)$-dynamics is given by
\begin{equation*}
\Dgamma = \re{e^{-i\Phi} F(\alpha, \alpha)} + \Delta_{\gamma}, \qquad
\DPhi = \frac{1}{\gamma} \im{e^{-i\Phi} F(\alpha, \alpha)} + \Delta_{\Phi} \,.
\end{equation*}
Here, $\DeltaGamma$ and $\DeltaPhi$ are defined in \eqref{eq:Delta_gamma_Phi}, and for any $\alpha, r \in \BBC$,
\begin{equation*}
F(\alpha, r) = 2i \left( B(r) \alpha^2 + C(r) \alpha + \overline{B}(r) \right),
\end{equation*}
where the functions $B$ and $C$ are given in \eqref{eq:coeff_B_C}.
We note that $B$ and $C$ depend on the matrices $A$ and $V$, and hence so does the function $F$.

As an immediate consequence of \eqref{eq:Delta_bound_gen} and Proposition \ref{prop:diff_r2_alpha}, we obtain
\begin{equation}\label{eq:Delta_Phi_control}
|\DeltaPhi(t)| = \left| \frac{1}{\gamma} \im{e^{-i\Phi}} \Delta(t) \right| \leq \frac{1}{\gamma} |\Delta(t)| \leq L_F (A,V) C_{\nu}(\gamma) \frac{1 - \gamma(t)}{\gamma(t)}\, .
\end{equation}
Here, $L_F (A,V)$ denotes the Lipschitz constant of $F$, which depends on the matrices $A$ and $V$. 

Next, we derive the explicit form of the $(\gamma, \Phi)$-dynamics and estimate $|\DeltaGamma|$ separately for the two cases under consideration.

\begin{itemize}
    \item \textbf{Case 1.}
    For the matrices $A$ and $V$ in \eqref{eq:matrices_A_V_case_1}, the functions $B$ and $C$ take the form
    \begin{equation*}
    \begin{aligned}
    B(r) = Q_1 \rOverline + Q_2, \qquad C(r) = -2 \re{Q_2 r} + Q_3,
    \end{aligned}
    \end{equation*}
    where
    \begin{equation*}
    Q_1 := \frac{1}{8} \left( - \aOMinus + i\aDPlus \right), \qquad Q_2 := \frac{1}{8} \left( \aOPlus + i \aDMinus \right), \qquad Q_3 := \frac{1}{4} a_{-}^o \, .
    \end{equation*}
    Substituting these expressions into $F$, we obtain
    \begin{equation*}
    F(\alpha, \alpha) = 2i (1 - |\alpha|^2) \left( -Q_1 \alpha + \overline{Q}_2 \right), 
    \end{equation*}
    and
    \begin{equation*}
    F(\alpha, r) - F(\alpha, \alpha) = 2i \left( \left( \overline{Q}_1 - \alpha Q_2 \right) \left( \CR_2 - \alpha \right) + \alpha\left(\alpha Q_1 - \overline{Q}_2 \right) \left( \overline{\CR}_2 - \alphaBar \right) \right) \, .
    \end{equation*}
    Using these identities together with \eqref{eq:gamma_Phi_dyn}, a direct computation yields the explicit $(\gamma,\Phi)$-system
    \begin{equation}\label{eq:WS_pc_dyn_case_1}
    \begin{aligned}
        \Dgamma &= 2(1 - \gamma^2) \left(  \im{Q_1} \gamma + \im{Q_2 e^{i\Phi}} \right) + \DeltaGamma,\\[4pt]
    \dot{\Phi} &= 2(1-\gamma^2) \left( \frac{Q_3}{2} + \frac{1}{\gamma} \re{Q_2 e^{i\Phi}} \right) + \DeltaPhi,
    \end{aligned}
    \end{equation}
    where
    \begin{equation}\label{eq:Delta_gamma_Phi_case1}
    \begin{aligned}
    \DeltaGamma &= 2(1-\gamma^2) \im{Q_1 \left(\rho e^{i(\Phi - \phi)} - \gamma \right)},\\[4pt]
    \DeltaPhi &= \frac{2(1+\gamma^2)}{\gamma} \re{Q_1 \left( \rho e^{i(\Phi - \phi)} - \gamma \right)} + 4 \left( \gamma \re{Q_2 e^{i\Phi}} - \rho \re{Q_2 e^{i\phi}} \right) \, .
    \end{aligned}
    \end{equation}
    
    Next, we estimate $|\DeltaGamma|$.
    Using the triangle inequality, the bound $\rho, \gamma \leq 1$, and Proposition \ref{prop:structural_stability}, we deduce
    \begin{equation*}
    \begin{aligned}
    |\DeltaGamma| &= \left|2(1-\gamma^2) \im{Q_1 \left(\rho e^{i(\Phi - \phi)} - \gamma \right)} \right|\\
    &\leq 2(1-\gamma^2) |Q_1| \left|\rho e^{i(\Phi - \phi)} - \gamma \right| \\
    &\leq 2(1-\gamma^2) |Q_1| \left( |\rho - \gamma| + \rho |e^{i(\Phi - \phi)} - 1| \right)\\
    & \leq 2 (1 - \gamma^2) |Q_1| \left( |\rho - \gamma| + |\Phi - \phi| \right)\\
    &\leq (1 - \gamma^2) C_{A,V} C_{\nu}(\gamma) \frac{1 - \gamma}{\gamma} \, .
    \end{aligned}
    \end{equation*}
    
    Finally, substituting the definitions of $Q_1, Q_2,$ and $Q_3$ into \eqref{eq:WS_pc_dyn_case_1} gives the claimed form of the dynamics.

    \item \textbf{Case 2.}
    For the matrices $A$ and $V$ given in \eqref{eq:matrices_A_V_case_2}, the functions $B$ and $C$ take the form
    \begin{equation*}
    B(r) = Q_1 \overline{r} + Q_2, \qquad C(r) = 2\re{Q_2 r}
    \end{equation*}
    with 
    \begin{equation*}
    Q_1 := \frac{i}{8} \vDPlus, \qquad Q_2 := \frac{1}{8} \left( \vOPlus + i \vDMinus \right) \, .
    \end{equation*}
    Substituting these expressions into $F$, we obtain
    \begin{equation*}
    F(\alpha, \alpha) = 2i \left( Q_1 |\alpha|^2 \alpha + 2 Q_2 \alpha^2 + \overline{Q}_1 \alpha + (1 + |\alpha|^2) \overline{Q}_2 \right),
    \end{equation*}
    and
    \begin{equation*}
    F(\alpha, r) - F(\alpha, \alpha) = 2i \left( \left( \overline{Q}_1 + Q_2 \alpha \right) (r - \alpha) + \alpha(Q_1 \alpha + \overline{Q}_2) (\overline{r} - \alphaBar) \right) \, .
    \end{equation*}
    Using these identities, together with \eqref{eq:gamma_Phi_dyn}, a direct computation yields the explicit $(\gamma,\Phi)$-system
    \begin{equation}\label{eq:WS_pc_dyn_case_2}
    \begin{aligned}
    \dot{\gamma} &= 2(1 - \gamma^2) \left(\im{Q_1} \gamma + \im{Q_2 e^{i\Phi}} \right) + \DeltaGamma,\\[4pt]
    \dot{\Phi} &= \frac{2(3\gamma^2 + 1)}{\gamma} \re{Q_2 e^{i\Phi}} + \DeltaPhi,
    \end{aligned}
    \end{equation}
    where
    \begin{equation}\label{eq:Delta_gamma_Phi_case2}
    \begin{aligned}
    \DeltaGamma &:= 2(1 - \gamma^2) \im{Q_1} \left( \rho \re{e^{i(\phi - \Phi)}} - \gamma \right),\\[4pt]
    \DeltaPhi &:= 2 \frac{1 + \gamma^2}{\gamma} \im{Q_1} \rho \im{e^{i(\phi - \Phi)}} + 4 \left( \rho \re{Q_2 e^{i\phi}} - \gamma \re{Q_2 e^{i\Phi}} \right) \, .
    \end{aligned}
    \end{equation}
    With a similar argument as for Case 1, we obtain
    \begin{equation*}
    \quad |\Delta_\gamma|\leq (1-\gamma^2) C_{A,V} C_\nu(\gamma)\frac{1-\gamma}{\gamma} \, .
    \end{equation*}

    Finally, substituting the definitions of $Q_1, Q_2,$ and $Q_3$ into \eqref{eq:WS_pc_dyn_case_2} gives the claimed form of the dynamics.
\end{itemize}
\end{proof}

\begin{lemma}\label{lem:nu-lp-unified}
Let $1 \le p \le \infty$, and let $\nu$ be a probability density on $[0,2\pi)$. Then there exists a constant $C_p<\infty$ such that, for every $\gamma \in [0,1)$, 
\[
C_\nu(\gamma) (1-\gamma) \leq
\begin{cases}
C_p \|\nu\|_{L^p([0,2\pi))} (1-\gamma)^{1-1/p}, & 1<p<\infty,\\[1ex]
C_\infty \|\nu\|_{L^\infty([0,2\pi))} (1-\gamma) \log \dfrac{2}{1-\gamma}, & p=\infty,\\[2ex]
C_1, & p=1.
\end{cases}
\]
In particular, when $1<p \leq \infty$ and $\lVert \nu \rVert_{L^p([0,2\pi)])}<\infty$, we have $C_\nu(\gamma)(1-\gamma) \to 0$ as $\gamma \to 1$.
\end{lemma}

\begin{proof}
Fix $\chi \in[0,2\pi)$. 
By Hölder's inequality and the fact that $\lVert \cdot \rVert_{L^p([0,2\pi))}$ is invariant by rotations we deduce
\begin{equation}
\label{eq:AuxCondnu}
\begin{aligned}
\int_0^{2\pi} \frac{1}{|1+\gamma e^{i(\varphi+\chi)}|}\,\nu(d\varphi)
&\leq \|\nu\|_{L^p([0,2\pi))} 
\left\|\frac{1}{|1+\gamma e^{i\varphi}|}\right\|_{L^q([0,2\pi))},
\end{aligned}
\end{equation}
where $q$ is $p$'s Hölder conjugate exponent. It thus suffices to estimate
\[
K_q(\gamma):=\left\|\frac{1}{|1+\gamma e^{i\varphi}|}\right\|_{L^q([0,2\pi))}.
\]

If $0\le \gamma \le \frac12$, then $|1+\gamma e^{i\varphi}|\ge 1-\gamma\ge \frac12$, and hence
\[
K_q(\gamma)\leq 2 \leq C_p (1-\gamma)^{-\frac{1}{p}},
\]
in case $p<\infty$, and
\[
K_1(\gamma) \leq 2 \leq C_p \log\frac{2}{1-\gamma},
\]
if case $p=\infty$.

Assume now that $\frac12\le \gamma<1$. Writing $\Wvarphi=\varphi-\pi$, we obtain
\[
|1+\gamma e^{i\varphi}|^2 = (1-\gamma)^2 + 2\gamma(1-\cos\Wvarphi)
\ge (1-\gamma)^2 + c_1 \Wvarphi^2
\]
for some $c_1>0$. Hence
\[
K_q(\gamma)^q \le C \int_{-\pi}^{\pi} \frac{d\Wvarphi}{\big((1-\gamma)^2+\Wvarphi^2\big)^{q/2}}.
\]

If $1<q<\infty$, a change of variables $\Wvarphi=(1-\gamma)u$ gives
\[
K_q(\gamma)^q \le C (1-\gamma)^{1-q} \int_{\mathbb R}\frac{du}{(1+u^2)^{q/2}},
\]
and since $q>1$, the integral is finite. Therefore
\[
K_q(\gamma)\le C_p (1-\gamma)^{-1/p}.
\]

If $q=1$ (i.e. $p=\infty$), we obtain, instead,
\[
K_1(\gamma) \le C \int_{-\pi}^{\pi} \frac{d\theta}{\sqrt{(1-\gamma)^2+\theta^2}}
\le C \log\frac{2}{1-\gamma}.
\]

Finally, if $q=\infty$ (i.e. $p=1$), we simply use
\[
K_\infty(\gamma)=\sup_\varphi \frac{1}{|1+\gamma e^{i\varphi}|}
= \frac{1}{1-\gamma}.
\]

Inserting the above bounds in \eqref{eq:AuxCondnu}, and multiplying both sides of \eqref{eq:AuxCondnu} by $1-\gamma$, we obtain the desired upper bound on $C_\nu(\gamma)(1-\gamma)$. 
\end{proof}

\begin{lemma}\label{lem:stability_estimate_control}
Under the assumptions of Proposition \ref{prop:structural_stability}, suppose, additionally, that $\nu$ satisfies Assumption \ref{asm:Lp}.
For any prescribed $\varsigma > 0$ and $T > 0$, there exists a constant $\varepsilon (\varsigma, T, p, M, L_F) > 0$, depending on $\varsigma, T$, the constants $p, M$ in Assumption \ref{asm:Lp}, and the Lipschitz constant $L_F$ in Proposition \ref{prop:structural_stability}, such that if
\begin{equation*}
d_{\BL} (\nu, \Unif) \leq \varepsilon (\varsigma, T, p, M, L_F),
\end{equation*}
then
\begin{equation*}
\sup_{t \in [0,T]} \left| \alpha(t) - \alpha_{\OA}(t) \right| \leq \varsigma \, .
\end{equation*}
\end{lemma}

\begin{proof}
Since $\nu$ satisfies Assumption \ref{asm:Lp}, by Lemma \ref{lem:nu-lp-unified} there exists a constant $\WGammaTH > 0$, depending on $T, \varsigma, p, M$, and $L_F$ so that
\begin{equation*}
T L_F e^{L_F T} C_{\nu} (\gamma) (1 - \gamma) \leq \frac{\varsigma}{2}, \qquad \text{for } \forall \gamma \in [\WGammaTH, 1] \, .
\end{equation*}
Then, by Proposition \ref{prop:structural_stability}, we have
\begin{equation*}
\begin{aligned}
&\sup_{t \in [0,T]} \left| \alpha(t) - \alpha_{\OA} (t) \right|\\
& \qquad  \leq L_{F}\,e^{L_{F} T} \int_0^T \min \left\{ C_{\nu} (\gamma(t)) (1 - \gamma(t)), \ \left( 1 + \frac{2}{1 - \gamma(t)} \right)
 d_{\BL}(\nu,\Unif) \right\} dt \\
& \qquad \leq L_{F}\,e^{L_{F} T} \int_0^T \left( C_{\nu} (\gamma(t)) (1 - \gamma(t)) \mathds{1}_{\{\gamma(t) \geq \WGammaTH\}} + \ \left( 1 + \frac{2}{1 - \gamma(t)} \right)
d_{\BL}(\nu,\Unif) \mathds{1}_{\{\gamma(t) \leq \WGammaTH\}} \right) dt\\
& \qquad \leq \frac{\varsigma}{2} + T L_F e^{L_F T} \frac{ 3 }{1 - \WGammaTH} d_{\BL} (\nu, \Unif) \leq \varsigma \,,
\end{aligned}
\end{equation*}
where in the final inequality we choose $\varepsilon > 0$, which depends on $T, L_F, \WGammaTH$, and $\varsigma$, sufficiently small so that the condition $d_{\BL} (\nu, \Unif) \leq \varepsilon$ ensures that the second term in the preceding inequality is bounded by $\varsigma/2$.
\end{proof}

\end{document}